\documentclass[twoside]{article}

\usepackage{amsmath,amsfonts,amsthm,bm}
\usepackage[pdfencoding=auto,psdextra]{hyperref}
\usepackage{url}
\usepackage[T1]{fontenc}
\usepackage{algorithm, algpseudocode}
\usepackage{graphicx}
\usepackage{enumitem}
\usepackage{wrapfig}
\usepackage{xfrac}
\usepackage{subcaption}
\usepackage{longtable}
\usepackage{array}
\usepackage{xcolor}

\def\eqref#1{equation~\ref{#1}}

\newcommand{\paren}[1]{\left(#1\right)}

\newcommand{\norm}[1]{\left\|#1\right\|}
\newcommand{\indy}[1]{\mathbb{I}\left\{#1\right\}}
\newcommand{\inner}[2]{\left\langle#1, #2\right\rangle}

\newcommand\EXP[2][]{
\ensuremath{\mathrm{\mathbb{E}}_{#1} 
\pmb{\left[\vphantom{#2}\right.}
{#2}
\pmb{\left.\vphantom{#2}\right]}
}}

\newcommand{\Tr}[1]{\texttt{Tr}\paren{#1}}
\renewcommand{\exp}[1]{\text{exp}\paren{#1}}

\newcommand{\R}{\mathbb{R}}
\newcommand{\N}{\mathbb{N}}

\def\calL{{\mathcal{L}}}

\def\bfA{{\mathbf{A}}}

\def\bfI{{\mathbf{I}}}

\def\bfM{{\mathbf{M}}}
\def\bfN{{\mathbf{N}}}

\def\bfV{{\mathbf{V}}}

\def\bfa{{\mathbf{a}}}

\def\bfk{{\mathbf{k}}}

\def\bfq{{\mathbf{q}}}

\def\bft{{\mathbf{t}}}
\def\bfu{{\mathbf{u}}}
\def\bfv{{\mathbf{v}}}
\def\bfw{{\mathbf{w}}}
\def\bfx{{\mathbf{x}}}

\def\calT{{\mathcal{T}}}
\def\calL{{\mathcal{L}}}

\DeclareMathOperator*{\argmax}{arg\,max}
\DeclareMathOperator*{\argmin}{arg\,min}

\definecolor{myblue2}{HTML}{4682B4}

\newtheorem{defin}{Definition}
\newtheorem{theorem}{Theorem}

\newtheorem{lemma}{Lemma}

\newtheorem{asump}{Assumption}

\newtheorem{cond}{Condition}

\usepackage[preprint]{aistats2026}
\usepackage{subcaption}
\usepackage[round]{natbib}

\newcommand{\bbfw}{\bar{\bfw}}
\newcommand{\bbfv}{\bar{\bfv}}

\newcommand{\calO}{\mathcal{O}}
\newcommand{\calN}{\mathcal{N}}
\newcommand{\derivt}{\frac{d}{dt}}
\newcommand{\qqquad}{\quad\quad\quad}
\newcommand{\qqqqquad}{\quad\quad\quad\quad\quad}
\newcommand{\gmij}[2]{\gamma_{i,j}^{(#1)#2}}
\newcommand{\ztij}[2]{\zeta_{i,j}^{(#1)#2}}
\renewcommand{\exp}[1]{\text{exp}\paren{#1}}

\newcommand{\forweps}[1]{\varepsilon_{\mathcal{F},#1}}
\newcommand{\backepso}[1]{\varepsilon_{\mathcal{B},#1}^{(1)}}
\newcommand{\backepst}[1]{\varepsilon_{\mathcal{B},#1}^{(2)}}
\newcommand{\backepsi}[1]{\varepsilon_{\mathcal{B},#1}^{(3)}}
\newcommand{\aggepso}[1]{\varepsilon_{\mathcal{A},#1}^{(1)}}
\newcommand{\aggepst}[1]{\varepsilon_{\mathcal{A},#1}^{(2)}}

\begin{document}

%

%

\twocolumn[

\aistatstitle{Guided by the Experts: Provable Feature Learning Dynamic of Soft-Routed Mixture-of-Experts}

\aistatsauthor{ Fangshuo Liao \And Anastasios Kyrillidis }

\aistatsaddress{Rice University } ]

\begin{abstract}
  Mixture-of-Experts (MoE) architectures have emerged as a cornerstone of modern AI systems. 
  In particular, MoEs route inputs dynamically to specialized experts whose outputs are aggregated through weighted summation. Despite their widespread application, theoretical understanding of MoE training dynamics remains limited to either separate expert-router optimization or only top-1 routing scenarios with carefully constructed datasets. 
  This paper advances MoE theory by providing convergence guarantees for joint training of soft-routed MoE models with non-linear routers and experts in a student-teacher framework. 
  We prove that, with moderate over-parameterization, the student network undergoes a feature learning phase, where the router's learning process is ``guided'' by the experts, that recovers the teacher's parameters. 
  Moreover, we show that a post-training pruning can effectively eliminate redundant neurons, followed by a provably convergent fine-tuning process that reaches global optimality. 
  To our knowledge, our analysis is the first to bring novel insights in understanding the optimization landscape of the MoE architecture.
\end{abstract}

\section{INTRODUCTION}

Mixture-of-Experts (MoE) architectures have become a fundamental building block in modern artificial intelligence systems, enabling significant advances in model capacity without corresponding increases in computational costs \citep{shazeer2017outrageously, fedus2022switch}. 
At its core, a MoE system treats a complicated task as a combination of multiple simpler tasks, which can be handled efficiently by smaller models. 
In concept, the conditional routing of the input to different sub-modules of the MoE allows each sub-module to ``specialize'' in its own domain, leading to an effective decoupling of the overall task complexity.\footnote{In this paper we consider a ``classic'' MoE instead of the MoE used to increase parameterization without increasing computational costs such as \cite{fedus2022switch}.}
This approach has achieved remarkable success in large language models (LLMs) \citep{fedus2022switch}, computer vision \citep{riquelme2021scaling}, and multi-modal agentic systems \citep{mustafa2022multimodal}, where conditional computation provides an efficient way to scale model capacity.

The fundamental structure of an MoE layer consists of a set of expert networks (the ``experts'') and a gating network that determines the contribution of each expert to the final output.
While simple, this architecture presents theoretical challenges, particularly regarding the joint optimization of both components. 
The gating function, typically implemented using a softmax activation, introduces non-convexity that makes the analysis difficult. 
This is further complicated by the interplay between expert specialization and router assignments: experts must specialize in certain inputs, while the router must correctly identify the correct combination of experts appropriate for each input.

Despite the widespread adoption of MoE architectures in practice, such a theoretical understanding of their optimization dynamic remains limited; see Related Works section. 
Existing theoretical work has focused on either simplified linear models or has analyzed the experts and gating networks separately; e.g. \cite{li2025theorymixtureofexpertscontinuallearning} and \cite{kawata2025mixtureexpertsprovablydetect} study the setting where the experts are trained first with the router parameter fixed, followed by a fine-tuning stage of the router itself. 
While such a setting simplifies the analysis by decoupling the updates of the router and expert parameters, it deviates from the more beneficial scenario, where a joint optimization is applied to handle intricate task combinations \citep{kong2025masteringmassivemultitaskreinforcement,zhang2025mixtureexpertslargelanguage}.

A prior work \citep{chen2022understandingmixtureexpertsdeep} studies the joint optimization of the expert and router parameters in a top-1 routed MoE on patched input data, thus reducing the interference between the learning process of the experts in each gradient step. 
Although top-1 routing has been a popular approach \citep{fedus2022switch}, most state-of-the-art language models such as Mixtral 8x7B \citep{jiang2024mixtral}, DeepSeek-V3 \citep{deepseekai2025deepseekv3technicalreport}, and Qwen 3 \citep{yang2025qwen3technicalreport} uses a top-$K$ routing with $K > 1$, leading to gaps between theory and practice. 

In general, it remains open to study the optimization dynamics of MoEs with more than one activated experts, where experts and router are jointly trained. 
This gap is increasingly significant as MoE architectures become fundamental components in state-of-the-art AI systems. 
Analysis of these dynamics would not only inform architectural improvements but also provide formal guarantees about model behavior and performance. 
Moreover, such an analysis could help us understand better Agentic AI systems \citep{hu2025automateddesignagenticsystems,zhang2025gdesignerarchitectingmultiagentcommunication,zhang2025symbioticcooperationwebagents}, where component orchestration mirrors MoE routing mechanisms \citep{bhatt2025orchestratemultipleagents}. These systems must dynamically select appropriate specialized modules (tools, APIs, or reasoning components) based on input context—functionally analogous to expert selection in MoE architectures. 

\textbf{Contributions.} Given the difficulty of the task, we focus on the learning of a MoE model with one-layer sigmoid router and non-linear experts over the mean-square-error (MSE) loss in a teacher-student set-up on high dimensional Gaussian input. 
In particular, we show that, with moderate over-parameterization and under the gradient flow training, the student MoE enjoys a near-perfect recovery of the feature from the teacher model's in a sequential order in $\mathcal{O}\paren{\sqrt{d}}$ time, where $d$ is the dimensionality of the input. 
Moreover, after the feature learning stage, a greedy pruning can be applied to remove the unused experts. 
Lastly, the post-pruning fine-tuning of the student model converges to zero loss.

\textbf{Notations.} Without further specification, we use regular lower-case letters (e.g. $a$) to denote scalars, bold-face lower-case letters (e.g. $\bfa$) to denote vectors, and bold-face capital letters (e.g. $\bfA$) to denote matrices. We use $\mathcal{N}\paren{\mu, \sigma^2}$ to denote the the Gaussian distribution with mean $\mu$ and (co)variance $\sigma^2$. For a function $f(x)$, we use $f'(x), f''(x)$ and $f'''(x)$ to denote its first three order derivatives, and $f^{(a)}(x)$ to denote is arbitray $a$th order derivative. We use $\text{poly}\paren{x_1, \dots, x_n}$ to denote the polynomial dependency in terms of $x_1,\dots,x_n$.

\section{RELATED WORKS}

\textbf{Theory of Mixture-of-Experts.} From an optimization perspective, \cite{chen2022understandingmixtureexpertsdeep} studies the convergence rate of top-1 MoE with CNN experts on patched input data. \cite{chowdhury2023patchlevelroutingmixtureofexpertsprovably} studies the patch-level routing under both the setting with a separately trained expert and router, and the setting of pre-trained experts. \cite{chowdhury2024provablyeffectivemethodpruning} shows the pruning effectiveness after fine-tuning a pre-trained MoE model. \cite{kawata2025mixtureexpertsprovablydetect} considers the training both a top-1 MoE and a ReLU routed MoE, but under a four-stage training algorithm. \cite{fruytier2025learningmixturesexpertsem} studies the convergence of the Expectation-Maximization algorithm for learning MoE. \cite{li2025theorymixtureofexpertscontinuallearning} studies the optimization of MoE in a continual learning set-up. From the perspective of sample complexity, \cite{nguyen2024generaltheorysoftmaxgating,nguyen2024squareestimationsoftmaxgating,nguyen2025convergenceratessoftmaxgating} studies the sample complexity of correctly identifying experts for softmax MoE under both the logistic loss and the MSE loss. \cite{nguyen2024sigmoidgatingsampleefficient} shows that sigmoid gated MoEs enjoy a better sample complexity compared with softmax gated MoEs. Other works \citep{kratsios2024mixtureexpertssoftencurse,wang2025expressivepowermixtureofexpertsstructured} studies MoE under operator learning, and the expressive power of MoEs, respectively. Following the expressivity line of work, \cite{boixadsera2025powerfinegrainedexpertsgranularity} studies how the granularity of the experts affects the expressive power of MoEs.

\textbf{Feature Learning of Neural Networks.} As its name suggested, the feature learning framework explores the ability of the neural network to learn intrinsic features of the dataset, which is an ability not present in the traditional Neural Tangent Kernel framework \citep{jacot2020neuraltangentkernelconvergence,du2019gradientdescentfindsglobal}. In particular \cite{shi2022theoreticalanalysisfeaturelearning,shi2023provableguaranteesneuralnetworks} studies the hidden-neuron evolution during training, and \cite{damian2022neuralnetworkslearnrepresentations,mousavihosseini2023neuralnetworksefficientlylearn} investigates how gradient-based learning discovers the intrinsic low-dimensional subspace of data. Along this line of work \cite{ba2023learning, mousavihosseini2023gradientbasedfeaturelearningstructured} studies the learning with data sampled from distribution with a spiked covariance matrix. Recently, a popular line of work studies the learning of Gaussian single/multi-index models \citep{bietti2022learningsingleindexmodelsshallow,lee2024neuralnetworklearnslowdimensional,ren2025emergencescalinglawssgd,bietti2023learninggaussianmultiindexmodels,ba2023learning,şimşek2025learninggaussianmultiindexmodels} Noticeably, this line of work adopts the Hermite expansion of the nonlinear function to transform the loss objective into a form similar to the tensor decomposition \citep{ge2017learningonehiddenlayerneuralnetworks}. In terms of the proof technique, our work is similar to \cite{ren2025emergencescalinglawssgd} by utilizing the sharp phase transition that occurs from the high information exponent of the activation function.

\begin{figure}[t!]
    \centering
    \includegraphics[width=\linewidth]{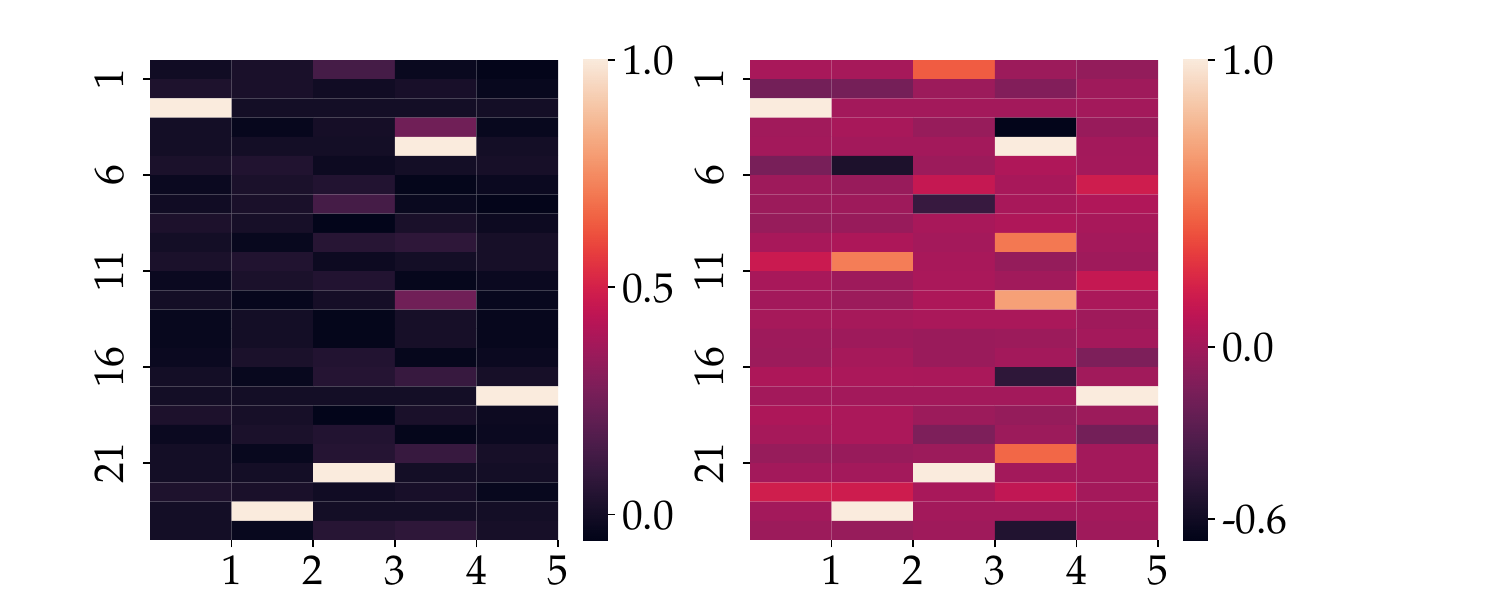}
    \caption{Training MoE in (\ref{eq:moe}) on (\ref{eq:population_loss}) with $m^\star = 5, m = 25$, and $d=  1000$ with online batch SGD simulating GF on the population loss. Left: alignment values of the router parameters $\bbfv_i^\top\bbfv_j^\star$. Right: alignment values of the expert parameters $\bbfw_i^\top\bbfw_j^\star$.}
    \label{fig:heatmap}
\end{figure}

\section{PRELIMINARY AND SET-UP}
\textbf{Student model.} In this paper, we consider the training of a normalized MoE with $m$ experts. In particular, given inputs $\bfx\in\R^d$, we study the setting of a one-layer router with parameter $\bfV$, given by $\pi\paren{\bar{\bfV}\bfx}$. Here $\pi\paren{\cdot}$ is an entry-wise sigmoid function, and $\bar{\bfV}$ denote the row-wise normalized version of $\bfV$. In short, we have
\[
    \pi\paren{\bar{\bfV}\bfx}_i := \pi\paren{\bar{\bfv}_i^\top\bfx} = \pi\paren{\frac{\bfv_i^\top\bfx}{\norm{\bfv_i}_2}};\;\forall i\in[m]
\]
where $\bfv_i$ is the $i$th row of $\bfV$. We consider each expert as a one-layer non-linear function with parameter $\bfw_i$ given by $\sigma\paren{\bbfw_i^\top\bfx} = \sigma\paren{\frac{\bfw_i^\top\bfx}{\norm{\bfw_i}_2}}$. In this paper, we choose $\sigma\paren{a} = a^3 - 3a$ to be the third-order Hermite polynomial.  Letting $\bm{\theta} = \left\{\paren{\bfv_i,\bfw_i}\right\}_{i=1}^{m}$, then the student model is given by
\begin{equation}
    \label{eq:moe}
    f\paren{\bm{\theta},\bfx} := \sum_{i=1}^{m}\pi\paren{\bbfv_i^\top\bfx}\sigma\paren{\bbfw_i^\top\bfx}
\end{equation}
Our choice of the sigmoid router is motivated by \cite{nguyen2024sigmoidgatingsampleefficient} which shows that sigmoid routing is more sample efficient than softmax routing. Moreover, using Hermite polynomial as the activation function has been a popular approach in previous study of the feature learning mechanism of neural networks \citep{arous2025learningquadraticneuralnetworks}. Lastly, the choice of normalizing the weights is also a popular choice in prior works \citep{wang2020lazytrainingoverparameterizedtensor,ren2025emergencescalinglawssgd}.

\textbf{Data and Teacher Model.} The teacher model $f^\star$ we consider has the same structure as in the the student model, but with $m^\star$ experts. In addition, we assume that the parameter of the teacher model's parameters $\bbfv_1^\star,\dots,\bbfv_{m^\star}^\star,\bbfw_1^\star,\dots,\bbfw_{m^\star}^\star$ forms an orthonormal list. We assume that the input data comes from a standard Gaussian distribution $\bfx\sim\mathcal{N}\paren{\bm{0}, \bfI_d}$, and labels are generated according to
\begin{equation}
    \label{eq:teacher}
    y = f^\star\paren{\bfx} := \sum_{i=1}^{m^\star}\pi\paren{\bbfv_i^{\star\top}\bfx}\sigma\paren{\bbfw_i^{\star\top}\bfx}
\end{equation}
Intuitively, this set-up implies that the input space $\R^d$ is softly partitioned by the teacher's router. The goal of the student model is to learn both the features $\bbfv_i^\star$s that gives a correct partitions, and the features $\bbfw_i^\star$s that leads to the effective specialization of experts.

We consider training the student model $f\paren{\bm{\theta},\bfx}$ on the population mean-squared error (MSE) loss $\calL\paren{\bm{\theta}}$ using gradient flow $\derivt \bm{\theta}(t) = -\nabla\calL\paren{\bm{\theta}(t)}$ over the data distribution defined by the teacher model. To be more specific, the MSE has the form
\begin{equation}
    \label{eq:population_loss}
    \begin{aligned}
        \calL\paren{\bm{\theta}} & = \frac{1}{2}\EXP[\bfx,y]{\paren{f\paren{\bm{\theta},\bfx} - y}^2}\\
        & = \frac{1}{2}\EXP[\bfx\sim\mathcal{N}\paren{\bm{0},\bfI_d}]{\paren{f\paren{\bm{\theta},\bfx} - f^\star\paren{\bfx}}^2}
    \end{aligned}
\end{equation}
We initialize the student model by $\hat{\bfv}_i(0),\bfw_i(0)\sim\mathcal{N}\paren{\bm{0},d^{-1}\bfI_d}$ and $\bfv_i(0) = \hat{\bfv}_i(0) - \hat{\bfv}_i(0)^\top\bbfw_i(0)\cdot\bbfw_i(0)$ to decouple the router and expert weights.\footnote{The behavior that $\bfv_i(t)^\top\bfw_i(t) = 0$ does not hold throughout gradient flow training.} We also need to make the following assumption on the sigmoid function, which is numerically checked in Appendix~\ref{sec:plot_sigmoid}.
\begin{asump}
    \label{asump:sigmoid}
    Let $z_1, z_2\sim\mathcal{N}\paren{0,1}$ with arbitrary covariance $\text{Cov}\paren{z_1, z_2}\in [-1, 1]$, it holds that $\EXP[z_1, z_2]{\pi'(z_1)\pi^{(3)}(z_2)} \leq 0$
\end{asump}
We empirically verify that this setting allows that each of the $m^\star$ experts and corresponding router parameter in the teacher model can be recovered by one and only one expert and corresponding router in the student model. According to Figure~\ref{fig:heatmap}, for each $\bbfv_j^\star,\bbfw_j^\star$ in the teacher (x-axis), there is one and only one expert and router $\bbfv_j, \bbfw_j$ that converges to $\bbfv_j^\star,\bbfw_j^\star$ (lighter color indicates that $\bbfv_i^\top\bbfv_j^\star$ and $\bbfw_i^\top\bbfw_j^\star$ are closer to one.)

\section{MAIN RESULT: SEQUENTIAL FEATURE LEARNING}

In this section, we present the main result of the feature learning phase. Recall the set-up of the student and teacher MoE models in (\ref{eq:moe}) and (\ref{eq:teacher}). An ideal feature learning result would be that, for each router-expert pair $\paren{\bbfv_i^\star,\bbfw_i^\star}$ in the teacher model, there is an exclusive router-expert pair $\paren{\bbfv_i,\bbfw_i}$ that converges to it. The theorem below states that the matching between the router-expert pair from the teacher model and the router-expert pair from the student model happens in a sequential order.

\begin{theorem}
    \label{thm:gf_conv}
    Consider training the MoE model $f\paren{\bm{\theta},\bfx}$ in (\ref{eq:moe}) with respect to a teacher model given by (\ref{eq:teacher}) using the gradient flow on the population MSE loss in (\ref{eq:population_loss}). Let $\delta_{\mathbb{P}}\in (0, \sfrac{1}{7})$ be given. If $m \geq \Omega\paren{m^\star\log \frac{m^\star}{\delta_{\mathbb{P}}}}$ and $d\geq \text{poly}\paren{m, \delta_{\mathbb{P}}^{-1}}$, then there exists an injective mapping $\mathcal{I}:[m^\star]\rightarrow [m]$ and time steps $0 \leq T_1 \leq \dots\leq T_{m^\star} \leq T^\star \leq \calO\paren{\sqrt{d}}$ such that for all $\ell\in[m^\star-1]$ and $t\in [T_{\ell}, T_{\ell+1})$, we have that
    \begin{itemize}
        \item (Recovered expert-router pairs) $\bbfv_{\mathcal{I}(i)}(t)^\top\bbfv_{i}^\star \geq 0.9$ and $\bbfw_{\mathcal{I}(i)}(t)^\top\bbfw_{i}^\star \geq 0.9$ for all $i\leq \ell$.
        \item (Unrecovered expert-router pairs) For all $i > \ell$, $\max\left\{\left|\bbfv_{\mathcal{I}(i)}(t)^\top\bbfv_{i}^\star\right|, \left|\bbfw_{\mathcal{I}(i)}(t)^\top\bbfw_{i}^\star\right|\right\} \leq \calO\paren{\frac{m^2}{\delta_{\mathbb{P}}\sqrt{d}}}$
    \end{itemize}
    Moreover, for $T^\star \leq t \leq T^\star + \calO\paren{\frac{\delta_{\mathbb{P}}\sqrt{d}}{m^2}}$, we have that
    \begin{itemize}
        \item (Learned features) $\bbfv_{\mathcal{I}(i)}(t)^\top\bbfv_{i}^\star, \bbfw_{\mathcal{I}(i)}(T)^\top\bbfw_{i}^\star \geq 1 - \calO\paren{\frac{m^7}{\delta_{\mathbb{P}}^3d^{\frac{3}{2}}}}$ for all $i\in[m^\star]$
        \item (Unused expert-router pairs) for all $i_1\in[m]\setminus \mathcal{I}\paren{[m^\star]}, i_2\in[m], i_1\neq i_2$  and $j\in[m^\star]$, the following quantities
        \begin{gather*}
            \hspace{-0.35cm}\left|\bbfv_{i_1}(t)^\top\bbfv_{i_2}(t)\right|, \left|\bbfv_{i_1}(t)^\top\bbfw_{i_1}(t)\right|, \left|\bbfw_{i_1}(t)^\top\bbfw_{i_2}(t)\right|\\
            \hspace{-0.35cm}\left|\bbfv_{i_1}(t)^\top\bbfw_{i_2}(t)\right|,  \left|\bbfv_{i_2}(t)^\top\bbfw_{i_i}(t)\right|\\
            \hspace{-0.35cm}\left|\bbfv_{i_1}(t)^\top\bbfv_j^\star\right|, \left|\bbfv_{i_1}(t)^\top\bbfw_j^\star\right|, \left|\bbfw_{i_1}(t)^\top\bbfw_j^\star\right|, \left|\bbfw_{i_1}(t)^\top\bbfv_j^\star\right|
        \end{gather*}
         are all upper bounded by $\calO\paren{\frac{m^2}{\delta_{\mathbb{P}}\sqrt{d}}}$
    \end{itemize}
\end{theorem}
In particular, Theorem~\ref{thm:gf_conv} states the result that the MoE training in our set-up undergoes a sequential feature learning phase. As requirements of the theorem, we need $m$ to be as large as $\Omega\paren{m^\star\log \frac{m^\star}{\delta_{\mathbb{P}}}}$ to ensure that at initialization , at least one of the router-expert pair from the student model have a good enough alignment for each router-expert pair in the teacher. Also we require $d$ to be polynomially large in terms of $m$ and $\delta_{\mathbb{P}}^{-1}$ to control the interference between the convergence of each router-expert pair, as well as between the convergence of the router parameter and the expert parameter. Due to the polynomial scaling of $d$ in terms of $m$ and $\delta_{\mathbb{P}}^{-1}$, quantities $\calO\paren{\frac{m^2}{\delta_{\mathbb{P}}\sqrt{d}}}$ and $\calO\paren{\frac{m^7}{\delta_{\mathbb{P}}^3d^{\frac{3}{2}}}}$ are in general small, and $\calO\paren{\frac{\delta_{\mathbb{P}}\sqrt{d}}{m^2}}$ is large.

 In the set-up of the theorem, the mapping $\mathcal{I}$ establishes the correct matching between the expert-router pair in the teacher model and the router-expert pair in the student model. Ideally, we expect $\bbfv_{\mathcal{I}(i)}$ and $\bbfw_{\mathcal{I}(i)}$ to converge to $\bbfv_i^\star$ and $\bbfw_i^\star$. Two key points of the theorem are outlined below:

\textbf{Sequential weak recovery.} The first part of the theorem states that such converges happens in a sequential order. By its set-up, $T_\ell$ denotes the time where the first $\ell$ pairs of $\paren{\bbfv_{\mathcal{I}(i)},\bbfw_{\mathcal{I}(i)}}$ for $i\leq \ell$ just achieved a weak convergence to $\bbfv_i^\star$ and $\bbfw_i^\star$ by achieving an inner product with of at least $0.9$. In the mean time, before $t$ reaches $T_{\ell+1}$, all the remaining pairs $\paren{\bbfv_{\mathcal{I}(i)},\bbfw_{\mathcal{I}(i)}}$ for $i > \ell$ still stays within a small alignment value with their reference.

\textbf{Near-perfect recovery.}The second part of the theorem shows that for any time $t$ that exceeds some $T^\star\leq \calO\paren{\sqrt{d}}$ but stays under $T^\star + \calO\paren{\frac{\delta_{\mathbb{P}}\sqrt{d}}{m^2}}$, the learned features have converged to inner product values of at least $1 - \calO\paren{\frac{m^7}{\delta_{\mathbb{P}}d^{\frac{3}{2}}}}$. In the meantime, all the router and expert parameters in the student model that did not converge to any teacher's parameter must stay nearly orthogonal both to the teacher model's parameter and to each other. In Section~\ref{sec:pruning_finetuning}, we will utilize this property to prove the theoretical guarantee of pruning these unused experts in the student model.

\begin{figure}[t!]
    \centering
    \includegraphics[width=0.9\linewidth]{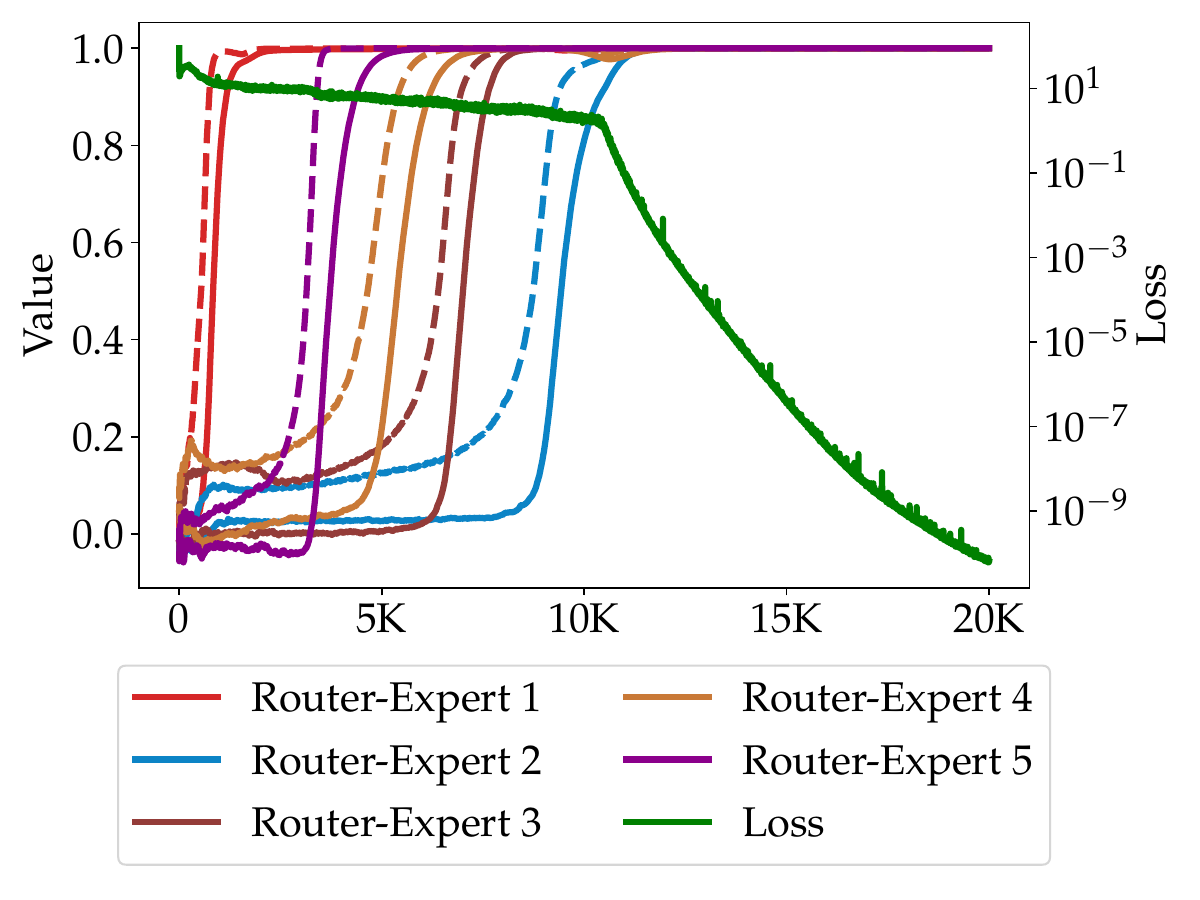}
    \caption{Dynamics of the routers' and experts' alignment value with the teacher's parameter under the same set-up as Figure~\ref{fig:heatmap}. The green curve denotes the loss value. Except for the green curve, dashed line and solid line of the same color denotes a pair of router and expert alignment value.}
    \label{fig:alignment_dynamic}
    \vspace{-0.5cm}
\end{figure}

\subsection{Guided by the Experts: A Proof Sketch of Theorem~\ref{thm:gf_conv}}
In this section, we will discuss the difficulties and techniques arises in the proof of Theorem~\ref{thm:gf_conv}. 

\textbf{Hermite expansion of the loss and gradient.} The starting point of our proof relies on the Hermite expansion of non-linear functions to study its property with Gaussian inputs. Let $He_k\paren{x}$ denote the $k$th-order probabilist's Hermite polynomial. It is known that the set of Hermite polynomials $\left\{He_k\paren{x}\right\}_{k=0}^{\infty}$ consists an basis of the square integrable functions under the Gaussian measure. Therefore, we can expand the sigmoid function as
\[
    \pi\paren{x} = \sum_{k=0}^{\infty}\frac{c_k}{k!}He_k\paren{x};\;c_k = \EXP[x\sim\mathcal{N}\paren{0,1}]{\pi^{(k)}\paren{x}}
\]
Here $c_k$'s are the Hermite coefficients of $\pi\paren{x}$. Since $\sigma\paren{x} = He_3\paren{x}$, $f\paren{\bm{\theta},\bfx}$ and $f^\star\paren{\bfx}$ has the form
\begin{gather*}
    f\paren{\bm{\theta},\bfx} = \sum_{k=0}^{\infty}\frac{c_k}{k!}He_k\paren{\bbfv_i^\top\bfx}He_3\paren{\bbfw_i^\top\bfx}\\
    f^\star\paren{\bfx} = \sum_{k=0}^{\infty}\frac{c_k}{k!}He_k\paren{\bbfv_i^{\star\top}\bfx}He_3\paren{\bbfw_i^{\star\top}\bfx}
\end{gather*}
Carrying this idea to the setting of minimizing the MSE in \ref{eq:population_loss}, we notice that $\calL\paren{\bm{\theta}}$ can be written as
\begin{align*}
    \calL\paren{\bm{\theta}} & = \frac{1}{2}\EXP[\bfx]{f\paren{\bm{\theta},\bfx}^2} - \EXP[\bfx]{f\paren{\bm{\theta},\bfx}f^\star\paren{\bfx}}\\
    & \qqquad + \frac{1}{2}\EXP[\bfx]{f^\star\paren{\bfx}^2}
\end{align*}
where one could notice that the last term does not depend on $\bm{\theta}$. However, the first two terms involves second-order terms on $f\paren{\bm{\theta},\bfx}$ and $f^\star\paren{\bfx}$. As an illustration, we expand the first term as
\begin{gather*}
    \EXP[\bfx]{f\paren{\bm{\theta},\bfx}^2} = \sum_{i,j=1}^m\sum_{k,\ell=0}^{\infty}\frac{c_kc_{\ell}}{k!\ell!}\mathcal{C}_{k,\ell,3,3}^{i,j}\\
    \mathcal{C}_{k,\ell,3,3}^{i,j} = \\
    \EXP[\bfx]{He_k\paren{\bbfv_i^\top\bfx}He_{\ell}\paren{\bbfv_j^\top\bfx}He_3\paren{\bbfw_i^\top\bfx}He_3\paren{\bbfw_j^\top\bfx}}
\end{gather*}
Although a large body of prior work has exploited the nice property of Hermite polynomial that $\EXP[\bfx]{He_k\paren{\bfu_1^\top\bfx}He_{\ell}\paren{\bfu_2^\top\bfx}} = k!\paren{\bfu_1^\top\bfu_2}^k\indy{k=\ell}$ for $\bfu_1,\bfu_2$ with unit norm, in our setting we have to deal with the expectation of the product of four Hermite polynomial. Our main tool of handling this difficulty is the lemma below.\footnote{We are not the first to introduce this result. However, we could not find a formal published source that proves the result.}
\begin{lemma}
    \label{lem:prod_hermite}
    Let $\bfx\sim\mathcal{N}\paren{\bm{0},\bm{\Sigma}}$. For some multi-index $\bfk\in\N^n$, we define the multi-variate Hermite polynomial as
    \[
        He_{\bfk}\paren{\bfx} = \prod_{i=1}^nHe_{\bfk[i]}\paren{\bfx[i]}
    \]
    Then we have that for $\bfx\sim\mathcal{N}\paren{\bm{0},\bm{\Sigma}}$,
    \[
        \EXP[\bfx]{He_{\bfk}\paren{\bfx}} = \paren{\prod_{i=1}^n\bfk[i]!}\sum_{\bfM\in\mathcal{S}}\prod_{i,j=1}^n\frac{\bm{\Sigma}[i,j]^{\bfM[i,j]}}{\bfM[i,j]!}
    \]
    where the set $\mathcal{S}$ is the set of symmetric matrices $\bfM\in\N^{n \times n}$ satisfying
    \[
        \bfM[i,i] = 0;\;\sum_{j=1}^n\bfM[i,j] = \bfk[i];\;\;\forall i\in[n]
    \]
\end{lemma}
In Lemma~\ref{lem:prod_hermite}, each $\bfM\in\mathcal{S}$ can be considered as the adjacency matrix of a graph with $n$ nodes and integer weights such that the degree of the $i$th node is $\bfk[i]$. Applying Lemma~\ref{lem:prod_hermite} to our case thus only requires to enumerate the graphs of four nodes with degree $(k, \ell, 3, 3)$. With Lemma~\ref{lem:prod_hermite}, we are able to derive the form of $\nabla_{\bfv_i}\calL\paren{\bm{\theta}}$ and $\nabla_{\bfw_i}\calL\paren{\bm{\theta}}$, whose exact form are omitted from the main text due to its intricacy.

\textbf{Gradient flow dynamic of target alignment.} Recall that our goal is to show that there is some $\bbfv_i,\bbfw_i$ that converges to $\bbfv_j^\star,\bbfw_j^\star$ for $j\in[m^\star]$. Thus, it is intuitive to start with tracking the dynamic of $\bbfv_i(t)^\top\bbfv_j^\star$ and $\bbfw_i(t)^\top\bbfw_j^\star$. With gradient flow, we have that
\begin{equation}
    \label{eq:target_dynamic}
    \begin{gathered}
        \derivt \bbfv_i(t)^\top\bbfv_j^\star = -\norm{\bfv_i(t)}_2^{-1}\nabla_{\bfv_i}\calL\paren{\bm{\theta}}^\top\bbfv_j^\star\\
        \derivt \bbfw_i(t)^\top\bbfw_j^\star = -\norm{\bfw_i(t)}_2^{-1}\nabla_{\bfw_i}\calL\paren{\bm{\theta}}^\top\bbfw_j^\star
    \end{gathered}
\end{equation}
due to the fact that $\norm{\bfv_i}_2$ and $\norm{\bfw_i}_2$ stays constant during the gradient flow process. Utilizing the fact that $\bbfv_i(t)^\top\bbfv_j^\star, \bbfw_i(t)^\top\bbfw_j^\star\ll 1$ when $\bbfv_i(t)$ and $\bbfw_i(t)$ are near their initialization, we can utilize Lemma~\ref{lem:prod_hermite} to approximate (\ref{eq:target_dynamic}) as
\begin{equation}
    \label{eq:approx_guided}
    \begin{aligned}
        \derivt \bbfv_i(t)^\top\bbfv_j^\star & \propto \paren{\bbfw_i(t)^\top\bbfw_j^\star}^3\\
        \derivt \bbfw_i(t)^\top\bbfw_j^\star & \propto \paren{\bbfw_i(t)^\top\bbfw_j^\star}^2
    \end{aligned}
\end{equation}
The approximation above exhibits two interesting behaviors near initialization. First, the improvements in both the router alignment $\bbfv_i(t)^\top\bbfv_j^\star$ and the expert alignment $\bbfw_i(t)^\top\bbfw_j^\star$ depends on the current magnitude of the expert alignment. This implies that, once the expert alignment reaches a magnitude of $\Omega\paren{1}$, it will take only constant time for the router and expert alignments to grow to a value near one (perfect alignment). This behavior can be observed in Figure~\ref{fig:alignment_dynamic} where the router alignment values (solid lines except for the gree on) follow closely as the expert alignment values (dashed lines) increases. 

Second, the quadratic dependency in the expert alignment dynamic induces a sharp phase transition where the alignment value starts off slow for a period of time, and suddenly increases with a fast speed (see dashed lines in Figure~\ref{fig:alignment_dynamic}), as studied in \cite{ren2025emergencescalinglawssgd}. This sharp phase transition is particularly helpful to prevent multiple experts from the student model to converge to the same expert in the teacher model. As an example, consider dynamics $\bbfw_1(t)^\top\bfw_j^\star$ and $\bbfw_2(t)^\top\bfw_j^\star$ with a small different $\Delta \geq 0$ at initialization
\[
    0 \leq (1+\Delta)\bbfw_1(t)^\top\bfw_j^\star \leq \bbfw_2(t)^\top\bfw_j^\star \leq \tilde{\calO}\paren{\frac{1}{\sqrt{d}}}
\]
Solving the ODE in (\ref{eq:approx_guided}) gives that for $i\in\{1,2\}$
\[
    \bbfw_i(t)^\top\bbfw_j^\star \approx \paren{\paren{\bbfw_i(0)^\top\bbfw_j^\star}^{-1} - t}^{-1}
\]
Thus, the time $T$ required for $\bbfw_2(t)^\top\bfw_j^\star \geq \frac{1}{2}$ is $T = \paren{\bbfw_2(0)^\top\bbfw_j^\star}^{-1} - 2 \leq \paren{\bbfw_2(0)^\top\bbfw_j^\star}^{-1}$.
However, at time $T$, one can compute that
\begin{align*}
    \bbfw_1(t)^\top\bbfw_j^\star & \leq \paren{\paren{\bbfw_1(0)^\top\bbfw_j^\star}^{-1} - \paren{\bbfw_2(0)^\top\bbfw_j^\star}^{-1}}^{-1}\\
    & \leq \frac{\bbfw_2(0)^\top\bbfw_j^\star}{\Delta} \leq \tilde{\calO}\paren{\frac{\Delta^{-1}}{\sqrt{d}}}
\end{align*}
When $\sqrt{d}\gg \Delta^{-1}$, we can conclude that $\bbfw_2(T)^\top\bfw_j^\star \geq \frac{1}{2}$ while $\bbfw_1(T)^\top\bfw_j^\star \ll \frac{1}{2}$. This behavior implies that, the expert in the student model that aligns with expert $j$ in the teacher best at initialization will converge to some $\Omega\paren{1}$ quickly while the other experts' alignment remains small. Below, we formalize this dominance determined by the initialization.

\textbf{Alignment gap at Initialization.} We show that there is a set of experts in the student model that aligns with each expert in the teacher model good enough to create a gap compared with other experts in the student model. At a high level, our goal here is to construct the mapping $\mathcal{I}$ in Theorem~\ref{thm:gf_conv} based on the initialization. Our approach is a greedy forward selection similar to \cite{ren2025emergencescalinglawssgd}. In particular, we define 
Define $\mathcal{R}_{\ell} = \{i_k^\star\}_{k=1}^{\ell}$ and $\mathcal{C}_{\ell} = \{j_\ell^\star\}_{k=1}^{\ell}$ recursively as follows
\begin{equation}
    \label{eq:greedy_pairing}
    i_{\ell+1}^\star,j_{\ell+1}^\star = \argmax_{i\in[m]\setminus \mathcal{R}_{\ell},j\in[m^\star]\setminus \mathcal{C}_{\ell}}\bfw_i(0)^\top\bbfw_j^\star
\end{equation}
We expect that $\mathcal{I}\paren{j_{\ell}^\star} = i_{\ell}^\star$. Namely, we expect $\bbfv_{i_\ell^\star}(t)^\top\bbfv_{j_\ell^\star}^\star$ and $\bbfw_{i_\ell^\star}(t)^\top\bbfw_{j_\ell^\star}^\star$ to converge to $1$. The index $\ell$ denotes the order of the sequential convergence. That is, we expect that $\bbfv_{i_1^\star}(t)^\top\bbfv_{j_1^\star}^\star$ to grow large first, followed by $\bbfv_{i_2^\star}(t)^\top\bbfv_{j_2^\star}^\star$, etc. Our theorem below shows that at initialization, the pairs $\bbfw_{i_\ell^\star}(t)^\top\bbfw_{j_\ell^\star}^\star$ has a gap compared with other alignment values.
\begin{lemma}
    \label{lem:init}
    Let $\bfw_1, \dots, \bfw_m\sim\mathcal{N}\paren{0,d^{-1}\bfI_d}$ be I.I.D. Gaussian random vectors. Define
    \begin{gather*}
        i_{\ell}^\star, j_{\ell}^\star = \argmax_{i\in[m]\setminus \mathcal{R}_{\ell-1},j\in[m^\star]\setminus \mathcal{C}_{\ell-1}}\bfw_i[j]\\ \mathcal{R}_{\ell} = \{i_k^\star\}_{k=1}^{\ell};\;\; \mathcal{C}_{\ell} = \{j_k^\star\}_{k=1}^{\ell}
    \end{gather*}
    Let any $\delta_{\mathbb{P}} \in (0, \sfrac{1}{2})$ be given. Then there exists some absolute constant $\beta_2, \beta_4 > 0$ such that if $m\geq  \beta_4m^\star\log \frac{m^\star}{\delta_{\mathbb{P}}}$, then for $\delta_s = \frac{\beta_2\delta_{\mathbb{P}}}{m^2}$, with probability at least $1 - 4\delta_{\mathbb{P}}$, it holds that
    \begin{itemize}
        \item (Row-wise Gap) $\bfw_{i^\star_\ell}[j^\star_{\ell}]\geq \paren{1+2\delta_s}\bfw_{i^\star_{\ell}}[j]$ for all $\ell\in[m^\star]$ and $j\in[m^\star]\setminus \mathcal{C}_{\ell}$
        \item (Column-wise Gap) $\bfw_{i^\star_{\ell}}[j^\star_{\ell}]\geq \paren{1+2\delta_s}\bfw_i[j^{\star}_{\ell}]$ for all $\ell\in[m^\star]$ and $i\in[m]\setminus\mathcal{R}_{\ell}$
        \item (Threshold Gap) $\bfw_{i^\star_{\ell}}[j^\star_{\ell}] \geq \paren{1+2\delta_s}\bfw_{i^\star_{\ell+1}}[j^\star_{\ell+1}]^2$ for all $\ell\in[m^\star-1]$
        \item (Magnitude Lower Bound) $\bfw_{i^\star_\ell}[j_\ell^\star]^2 \geq \frac{\log m^\star}{d}$ for all $\ell\in[m]$
    \end{itemize}
\end{lemma}
Since the standard Gaussian distribution is rotational invariant, we can the gaps and lower bound shown in Lemma~\ref{lem:init} to the initial alignment scores $\bbfw_i(0)^\top\bbfw_j^\star$. Roughly speaking, the \textit{row-wise gap} facilitates that $\bbfw_{i_\ell^\star}$ will not converge to $\bbfw_j$s with $j\neq j_{\ell}^\star$; the \textit{column-wise gap} induces the fact that no $\bbfw_i$ will converge to $\bbfw_{j_\ell^\star}$ except for $\bbfw_{i_\ell^\star}$. Moreover, the \textit{threshold gap} leads to the sequential recovery as stated in Theorem~\ref{thm:gf_conv}. Finally, the \textit{magnitude lower bound} guarantees that at the target alignment values at initialization are not too small for the whole convergence process to be too long.

\textbf{Induction-based Proof.} With the goal of tracking the growth of $\bbfv_i(t)^\top\bbfv_j^\star$ and $\bbfw_i(t)^\top\bbfw_j^\star$ in mind, however, we also have to track the ``mis-alignments'' including $\bbfv_i(t)^\top\bbfw_j^\star, \bbfw_i(t)^\top\bbfv_j^\star$ and ``self-alignments'' $\bbfv_i(t)^\top\bbfv_j(t), \bbfw_i(t)^\top\bbfw_j(t), \bbfv_i(t)\bbfw_j(t)$ due to the complicated form of the gradient, as can be seem from (\ref{eq:grad_full_form}) from the Appendix, so that their value does not interrupt with the target dynamics. To this end, our proof is an induction on $\ell\in[m^\star]$ that assumes
\begin{itemize}
    \item $\bbfv_{i_{\ell'}^\star}(t)^\top\bbfv_j^\star$ and $\bbfw_{i_{\ell'}^\star}(t)^\top\bbfw_j^\star$ are close to one for $\ell' < \ell$, i.e., the top-$\ell-1$ router and experts are recovered well while the $\ell$th router-expert pair still remains not learned.
    \item The ``mis-alignments'' and ``self-alignments'' associated with the recovered router-expert pairs must be small throughout the process.
\end{itemize}
to show that $\bbfv_{i_{\ell}^\star}(t)^\top\bbfv_j^\star$ and $\bbfw_{i_{\ell}^\star}(t)^\top\bbfw_j^\star$ converge to a close-to-one value. A formal statement of the inductive hypothesis is provided in Appendix~\ref{sec:proof_outline}, and the complete proof is provided in Appendix~\ref{sec:main_proof}.

\section{PRUNING AND FINE-TUNING}
\label{sec:pruning_finetuning}
Theorem~\ref{thm:gf_conv} guarantees that in $\calO\paren{\sqrt{d}}$ time, the student MoE model trained with gradient flow extracts $m^\star$-pairs of near-perfect features from the teacher model. However, recall that the student have an over-parameterization of $m \geq \Omega\paren{m^\star\log \frac{m^\star}{\delta_{\mathbb{P}}}}$. Despite being moderate, the $\log\frac{m^\star}{\delta_{\mathbb{P}}}$ factor still leads to a large number of excessive parameters. Continuing to train these unused experts together with their corresponding router parameter results in a wast of the computation resource, regardless of whether they can converge to zero output or not. This theoretical insight corresponds with existing empirical works \citep{lu2024expertsequalefficientexpert,chowdhury2024provablyeffectivemethodpruning,zhang2025diversifyingexpertknowledgetaskagnostic} which discovers the existence of redundant experts in pre-trained LLMs.

\subsection{Pruning the Redundant Experts}
\label{sec:pruning}
In this section, we adopt a greedy pruning algorithm based on the test loss similar to \cite{lu2024expertsequalefficientexpert} to remove the redundant experts, and show that, if we apply the algorithm at $T^\star \leq t \leq T^\star +\calO\paren{\frac{\sqrt{d}}{\delta_{\mathbb{P}}m^2}}$, then we can provably remove all the unused experts and keep all the correctly learned router-expert pairs as stated in Theorem~\ref{thm:gf_conv}. To state the algorithm, we first define the sub-model MoE induced by $\mathcal{S}\subseteq [m]$ as
\[
    f_{\mathcal{S}}\paren{\bm{\theta},\bfx} = \sum_{i\in[m]\setminus\mathcal{S}}\pi\paren{\bbfv_i^\top\bfx}\sigma\paren{\bbfw_i^\top\bfx}
\]
We consider the following pruning procedure that iteratively constructs the pruned set $\mathcal{S}$. In the $\tau$th step, we identify an index $r_{\tau}\in[m]\setminus \mathcal{S}_{\tau-1}$
\begin{equation}
    \label{eq:pruning_procedure}
    \begin{gathered}
        r_{\tau} = \argmin_{r\in[m]\setminus \mathcal{S}_{\tau-1}}\EXP[\bfx]{\paren{f_{S_{\tau-1}\cup\{r\}}\paren{\bm{\theta},\bfx} - f^\star\paren{\bfx}}^2}\\ \mathcal{S}_{\tau} = \mathcal{S}_{\tau-1} \cup\left\{r_{\tau}\right\}
    \end{gathered}
\end{equation}
The procedure will stop when pruning one more expert does not improve the population loss. In particular, we define the stopping step $\tau^\star$ be such that
\begin{equation}
    \label{eq:stopping_criteria}
    \begin{aligned}
        & \min_{r\in[m]}\EXP[\bfx]{\paren{f_{S_{\tau^\star}\cup\{r\}}\paren{\bm{\theta},\bfx} - f^\star\paren{\bfx}}^2}\\
        & \qqquad \geq \EXP[\bfx]{\paren{f_{\mathcal{S}_{\tau^\star}}\paren{\bm{\theta},\bfx} - f^\star\paren{\bfx}}^2}
    \end{aligned}
\end{equation}
For the simplicity of the analysis, we assume that $\mathcal{I}(i) = i$, as reordering the router-expert pairs does not change $f\paren{\bm{\theta},\bfx}$. To facilitate the analysis of the pruning procedure, we make the following assumption
\begin{asump}
    \label{asump:post-training}
    Let $\left\{\bfv_i\right\}_{i=1}^m$ and $\left\{\bfw_i\right\}_{i=1}^m$ be the router and expert weights of the MoE model in (\ref{eq:moe}). Let $\left\{\bbfv_i^\star\right\}_{i=1}^{m^\star}$ and $\left\{\bbfw_i^\star\right\}_{i=1}^{m^\star}$ be the router and expert weights of the teacher model in (\ref{eq:teacher}). There exists $\varepsilon \leq o\paren{\frac{1}{\sqrt{m}}}$ such that for all $i\in[m^\star]$ it holds that
    \[
        \min\left\{\bbfv_i^\top\bbfv_i^\star,\bbfw_i^\top\bbfw_i^\star\right\} \geq 1 - \varepsilon,
    \]
    for all $i_1\in[m]\setminus [m^\star], i_2\in[m],i_1\neq i_2$ and $j\in[m^\star]$ it holds that
    \begin{gather*}
        \left|\bbfv_{i_1}^\top\bbfv_{i_2}\right|, \left|\bbfw_{i_1}^\top\bbfw_{i_2}\right|, \left|\bbfv_{i_1}^\top\bbfw_{i_2}\right|, \left|\bbfv_{i_2}^\top\bbfw_{i_1}\right|, \left|\bbfv_{i_1}^\top\bbfw_{i_1}\right| \leq \varepsilon\\
        \left|\bbfv_i^\top\bbfv_j^\star\right|, \left|\bbfw_i^\top\bbfw_j^\star\right|, \left|\bbfv_i^\top\bbfw_j^\star\right|, \left|\bbfw_i^\top\bbfv_j^\star\right| \leq \varepsilon
    \end{gather*}
\end{asump}
\begin{theorem}
    \label{thm:pruning_guarantee}
    Let $f_{\mathcal{S}}\paren{\bm{\theta},\bfx},\mathcal{S}_{\tau}$, and $\tau^\star$ be defined above. If Assumption~\ref{asump:post-training} holds, then we have that\vspace{-0.2cm}
    \[
        f_{\mathcal{S}_{\tau^\star}}\paren{\bm{\theta},\bfx} = \sum_{i=1}^{m^\star}\pi\paren{\bbfv_i^\top\bfx}\sigma\paren{\bbfw_i^\top\bfx}\vspace{-0.2cm}
    \]
\end{theorem}
Theorem~\ref{thm:pruning_guarantee} states that, after $\tau^\star$ steps of pruning, the resulting model contains the exact $m^\star$ router-expert pairs with learned features from the teacher model. As a condition of Theorem~\ref{thm:pruning_guarantee}, Assumption~\ref{asump:post-training} is satisfied by Theorem~\ref{thm:gf_conv} with $\varepsilon = \calO\paren{\frac{m^2}{\delta_{\mathbb{P}}\sqrt{d}}} \leq o\paren{\frac{1}{\sqrt{m}}}$ since $d\gg m$. This implies that, if we perform the pruning at $T^\star \leq t \leq T^\star + \calO\paren{\frac{m^2}{\delta_{\mathbb{P}}\sqrt{d}}}$ in the gradient flow process, we are guaranteed to remove all unused router-expert pairs and keep all necessary ones.

Notice that in the pruning procedure we evaluate the model on the population loss. To apply the algorithm in practice, one can effectively approximate the population loss with the sample loss. We use the population loss for the succinctness of the theoretical analysis. 

\textit{Sketch of Proof.} From a high level perspective, our proof relies on the observation that for two nonlinear function $h_1, h_2:\R\rightarrow \R$ and vectors $\bfu_1,\bfu_2$ with $\bfu_1^\top\bfu_2\approx 0$, we have that
\begin{align*}
    & \EXP[\bfx]{h_1\paren{\bfu_1^\top\bfx}h_2\paren{\bfu_2^\top\bfx}} \\
    & \qqquad \approx\EXP[\bfx]{h_1\paren{\bfu_1^\top\bfx}}\EXP[\bfx]{h_2\paren{\bfu_2^\top\bfx}}
\end{align*}
Let $q\paren{\bbfv_i, \bbfw_i,\bfx}= \pi\paren{\bbfv_i^\top\bfx}\sigma\paren{\bbfw_i^\top\bfx}$. This allows us to approximate the loss as
\begin{align*}
    \calL\paren{\theta} & \approx \EXP[\bfx]{\paren{\sum_{i=1}^{m^\star}\paren{q\paren{\bbfv_i, \bbfw_i,\bfx} - q\paren{\bbfv_i^\star, \bbfw_i^\star,\bfx}}}^2}\\
    & \qqquad + \sum_{i=m^\star+1}^m\EXP[\bfx]{\pi\paren{\bbfv_i^\top\bfx}^2}\EXP[\bfx]{\sigma\paren{\bbfw_i^\top\bfx}^2}
\end{align*}
The first term is naturally small due to the fact that $\bbfv_i^\top\bbfv_i^\star$ and $\bbfw_i^\top\bbfw_i^\star$ are close to one for $i\in[m^\star]$. The second term involves a summation of positive terms, which depends on the redundant router-expert pairs. Thus, removing each one of these will decrease $\calL\paren{\bm{\theta}}$. The proof of Theorem~\ref{thm:pruning_guarantee} is provided in Appendix~\ref{sec:pruning_guarantee}.
\subsection{Fine-Tuning the Pruned Model}
Recall from Theorem~\ref{thm:gf_conv} that, although the $m^\star$ router-expert pairs in the student model extracted near-perfect features from the teacher model, there is still an $\calO\paren{\frac{m^7}{\delta_{\mathbb{P}}^3d^{\frac{3}{2}}}}$ error for each $\bbfv_{\mathcal{I}(i)}^\top\bbfv_i^\star$ and $\bbfw_{\mathcal{I}(i)}^\top\bbfw_i^\star$. This results in a non-zero loss even after the pruning in Section~\ref{sec:pruning}.
In this section, we study the convergence guarantee of fine-tuning the pruned model with gradient flow on the population MSE. In particular, we assume that $f\paren{\bm{\theta},\bfx}$ is the pruned model from Section~\ref{sec:pruning} given by\vspace{-0.2cm}
\[
    f\paren{\bm{\theta},\bfx} = \sum_{i=1}^{m^\star}\pi\paren{\bbfv_i^\top\bfx}\sigma\paren{\bbfw_i^\top\bfx}\vspace{-0.2cm}
\]
and the fine-tuning starts at $\bm{\theta}(T_0)$ learned from Theorem~\ref{thm:gf_conv} at time $T^\star\leq T_0 \leq \calO\paren{\frac{m^4}{\delta_{\mathbb{P}}^2d}}$.
We slightly abuse the notation by denoting
\[
    \bm{\theta} = [\bfv_1^\top, \dots,\bfv_{m^\star}^\top, \bfw_1^\top, \dots,\bfw_{m^\star}^\top]^\top \in\R^{2m^\star d}
\]
Moreover, we denote the normalized version of $\bm{\theta}$ as 
\[
    \bar{\bm{\theta}} = [\bbfv_1^\top, \dots,\bbfv_{m^\star}^\top, \bbfw_1^\top, \dots,\bbfw_{m^\star}^\top]^\top \in\R^{2m^\star d}
\]

The following theorem shows the convergence of gradient flow in the fine-tuning phase.
\begin{theorem}
    \label{thm:fine-tuning}
    Let $\bm{\theta}(T_0)$ that satisfy $\norm{\bbfv_i - \bbfv_i^\star}_2\leq \varepsilon$ and $\norm{\bbfw_i - \bbfw_i^\star}_2\leq \varepsilon$ for some $\varepsilon \leq o\paren{\frac{1}{m^{\star 2}}}$. Let $C_{S,0} = 2\EXP[x\sim\mathcal{N}(0,1)]{\pi\paren{x}^2}$ and $C_{S,1} = 6\EXP[x\sim\mathcal{N}(0,1)]{\pi'\paren{x}^2}$. If $C_{S,0} \geq 1.1C_{S,1}$, then there exists some constant $\kappa > 0$ that only depends on the property of the sigmoid function $\pi\paren{\cdot}$ such that
    \[
        \norm{\bar{\bm{\theta}}\paren{t+T_0} - \bm{\theta}^\star}_2^2 \leq \exp{-\frac{\kappa t}{2}}\norm{\bar{\bm{\theta}}(T_0) - \bm{\theta}^\star}_2^2\vspace{-0.2cm}
    \]
\end{theorem}
Under the condition that $\bbfv_i$'s and $\bbfw_i$'s are $\varepsilon$-close to $\bbfv_i^\star$'s and $\bbfw_i^\star$'s, Theorem~\ref{thm:fine-tuning} shows a linear convergence rate in terms of the difference between the pruned model's normalized parameters $\bar{\bm{\theta}}$ and the optimal parameters $\bm{\theta}^\star$. In this fine-tuning stage, the convergence rate $\kappa$ is independent of the dimension $d$ or the number of experts $m^\star$. Instead, it only depends on the property of the router's non-linear function $\pi(\cdot)$. Since $\bm{\theta}\paren{T_0}$ is given by the learned result in Theorem~\ref{thm:gf_conv}, the condition that $\norm{\bbfv_i - \bbfv_i^\star}_2\leq \varepsilon$ and $\norm{\bbfw_i - \bbfw_i^\star}_2\leq \varepsilon$ for some $\varepsilon\leq \calO\paren{\frac{m}{\sqrt{\delta_{\mathbb{P}}}d^{\frac{1}{4}}}} \leq o\paren{\frac{1}{m^{\star 2}}}$ are automatically satisfied under $d\gg m$, since
\[
    \norm{\bbfv_i - \bbfv_i^\star}_2^2 = 2 - 2\bbfv_i^\top\bbfv_i^\star \leq \calO\paren{\frac{m^2}{\delta_{\mathbb{P}}\sqrt{d}}}
\]
The assumption that $C_{S,0}\geq 1.1C_{S,1}$ only depends on the property of $\pi(\cdot)$ and is checked in Appendix~\ref{sec:plot_sigmoid}.

\textit{Sketch of Proof.} Our proof relies on the idea that, near the global minimum, the Hessian matrix is positive definite. In particular, we show that for any vector $\bfu_1,\bfu_2\in\R^{2m^\star d}$ such that $\cos\inner{\bfu_1}{\bfu_2} \approx 1$, it holds that $\bfu_1^\top\nabla^2\calL\paren{\bm{\theta}}\bfu_2 \geq \kappa\norm{\bfu_1}_2\norm{\bfu_2}$ for some constant $\kappa > 0$ and $\bm{\theta}$ satisfying $\norm{\bbfv_i - \bbfv_i^\star}_2\leq \varepsilon$ and $\norm{\bbfw_i - \bbfw_i^\star}_2\leq \varepsilon$ with $\varepsilon \leq o\paren{\frac{1}{m^{\star 2}}}$. Theorem~\ref{thm:hessian_pd} in Appendix~\ref{sec:proof_fine_tuning} provides a formal statement of the result. Based on the positive-definiteness of the Hessian matrix, we leverage the classic convex optimization technique to show that the trajectory never leaves the neighborhood near the global minima, and that the distance to the global minimum converges linearly. The full proof is deferred to Appendix~\ref{sec:proof_fine_tuning}.

\section{CONCLUSION}

Under the teacher-student set-up, we study the learning dynamics of the sigmoid-routed MoE with nonlinear experts defined in (\ref{eq:moe}) when trained with gradient flow on the population MSE over high dimensional Gaussian inputs. In particular, our main result is a characterization of the feature learning stage, where proper features of the router-expert pairs are discovered in sequential order, with the expert's recovery leading the router's recovery. At the end of the feature learning stage, we show that a pruning procedure can be conducted to provably remove all the redundant experts and keep all necessary ones. Lastly, we show a linear convergence rate to the global minima for the the post-pruning fine-tuning with gradient flow. To the best of our knowledge, our work is the first to provide theoretical understanding on the joint training guarantee of MoEs with more than one activated experts and a general data assumption. In general, our paper is a further step into understanding the complicated dynamics of MoE training, and leads to the following open problems:

\textbf{Online SGD and Sample Complexity.} Due to the already sophisticated proof, our study is restricted to the setting of gradient flow on the population loss. However, as the main idea of the proof consists of an ODE based dynamic analysis, one could discretize the dynamic and apply martingale-based analysis to extend the theory to online SGD, as in \cite{ren2025emergencescalinglawssgd}. This extension may lead to a sample complexity bound of learning $m^\star$ experts on $d$-dimensional data.

\textbf{Experts with different importance.} In our work we considered the teacher's router and expert parameter $\bbfv_i^\star$ and $\bbfw_i^\star$ to be an orthonormal list. Due to the rotational invariance, this set-up puts equal importance to each router-expert pairs. Future work can investigate the scenario where the $i$th expert is scaled with a factor of $\alpha_i$, and study the explicit ordering of the recovered experts in the student model.

\textbf{Relax the dependency of $d$ on $m$.} Our current theory relies on the fact that $d\gg m$. While in the practical application of MoE we rarely set the number of experts to be larger than the input dimension, in most cases the scale of the two remains relatively the same. A meaningful future direction is to bridge the gap by studying the setting where $d$ is only moderately larger than $m$.

\bibliography{references}

\begin{thebibliography}{44}
\providecommand{\natexlab}[1]{#1}
\providecommand{\url}[1]{\texttt{#1}}
\expandafter\ifx\csname urlstyle\endcsname\relax
  \providecommand{\doi}[1]{doi: #1}\else
  \providecommand{\doi}{doi: \begingroup \urlstyle{rm}\Url}\fi

\bibitem[Arous et~al.(2025)Arous, Erdogdu, Vural, and Wu]{arous2025learningquadraticneuralnetworks}
Gérard~Ben Arous, Murat~A. Erdogdu, N.~Mert Vural, and Denny Wu.
\newblock Learning quadratic neural networks in high dimensions: Sgd dynamics and scaling laws, 2025.
\newblock URL \url{https://arxiv.org/abs/2508.03688}.

\bibitem[Ba et~al.(2023)Ba, Erdogdu, Suzuki, Wang, and Wu]{ba2023learning}
Jimmy Ba, Murat~A Erdogdu, Taiji Suzuki, Zhichao Wang, and Denny Wu.
\newblock Learning in the presence of low-dimensional structure: A spiked random matrix perspective.
\newblock In \emph{Thirty-seventh Conference on Neural Information Processing Systems}, 2023.
\newblock URL \url{https://openreview.net/forum?id=HlIAoCHDWW}.

\bibitem[Bhatt et~al.(2025)Bhatt, Kapoor, Upadhyay, Sucholutsky, Quinzan, Collins, Weller, Wilson, and Zafar]{bhatt2025orchestratemultipleagents}
Umang Bhatt, Sanyam Kapoor, Mihir Upadhyay, Ilia Sucholutsky, Francesco Quinzan, Katherine~M. Collins, Adrian Weller, Andrew~Gordon Wilson, and Muhammad~Bilal Zafar.
\newblock When should we orchestrate multiple agents?, 2025.
\newblock URL \url{https://arxiv.org/abs/2503.13577}.

\bibitem[Bietti et~al.(2022)Bietti, Bruna, Sanford, and Song]{bietti2022learningsingleindexmodelsshallow}
Alberto Bietti, Joan Bruna, Clayton Sanford, and Min~Jae Song.
\newblock Learning single-index models with shallow neural networks, 2022.
\newblock URL \url{https://arxiv.org/abs/2210.15651}.

\bibitem[Bietti et~al.(2023)Bietti, Bruna, and Pillaud-Vivien]{bietti2023learninggaussianmultiindexmodels}
Alberto Bietti, Joan Bruna, and Loucas Pillaud-Vivien.
\newblock On learning gaussian multi-index models with gradient flow, 2023.
\newblock URL \url{https://arxiv.org/abs/2310.19793}.

\bibitem[Boix-Adsera and Rigollet(2025)]{boixadsera2025powerfinegrainedexpertsgranularity}
Enric Boix-Adsera and Philippe Rigollet.
\newblock The power of fine-grained experts: Granularity boosts expressivity in mixture of experts, 2025.
\newblock URL \url{https://arxiv.org/abs/2505.06839}.

\bibitem[Chen et~al.(2022)Chen, Deng, Wu, Gu, and Li]{chen2022understandingmixtureexpertsdeep}
Zixiang Chen, Yihe Deng, Yue Wu, Quanquan Gu, and Yuanzhi Li.
\newblock Towards understanding mixture of experts in deep learning, 2022.
\newblock URL \url{https://arxiv.org/abs/2208.02813}.

\bibitem[Chowdhury et~al.(2023)Chowdhury, Zhang, Wang, Liu, and Chen]{chowdhury2023patchlevelroutingmixtureofexpertsprovably}
Mohammed Nowaz~Rabbani Chowdhury, Shuai Zhang, Meng Wang, Sijia Liu, and Pin-Yu Chen.
\newblock Patch-level routing in mixture-of-experts is provably sample-efficient for convolutional neural networks, 2023.
\newblock URL \url{https://arxiv.org/abs/2306.04073}.

\bibitem[Chowdhury et~al.(2024)Chowdhury, Wang, Maghraoui, Wang, Chen, and Carothers]{chowdhury2024provablyeffectivemethodpruning}
Mohammed Nowaz~Rabbani Chowdhury, Meng Wang, Kaoutar~El Maghraoui, Naigang Wang, Pin-Yu Chen, and Christopher Carothers.
\newblock A provably effective method for pruning experts in fine-tuned sparse mixture-of-experts, 2024.
\newblock URL \url{https://arxiv.org/abs/2405.16646}.

\bibitem[Damian et~al.(2022)Damian, Lee, and Soltanolkotabi]{damian2022neuralnetworkslearnrepresentations}
Alex Damian, Jason~D. Lee, and Mahdi Soltanolkotabi.
\newblock Neural networks can learn representations with gradient descent, 2022.
\newblock URL \url{https://arxiv.org/abs/2206.15144}.

\bibitem[DeepSeek-AI et~al.(2025)DeepSeek-AI, Liu, Feng, Xue, Wang, Wu, Lu, Zhao, Deng, Zhang, Ruan, Dai, Guo, Yang, Chen, Ji, Li, Lin, Dai, Luo, Hao, Chen, Li, Zhang, Bao, Xu, Wang, Zhang, Ding, Xin, Gao, Li, Qu, Cai, Liang, Guo, Ni, Li, Wang, Chen, Chen, Yuan, Qiu, Li, Song, Dong, Hu, Gao, Guan, Huang, Yu, Wang, Zhang, Xu, Xia, Zhao, Wang, Zhang, Li, Wang, Zhang, Zhang, Tang, Li, Tian, Huang, Wang, Zhang, Wang, Zhu, Chen, Du, Chen, Jin, Ge, Zhang, Pan, Wang, Xu, Zhang, Chen, Li, Lu, Zhou, Chen, Wu, Ye, Ye, Ma, Wang, Zhou, Yu, Zhou, Pan, Wang, Yun, Pei, Sun, Xiao, Zeng, Zhao, An, Liu, Liang, Gao, Yu, Zhang, Li, Jin, Wang, Bi, Liu, Wang, Shen, Chen, Zhang, Chen, Nie, Sun, Wang, Cheng, Liu, Xie, Liu, Yu, Song, Shan, Zhou, Yang, Li, Su, Lin, Li, Wang, Wei, Zhu, Zhang, Xu, Xu, Huang, Li, Zhao, Sun, Li, Wang, Yu, Zheng, Zhang, Shi, Xiong, He, Tang, Piao, Wang, Tan, Ma, Liu, Guo, Wu, Ou, Zhu, Wang, Gong, Zou, He, Zha, Xiong, Ma, Yan, Luo, You, Liu, Zhou, Wu, Ren, Ren, Sha, Fu, Xu, Huang, Zhang, Xie, Zhang, Hao,
  Gou, Ma, Yan, Shao, Xu, Wu, Zhang, Li, Gu, Zhu, Liu, Li, Xie, Song, Gao, and Pan]{deepseekai2025deepseekv3technicalreport}
DeepSeek-AI, Aixin Liu, Bei Feng, Bing Xue, Bingxuan Wang, Bochao Wu, Chengda Lu, Chenggang Zhao, Chengqi Deng, Chenyu Zhang, Chong Ruan, Damai Dai, Daya Guo, Dejian Yang, Deli Chen, Dongjie Ji, Erhang Li, Fangyun Lin, Fucong Dai, Fuli Luo, Guangbo Hao, Guanting Chen, Guowei Li, H.~Zhang, Han Bao, Hanwei Xu, Haocheng Wang, Haowei Zhang, Honghui Ding, Huajian Xin, Huazuo Gao, Hui Li, Hui Qu, J.~L. Cai, Jian Liang, Jianzhong Guo, Jiaqi Ni, Jiashi Li, Jiawei Wang, Jin Chen, Jingchang Chen, Jingyang Yuan, Junjie Qiu, Junlong Li, Junxiao Song, Kai Dong, Kai Hu, Kaige Gao, Kang Guan, Kexin Huang, Kuai Yu, Lean Wang, Lecong Zhang, Lei Xu, Leyi Xia, Liang Zhao, Litong Wang, Liyue Zhang, Meng Li, Miaojun Wang, Mingchuan Zhang, Minghua Zhang, Minghui Tang, Mingming Li, Ning Tian, Panpan Huang, Peiyi Wang, Peng Zhang, Qiancheng Wang, Qihao Zhu, Qinyu Chen, Qiushi Du, R.~J. Chen, R.~L. Jin, Ruiqi Ge, Ruisong Zhang, Ruizhe Pan, Runji Wang, Runxin Xu, Ruoyu Zhang, Ruyi Chen, S.~S. Li, Shanghao Lu, Shangyan Zhou, Shanhuang
  Chen, Shaoqing Wu, Shengfeng Ye, Shengfeng Ye, Shirong Ma, Shiyu Wang, Shuang Zhou, Shuiping Yu, Shunfeng Zhou, Shuting Pan, T.~Wang, Tao Yun, Tian Pei, Tianyu Sun, W.~L. Xiao, Wangding Zeng, Wanjia Zhao, Wei An, Wen Liu, Wenfeng Liang, Wenjun Gao, Wenqin Yu, Wentao Zhang, X.~Q. Li, Xiangyue Jin, Xianzu Wang, Xiao Bi, Xiaodong Liu, Xiaohan Wang, Xiaojin Shen, Xiaokang Chen, Xiaokang Zhang, Xiaosha Chen, Xiaotao Nie, Xiaowen Sun, Xiaoxiang Wang, Xin Cheng, Xin Liu, Xin Xie, Xingchao Liu, Xingkai Yu, Xinnan Song, Xinxia Shan, Xinyi Zhou, Xinyu Yang, Xinyuan Li, Xuecheng Su, Xuheng Lin, Y.~K. Li, Y.~Q. Wang, Y.~X. Wei, Y.~X. Zhu, Yang Zhang, Yanhong Xu, Yanhong Xu, Yanping Huang, Yao Li, Yao Zhao, Yaofeng Sun, Yaohui Li, Yaohui Wang, Yi~Yu, Yi~Zheng, Yichao Zhang, Yifan Shi, Yiliang Xiong, Ying He, Ying Tang, Yishi Piao, Yisong Wang, Yixuan Tan, Yiyang Ma, Yiyuan Liu, Yongqiang Guo, Yu~Wu, Yuan Ou, Yuchen Zhu, Yuduan Wang, Yue Gong, Yuheng Zou, Yujia He, Yukun Zha, Yunfan Xiong, Yunxian Ma, Yuting Yan, Yuxiang
  Luo, Yuxiang You, Yuxuan Liu, Yuyang Zhou, Z.~F. Wu, Z.~Z. Ren, Zehui Ren, Zhangli Sha, Zhe Fu, Zhean Xu, Zhen Huang, Zhen Zhang, Zhenda Xie, Zhengyan Zhang, Zhewen Hao, Zhibin Gou, Zhicheng Ma, Zhigang Yan, Zhihong Shao, Zhipeng Xu, Zhiyu Wu, Zhongyu Zhang, Zhuoshu Li, Zihui Gu, Zijia Zhu, Zijun Liu, Zilin Li, Ziwei Xie, Ziyang Song, Ziyi Gao, and Zizheng Pan.
\newblock Deepseek-v3 technical report, 2025.
\newblock URL \url{https://arxiv.org/abs/2412.19437}.

\bibitem[Du et~al.(2019)Du, Lee, Li, Wang, and Zhai]{du2019gradientdescentfindsglobal}
Simon~S. Du, Jason~D. Lee, Haochuan Li, Liwei Wang, and Xiyu Zhai.
\newblock Gradient descent finds global minima of deep neural networks, 2019.
\newblock URL \url{https://arxiv.org/abs/1811.03804}.

\bibitem[Fedus et~al.(2022)Fedus, Zoph, and Shazeer]{fedus2022switch}
William Fedus, Barret Zoph, and Noam Shazeer.
\newblock Switch transformers: Scaling to trillion parameter models with simple and efficient sparsity.
\newblock \emph{Journal of Machine Learning Research}, 23\penalty0 (120):\penalty0 1--39, 2022.

\bibitem[Fruytier et~al.(2025)Fruytier, Mokhtari, and Sanghavi]{fruytier2025learningmixturesexpertsem}
Quentin Fruytier, Aryan Mokhtari, and Sujay Sanghavi.
\newblock Learning mixtures of experts with em: A mirror descent perspective, 2025.
\newblock URL \url{https://arxiv.org/abs/2411.06056}.

\bibitem[Ge et~al.(2017)Ge, Lee, and Ma]{ge2017learningonehiddenlayerneuralnetworks}
Rong Ge, Jason~D. Lee, and Tengyu Ma.
\newblock Learning one-hidden-layer neural networks with landscape design, 2017.
\newblock URL \url{https://arxiv.org/abs/1711.00501}.

\bibitem[Hu et~al.(2025)Hu, Lu, and Clune]{hu2025automateddesignagenticsystems}
Shengran Hu, Cong Lu, and Jeff Clune.
\newblock Automated design of agentic systems, 2025.
\newblock URL \url{https://arxiv.org/abs/2408.08435}.

\bibitem[Jacot et~al.(2020)Jacot, Gabriel, and Hongler]{jacot2020neuraltangentkernelconvergence}
Arthur Jacot, Franck Gabriel, and Clément Hongler.
\newblock Neural tangent kernel: Convergence and generalization in neural networks, 2020.
\newblock URL \url{https://arxiv.org/abs/1806.07572}.

\bibitem[Jiang et~al.(2024)Jiang, Sablayrolles, Roux, Mensch, Savary, Bamford, Chaplot, Casas, Hanna, Bressand, et~al.]{jiang2024mixtral}
Albert~Q Jiang, Alexandre Sablayrolles, Antoine Roux, Arthur Mensch, Blanche Savary, Chris Bamford, Devendra~Singh Chaplot, Diego de~las Casas, Emma~Bou Hanna, Florian Bressand, et~al.
\newblock Mixtral of experts.
\newblock \emph{arXiv preprint arXiv:2401.04088}, 2024.

\bibitem[Kawata et~al.(2025)Kawata, Matsutani, Kinoshita, Nishikawa, and Suzuki]{kawata2025mixtureexpertsprovablydetect}
Ryotaro Kawata, Kohsei Matsutani, Yuri Kinoshita, Naoki Nishikawa, and Taiji Suzuki.
\newblock Mixture of experts provably detect and learn the latent cluster structure in gradient-based learning, 2025.
\newblock URL \url{https://arxiv.org/abs/2506.01656}.

\bibitem[Kong et~al.(2025)Kong, Ma, Zhao, Wang, Shen, Wang, and Tao]{kong2025masteringmassivemultitaskreinforcement}
Yilun Kong, Guozheng Ma, Qi~Zhao, Haoyu Wang, Li~Shen, Xueqian Wang, and Dacheng Tao.
\newblock Mastering massive multi-task reinforcement learning via mixture-of-expert decision transformer, 2025.
\newblock URL \url{https://arxiv.org/abs/2505.24378}.

\bibitem[Kratsios et~al.(2024)Kratsios, Furuya, Benitez, Lassas, and de~Hoop]{kratsios2024mixtureexpertssoftencurse}
Anastasis Kratsios, Takashi Furuya, Jose Antonio~Lara Benitez, Matti Lassas, and Maarten de~Hoop.
\newblock Mixture of experts soften the curse of dimensionality in operator learning, 2024.
\newblock URL \url{https://arxiv.org/abs/2404.09101}.

\bibitem[Lee et~al.(2024)Lee, Oko, Suzuki, and Wu]{lee2024neuralnetworklearnslowdimensional}
Jason~D. Lee, Kazusato Oko, Taiji Suzuki, and Denny Wu.
\newblock Neural network learns low-dimensional polynomials with sgd near the information-theoretic limit, 2024.
\newblock URL \url{https://arxiv.org/abs/2406.01581}.

\bibitem[Li et~al.(2025)Li, Lin, Duan, Liang, and Shroff]{li2025theorymixtureofexpertscontinuallearning}
Hongbo Li, Sen Lin, Lingjie Duan, Yingbin Liang, and Ness~B. Shroff.
\newblock Theory on mixture-of-experts in continual learning, 2025.
\newblock URL \url{https://arxiv.org/abs/2406.16437}.

\bibitem[Lu et~al.(2024)Lu, Liu, Xu, Zhou, Huang, Zhang, Yan, and Li]{lu2024expertsequalefficientexpert}
Xudong Lu, Qi~Liu, Yuhui Xu, Aojun Zhou, Siyuan Huang, Bo~Zhang, Junchi Yan, and Hongsheng Li.
\newblock Not all experts are equal: Efficient expert pruning and skipping for mixture-of-experts large language models, 2024.
\newblock URL \url{https://arxiv.org/abs/2402.14800}.

\bibitem[Mousavi-Hosseini et~al.(2023{\natexlab{a}})Mousavi-Hosseini, Park, Girotti, Mitliagkas, and Erdogdu]{mousavihosseini2023neuralnetworksefficientlylearn}
Alireza Mousavi-Hosseini, Sejun Park, Manuela Girotti, Ioannis Mitliagkas, and Murat~A. Erdogdu.
\newblock Neural networks efficiently learn low-dimensional representations with sgd, 2023{\natexlab{a}}.
\newblock URL \url{https://arxiv.org/abs/2209.14863}.

\bibitem[Mousavi-Hosseini et~al.(2023{\natexlab{b}})Mousavi-Hosseini, Wu, Suzuki, and Erdogdu]{mousavihosseini2023gradientbasedfeaturelearningstructured}
Alireza Mousavi-Hosseini, Denny Wu, Taiji Suzuki, and Murat~A. Erdogdu.
\newblock Gradient-based feature learning under structured data, 2023{\natexlab{b}}.
\newblock URL \url{https://arxiv.org/abs/2309.03843}.

\bibitem[Mustafa et~al.(2022)Mustafa, Riquelme, Puigcerver, Jenatton, and Houlsby]{mustafa2022multimodal}
Basil Mustafa, Carlos Riquelme, Joan Puigcerver, Rodolphe Jenatton, and Neil Houlsby.
\newblock Multimodal contrastive learning with {LIM}o{E}: the language-image mixture of experts.
\newblock \emph{Advances in Neural Information Processing Systems}, 35:\penalty0 9564--9576, 2022.

\bibitem[Nguyen et~al.(2024{\natexlab{a}})Nguyen, Akbarian, Nguyen, and Ho]{nguyen2024generaltheorysoftmaxgating}
Huy Nguyen, Pedram Akbarian, TrungTin Nguyen, and Nhat Ho.
\newblock A general theory for softmax gating multinomial logistic mixture of experts, 2024{\natexlab{a}}.
\newblock URL \url{https://arxiv.org/abs/2310.14188}.

\bibitem[Nguyen et~al.(2024{\natexlab{b}})Nguyen, Ho, and Rinaldo]{nguyen2024sigmoidgatingsampleefficient}
Huy Nguyen, Nhat Ho, and Alessandro Rinaldo.
\newblock Sigmoid gating is more sample efficient than softmax gating in mixture of experts, 2024{\natexlab{b}}.
\newblock URL \url{https://arxiv.org/abs/2405.13997}.

\bibitem[Nguyen et~al.(2024{\natexlab{c}})Nguyen, Ho, and Rinaldo]{nguyen2024squareestimationsoftmaxgating}
Huy Nguyen, Nhat Ho, and Alessandro Rinaldo.
\newblock On least square estimation in softmax gating mixture of experts, 2024{\natexlab{c}}.
\newblock URL \url{https://arxiv.org/abs/2402.02952}.

\bibitem[Nguyen et~al.(2025)Nguyen, Ho, and Rinaldo]{nguyen2025convergenceratessoftmaxgating}
Huy Nguyen, Nhat Ho, and Alessandro Rinaldo.
\newblock Convergence rates for softmax gating mixture of experts, 2025.
\newblock URL \url{https://arxiv.org/abs/2503.03213}.

\bibitem[Ren et~al.(2025)Ren, Nichani, Wu, and Lee]{ren2025emergencescalinglawssgd}
Yunwei Ren, Eshaan Nichani, Denny Wu, and Jason~D. Lee.
\newblock Emergence and scaling laws in sgd learning of shallow neural networks, 2025.
\newblock URL \url{https://arxiv.org/abs/2504.19983}.

\bibitem[Riquelme et~al.(2021)Riquelme, Puigcerver, Mustafa, Neumann, Jenatton, Susano~Pinto, Keysers, and Houlsby]{riquelme2021scaling}
Carlos Riquelme, Joan Puigcerver, Basil Mustafa, Maxim Neumann, Rodolphe Jenatton, Andr{\'e} Susano~Pinto, Daniel Keysers, and Neil Houlsby.
\newblock Scaling vision with sparse mixture of experts.
\newblock \emph{Advances in Neural Information Processing Systems}, 34:\penalty0 8583--8595, 2021.

\bibitem[Shazeer et~al.(2017)Shazeer, Mirhoseini, Maziarz, Davis, Le, Hinton, and Dean]{shazeer2017outrageously}
Noam Shazeer, Azalia Mirhoseini, Krzysztof Maziarz, Andy Davis, Quoc Le, Geoffrey Hinton, and Jeff Dean.
\newblock Outrageously large neural networks: The sparsely-gated mixture-of-experts layer.
\newblock In \emph{International Conference on Learning Representations}, 2017.

\bibitem[Shi et~al.(2022)Shi, Wei, and Liang]{shi2022theoreticalanalysisfeaturelearning}
Zhenmei Shi, Junyi Wei, and Yingyu Liang.
\newblock A theoretical analysis on feature learning in neural networks: Emergence from inputs and advantage over fixed features, 2022.
\newblock URL \url{https://arxiv.org/abs/2206.01717}.

\bibitem[Shi et~al.(2023)Shi, Wei, and Liang]{shi2023provableguaranteesneuralnetworks}
Zhenmei Shi, Junyi Wei, and Yingyu Liang.
\newblock Provable guarantees for neural networks via gradient feature learning, 2023.
\newblock URL \url{https://arxiv.org/abs/2310.12408}.

\bibitem[Wang and E(2025)]{wang2025expressivepowermixtureofexpertsstructured}
Mingze Wang and Weinan E.
\newblock On the expressive power of mixture-of-experts for structured complex tasks, 2025.
\newblock URL \url{https://arxiv.org/abs/2505.24205}.

\bibitem[Wang et~al.(2020)Wang, Wu, Lee, Ma, and Ge]{wang2020lazytrainingoverparameterizedtensor}
Xiang Wang, Chenwei Wu, Jason~D. Lee, Tengyu Ma, and Rong Ge.
\newblock Beyond lazy training for over-parameterized tensor decomposition, 2020.
\newblock URL \url{https://arxiv.org/abs/2010.11356}.

\bibitem[Yang et~al.(2025)Yang, Li, Yang, Zhang, Hui, Zheng, Yu, Gao, Huang, Lv, Zheng, Liu, Zhou, Huang, Hu, Ge, Wei, Lin, Tang, Yang, Tu, Zhang, Yang, Yang, Zhou, Zhou, Lin, Dang, Bao, Yang, Yu, Deng, Li, Xue, Li, Zhang, Wang, Zhu, Men, Gao, Liu, Luo, Li, Tang, Yin, Ren, Wang, Zhang, Ren, Fan, Su, Zhang, Zhang, Wan, Liu, Wang, Cui, Zhang, Zhou, and Qiu]{yang2025qwen3technicalreport}
An~Yang, Anfeng Li, Baosong Yang, Beichen Zhang, Binyuan Hui, Bo~Zheng, Bowen Yu, Chang Gao, Chengen Huang, Chenxu Lv, Chujie Zheng, Dayiheng Liu, Fan Zhou, Fei Huang, Feng Hu, Hao Ge, Haoran Wei, Huan Lin, Jialong Tang, Jian Yang, Jianhong Tu, Jianwei Zhang, Jianxin Yang, Jiaxi Yang, Jing Zhou, Jingren Zhou, Junyang Lin, Kai Dang, Keqin Bao, Kexin Yang, Le~Yu, Lianghao Deng, Mei Li, Mingfeng Xue, Mingze Li, Pei Zhang, Peng Wang, Qin Zhu, Rui Men, Ruize Gao, Shixuan Liu, Shuang Luo, Tianhao Li, Tianyi Tang, Wenbiao Yin, Xingzhang Ren, Xinyu Wang, Xinyu Zhang, Xuancheng Ren, Yang Fan, Yang Su, Yichang Zhang, Yinger Zhang, Yu~Wan, Yuqiong Liu, Zekun Wang, Zeyu Cui, Zhenru Zhang, Zhipeng Zhou, and Zihan Qiu.
\newblock Qwen3 technical report, 2025.
\newblock URL \url{https://arxiv.org/abs/2505.09388}.

\bibitem[Zhang et~al.(2025{\natexlab{a}})Zhang, Song, Bi, Yuan, Wang, Yeong, and Hao]{zhang2025mixtureexpertslargelanguage}
Danyang Zhang, Junhao Song, Ziqian Bi, Yingfang Yuan, Tianyang Wang, Joe Yeong, and Junfeng Hao.
\newblock Mixture of experts in large language models, 2025{\natexlab{a}}.
\newblock URL \url{https://arxiv.org/abs/2507.11181}.

\bibitem[Zhang et~al.(2025{\natexlab{b}})Zhang, Yue, Sun, Wan, Yu, Fang, Wang, Chen, and Cheng]{zhang2025gdesignerarchitectingmultiagentcommunication}
Guibin Zhang, Yanwei Yue, Xiangguo Sun, Guancheng Wan, Miao Yu, Junfeng Fang, Kun Wang, Tianlong Chen, and Dawei Cheng.
\newblock G-designer: Architecting multi-agent communication topologies via graph neural networks, 2025{\natexlab{b}}.
\newblock URL \url{https://arxiv.org/abs/2410.11782}.

\bibitem[Zhang et~al.(2025{\natexlab{c}})Zhang, Qiu, Tan, Zhang, Lu, Peng, Xu, Agudelo, Qian, and Chen]{zhang2025symbioticcooperationwebagents}
Ruichen Zhang, Mufan Qiu, Zhen Tan, Mohan Zhang, Vincent Lu, Jie Peng, Kaidi Xu, Leandro~Z. Agudelo, Peter Qian, and Tianlong Chen.
\newblock Symbiotic cooperation for web agents: Harnessing complementary strengths of large and small llms, 2025{\natexlab{c}}.
\newblock URL \url{https://arxiv.org/abs/2502.07942}.

\bibitem[Zhang et~al.(2025{\natexlab{d}})Zhang, Liu, Cheng, Xu, and Gao]{zhang2025diversifyingexpertknowledgetaskagnostic}
Zeliang Zhang, Xiaodong Liu, Hao Cheng, Chenliang Xu, and Jianfeng Gao.
\newblock Diversifying the expert knowledge for task-agnostic pruning in sparse mixture-of-experts, 2025{\natexlab{d}}.
\newblock URL \url{https://arxiv.org/abs/2407.09590}.

\bibitem[Şimşek et~al.(2025)Şimşek, Bendjeddou, and Hsu]{şimşek2025learninggaussianmultiindexmodels}
Berfin Şimşek, Amire Bendjeddou, and Daniel Hsu.
\newblock Learning gaussian multi-index models with gradient flow: Time complexity and directional convergence, 2025.
\newblock URL \url{https://arxiv.org/abs/2411.08798}.

\end{thebibliography}
\bibliographystyle{plainnat}
\section*{Checklist}



\begin{enumerate}

  \item For all models and algorithms presented, check if you include:
  \begin{enumerate}
    \item A clear description of the mathematical setting, assumptions, algorithm, and/or model. [Yes/No/Not Applicable] \textbf{Yes}
    \item An analysis of the properties and complexity (time, space, sample size) of any algorithm. [Yes/No/Not Applicable] \textbf{Yes}
    \item (Optional) Anonymized source code, with specification of all dependencies, including external libraries. [Yes/No/Not Applicable] \textbf{Yes}
  \end{enumerate}

  \item For any theoretical claim, check if you include:
  \begin{enumerate}
    \item Statements of the full set of assumptions of all theoretical results. [Yes/No/Not Applicable] \textbf{Yes}
    \item Complete proofs of all theoretical results. [Yes/No/Not Applicable] \textbf{Yes}
    \item Clear explanations of any assumptions. [Yes/No/Not Applicable] \textbf{Yes}  
  \end{enumerate}

  \item For all figures and tables that present empirical results, check if you include:
  \begin{enumerate}
    \item The code, data, and instructions needed to reproduce the main experimental results (either in the supplemental material or as a URL). [Yes/No/Not Applicable] \textbf{Yes}
    \item All the training details (e.g., data splits, hyperparameters, how they were chosen). [Yes/No/Not Applicable] \textbf{Yes}
    \item A clear definition of the specific measure or statistics and error bars (e.g., with respect to the random seed after running experiments multiple times). [Yes/No/Not Applicable] \textbf{No}
    \item A description of the computing infrastructure used. (e.g., type of GPUs, internal cluster, or cloud provider). [Yes/No/Not Applicable] \textbf{No}
  \end{enumerate}

  \item If you are using existing assets (e.g., code, data, models) or curating/releasing new assets, check if you include:
  \begin{enumerate}
    \item Citations of the creator If your work uses existing assets. [Yes/No/Not Applicable] \textbf{Not Applicable}
    \item The license information of the assets, if applicable. [Yes/No/Not Applicable] \textbf{Not Applicable}
    \item New assets either in the supplemental material or as a URL, if applicable. [Yes/No/Not Applicable] \textbf{Not Applicable}
    \item Information about consent from data providers/curators. [Yes/No/Not Applicable] \textbf{Not Applicable}
    \item Discussion of sensible content if applicable, e.g., personally identifiable information or offensive content. [Yes/No/Not Applicable] \textbf{Not Applicable}
  \end{enumerate}

  \item If you used crowdsourcing or conducted research with human subjects, check if you include:
  \begin{enumerate}
    \item The full text of instructions given to participants and screenshots. [Yes/No/Not Applicable] \textbf{Not Applicable}
    \item Descriptions of potential participant risks, with links to Institutional Review Board (IRB) approvals if applicable. [Yes/No/Not Applicable] \textbf{Not Applicable}
    \item The estimated hourly wage paid to participants and the total amount spent on participant compensation. [Yes/No/Not Applicable] \textbf{Not Applicable}
  \end{enumerate}

\end{enumerate}

\clearpage
\appendix
\thispagestyle{empty}

\onecolumn
\aistatstitle{Supplementary Materials}
\section{Proof of Theorem~\ref{thm:gf_conv}}
\label{sec:main_proof}
\subsection{Proof Outline}
\label{sec:proof_outline}
\textbf{Initialization Property.} At initialization, the following property needs to be satisfied. 

\begin{cond}[Initialization]
    \label{cond:init}
    At initialization $\left\{\bfw_i(0)\right\}_{i=1}^m$ and $\left\{\bfv_i(0)\right\}_{i=1}^m$ satisfies
    \begin{itemize}
        \item $\bfw_{i_\ell^\star}(0)^\top\bbfw_{j_\ell^\star}^\star \geq \paren{1 + 2\delta_s}\bfw_{i_\ell^\star}(0)^\top\bbfw_{j}^\star$ for all $\ell\in[m^\star]$ and $j\in[m^\star]\setminus \mathcal{C}_{\ell+1}$.
        \item $\bfw_{i_\ell^\star}(0)^\top\bbfw_{j_\ell^\star}^\star \geq \paren{1 + 2\delta_s}\bfw_{i}(0)^\top\bbfw_{j_\ell^\star}^\star$ for all $\ell\in[m^\star]$ and $i\in[m]\setminus \mathcal{R}_{\ell+1}$.
        \item $\bfw_{i_\ell^\star}(0)^\top\bbfw_{j_\ell^\star}^\star \geq \paren{1 + 2\delta_s}\bfw_{i_{\ell+1}^\star}(0)^\top\bbfw_{j_{\ell+1}^\star}^\star$ for all $\ell\in[m^\star-1]$.
        \item $\paren{\bfw_{i_\ell^\star}(0)^\top\bbfw_{j_\ell^\star}^\star}^2\geq \frac{\log m^\star}{d}$ for all $\ell\in[m^\star]$.
        \item $\norm{\bfw_i}_2, \norm{\bfv_i}_2 \in \left[1-\beta_2\delta_s,1+\beta_2\delta_s\right]$ for all $i\in[m]$.
        \item $\max\left\{\paren{\bfv_i(0)^\top\bfv_j^\star}^2,\paren{\bfv_i(0)^\top\bfw_j^\star}^2,\paren{\bfv_i(0)^\top\bfw_j^\star}^2,\paren{\bfw_i(0)^\top\bfv_j^\star}^2\right\} \leq \frac{\beta_3}{d}\log \frac{m}{\delta_{\mathbb{P}}}$ for all $i\in[m],j\in[m^\star]$.
        \item $\max\left\{\paren{\bfv_i(0)^\top\bfv_j(0)}^2, \paren{\bfw_i(0)^\top\bfw_j(0)}^2, \paren{\bfv_i(0)^\top\bfw_j(0)}^2\right\}\leq \frac{\beta_3}{d}\log \frac{m}{\delta_{\mathbb{P}}}$ for all $i,j\in[m]$ and $i\neq j$. Moreover, $\bfv_i(0)^\top\bfw_i(0) = 0$ for all $i\in[m]$.
    \end{itemize}
    where $\delta_s = \frac{\beta_1\delta_{\mathbb{P}}}{mm^\star}$ for some absolute constant $\beta_1,\beta_3 > 0$ and $\beta_2 \leq o(1)$ and any $\delta_{\mathbb{P}}\in (0,\sfrac{1}{7})$. 
\end{cond}
By Lemma~\ref{lem:init} and Lemma~\ref{lem: init_ub}, the above condition holds with probability at least $1 - 7\delta_{\mathbb{P}}$ as long as $d\geq \frac{\beta_5m^4}{\delta_{\mathbb{P}}^2}\log\frac{m}{\delta_{\mathbb{P}}}$ and $m \geq \beta_4 m^\star\log\frac{m^\star}{\delta_{\mathbb{P}}}$.

\textbf{Inductive Hypothesis.} 
Now we are going to show that $\bbfw_{i_\ell^\star}(t)^\top\bbfw_{j_\ell^\star}^\star$ converges to at least $1 - \frac{c}{\sqrt{d}}$ for all $\ell\in[m^\star]$ by induction. To start, we denote the values of interest as follows
\begin{equation}
    \label{eq:target_dynamic1}
    \begin{gathered}
        \gmij{1}{}(t) = \bbfv_i(t)^\top\bbfv_j^\star;\quad \gmij{2}{}(t) = \bbfw_i(t)^\top\bbfw_j^\star;\quad \ztij{1}{}(t) = \bbfv_i(t)^\top\bbfw_j^\star;\quad \ztij{2}{}(t) = \bbfw_i(t)^\top\bbfv_j^\star\\
        I_{i,j}^{(1)}(t) = \bbfv_i(t)^\top\bbfv_j(t);\quad I_{i,j}^{(2)}(t) = \bbfw_i(t)^\top\bbfw_j(t);\quad I_{i,j}^{(3)}(t) = \bbfv_i(t)^\top\bbfw_j(t)
    \end{gathered}
\end{equation}
To state the inductive hypothesis, we need the following error bounds.
\begin{defin}[Error Bounds]
    \label{def:error_bound}
    For each $\ell\in[m^\star]$, we define the following:
    \begin{gather*}
        \varepsilon_{1,\ell}(t) := \max_{i\in[m]\setminus \mathcal{R}_{\ell}, j\in[m^\star]}\left|\gamma_{i,j}^{(1)}(t)\right|;\; \varepsilon_{2,\ell}(t) := \max_{j\in[m^\star]\setminus \{j_\ell^\star\}}\left|\gamma_{i_\ell^\star,j}^{(1)}(t)\right|;\\
        \varepsilon_{3,\ell}(t) := \max_{i\in[m]\setminus \mathcal{R}_{\ell}, j\in[m^\star]}\left|\gamma_{i,j}^{(2)}(t)\right|;\; \varepsilon_{4,\ell}(t) := \max_{j\in[m^\star]\setminus \{j_\ell^\star\}}\left|\gamma_{i_\ell^\star,j}^{(2)}(t)\right|;\;\varepsilon_{5,\ell}(t) := \left|I_{i_\ell^\star,i_\ell^\star}^{(3)}(t)\right|
    \end{gather*}
    Moreover, we also define the forward error, the backward error, and the aggregated error as
    \begin{gather*}
        \forweps{\ell}(t) = \max\left\{\varepsilon_{\mathcal{F},\ell,1}(t), \varepsilon_{\mathcal{F},\ell,2}(t)\right\};\;
        \backepso{\ell}(t) = \max\left\{\varepsilon_{\mathcal{B},\ell,1}(t), \varepsilon_{\mathcal{B},\ell,2}(t), \varepsilon_{2,\ell}(t)\right\}\\
        \backepst{\ell}(t) = \max\left\{\backepso{\ell}(t),\varepsilon_{1,\ell}(t)\right\};\;\backepsi{\ell}(t) = \max_{i\in[m]\setminus \mathcal{R}_{\ell}}\left|I_{i,i}^{(3)}(t)\right|\\
        \hat{\varepsilon}_{\mathcal{A},\ell}^{(1)}(t) = \max\left\{\varepsilon_{4,\ell}(t), \backepso{\ell}(t), \forweps{\ell}(t)\right\};\;
        \hat{\varepsilon}_{\mathcal{A},\ell}^{(2)}(t) = \max\left\{\hat{\varepsilon}_{\mathcal{A},\ell}^{(1)}(t),\varepsilon_{1,\ell}(t), \varepsilon_{3,\ell}(t)\right\}
    \end{gather*}
    where $\varepsilon_{\mathcal{F},\ell,1}(t), \varepsilon_{\mathcal{F},\ell,2}(t), \varepsilon_{\mathcal{F},\ell,3}(t)$ and $\varepsilon_{\mathcal{B},\ell,1}(t), \varepsilon_{\mathcal{B},\ell,2}(t)$ are defined as
    \begin{gather*}
        \varepsilon_{\mathcal{F},\ell,1}(t) := \max_{\ell'\leq \ell,j\in[m^\star]\setminus \{j_{\ell'}^\star\}}\max\left\{\left|\gamma_{i_{\ell'}^\star,j}^{(1)}(t)\right|, \left|\gamma_{i_{\ell'}^\star,j}^{(2)}(t)\right|\right\};\\
        \varepsilon_{\mathcal{F},\ell,2}(t) := \max_{i\in\mathcal{R}_{\ell-1},j\in[m^\star]}\max\left\{ \left|\zeta_{i,j}^{(1)}(t)\right|, \left|\zeta_{i,j}^{(2)}(t)\right|\right\}\\
        \varepsilon_{\mathcal{B},\ell,1}(t) := \max_{i\in[m]\setminus \mathcal{R}_{\ell-1},j\in[m^\star]}\max\left\{ \left|\zeta_{i,j}^{(1)}(t)\right|, \left|\zeta_{i,j}^{(2)}(t)\right|\right\}\\
        \varepsilon_{\mathcal{B},\ell,2}(t) := \max_{i,j\in[m],i\neq j}\max\left\{\left|I_{i,j}^{(1)}(t)\right|, \left|I_{i,j}^{(2)}(t)\right|, \left|I_{i,j}^{(3)}(t)\right|\right\}
    \end{gather*}
    Lastly, we are going to define the monotonic upper bound of $\hat{\varepsilon}_{\mathcal{A},\ell}^{(1)}(t)$ and $\hat{\varepsilon}_{\mathcal{A},\ell}^{(2)}(t)$
    \[
        \aggepso{\ell}(t) = \sup_{t'\in[0,t]}\hat{\varepsilon}_{\mathcal{A},\ell}^{(1)}(t);\; \aggepst{\ell}(t) = \sup_{t'\in[0,t]}\hat{\varepsilon}_{\mathcal{A},\ell}^{(2)}(t)
    \]
\end{defin}

\begin{defin}[Recovery Time]
    \label{def:recovery_time}
    Define the $\xi$-recovery time of $\gamma_{i_{\ell}^\star,j_{\ell}^\star}^{(2)}(t)$, denoted as $T_{\ell}\paren{\xi}$, as
    \[
        T_\ell\paren{\xi} = \min\left\{t\geq 0: \gamma_{i_{\ell}^\star,j_{\ell}^\star}^{(2)}(t) \geq \xi\right\}
    \]
    Based on $T_{\ell}(\xi)$, we define the \textbf{constant-recovery time} and the \textbf{near-perfect-recovery time} as $T_{r,\ell} = T_\ell\paren{0.9}$ and $T_{p,\ell} = T_\ell\paren{1 - \frac{\beta_9m^7}{\delta_{\mathbb{P}}^3d^{\frac{3}{2}}}}$, respectively.
\end{defin}

\begin{cond}[Inductive Hypothesis]
    \label{cond:inductive_hypo}
    Let $\ell\in[m^\star]$. Then we have that
    \begin{itemize}
        \item (Sequential recovery) For all $t\geq T_{p,\ell-1}$ such that $\aggepso{\ell}(t)\leq \calO\paren{\frac{m^2}{\delta_{\mathbb{P}}\sqrt{d}}}$ we have $\gamma_{i_\ell'^\star,j_\ell'^\star}^{(2)}(t)\geq 1 - \frac{\beta_9m^7}{\delta_{\mathbb{P}}^3d^{\frac{3}{2}}}$, and $\gamma_{i_\ell'^\star,j_\ell'^\star}^{(1)}(t)\geq 1 - \frac{\beta_9m^7}{\delta_{\mathbb{P}}^3d^{\frac{3}{2}}}$ for all $\ell' < \ell$.
        \item (Error bound of $\gamma_{i,j}^{(1)},\gamma_{i,j}^{(2)}$) For $t\leq T_{p,\ell-1}$, we have 
        \[
            \max_{i\in[m]\setminus \mathcal{R}_{\ell-1},j\in [m^\star]}\max\left\{\left|\gamma_{i,j}^{(1)}(t)\right|, \left|\gamma_{i,j}^{(2)}(t)\right| \leq \frac{\beta_6m^2}{\sqrt{d}}\right\}
        \]
        .
        \item (Error bound of remaining items) $\forweps{\ell}(t) \leq \calO\paren{\frac{m^2}{\delta_{\mathbb{P}}\sqrt{d}}}$ for all $t$ such that $\aggepso{\ell}(t)\leq \calO\paren{\frac{m^2}{\delta_{\mathbb{P}}\sqrt{d}}}$.
    \end{itemize}
    Here $\beta_6,\beta_9 > 0$ are some absolute constant.
\end{cond}

The proof proceeds by establishing the inductive hypothesis.

\subsection{Initialization Property}
\begin{lemma}
    \label{lem: init_ub}
    Let $\hat{\bfv}_1,\dots,\hat{\bfv}_m$ and $\bfw_1,\dots,\bfw_m$ be I.I.D. random vectors from $\mathcal{N}\paren{\bm{0},d^{-1}\bfI_d}$. Define $\bfv_i = \paren{\bfI - \frac{1}{\norm{\bfw_i}_2^2}\bfw_i\bfw_i^\top}\hat{\bfv}_i$. Then there exists some absolute constant $\beta_1,\beta_3,\beta_5 > 0, \beta_2 \leq o\paren{1}$, and $\delta_{\mathbb{P}} \in (0,\sfrac{1}{3})$ such that if $d \geq  \frac{\beta_5m^4}{\delta_{\mathbb{P}}^2}\log\frac{m}{\delta_{\mathbb{P}}}$ and $\delta_s = \frac{\beta_1\delta_{\mathbb{P}}}{m^2}$, with probability at least $1 - 3\delta_{\mathbb{P}}$ we have that
    \begin{itemize}
        \item $\norm{\bfv_i}_2^2, \norm{\bfw_i}_2^2 \in [1-\beta_2\delta_s, 1+\beta_2\delta_s]$ for all $i\in[m]$;
        \item $\max\left\{\bfv_i[j]^2, \bfw_{i}[j]^2\right\}\leq \frac{\beta_3}{d}\log\frac{m}{\delta_{\mathbb{P}}}$ for all $i\in[m], j\in[m^\star]$;
        \item $\max\left\{\paren{\bfv_i^\top\bfv_j}^2,\paren{\bfw_i^\top\bfw_j}^2, \paren{\bfv_i^\top\bfw_j}^2\right\} \leq \frac{\beta_3}{d}\log\frac{m}{\delta_{\mathbb{P}}}$ for all $i,j\in[m]$ with $i\neq j$.
    \end{itemize}
\end{lemma}
\begin{proof}
    Our proof starts with showing the concentration for $\norm{\hat{\bfv}_i}_2^2, \norm{\bfw_i}_2^2$ and $\hat{\bfv}_i^\top\bfw_i$, and then moves to the proof of the desired statement.

    \textbf{Concentration of $\norm{\hat{\bfv}_i}_2^2, \norm{\bfw_i}_2^2$.} Let $\bfv\sim\mathcal{N}\paren{\bm{0},d^{-1}\bfI_d}$. Then $\bfv[i]\sim \mathcal{N}\paren{0, d^{-1}}$. Thus, $d\bfv[i]^2 \in\texttt{subE}\paren{2,2}$. Since $\bfv[i]$'s are I.I.D., we have that
    \[
        d\norm{\bfv}_2^2 = d\sum_{i=1}^d\bfv[i]^2 \in \texttt{subE}\paren{2d, 2}
    \]
    Therefore, by the tail bound of sub-exponential random variables, we have that
    \[
        \Pr\paren{\left|\norm{\bfv}_2^2 - 1\right| \geq \frac{t}{d}} = \Pr\paren{\left|d\norm{\bfv}_2^2 - d\EXP{\norm{\bfv}_2^2}\right| \geq t} \leq 2\exp{-\frac{1}{4}\min\left\{\frac{t^2}{d}, t\right\}}
    \]
    We are going to focus on the case where $t\leq d$. In particular, we set $t = \frac{1}{2}\beta_2\delta_s d$. Then we have that
    \[
        \Pr\paren{\left|\norm{\bfv}_2^2 - 1\right| \geq \frac{1}{2}\beta_2\delta_s}\leq 2\exp{-\frac{1}{16}\beta_2^2\delta_s^2 d}
    \]
    Take a union bound over all $i\in[m]$ for $\hat{\bfv}_i$ and $\bfw_i$ gives that, with probability at least $1 - 4m \exp{-\frac{1}{16}\beta_2^2\delta_s^2 d}$, it holds that
    \[
        \norm{\hat{\bfv}_i}_2^2, \norm{\bfw_i}_2^2 \in \left[1 - \frac{1}{2}\beta_2\delta_s, 1+\frac{1}{2}\beta_2\delta_s\right]
    \]
    Setting $d \geq  \frac{\beta_5m^4}{\delta_{\mathbb{P}}^2}\log\frac{m}{\delta_{\mathbb{P}}} \geq \frac{16}{\beta_2^2\delta_s^2}\log\frac{4m}{\delta_{\mathbb{P}}}$ guarantees that the failing probability is within $\delta_{\mathbb{P}}$.

    \textbf{Concentration of $\hat{\bfv}_i^\top\bfw_j, \hat{\bfv}_i^\top\hat{\bfv}_j$, and $\bfw_i^\top\bfw_j$.} Due to the independence between $\hat{\bfv}_i$'s and $\bfw_j$'s, we have that
    \[
        \EXP{\hat{\bfv}_i^\top\bfw_j} = 0;\quad d\hat{\bfv}_i[\ell]\bfw_j[\ell]\in \texttt{subE}\paren{1,1}. 
    \]
    Thus, we have that $\hat{\bfv}_i^\top\bfw_j\in \texttt{subE}\paren{d, 1}$. Applying the tail bound of sub-exponential random variable gives that
    \[
        \Pr\paren{\left|\hat{\bfv}_i^\top\bfw_j\right|\geq \frac{t}{d}} =\Pr\paren{d\left|\hat{\bfv}_i^\top\bfw_j\right|\geq t} \leq 2\exp{-\frac{1}{2}\min\left\{\frac{t^2}{d}, t\right\}}
    \]
    The same concentration also holds for $\hat{\bfv}_i^\top\hat{\bfv}_j$ and $\bfw_i^\top\bfw_j$. Again we are going to focus on the case $t \leq d$. Take a union bound over all $i,j\in[m]$ and $\hat{\bfv}_i,\bfw_i$ gives that
    \[
        \Pr\paren{\max\left\{\left|\hat{\bfv}_i^\top\bfw_j\right|,\left|\hat{\bfv}_i^\top\hat{\bfv}_j\right|,\left|\bfw_i^\top\bfw_j\right|\right\}\geq \frac{t}{d};\;\forall i,j\in[m]} \leq 8m^2\exp{-\frac{t^2}{2d}}
    \]
    Setting the failing probability to $\delta_{\mathbb{P}}$ gives that with probability at least $1 - \delta_{\mathbb{P}}$, we have that
    \[
        \hat{\bfv}_i^\top\bfw_j,\hat{\bfv}_i^\top\hat{\bfv}_j, \bfw_i^\top\bfw_j \in \left[-\frac{1}{\sqrt{d}}\paren{\log \frac{8m^2}{\delta_{\mathbb{P}}}}^{\frac{1}{2}},-\frac{1}{\sqrt{d}}\paren{\log \frac{8m^2}{\delta_{\mathbb{P}}}}^{\frac{1}{2}}\right]
    \]

    \textbf{Proof of the first statement.} Notice that the bound for $\norm{\bfw_i}_2^2$ is already implied by its concentration property. To prove the bound for $\norm{\bfv_i}_2^2$, we write
    \begin{align*}
        \norm{\bfv_i}_2^2 & = \norm{\paren{\bfI -\frac{1}{\norm{\bfw_i}_2^2}\bfw_i\bfw_i^\top}\hat{\bfv}_i}_2^2 = \norm{\hat{\bfv}_i}_2^2 - \frac{1}{\norm{\bfw}_i^2}\paren{\bfw_i^\top\hat{\bfv}_i}^2
    \end{align*}
    By the concentration property of $\norm{\bfw}_i^2$ and $\bfw_i^\top\hat{\bfv}_j$, we have that
    \[
        \frac{1}{\norm{\bfw}_i^2}\paren{\bfw_i^\top\hat{\bfv}_i}^2 \leq \frac{1}{d\paren{1-\frac{1}{2}\beta_s\delta_s}}\log\frac{8m^2}{\delta{\mathbb{P}}}
    \]
    For $d\geq\frac{\beta_5m^2}{\delta_{\mathbb{P}}}\log \frac{m}{\delta_{\mathbb{P}}} \geq \frac{2}{\beta_2\delta_s\paren{1 - \frac{1}{2}\beta_2\delta_s}}\log\frac{8m^2}{\delta_{\mathbb{P}}}$, we have that $\frac{1}{\norm{\bfw}_i^2}\paren{\bfw_i^\top\hat{\bfv}_i}^2 \leq \frac{1}{2}\beta_2\delta_s$. Combined with the concentration property of $\hat{\bfv}_i$, we have that $\norm{\bfv_i}_2^2 \in [1 - \beta_2\delta_s, 1 +\beta_2\delta_s]$.

    \textbf{Proof of second statement.} 
    By the tail bound of Gaussian random variable, we have that for all $z\sim\mathcal{N}\paren{0,1}$, it holds that
    \[
        \Pr\paren{z^2 \geq t} = 2\Pr\paren{z \geq \sqrt{t}} \leq \exp{-\frac{t}{2}}
    \]
    Apply the above to $z = \sqrt{d}\cdot \bfw_i[j]$ and $\sqrt{d}\cdot \hat{\bfv}_i[j]$ with  a union bound over all $i\in[m]$ and $j\in[m^\star]$ gives
    \[
        \Pr\paren{\max_{i\in[m],j\in[m^\star]}\hat{\bfv}_i[j]^2\geq \frac{t}{d};\; \max_{i\in[m],j\in[m^\star]}\bfw_i[j]^2 \geq \frac{t}{d}} \leq 2mm^\star\exp{-\frac{t}{2}}
    \]
    Set $t = \frac{\beta_3}{4}\log\frac{m}{\delta_{\mathbb{P}}} \geq 2\log \frac{2mm^\star}{\delta_{\mathbb{P}}}$ gives the desired result for $\bfw_i[j]^2$s. To bound $\bfv_i[j]^2$, we notice that
    \[
        \left|\bfv_i[j]\right| \leq \left|\hat{\bfv}_i[j]\right| + \frac{1}{\norm{\bfw_i}_2^2}\left|\bfw_i^\top\hat{\bfv}_i\cdot \bfw_i[j]\right|
    \]
    By the previous bounds, we have that
    \[
        \left|\hat{\bfv}_i[j]\right| \leq \frac{\sqrt{\beta_3}}{2\sqrt{d}}\paren{\log \frac{m}{\delta_{\mathbb{P}}}};\; \frac{1}{\norm{\bfw_i}_2^2}\left|\bfw_i^\top\hat{\bfv}_i\cdot \bfw_i[j]\right| \leq \frac{1}{\sqrt{d}\paren{1 - \frac{1}{2}\beta_2\delta_s}}\paren{\log\frac{8m^2}{\delta_{\mathbb{P}}}}^{\frac{1}{2}}\cdot \frac{\sqrt{\beta_3}}{2\sqrt{d}}\paren{\log \frac{m}{\delta_{\mathbb{P}}}}
    \]
    With $d\geq \frac{\beta_5m^4}{\delta_{\mathbb{P}}}\log \frac{m}{\delta_{\mathbb{P}}}$ we can guarantee that the latter is also upper bounded by $\frac{\sqrt{\beta_3}}{2\sqrt{d}}\paren{\log \frac{m}{\delta_{\mathbb{P}}}}$. Combining the two bounds and square both sides gives the desired result for $\bfv_i[j]^2$.
    
    \textbf{Proof of the third statement.} The bound of $\bfw_i^\top\bfw_j$ is again implied by the concentration we showed above. To show the bound of $\bfv_i^\top\bfv_j$, we write
    \begin{align*}
        \bfv_i^\top\bfv_j & = \hat{\bfv}_i^\top\paren{\bfI -\frac{1}{\norm{\bfw_i}_2^2}\bfw_i\bfw_i^\top}\paren{\bfI -\frac{1}{\norm{\bfw_j}_2^2}\bfw_j\bfw_j^\top}\hat{\bfv}_j\\
        & = \hat{\bfv}_i^\top\hat{\bfv}_j - \frac{1}{\norm{\bfw_i}_2^2}\bfw_i^\top\hat{\bfv}_i\cdot \bfw_i\top\hat{\bfv}_j - \frac{1}{\norm{\bfw_j}_2^2}\bfw_j^\top\hat{\bfv}_j\cdot \bfw_j^\top\hat{\bfv}_i  + \frac{1}{\norm{\bfw_i}_2^2\norm{\bfw_j}_2^2} \bfw_i^\top\hat{\bfv}_i\cdot \bfw_j^\top\hat{\bfv}_j\cdot \bfw_i^\top\bfw_j
    \end{align*}
    By the concentration of the norms and inner-products, we have that
    \begin{gather*}
        \frac{1}{\norm{\bfw_i}_2^2}\left|\bfw_i^\top\hat{\bfv}_i\cdot \bfw_i\top\hat{\bfv}_j\right|, \frac{1}{\norm{\bfw_j}_2^2}\left|\bfw_j^\top\hat{\bfv}_j\cdot \bfw_j^\top\hat{\bfv}_i\right|\leq \frac{1}{d\paren{1-\frac{1}{2}\beta_s\delta_s}}\log\frac{8m^2}{\delta_{\mathbb{P}}}\\
        \frac{1}{\norm{\bfw_i}_2^2\norm{\bfw_j}_2^2} \left|\bfw_i^\top\hat{\bfv}_i\cdot \bfw_j^\top\hat{\bfv}_j\cdot \bfw_i^\top\bfw_j\right| \leq \frac{1}{d^{\frac{3}{2}}\paren{1-\frac{1}{2}\beta_s\delta_s}^2}\paren{\log\frac{8m^2}{\delta_{\mathbb{P}}}}^{\frac{3}{2}}
    \end{gather*}
    With the condition $d\geq\frac{\beta_5m^2}{\delta_{\mathbb{P}}}\log \frac{m}{\delta_{\mathbb{P}}}$ we have that
    \[
        \frac{1}{d\paren{1-\frac{1}{2}\beta_s\delta_s}^2}\log\frac{8m^2}{\delta_{\mathbb{P}}} \leq \frac{1}{6\sqrt{d}}\paren{\log\frac{8m^2}{\delta_{\mathbb{P}}}}^{\frac{1}{2}} \leq \frac{\sqrt{\beta_3}}{6\sqrt{d}}\paren{\log\frac{m}{\delta_{\mathbb{P}}}}^{\frac{1}{2}} \leq 1
    \]
    Therefore, we can conclude that
    \[
        \frac{1}{\norm{\bfw_i}_2^2}\left|\bfw_i^\top\hat{\bfv}_i\cdot \bfw_i\top\hat{\bfv}_j\right|, \frac{1}{\norm{\bfw_j}_2^2}\left|\bfw_j^\top\hat{\bfv}_j\cdot \bfw_j^\top\hat{\bfv}_i\right|, \frac{1}{\norm{\bfw_i}_2^2\norm{\bfw_j}_2^2} \left|\bfw_i^\top\hat{\bfv}_i\cdot \bfw_j^\top\hat{\bfv}_j\cdot \bfw_i^\top\bfw_j\right| \leq \frac{\sqrt{\beta_3}}{6\sqrt{d}}\paren{\log \frac{m}{\delta_{\mathbb{P}}}}^{\frac{1}{2}}
    \]
    Combined with the bound on $\hat{\bfv}_i^\top\hat{\bfv}_j$ gives that
    \[
        \bfv_i^\top\bfv_j \in \left[-\frac{\sqrt{\beta_3}}{\sqrt{d}}\paren{\log\frac{m}{\delta_{\mathbb{P}}}}^{\frac{1}{2}}, \frac{\sqrt{\beta_3}}{\sqrt{d}}\paren{\log\frac{m}{\delta_{\mathbb{P}}}}^{\frac{1}{2}}\right]
    \]
    Squaring both sides gives the desired result. Lastly, to bound $\bfv_i^\top\bfw_j$, we write
    \begin{align*}
        \bfv_i^\top\bfw_j & = \bfw_j^\top\paren{\bfI - \frac{1}{\norm{\bfw_i}_2^2}\bfw_i\bfw_i^\top}\hat{\bfv}_i = \hat{\bfv}_i^\top\bfw_j - \frac{1}{\norm{\bfw_i}_2^2}\bfw_i^\top\hat{\bfv}_i\cdot \bfw_i^\top\bfw_j
    \end{align*}
    Similar to the above, we have that
    \[
        \frac{1}{\norm{\bfw_i}_2^2}\bfw_i^\top\hat{\bfv}_i\cdot \bfw_i^\top\bfw_j \leq \frac{1}{d\paren{1-\frac{1}{2}\beta_s\delta_s}}\log\frac{8m^2}{\delta_{\mathbb{P}}} \leq \frac{\sqrt{\beta_3}}{6\sqrt{d}}\paren{\log \frac{m}{\delta_{\mathbb{P}}}}^{\frac{1}{2}}
    \]
    Therefore, we can conclude that
    \[
        \bfv_i^\top\bfw_j \in \left[-\frac{\sqrt{\beta_3}}{\sqrt{d}}\paren{\log\frac{m}{\delta_{\mathbb{P}}}}^{\frac{1}{2}}, \frac{\sqrt{\beta_3}}{\sqrt{d}}\paren{\log\frac{m}{\delta_{\mathbb{P}}}}^{\frac{1}{2}}\right]
    \]
    Squaring both sides gives the desired result.
\end{proof}

\begin{lemma}
    \label{lem:init1}[Restatement of Lemma~\ref{lem:init}]
    Let $\bfw_1, \dots, \bfw_m\sim\mathcal{N}\paren{0,d^{-1}\bfI_d}$ be I.I.D. Gaussian random vectors. Define
    \[
        i_{\ell}^\star, j_{\ell}^\star = \argmax_{i\in[m]\setminus \mathcal{R}_{\ell-1},j\in[m^\star]\setminus \mathcal{C}_{\ell-1}}\bfw_i[j];\;\; \mathcal{R}_{\ell} = \{i_k^\star\}_{k=1}^{\ell};\;\; \mathcal{C}_{\ell} = \{j_k^\star\}_{k=1}^{\ell}
    \]
    Let any $\delta_{\mathbb{P}} \in (0, \sfrac{1}{2})$ be given. Then there exists some absolute constant $\beta_2, \beta_4 > 0$ such that if $m\geq  \beta_4m^\star\log \frac{m^\star}{\delta_{\mathbb{P}}}$, then for $\delta_s = \frac{\beta_2\delta_{\mathbb{P}}}{m^2}$, with probability at least $1 - 4\delta_{\mathbb{P}}$, it holds that
    \begin{itemize}
        \item (Row-wise Gap) $\bfw_{i^\star_\ell}[j^\star_{\ell}]\geq \paren{1+2\delta_s}\bfw_{i^\star_{\ell}}[j]$ for all $\ell\in[m^\star]$ and $j\in[m^\star]\setminus \mathcal{C}_{\ell}$
        \item (Column-wise Gap) $\bfw_{i^\star_{\ell}}[j^\star_{\ell}]\geq \paren{1+2\delta_s}\bfw_i[j^{\star}_{\ell}]$ for all $\ell\in[m^\star]$ and $i\in[m]\setminus\mathcal{R}_{\ell}$
        \item (Threshold Gap) $\bfw_{i^\star_{\ell}}[j^\star_{\ell}] \geq \paren{1+2\delta_s}\bfw_{i^\star_{\ell+1}}[j^\star_{\ell+1}]^2$ for all $\ell\in[m^\star-1]$
        \item (Magnitude Lower Bound) $\bfv_{i^\star_\ell}[j_\ell^\star]^2 \geq \frac{\log m^\star}{d}$ for all $\ell\in[m]$
    \end{itemize}
\end{lemma}
\begin{proof}
    We start by proving an auxiliary result that, with high probability, there are at least $\frac{m}{3}$ out of the m $\bfw_i[j]$'s that are positive for each $j\in[m]^\star$. Define
    \[
        s_{i,j} = \indy{\bfw_i[j] > 0};\; S_j = \sum_{i=1}^m s_{i,j}
    \]
    Due to the symmetry of Gaussian, we have that $s_{i,j}\sim\texttt{Bern}\paren{0.5}$ independently. Therefore, $\EXP{S_j} = \frac{m}{2}$. By Hoeffding's inequality, we have that
    \[
        \Pr\paren{S_j - \frac{m}{2} \leq -t} \leq \exp{-\frac{2t^2}{m}}
    \]
    Setting $t = \frac{m}{6}$ and take a union bound over all $j\in[m^\star]$ gives that
    \[
        \Pr\paren{S_j\geq \frac{m}{3};\;\forall j\in[m^\star]} \leq m^\star\exp{-\frac{m}{18}}
    \]
    Since $m \geq \beta_4m^\star\log \frac{m^\star}{\delta_{\mathbb{P}}} \geq 18\log \frac{m^\star}{\delta_{\mathbb{P}}}$, with probability at least $1 - \delta_{\mathbb{P}}$, we have that at least $\frac{m}{3}$ out of $\left\{\bfw_i[j]\right\}_{i=1}^m$ are positive for all $j\in[m]$.

    \textbf{Proof of the first statement.}
    Let any $i_1, i_2\in[m]$ and $j_1, j_2\in[m^\star]$ such that $(i_1, j_1)$ and $(i_2, j_2)$ differ in at least one coordinate. Then we have that $\bfw_{i_1}[j_1]$ and $\bfw_{i_2}[j_2]$ are I.I.D. Gaussian random variables in $\mathcal{N}\paren{0, d^{-1}}$. Therefore, $\frac{\bfw_{i_1}[j_1]}{\bfw_{i_2}[j_2]}$ is a standard Cauchy random variable. Given the condition that $\bfw_{i_1}[j_1],\bfw_{i_2}[j_2] \geq 0$, we have that $\bfw_{i_1}[j_1],\bfw_{i_2}[j_2]$ are half-Gaussian, and thus $\frac{\bfw_{i_1}[j_1]}{\bfw_{i_2}[j_2]}$ is half-Cauchy. Using the CDF of Cauchy random variables, we have that for any $\delta\in(0,\sfrac{1}{2})$
    \begin{align*}
        & \Pr\paren{\bfw_{i_1}[j_1]\in \paren{\bfw_{i_2}[j_2], \paren{1+ \delta}\bfw_{i_2}[j_2]} \mid \bfw_{i_1}[j_1],\bfw_{i_2}[j_2] \geq 0}\\
        & \qqquad = \Pr\paren{\left|\frac{\bfw_{i_1}[j_1]}{\bfw_{i_2}[j_2]}\right| \in (1, 1+\delta)}\\
        & \qqquad = 2\Pr\paren{\frac{\bfw_{i_1}[j_1]}{\bfw_{i_2}[j_2]}\in (1, 1+\delta)}\\
        & \qqquad = \frac{2}{\pi}\paren{\arctan \paren{1+\delta} - \arctan 1}\\
        & \qqquad = \frac{2}{\pi}\arctan\frac{\delta}{2 + \delta}\\
        & \qqquad \leq \frac{\delta}{\pi}
    \end{align*}
    where the last inequality follows from $\arctan x \leq x$ with $x\geq 0$ and $\delta > 0$. 
    To prove the first property, we let $j_1 = j_{\ell}^\star, j_2= j$ and $i_1 = i_2 = i_{\ell}^\star$. Fix any $\ell\in[m^\star]$, we have that
    \[
        \Pr\paren{\bfw_{i^\star_{\ell}}[j^\star_{\ell}]\in \paren{\bfw_{i_{\ell}^\star}[j], \paren{1+ \delta} \bfw_{i_{\ell}^\star}[j]};\;\forall j\in [m^\star] \setminus \mathcal{C}_{\ell} \text{ s.t. } \bfw_{i_{\ell}^\star}[j]\geq 0} \leq m^\star\delta
    \]
    Recall that $\bfw_{i^\star_{\ell}}[j^\star_{\ell}] \geq \bfw_{i^\star_{\ell}}[j]$ for all $j \in [m^\star] \setminus \mathcal{C}_{\ell}$. If $\bfw_{i_{\ell}^\star}[j] < 0$, since by definition $\bfw_{i^\star_{\ell}}[j^\star_{\ell}] \geq 0$, it must hold that $\bfw_{i^\star_{\ell}}[j^\star_{\ell}] \geq (1+2\delta_s)\bfw_{i_{\ell}^\star}[j]$. Take a union bound over $\ell\in[m^\star]$, and set $\delta = 2\delta_s$ with $\delta_s \leq \frac{\delta_{\mathbb{P}}}{2mm^{\star}}$ gives the first property. 
    
    \textbf{Proof of the second statement.} To prove the second property, we set $i_1 = i_\ell^\star, i_2 = i$ and $j_1 = j_2 = j_{\ell}^\star$. Fix any $\ell\in[m^\star]$, we have that
    \[
        \Pr\paren{\bfw_{i^\star_{\ell}}[j^\star_{\ell}]\in \paren{\bfw_{i}[j_{\ell}^\star], \paren{1+ \delta}\bfw_{i}[j_{\ell}^\star]};\;\forall i\in [m] \setminus \mathcal{R}_{\ell} \text{ s.t. } \bfw_{i}[j_{\ell}^\star]\geq 0} \leq m\delta
    \]
    Recall that $\bfw_{i^\star_{\ell}}[j^\star_{\ell}] \geq \bfw_{i}[j^\star_{\ell}]$ for all $i \in [m] \setminus \mathcal{R}_{\ell}$. If $\bfw_{i}[j_{\ell}^\star] < 0$, since by definition $\bfw_{i^\star_{\ell}}[j^\star_{\ell}] \geq 0$, it must hold that $\bfw_{i^\star_{\ell}}[j^\star_{\ell}] \geq (1+2\delta_s)\bfw_{i}[j_{\ell}^\star]$. Take a union bound over $\ell\in[m^\star]$, and set $\delta = 2\delta_s$ with $\delta_s \leq \frac{\delta_{\mathbb{P}}}{2m^{\star 2}}$ gives the second property. 

    \textbf{Proof of the third statement.} To prove the third property, we set $i_1 = i_\ell^\star, i_2 = i_{\ell+1}^\star$ and $j_1 = j_{\ell}^\star, j_1 = j_{\ell+1}^\star$. Fix any $\ell\in[m^\star-1]$, we have that
    \[
        \Pr\paren{\bfw_{i^\star_{\ell}}[j^\star_{\ell}]\in \paren{\bfw_{i_{\ell+1}^\star}[j_{\ell+1}^\star], \paren{1+ \delta}\bfw_{i_{\ell+1}^\star}[j_{\ell+1}^\star]}} \leq \delta
    \]
    Recall that $\bfw_{i^\star_{\ell}}[j^\star_{\ell}] \geq \bfw_{i_{\ell+1}^\star}[j_{\ell+1}^\star] \geq 0$ for all $\ell \in [m^\star -1]$. Take a union bound over $\ell\in[m^\star-1]$, and set $\delta = 2\delta_s$ with $\delta_s \leq \frac{\delta_{\mathbb{P}}}{2m^{\star}}$ gives the third property.

    \textbf{Proof of the fourth statement.} To show the last result, we notice that by definition of $i_{\ell}^\star$, it must holds that
    \[
        \bfw_{i_{\ell}^\star}[j_{\ell}^\star]\geq  \bfw_i[j_{\ell}^\star];\quad \forall i\in[m]\setminus \mathcal{R}_{\ell}
    \]
    This gives that for any $\gamma > 0$
    \begin{align*}
        \Pr\paren{\bfw_{i_{\ell}^\star}[j_{\ell}^\star] < \frac{\gamma}{\sqrt{d}}} & \leq \prod_{i\in[m]\setminus\mathcal{R}_{\ell}}\Pr\paren{\bfw_i[j_{\ell}^\star] < \frac{\gamma}{\sqrt{d}}}\\
        & = \prod_{i\in[m]\setminus\mathcal{R}_{\ell}}\Pr\paren{\bfw_i[j_{\ell}^\star] < \frac{\gamma}{\sqrt{d}}}\\
        & = \Pr_{z\sim\mathcal{N}(0,1)}\paren{z < \gamma}^{m-m^\star}\\
        & \leq \paren{1 - \Pr_{z\sim\mathcal{N}(0,1)}\paren{z \geq \gamma}}^{m-m^\star}
    \end{align*}
    By the tail bound of Gaussian random variable, we have that
    \[
        \Pr_{z\sim\mathcal{N}(0,1)}\paren{z \geq \gamma} \geq \frac{1}{\sqrt{2\pi}\cdot \gamma}\exp{-\frac{\gamma^2}{2}}
    \]
    Therefore
    \[
        \Pr\paren{\bfw_{i_{\ell}^\star}[j_{\ell}^\star] < \frac{\gamma}{\sqrt{d}}} \leq \paren{1 - \frac{1}{\sqrt{2\pi}\cdot \gamma}\exp{-\frac{\gamma^2}{2}}}^{m - m^\star}
    \]
    Take a union bound over all $\ell\in[m^\star]$ gives that
    \[
        \Pr\paren{\bfw_{i_{\ell}^\star}[j_{\ell}^\star]^2 \geq  \frac{\gamma^2}{d}} \geq 1 -  m^\star \paren{1 - \frac{1}{\sqrt{2\pi}\cdot \gamma}\exp{-\frac{\gamma^2}{2}}}^{m - m^\star}
    \]
    For the failing probability to be upper bounded by $\delta_{\mathbb{P}}$, we simply need
    \[
        m \geq m^\star + \frac{\log \frac{m^\star}{\delta_{\mathbb{P}}}}{\log \paren{1 - \frac{1}{\sqrt{2\pi}\cdot \gamma}\exp{-\frac{\gamma^2}{2}}}^{-1}}
    \]
    Using $\log \frac{1}{x}\geq 1 - x$ for all $x > 0$, it suffice to guarantee that
    \[
        m \geq m^\star + \sqrt{2\pi }\cdot \gamma \exp{\frac{\gamma^2}{2}}\log \frac{m^\star}{\delta_{\mathbb{P}}}
    \]
    Setting $\gamma = \paren{\log m^\star}^\frac{1}{2}$ gives the desired result. 
\end{proof}

\subsection{Hermite Expansion of the Gradient and the Gradient Flow Dynamics}
Notice that the population MSE has the following form
\[
    \calL\paren{\bm{\theta}} = \frac{1}{2}\EXP[\bfx]{f\paren{\bm{\theta},\bfx}^2 - 2f\paren{\bm{\theta},\bfx}f\paren{\bm{\theta}^\star,\bfx} + f\paren{\bm{\theta}^\star,\bfx}^2}
\]
where the last term is independent of $\bm{\theta}^\star$, and is thus omitted from our analysis. Consider the Hermite expansion of $\pi\paren{\cdot}$ and $\sigma\paren{\cdot}$, respectively
\[
    \pi\paren{x} = \sum_{k=0}^{\infty}c_kHe_k\paren{x};\quad \sigma\paren{x} = He_3\paren{x}
\]
where $c_k = \EXP[x\sim\calN\paren{0,1}]{\pi^{(k)}\paren{x}}$. Then we have that
\begin{align*}
    f\paren{\bm{\theta},\bfx}^2 & = \sum_{i,j=1}^m\sum_{k,\ell=0}^\infty\frac{c_kc_{\ell}}{k!\ell!}He_k\paren{\bfx^\top\bbfv_i}He_{\ell}\paren{\bfx^\top\bbfv_j}He_3\paren{\bfx^\top\bbfw_i}He_3\paren{\bfx^\top\bbfw_j}\\
    f\paren{\bm{\theta},\bfx}f\paren{\bm{\theta}^\star\bfx} & = \sum_{i=1}^m\sum_{j=1}^{m^\star}\sum_{k,\ell=0}^\infty\frac{c_kc_{\ell}}{k!\ell!}He_k\paren{\bfx^\top\bbfv_i}He_{\ell}\paren{\bfx^\top\bbfv_j^\star}He_3\paren{\bfx^\top\bbfw_i}He_3\paren{\bfx^\top\bbfw_j^\star}
\end{align*}
This gives that
\begin{align*}
    \frac{\partial}{\partial\bfv_i}f\paren{\bm{\theta},\bfx}^2 & = 2\sum_{j=1}^{m}\sum_{k,\ell=0}^{\infty}\frac{c_{k+1}c_{\ell}}{k!\ell!}\cdot He_{k}\paren{\bfx^\top\bbfv_i}He_{\ell}\paren{\bfx^\top\bbfv_j}He_3\paren{\bfx^\top\bbfw_i}He_3\paren{\bfx^\top\bbfw_j}\paren{\bfI - \bbfv_i\bbfv_i^\top}\frac{\bfx}{\norm{\bfv_i}_2}
\end{align*}
Taking the expectation gives
\begin{align*}
    & \EXP[\bfx]{\frac{\partial}{\partial\bfv_i}f\paren{\bm{\theta},\bfx}^2}\\
    &\qqquad = \frac{2\paren{\bfI - \bbfv_i\bbfv_i^\top}}{\norm{\bfv_i}_2}\sum_{j=1}^m\sum_{k,\ell=0}^{\infty}\frac{c_{k+1}c_{\ell}}{k!\ell!}\EXP[\bfx]{He_{k}\paren{\bfx^\top\bbfv_i}He_{\ell}\paren{\bfx^\top\bbfv_j}He_3\paren{\bfx^\top\bbfw_i}He_3\paren{\bfx^\top\bbfw_j}\bfx}\\
    & \qqquad = \frac{2\paren{\bfI - \bbfv_i\bbfv_i^\top}}{\norm{\bfv_i}_2}\sum_{j=1}^m\sum_{k,\ell=0}^{\infty}\frac{c_{k+1}c_{\ell+1}}{k!\ell!}\EXP[\bfx]{He_{k}\paren{\bfx^\top\bbfv_i}He_{\ell}\paren{\bfx^\top\bbfv_j}He_3\paren{\bfx^\top\bbfw_i}He_3\paren{\bfx^\top\bbfw_j}}\bbfv_j\\
    & \qqqqquad + \frac{6\paren{\bfI - \bbfv_i\bbfv_i^\top}}{\norm{\bfv_i}_2}\sum_{j=1}^m\sum_{k,\ell=0}^{\infty}\frac{c_{k+1}c_{\ell}}{k!\ell!}\EXP[\bfx]{He_{k}\paren{\bfx^\top\bbfv_i}He_{\ell}\paren{\bfx^\top\bbfv_j}He_2\paren{\bfx^\top\bbfw_i}He_3\paren{\bfx^\top\bbfw_j}}\bbfw_i\\
    & \qqqqquad + \frac{6\paren{\bfI - \bbfv_i\bbfv_i^\top}}{\norm{\bfv_i}_2}\sum_{j=1}^m\sum_{k,\ell=0}^{\infty}\frac{c_{k+1}c_{\ell}}{k!\ell!}\EXP[\bfx]{He_{k}\paren{\bfx^\top\bbfv_i}He_{\ell}\paren{\bfx^\top\bbfv_j}He_3\paren{\bfx^\top\bbfw_i}He_2\paren{\bfx^\top\bbfw_j}}\bbfw_j
\end{align*}
Similar, we can obtain that
\begin{align*}
    & 2\EXP[\bfx]{\frac{\partial}{\partial\bfv_i}f\paren{\bm{\theta},\bfx}f\paren{\bm{\theta}^\star,\bfx}}\\
    & \qqquad = \frac{2\paren{\bfI - \bbfv_i\bbfv_i^\top}}{\norm{\bfv_i}_2}\sum_{j=1}^{m^\star}\sum_{k,\ell=0}^{\infty}\frac{c_{k+1}c_{\ell+1}}{k!\ell!}\EXP[\bfx]{He_{k}\paren{\bfx^\top\bbfv_i}He_{\ell}\paren{\bfx^\top\bbfv_j^\star}He_3\paren{\bfx^\top\bbfw_i}He_3\paren{\bfx^\top\bbfw_j^\star}}\bbfv_j^\star\\
    & \qqqqquad + \frac{6\paren{\bfI - \bbfv_i\bbfv_i^\top}}{\norm{\bfv_i}_2}\sum_{j=1}^{m^\star}\sum_{k,\ell=0}^{\infty}\frac{c_{k+1}c_{\ell}}{k!\ell!}\EXP[\bfx]{He_{k}\paren{\bfx^\top\bbfv_i}He_{\ell}\paren{\bfx^\top\bbfv_j^\star}He_2\paren{\bfx^\top\bbfw_i}He_3\paren{\bfx^\top\bbfw_j^\star}}\bbfw_i\\
    & \qqqqquad + \frac{6\paren{\bfI - \bbfv_i\bbfv_i^\top}}{\norm{\bfv_i}_2}\sum_{j=1}^{m^\star}\sum_{k,\ell=0}^{\infty}\frac{c_{k+1}c_{\ell}}{k!\ell!}\EXP[\bfx]{He_{k}\paren{\bfx^\top\bbfv_i}He_{\ell}\paren{\bfx^\top\bbfv_j^\star}He_3\paren{\bfx^\top\bbfw_i}He_2\paren{\bfx^\top\bbfw_j^\star}}\bbfw_j^\star
\end{align*}
For the gradient of $\bfw_i$, we can compute that
\[
    \frac{\partial}{\partial\bfw_i}f\paren{\bm{\theta},\bfx}^2 = 6\sum_{j=1}^m\sum_{k,\ell=0}^\infty\frac{c_kc_{\ell}}{k!\ell!}He_k\paren{\bfx^\top\bbfv_i}He_{\ell}\paren{\bfx^\top\bbfv_j}He_2\paren{\bfx^\top\bbfw_i}He_3\paren{\bfx^\top\bbfw_j}\paren{\bfI - \bbfw_i\bbfw_i^\top}\frac{\bfx}{\norm{\bbfw_i}_2}
\]
Taking the expectation gives
\begin{align*}
    & \EXP[\bfx]{\frac{\partial}{\partial\bfw_i}f\paren{\bm{\theta},\bfx}^2}\\
    & \qqquad = \frac{6\paren{\bfI - \bbfw_i\bbfw_i^\top}}{\norm{\bfw_i}_2}\sum_{j=1}^m\sum_{k,\ell=0}^{\infty}\frac{c_{k+1}c_{\ell}}{k!\ell!}\EXP[\bfx]{He_{k}\paren{\bfx^\top\bbfv_i}He_{\ell}\paren{\bfx^\top\bbfv_j}He_2\paren{\bfx^\top\bbfw_i}He_3\paren{\bfx^\top\bbfw_j}}\bbfv_i\\
    & \qqqqquad + \frac{6\paren{\bfI - \bbfw_i\bbfw_i^\top}}{\norm{\bfw_i}_2}\sum_{j=1}^m\sum_{k,\ell=0}^{\infty}\frac{c_{k}c_{\ell+1}}{k!\ell!}\EXP[\bfx]{He_{k}\paren{\bfx^\top\bbfv_i}He_{\ell}\paren{\bfx^\top\bbfv_j}He_2\paren{\bfx^\top\bbfw_i}He_3\paren{\bfx^\top\bbfw_j}}\bbfv_j\\
    & \qqqqquad + \frac{18\paren{\bfI - \bbfw_i\bbfw_i^\top}}{\norm{\bfw_i}_2}\sum_{j=1}^m\sum_{k,\ell=0}^{\infty}\frac{c_{k}c_{\ell}}{k!\ell!}\EXP[\bfx]{He_{k}\paren{\bfx^\top\bbfv_i}He_{\ell}\paren{\bfx^\top\bbfv_j}He_2\paren{\bfx^\top\bbfw_i}He_2\paren{\bfx^\top\bbfw_j}}\bbfw_j
\end{align*}
Similarly, we have
\begin{align*}
    & 2\EXP[\bfx]{\frac{\partial}{\partial\bfw_i}f\paren{\bm{\theta},\bfx}f\paren{\bm{\theta}^\star,\bfx}}\\
    & \qqquad = \frac{6\paren{\bfI - \bbfw_i\bbfw_i^\top}}{\norm{\bfw_i}_2}\sum_{j=1}^{m^\star}\sum_{k,\ell=0}^{\infty}\frac{c_{k+1}c_{\ell}}{k!\ell!}\EXP[\bfx]{He_{k}\paren{\bfx^\top\bbfv_i}He_{\ell}\paren{\bfx^\top\bbfv_j^\star}He_2\paren{\bfx^\top\bbfw_i}He_3\paren{\bfx^\top\bbfw_j^\star}}\bbfv_i\\
    & \qqqqquad + \frac{6\paren{\bfI - \bbfw_i\bbfw_i^\top}}{\norm{\bfw_i}_2}\sum_{j=1}^{m^\star}\sum_{k,\ell=0}^{\infty}\frac{c_{k}c_{\ell+1}}{k!\ell!}\EXP[\bfx]{He_{k}\paren{\bfx^\top\bbfv_i}He_{\ell}\paren{\bfx^\top\bbfv_j^\star}He_2\paren{\bfx^\top\bbfw_i}He_3\paren{\bfx^\top\bbfw_j^\star}}\bbfv_j^\star\\
    & \qqqqquad + \frac{18\paren{\bfI - \bbfw_i\bbfw_i^\top}}{\norm{\bfw_i}_2}\sum_{j=1}^{m^\star}\sum_{k,\ell=0}^{\infty}\frac{c_{k}c_{\ell}}{k!\ell!}\EXP[\bfx]{He_{k}\paren{\bfx^\top\bbfv_i}He_{\ell}\paren{\bfx^\top\bbfv_j^\star}He_2\paren{\bfx^\top\bbfw_i}He_2\paren{\bfx^\top\bbfw_j^\star}}\bbfw_j^\star
\end{align*}
This gives that
\begin{gather*}
    \EXP[\bfx]{\frac{\partial}{\partial\bfv_i}f\paren{\bm{\theta},\bfx}f\paren{\bm{\theta}^\star,\bfx}}^\top\bfv_i = \EXP[\bfx]{\frac{\partial}{\partial\bfv_i}f\paren{\bm{\theta},\bfx}^2}^\top\bfv_i = 0\\
    \EXP[\bfx]{\frac{\partial}{\partial\bfw_i}f\paren{\bm{\theta},\bfx}f\paren{\bm{\theta}^\star,\bfx}}^\top\bfw_i = \EXP[\bfx]{\frac{\partial}{\partial\bfw_i}f\paren{\bm{\theta},\bfx}^2}^\top\bfw_i = 0
\end{gather*}
According to the gradient flow dynamics, we can conclude that
\[
    \frac{d}{dt}\norm{\bfv_{i}(t)}_2^2 = \frac{d}{dt}\norm{\bfw_{i}(t)}_2^2 = 0
\]
which implies that the norm of each $\bfv_i$ and $\bfw_i$ are fixed at initialization. For the convenience of the analysis, we shall denote
\begin{align*}
    \mathcal{C}_{k,\ell,a,b}^{i,j} & = \EXP[\bfx]{He_k\paren{\bfx^\top\bbfv_i}He_{\ell}\paren{\bfx^\top\bbfv_j}He_a\paren{\bfx^\top\bbfw_i}He_b\paren{\bfx^\top\bbfw_j}}\\
    \hat{\mathcal{C}}_{k,\ell,a,b}^{i,j} & = \EXP[\bfx]{He_k\paren{\bfx^\top\bbfv_i}He_{\ell}\paren{\bfx^\top\bbfv_j^\star}He_a\paren{\bfx^\top\bbfw_i}He_b\paren{\bfx^\top\bbfw_j^\star}}
\end{align*}
and that $\norm{\bbfv_i(t)}_2 = a_i, \norm{\bbfw_i(t)}_2 = b_i$
Then we have that
\begin{equation}
    \label{eq:grad_full_form}
    \begin{aligned}
        \nabla_{\bfv_i}\calL\paren{\bm{\theta}} & = \frac{\bfI - \bbfv_i\bbfv_i^\top}{\norm{\bfv_i}_2}\paren{\sum_{r=1}^m\sum_{k,\ell=0}^{\infty}\frac{c_{k+1}c_{\ell+1}}{k!\ell!}\mathcal{C}_{k,\ell,3,3}^{i,r}\bbfv_r - \sum_{r=1}^{m^\star}\sum_{k,\ell=0}^{\infty}\frac{c_{k+1}c_{\ell+1}}{k!\ell!}\hat{\mathcal{C}}_{k,\ell,3,3}^{i,r}\bbfv_r^\star}\\
        & \qqquad + \frac{3\paren{\bfI - \bbfv_i\bbfv_i^\top}}{\norm{\bfv_i}_2}\paren{\sum_{r=1}^m\sum_{k,\ell=0}^{\infty}\frac{c_{k+1}c_{\ell}}{k!\ell!}\mathcal{C}^{i,r}_{k,\ell,2,3} - \sum_{r=1}^{m^\star}\sum_{k,\ell=0}^{\infty}\frac{c_{k+1}c_{\ell}}{k!\ell!}\hat{\mathcal{C}}^{i,r}_{k,\ell,2,3}}\bbfw_i\\
        & \qqquad + \frac{3\paren{\bfI - \bbfv_i\bbfv_i^\top}}{\norm{\bfv_i}_2}\paren{\sum_{r=1}^m\sum_{k,\ell=0}^{\infty}\frac{c_{k+1}c_{\ell}}{k!\ell!}\mathcal{C}^{i,r}_{k,\ell,3,2}\bbfw_r - \sum_{r=1}^{m^\star}\sum_{k,\ell=0}^{\infty}\frac{c_{k+1}c_{\ell}}{k!\ell!}\hat{\mathcal{C}}^{i,r}_{k,\ell,3,2}\bbfw_r^\star}\\
        \nabla_{\bfw_i}\calL\paren{\bm{\theta}} & = \frac{3\paren{\bfI - \bbfw_i\bbfw_i^\top}}{\norm{\bfw_i}_2}\paren{\sum_{r=1}^m\sum_{k,\ell=0}^{\infty}\frac{c_{k+1}c_{\ell}}{k!\ell!}\mathcal{C}_{k,\ell,2,3}^{i,r} - \sum_{r=1}^{m^\star}\sum_{k,\ell=0}^{\infty}\frac{c_{k+1}c_{\ell}}{k!\ell!}\hat{\mathcal{C}}_{k,\ell,2,3}^{i,r}}\bbfv_i\\
        & \qqquad + \frac{3\paren{\bfI - \bbfw_i\bbfw_i^\top}}{\norm{\bfv_i}_2}\paren{\sum_{r=1}^m\sum_{k,\ell=0}^{\infty}\frac{c_{k}c_{\ell+1}}{k!\ell!}\mathcal{C}_{k,\ell,2,3}^{i,r}\bbfv_r - \sum_{r=1}^{m^\star}\sum_{k,\ell=0}^{\infty}\frac{c_{k}c_{\ell+1}}{k!\ell!}\hat{\mathcal{C}}_{k,\ell,2,3}^{i,r}\bbfv_r^\star}\\
        & \qqquad + \frac{9\paren{\bfI - \bbfw_i\bbfw_i^\top}}{\norm{\bfw_i}_2}\paren{\sum_{r=1}^m\sum_{k,\ell=0}^{\infty}\frac{c_kc_{\ell}}{k!\ell!}\mathcal{C}_{k,\ell,2,2}^{i,r}\bbfw_r - \sum_{r=1}^{m^\star}\sum_{k,\ell=0}^{\infty}\frac{c_kc_{\ell}}{k!\ell!}\hat{\mathcal{C}}_{k,\ell,2,2}^{i,r}\bbfw_r^\star}
    \end{aligned}
\end{equation}
In particular, we notice that there there are several quantities that appears in the form of the gradient. We make the following definition for the convenience of the analysis
\begin{gather*}
    \lambda_{i,j,1} = \sum_{k,\ell=0}^{\infty}\frac{c_{k+1}c_{\ell+1}}{k!\ell!}\mathcal{C}_{k,\ell,3,3}^{i,j};\quad \hat{\lambda}_{i,j,1} = \sum_{k,\ell=0}^{\infty}\frac{c_{k+1}c_{\ell+1}}{k!\ell!}\hat{\mathcal{C}}_{k,\ell,3,3}^{i,j};\\
    \lambda_{i,j,2}= \sum_{k,\ell=0}^{\infty}\frac{c_{k+1}c_{\ell}}{k!\ell!}\mathcal{C}_{k,\ell,2,3}^{i,j};\quad \hat{\lambda}_{i,j,2}= \sum_{k,\ell=0}^{\infty}\frac{c_{k+1}c_{\ell}}{k!\ell!}\hat{\mathcal{C}}_{k,\ell,2,3}^{i,j}\\
    \lambda_{i,j,3}= \sum_{k,\ell=0}^{\infty}\frac{c_{k}c_{\ell+1}}{k!\ell!}\mathcal{C}_{k,\ell,2,3}^{i,j};\quad \hat{\lambda}_{i,j,3}= \sum_{k,\ell=0}^{\infty}\frac{c_{k}c_{\ell+1}}{k!\ell!}\hat{\mathcal{C}}_{k,\ell,2,3}^{i,j}\\
    \lambda_{i,j,4}= \sum_{k,\ell=0}^{\infty}\frac{c_{k+1}c_{\ell}}{k!\ell!}\mathcal{C}_{k,\ell,3,2}^{i,j};\quad \hat{\lambda}_{i,j,4}= \sum_{k,\ell=0}^{\infty}\frac{c_{k+1}c_{\ell}}{k!\ell!}\hat{\mathcal{C}}_{k,\ell,3,2}^{i,j}\\
    \lambda_{i,j,5}= \sum_{k,\ell=0}^{\infty}\frac{c_{k}c_{\ell}}{k!\ell!}\mathcal{C}_{k,\ell,2,2}^{i,j};\quad \hat{\lambda}_{i,j,5}= \sum_{k,\ell=0}^{\infty}\frac{c_{k}c_{\ell}}{k!\ell!}\hat{\mathcal{C}}_{k,\ell,2,2}^{i,j}
\end{gather*}
Recall that our goal is to study the dynamics of the following alignment scores
\begin{equation}
    \begin{gathered}
        \gmij{1}{}(t) = \bbfv_i(t)^\top\bbfv_j^\star;\quad \gmij{2}{}(t) = \bbfw_i(t)^\top\bbfw_j^\star;\quad \ztij{1}{}(t) = \bbfv_i(t)^\top\bbfw_j^\star;\quad \ztij{2}{}(t) = \bbfw_i(t)^\top\bbfv_j^\star\\
        I_{i,j}^{(1)}(t) = \bbfv_i(t)^\top\bbfv_j(t);\quad I_{i,j}^{(2)}(t) = \bbfw_i(t)^\top\bbfw_j(t);\quad I_{i,j}^{(3)}(t) = \bbfv_i(t)^\top\bbfw_j(t)
    \end{gathered}
\end{equation}
This allows us to rewrite the gradient as
\begin{align*}
    \nabla_{\bfv_i}\calL\paren{\bm{\theta}} & = \frac{1}{a_i}\paren{\bfI - \bbfv_i\bbfv_i^\top}\paren{\sum_{r=1}^m\paren{\lambda_{i,r,1}\bbfv_r + 3\lambda_{i,r,2}\bbfw_i + 3\lambda_{i,r,4}\bbfw_r} - \sum_{r=1}^{m^\star}\paren{\hat{\lambda}_{i,r,1}\bbfv_r^\star + 3\hat{\lambda}_{i,r,2}\bbfw_i + 3\hat{\lambda}_{i,r,4}\bbfw_r^\star}}\\
    & = \frac{1}{a_i}\sum_{r=1}^m\paren{\lambda_{i,r,1}\bbfv_r + 3\lambda_{i,r,2}\bbfw_i + 3\lambda_{i,r,4}\bbfw_r} - \frac{1}{a_i}\sum_{r=1}^{m^\star}\paren{\hat{\lambda}_{i,r,1}\bbfv_r^\star + 3\hat{\lambda}_{i,r,2}\bbfw_i + 3\hat{\lambda}_{i,r,4}\bbfw_r^\star}\\
    & \qqquad -\frac{1}{a_i}\sum_{r=1}^m\paren{\lambda_{i,r,1}I_{i,r}^{(1)} + 3\lambda_{i,r,2}I_{i,i}^{(3)} + 3\lambda_{i,r,4}I_{i,r}^{(3)}}\bbfv_i\\
    & \qqquad + \frac{1}{a_i}\sum_{r=1}^{m^\star}\paren{\hat{\lambda}_{i,r,1}\gamma_{i,r}^{(1)} + 3\hat{\lambda}_{i,r,2}I_{i,i}^{(3)} + 3\hat{\lambda}_{i,r,4}\zeta_{i,r}^{(1)}}\bbfv_i\\
    \nabla_{\bfw_i}\calL\paren{\bm{\theta}} & = \frac{3}{b_i}\paren{\bfI - \bbfw_i\bbfw_i^\top}\paren{\sum_{r=1}^m\paren{\lambda_{i,r,2}\bbfv_i + \lambda_{i,r,3}\bbfv_r + 3\lambda_{i,r,5}\bbfw_r} - \sum_{r=1}^{m^\star}\paren{\hat{\lambda}_{i,r,2}\bbfv_i + \hat{\lambda}_{i,r,3}\bbfv_r^\star + 3\hat{\lambda}_{i,r,5}\bbfw_r^\star}}\\
    & = \frac{3}{b_i}\sum_{r=1}^m\paren{\lambda_{i,r,2}\bbfv_i + \lambda_{i,r,3}\bbfv_r + 3\lambda_{i,r,5}\bbfw_r} - \frac{3}{b_i}\sum_{r=1}^{m^\star}\paren{\hat{\lambda}_{i,r,2}\bbfv_i + \hat{\lambda}_{i,r,3}\bbfv_r^\star + 3\hat{\lambda}_{i,r,5}\bbfw_r^\star}\\
    & \qqquad - \frac{3}{b_i}\sum_{r=1}^m\paren{\lambda_{i,r,2}I_{i,i}^{(3)} + \lambda_{i,r,3}I_{r,i}^{(3)} + 3\lambda_{i,r,5}I_{r,i}^{(2)}}\bbfw_i\\
    & \qqquad + \frac{3}{b_i}\sum_{r=1}^{m^\star}\paren{\hat{\lambda}_{i,r,2}I_{i,i}^{(3)} + \hat{\lambda}_{i,r,3}\zeta_{i,r}^{(2)} + 3\hat{\lambda}_{i,r,5}\gamma_{i,r}^{(2)}}\bbfw_i
\end{align*}
Recall that our goal is to  study the dynamics in (\ref{eq:target_dynamic1}). By the gradient flow dynamic, we have that
\begin{gather*}
    \derivt\gmij{1}{}(t) = - \frac{1}{a_i}\cdot \nabla_{\bfv_i}\calL\paren{\bm{\theta}(t)}^\top\bbfv_j^\star;\quad \derivt\gmij{2}{}(t) = - \frac{1}{b_i}\cdot \nabla_{\bfw_i}\calL\paren{\bm{\theta}(t)}^\top\bbfw_j^\star\\
    \derivt\ztij{1}{}(t) = - \frac{1}{a_i}\cdot \nabla_{\bfv_i}\calL\paren{\bm{\theta}(t)}^\top\bbfw_j^\star;\quad \derivt\ztij{2}{}(t) = - \frac{1}{b_i}\cdot \nabla_{\bfw_i}\calL\paren{\bm{\theta}(t)}^\top\bbfv_j^\star
\end{gather*}
Moreover, we also have that
\begin{align*}
    \derivt I_{i,j}^{(1)}(t) & = - \frac{1}{a_i}\nabla_{\bfv_i}\calL\paren{\bm{\theta}(t)}^\top\bbfv_j - \frac{1}{a_j}\nabla_{\bfv_j}\calL\paren{\bm{\theta}(t)}^\top\bbfv_i\\
    \derivt I_{i,j}^{(2)}(t) & = - \frac{1}{b_i}\nabla_{\bfw_i}\calL\paren{\bm{\theta}(t)}^\top\bbfw_j - \frac{1}{b_j} \nabla_{\bfw_j}\calL\paren{\bm{\theta}(t)}^\top\bbfw_i\\
    \derivt I_{i,j}^{(3)}(t) & = - \frac{1}{a_i}\nabla_{\bfv_i}\calL\paren{\bm{\theta}(t)}^\top\bbfw_j - \frac{1}{b_j} \nabla_{\bfw_j}\calL\paren{\bm{\theta}(t)}^\top\bbfv_i
\end{align*}
Therefore, we should consider inner product between the gradient and the vectors $\bbfv_i,\bbfw_i,\bbfv_j^\star,\bbfw_j^\star$ above, which will give us exactly eight terms to analyze. The inner product between $\nabla_{\bfv_i}\calL\paren{\bm{\theta}(t)}$ and $\bbfv_j^\star,\bbfw_j^\star$ can be written as
\begin{align*}
    - \frac{1}{a_i}\nabla_{\bfv_i}\calL\paren{\bm{\theta}(t)}^\top\bbfv_j^\star & = \frac{1}{a_i^2}\paren{\hat{\lambda}_{i,j,1}(t)- \sum_{r=1}^m\lambda_{i,r,1}(t)\gamma_{r,j}^{(1)}(t)}\\
    & \qqquad - \frac{3}{a_i^2}\paren{\sum_{r=1}^m\paren{\lambda_{i,r,2}(t)\zeta_{i,j}^{(2)}(t) + \lambda_{i,r,4}\zeta_{r,j}^{(2)}(t)} - \sum_{r=1}^{m^\star}\hat{\lambda}_{i,r,2}(t)\zeta_{i,j}^{(2)}(t)}\\
    & \qqquad + \frac{1}{a_i^2}\sum_{r=1}^m\paren{\lambda_{i,r,1}(t)I_{i,r}^{(1)}(t) + 3\lambda_{i,r,2}(t)I_{i,i}^{(3)}(t) + 3\lambda_{i,r,4}(t)I_{i,r}^{(3)}(t)}\gamma_{i,j}^{(1)}(t)\\
    & \qqquad - \frac{1}{a_i^2}\sum_{r=1}^{m^\star}\paren{\hat{\lambda}_{i,r,1}(t)\gamma_{i,r}^{(1)}(t) + 3\hat{\lambda}_{i,r,2}(t)I_{i,i}^{(3)}(t) + 3\hat{\lambda}_{i,r,4}(t)\zeta_{i,r}^{(1)}(t)}\gamma_{i,j}^{(1)}(t)\\
    - \frac{1}{a_i}\nabla_{\bfv_i}\calL\paren{\bm{\theta}(t)}^\top\bbfw_j^\star & = \frac{3}{a_i^2}\paren{\hat{\lambda}_{i,j,4}(t) - \sum_{r=1}^m\lambda_{i,r,4}(t)\gamma_{r,j}^{(2)}(t)}\\
    & \qqquad - \frac{1}{a_i^2}\paren{\sum_{r=1}^m\paren{\lambda_{i,r,1}(t)\zeta_{r,j}^{(1)}(t) + 3\lambda_{i,r,2}(t)\gamma_{i,j}^{(2)}(t)} - 3\sum_{r=1}^{m^\star}\hat{\lambda}_{i,r,2}(t)\gamma_{i,j}^{(2)}(t)}\\
    & \qqquad + \frac{1}{a_i^2}\sum_{r=1}^m\paren{\lambda_{i,r,1}(t)I_{i,r}^{(1)}(t) + 3\lambda_{i,r,2}(t)I_{i,i}^{(3)}(t) + 3\lambda_{i,r,4}(t)I_{i,r}^{(3)}(t)}\zeta_{i,j}^{(1)}(t)\\
    & \qqquad - \frac{1}{a_i^2}\sum_{r=1}^{m^\star}\paren{\hat{\lambda}_{i,r,1}(t)\gamma_{i,r}^{(1)}(t) + 3\hat{\lambda}_{i,r,2}(t)I_{i,i}^{(3)}(t) + 3\hat{\lambda}_{i,r,4}\zeta_{i,r}^{(1)}(t)}\zeta_{i,j}^{(1)}(t)
\end{align*}
The inner product between $\nabla_{\bfw_i}\calL\paren{\bm{\theta}(t)}$ and $\bbfv_j^\star,\bbfw_j^\star$ can be written as
\begin{align*}
    - \frac{1}{b_i}\nabla_{\bfw_i}\calL\paren{\bm{\theta}(t)}^\top\bbfw_j^\star & = \frac{9}{b_i^2}\paren{\hat{\lambda}_{i,j,5}(t) - \sum_{r=1}^m\lambda_{i,r,5}(t)\gamma_{r,j}^{(2)}(t)}\\
    & \qqquad - \frac{3}{b_i^2}\paren{\sum_{r=1}^m\paren{\lambda_{i,r,2}(t)\zeta_{i,j}^{(1)}(t) + \lambda_{i,r,3}(t)\zeta_{r,j}^{(1)}(t)} - \sum_{r=1}^{m^\star}\hat{\lambda}_{i,r,2}(t)\zeta_{i,j}^{(1)}(t)}\\
    & \qqquad + \frac{3}{b_i^2}\sum_{r=1}^m\paren{\lambda_{i,r,2}(t)I_{i,i}^{(3)}(t) + \lambda_{i,r,3}(t)I_{r,i}^{(3)}(t) + 3\lambda_{i,r,5}(t)I_{r,i}^{(2)}(t)}\gamma_{i,j}^{(2)}(t)\\
    & \qqquad - \frac{3}{b_i^2}\sum_{r=1}^{m^\star}\paren{\hat{\lambda}_{i,r,2}(t)I_{i,i}^{(3)}(t) + \hat{\lambda}_{i,r,3}(t)\zeta_{i,r}^{(2)}(t) + 3\hat{\lambda}_{i,r,5}(t)\gamma_{i,r}^{(2)}(t)}\gamma_{i,j}^{(2)}(t)\\
    - \frac{1}{b_i}\cdot \nabla_{\bfw_i}\calL\paren{\bm{\theta}(t)}^\top\bbfv_j^\star & = \frac{3}{b_i^2}\paren{\hat{\lambda}_{i,j,3}(t) - \sum_{r=1}^m\lambda_{i,r,3}(t)\gamma_{r,j}^{(1)}(t)}\\
    & \qqquad - \frac{3}{b_i^2}\paren{\sum_{r=1}^m\paren{\lambda_{i,r,2}(t)\gamma_{i,j}^{(1)}(t) + 3\lambda_{i,r,5}(t)\zeta_{r,j}^{(2)}(t)} - \sum_{r=1}^{m^\star}\hat{\lambda}_{i,r,2}(t)\gamma_{i,j}^{(1)}(t)}\\
    & \qqquad + \frac{3}{b_i^2}\sum_{r=1}^m\paren{\lambda_{i,r,2}(t)I_{i,i}^{(3)}(t) + \lambda_{i,r,3}(t)I_{r,i}^{(3)}(t) + 3\lambda_{i,r,5}(t)I_{r,i}^{(2)}(t)}\zeta_{i,j}^{(2)}(t)\\
    & \qqquad - \frac{3}{b_i^2}\sum_{r=1}^{m^\star}\paren{\hat{\lambda}_{i,r,2}(t)I_{i,i}^{(3)}(t) + \hat{\lambda}_{i,r,3}(t)\zeta_{i,r}^{(2)}(t) + 3\hat{\lambda}_{i,r,5}(t)\gamma_{i,r}^{(2)}(t)}\zeta_{i,j}^{(2)}(t)
\end{align*}
The inner product between $\nabla_{\bfv_i}\calL\paren{\bm{\theta}(t)}$ and $\bbfv_j,\bbfw_j$ can be written as
\begin{align*}
    - \frac{1}{a_i}\nabla_{\bfv_i}\calL\paren{\bm{\theta}(t)}^\top\bbfv_j & = \frac{1}{a_i^2}\sum_{r=1}^{m^\star}\paren{\hat{\lambda}_{i,r,1}(t)\gamma_{j,r}^{(1)}(t) + 3\hat{\lambda}_{i,r,2}(t)I_{j,i}^{(3)}(t) + 3\hat{\lambda}_{i,r,4}(t)\zeta_{j,r}^{(1)}(t)}\\
    & \qqquad- \frac{1}{a_i^2}\sum_{r=1}^m\paren{\lambda_{i,r,1}(t)I_{r,j}^{(1)}(t) + 3\lambda_{i,r,2}(t)I_{j,i}^{(3)}(t) + 3\lambda_{i,r,4}(t)I_{j,r}^{(3)}(t)}\\
    & \qqquad  + \frac{1}{a_i^2}\sum_{r=1}^m\paren{\lambda_{i,r,1}(t)I_{i,r}^{(1)}(t) + 3\lambda_{i,r,2}(t)I_{i,i}^{(3)}(t) + 3\lambda_{i,r,4}(t)I_{i,r}^{(3)}(t)}I_{i,j}^{(1)}(t)\\
    & \qqquad - \frac{1}{a_i^2}\sum_{r=1}^{m^\star}\paren{\hat{\lambda}_{i,r,1}(t)\gamma_{i,r}^{(1)}(t) + 3\hat{\lambda}_{i,r,2}(t)I_{i,i}^{(3)}(t) + 3\hat{\lambda}_{i,r,4}(t)\zeta_{i,r}^{(1)}(t)}I_{i,j}^{(1)}(t)\\
    - \frac{1}{a_i}\nabla_{\bfv_i}\calL\paren{\bm{\theta}(t)}^\top\bbfw_j & = \frac{1}{a_i^2}\sum_{r=1}^{m^\star}\paren{\hat{\lambda}_{i,r,1}(t)\zeta_{j,r}^{(2)} + 3\hat{\lambda}_{i,r,2}(t)I_{i,j}^{(2)} + 3\hat{\lambda}_{i,r,4}(t)\gamma_{j,r}^{(2)}(t)}\\
    & \qqquad - \frac{1}{a_i^2}\sum_{r=1}^m\paren{\lambda_{i,r,1}(t)I_{r,j}^{(3)}(t) + 3\lambda_{i,r,2}(t)I_{i,j}^{(2)}(t) + 3\lambda_{i,r,4}(t)I_{r,j}^{(2)}(t)}\\
    & \qqquad + \frac{1}{a_i^2}\sum_{r=1}^m\paren{\lambda_{i,r,1}(t)I_{i,r}^{(1)}(t) + 3\lambda_{i,r,2}(t)I_{i,i}^{(3)}(t) + 3\lambda_{i,r,4}(t)I_{i,r}^{(3)}(t)}I_{i,j}^{(3)}(t)\\
    & \qqquad - \frac{1}{a_i^2}\sum_{r=1}^{m^\star}\paren{\hat{\lambda}_{i,r,1}(t)\gamma_{i,r}^{(1)}(t) + 3\hat{\lambda}_{i,r,2}(t)I_{i,i}^{(3)}(t) + 3\hat{\lambda}_{i,r,4}(t)\zeta_{i,r}^{(1)}(t)}I_{i,j}^{(3)}(t)
\end{align*}
Lastly, the inner product between $\nabla_{\bfw_i}\calL\paren{\bm{\theta}(t)}$ and $\bbfv_j,\bbfw_j$ can be written as
\begin{align*}
    - \frac{1}{b_i}\nabla_{\bfw_i}\calL\paren{\bm{\theta}(t)}^\top\bbfv_j & = \frac{3}{b_j^2}\sum_{r=1}^{m^\star}\paren{\hat{\lambda}_{i,r,2}(t)I_{i,j}^{(1)}(t) + \hat{\lambda}_{i,r,3}(t)\gamma_{j,r}^{(1)}(t) + 3\hat{\lambda}_{i,r,5}(t)\zeta_{j,r}^{(1)}(t)}\\
    & \qqquad - \frac{3}{b_j^2}\sum_{r=1}^m\paren{\lambda_{i,r,2}(t)I_{i,j}^{(1)}(t) + \lambda_{i,r,3}(t)I_{j,r}^{(1)}(t) + 3\lambda_{i,r,5}(t)I_{j,r}^{(3)}(t)}\\
    & \qqquad + \frac{3}{b_j^2}\sum_{r=1}^m\paren{\lambda_{i,r,2}(t)I_{i,i}^{(3)}(t) + \lambda_{i,r,3}(t)I_{r,i}^{(3)}(t) + 3\lambda_{i,r,5}(t)I_{r,i}^{(2)}(t)}I_{i,j}^{(3)}(t)\\
    & \qqquad - \frac{3}{b_j^2}\sum_{r=1}^{m^\star}\paren{\hat{\lambda}_{i,r,2}(t)I_{i,i}^{(3)}(t) + \hat{\lambda}_{i,r,3}(t)\zeta_{i,r}^{(2)}(t) + 3\hat{\lambda}_{i,r,5}(t)\gamma_{i,r}^{(2)}(t)}I_{i,j}^{(3)}(t)\\
    - \frac{1}{b_i}\nabla_{\bfw_i}\calL\paren{\bm{\theta}(t)}^\top\bbfw_j & = \frac{3}{b_i^2}\sum_{r=1}^{m^\star}\paren{\hat{\lambda}_{i,r,2}(t)I_{i,j}^{(3)}(t) + \hat{\lambda}_{i,r,3}(t)\zeta_{j,r}^{(2)}(t) + 3\hat{\lambda}_{i,r,5}(t)\gamma_{j,r}^{(2)}(t)}\\
    & \qqquad- \frac{3}{b_i^2}\sum_{r=1}^m\paren{\lambda_{i,r,2}(t)I_{i,j}^{(3)}(t) + \lambda_{i,r,3}(t)I_{r,j}^{(3)}(t) + 3\lambda_{i,r,5}(t)I_{r,j}^{(2)}(t)}\\
    & \qqquad  + \frac{3}{b_i^2}\sum_{r=1}^m\paren{\lambda_{i,r,2}(t)I_{i,i}^{(3)}(t) + \lambda_{i,r,3}(t)I_{r,i}^{(3)}(t) + 3\lambda_{i,r,5}(t)I_{r,i}^{(2)}(t)}I_{i,j}^{(2)}(t)\\
    & \qqquad - \frac{3}{b_i^2}\sum_{r=1}^{m^\star}\paren{\hat{\lambda}_{i,r,2}(t)I_{i,i}^{(3)}(t) + \hat{\lambda}_{i,r,3}(t)\zeta_{i,r}^{(2)}(t) + 3\hat{\lambda}_{i,r,5}(t)\gamma_{i,r}^{(2)}(t)}I_{i,j}^{(2)}(t)
\end{align*}

\subsection{Approximating the Gradient Flow Dynamics}
In order to understand the property of the GF induced dynamics given in the previous section, we need to first compute $\lambda_{i,j,\ell}$s and $\hat{\lambda}_{i,j,\ell}$s. The following lemma provides such property.
\begin{lemma}
    \label{lem:approx_lambda_hat}
    Fix $i\in[m], j\in[m^\star]$ and $t\geq 0$. If for any $\left|\zeta_{i,j}^{(1)}(t)\right|, \left|\zeta_{i,j}^{(1)}(t)\right|, \left|I_{i,i}^{(3)}(t)\right| \leq \delta_r$ for some $\delta_r > 0$, then we have that
    \begin{align*}
        \hat{\lambda}_{i,j,1}(t) & = 6\sum_{k=0}^{\infty}\frac{c_{k+1}^2}{k!}\gamma_{i,j}^{(1)}(t)^k\gamma_{i,j}^{(2)}(t)^3 \pm \calO\paren{\delta_r^2}\gamma_{i,j}^{(2)}(t)^2 \pm \calO\paren{\delta_r^4}\\
        \hat{\lambda}_{i,j,2}(t) & = 6\sum_{k=1}^{\infty}\frac{c_{k+2}c_{k}}{k!}\gamma_{i,j}^{(1)}(t)^k\gamma_{i,j}^{(2)}(t)^2\zeta_{i,j}^{(1)}(t) \pm \calO\paren{\delta_r^3}\\
        \hat{\lambda}_{i,j,3}(t) & = 6\sum_{k=0}^{\infty}\frac{c_{k+1}^2}{k!}\gamma_{i,j}^{(1)}(t)^k\gamma_{i,j}^{(2)}(t)^2\zeta_{i,j}^{(1)}(t) \pm \calO\paren{\delta_r^3}\\
        \hat{\lambda}_{i,j,4}(t) & = 6\sum_{k=1}^{\infty}\frac{c_{k+2}c_k}{k!}\gamma_{i,j}^{(1)}(t)^k\gamma_{i,j}^{(2)}(t)^2I_{i,i}^{(3)}(t) + 6\sum_{k=0}^{\infty}\frac{c_{k+1}^2}{k!}\gamma_{i,j}^{(1)}(t)^k\gamma_{i,j}^{(2)}(t)^2\zeta_{i,j}^{(2)}(t) \pm \calO\paren{\delta_r^3}\\
        \hat{\lambda}_{i,j,5}(t) & = 2\sum_{k=0}^{\infty}\frac{c_k^2}{k!}\gamma_{i,j}^{(1)}(t)^k\gamma_{i,j}^{(2)}(t)^2 \pm \calO\paren{\delta_r^2}\gamma_{i,j}^{(2)}(t) \pm \calO\paren{\delta_r^4}
    \end{align*}
\end{lemma}
\begin{proof}
    We are going to Lemma~\ref{lem:cs_hermite} with $\bfv_1 = \bbfv_i(t),\bfv_2 = \bbfv_j^\star$ and $\bfw_1 = \bbfw_i(t), \bfw_2 = \bbfw_j^\star$. In this case, we have that $\bfv_1^\top\bfv_2 = \gamma_{i,j}^{(1)}(t)$ and $\bfw_1^\top\bfw_2 = \gamma_{i,j}^{(2)}(t)$. We start with $\hat{\lambda}_{i,j,1}(t)$. By definition, we have that
    \[
        \hat{\lambda}_{i,j,1}(t) = \sum_{k,\ell=0}^{\infty}\frac{c_{k+1}c_{\ell+1}}{k!\ell!}\EXP{He_k\paren{\bbfv_i(t)^\top\bfx}He_{\ell}\paren{\bbfv_j^{\star\top}\bfx}He_3\paren{\bbfw_i(t)^\top\bfx}He_3\paren{\bbfw_j^{\star\top}\bfx}}
    \]
    Here, invoking Lemma~\ref{lem:cs_hermite} with $h_k = c_{k+1}, h_{\ell}' = c_{\ell+1}$ gives
    \[
        \hat{\lambda}_{i,j,1}(t) = 6\sum_{k=0}^{\infty}\frac{c_{k+1}^2}{k!}\gamma_{i,j}^{(1)}(t)^k\gamma_{i,j}^{(2)}(t)^3 \pm \calO\paren{\delta_r^2}\gamma_{i,j}^{(2)}(t)^2 \pm \calO\paren{\delta_r^4}
    \]
    For $\hat{\lambda}_{i,j,5}(t)$, by definition we have
    \[
        \hat{\lambda}_{i,j,5}(t) = \sum_{k,\ell=0}^{\infty}\frac{c_{k}c_{\ell}}{k!\ell!}\EXP{He_k\paren{\bbfv_i(t)^\top\bfx}He_{\ell}\paren{\bbfv_j^{\star\top}\bfx}He_2\paren{\bbfw_i(t)^\top\bfx}He_2\paren{\bbfw_j^{\star\top}\bfx}}
    \]
    Invoking Lemma~\ref{lem:cs_hermite} with $h_k = c_k, h_\ell' = c_\ell$ gives
    \[
        \hat{\lambda}_{i,j,5}(t) = 2\sum_{k=0}^{\infty}\frac{c_k^2}{k!}\gamma_{i,j}^{(1)}(t)^k\gamma_{i,j}^{(2)}(t)^2 \pm \calO\paren{\delta_r^2}\gamma_{i,j}^{(2)}(t) \pm \calO\paren{\delta_r^4}
    \]
    For $\hat{\lambda}_{i,j,2}(t)$, by definition we have
    \[
        \hat{\lambda}_{i,j,2}(t) = \sum_{k,\ell=0}^{\infty}\frac{c_{k+1}c_{\ell}}{k!\ell!}\EXP{He_k\paren{\bbfv_i(t)^\top\bfx}He_{\ell}\paren{\bbfv_j^{\star\top}\bfx}He_2\paren{\bbfw_i(t)^\top\bfx}He_3\paren{\bbfw_j^{\star\top}\bfx}}
    \]
    Here, invoking Lemma~\ref{lem:cs_hermite} with $h_k = c_{k+1}, h_{\ell}' = c_{\ell}$, and noticing that $\bbfv_j^{\star\top}\bbfw_j = 0$, gives
    \[
        \hat{\lambda}_{i,j,2}(t) = 6\sum_{k=0}^{\infty}\frac{c_{k+2}c_{k}}{k!}\gamma_{i,j}^{(1)}(t)^k\gamma_{i,j}^{(2)}(t)^2\zeta_{i,j}^{(1)}(t) \pm \calO\paren{\delta_r^3}
    \]
    Noticing that $c_{2} = 0$ gives that
    \[
        \hat{\lambda}_{i,j,2}(t) = 6\sum_{k=1}^{\infty}\frac{c_{k+2}c_{k}}{k!}\gamma_{i,j}^{(1)}(t)^k\gamma_{i,j}^{(2)}(t)^2\zeta_{i,j}^{(1)}(t) \pm \calO\paren{\delta_r^3}
    \]
    For $\hat{\lambda}_{i,j,3}(t)$, by definition, we have that
    \[
        \hat{\lambda}_{i,j,3}(t) = \sum_{k,\ell=0}^{\infty}\frac{c_{k}c_{\ell+1}}{k!\ell!}\EXP{He_k\paren{\bbfv_i(t)^\top\bfx}He_{\ell}\paren{\bbfv_j^{\star\top}\bfx}He_2\paren{\bbfw_i(t)^\top\bfx}He_3\paren{\bbfw_j^{\star\top}\bfx}}
    \]
    Invoking Lemma~\ref{lem:cs_hermite} with $h_k = c_k, h_{\ell}' = c_{\ell+1}$, and noticing that $\bbfv_j^{\star\top}\bbfw_j = 0$, gives
    \[
        \hat{\lambda}_{i,j,3}(t) = 6\sum_{k=0}^{\infty}\frac{c_{k+1}^2}{k!}\gamma_{i,j}^{(1)}(t)^k\gamma_{i,j}^{(2)}(t)^2\zeta_{i,j}^{(1)}(t) \pm \calO\paren{\delta_r^3}
    \]
    Lastly, for $\hat{\lambda}_4$, by definition we have
    \[
        \hat{\lambda}_{i,j,4}(t) = \sum_{k,\ell=0}^{\infty}\frac{c_{k+1}c_{\ell}}{k!\ell!}\EXP{He_k\paren{\bbfv_i(t)^\top\bfx}He_{\ell}\paren{\bbfv_j^{\star\top}\bfx}He_3\paren{\bbfw_i(t)^\top\bfx}He_2\paren{\bbfw_j^{\star\top}\bfx}}
    \]
    Therefore, we need to consider $\bfv_1 = \bbfv_i(t),\bfv_2 = \bbfv_j^\star$ and $\bfw_1 = \bbfw_j^\star,\bfw_2 = \bbfw_i(t)$. Moreover, we need to set $h_k = c_{k+1}$ and $h_\ell' = c_{\ell}$. In this case, $\bfv_1^\top\bfw_2 = I_{i,i}^{(3)}(t)$ and $\bfv_2^\top\bfw_2 = \zeta_{i,j}^{(2)}(t)$. Therefore
    \[
        \hat{\lambda}_{i,j,4}(t) = 6\sum_{k=0}^{\infty}\frac{c_{k+2}c_k}{k!}\gamma_{i,j}^{(1)}(t)^k\gamma_{i,j}^{(2)}(t)^2I_{i,i}^{(3)}(t) + 6\sum_{k=0}^{\infty}\frac{c_{k+1}^2}{k!}\gamma_{i,j}^{(1)}(t)^k\gamma_{i,j}^{(2)}(t)^2\zeta_{i,j}^{(2)}(t) \pm \calO\paren{\delta_r^3}
    \]
    Noticing that $c_2 = 0$ gives that
    \[
        \hat{\lambda}_{i,j,4}(t) = 6\sum_{k=1}^{\infty}\frac{c_{k+2}c_k}{k!}\gamma_{i,j}^{(1)}(t)^k\gamma_{i,j}^{(2)}(t)^2I_{i,i}^{(3)}(t) + 6\sum_{k=0}^{\infty}\frac{c_{k+1}^2}{k!}\gamma_{i,j}^{(1)}(t)^k\gamma_{i,j}^{(2)}(t)^2\zeta_{i,j}^{(2)}(t) \pm \calO\paren{\delta_r^3}
    \]
\end{proof}

\begin{lemma}
    \label{lem:approx_lambda}
    Fix $i,j\in[m]$ and $t\geq 0$. If for any $\left|I_{i,j}^{(1)}(t)\right|, \left|I_{i,j}^{(2)}(t)\right|, \left|I_{i,j}^{(3)}(t)\right| \leq \delta_r$ when $i\neq j$, and $\left|I_{i,i}^{(3)}(t)\right| \leq \delta_p \leq \calO\paren{\delta_r}$, then we have that
    \begin{align*}
        \lambda_{i,j,1}(t) & = \begin{cases}
            \pm \calO\paren{\delta_r^3} & \text{ if } i \neq j\\
            6\sum_{k=0}^{\infty}\frac{c_{k+1}^2}{k!} \pm \calO\paren{\delta_p^2} & \text{ if } i = j
        \end{cases}\\
        \lambda_{i,j,2}(t) & = \begin{cases}
            \pm \calO\paren{\delta_r^3} & \text{ if } i \neq j\\
            6C_{S,2}I_{i,i}^{(3)}(t) \pm \calO\paren{\delta_p^3} & \text{ if } i = j\\
        \end{cases}\\
        \lambda_{i,j,3}(t) & = \begin{cases}
            \pm \calO\paren{\delta_r^3} & \text{ if } i \neq j\\
            6C_{S,2}I_{i,i}^{(3)}(t) \pm \calO\paren{\delta_p^3}& \text{ if } i = j\\
        \end{cases}\\
        \lambda_{i,j,4}(t) & = \begin{cases}
            \pm \calO\paren{\delta_r^3} & \text{ if } i \neq j\\
            6C_{S,2}I_{i,i}^{(3)}(t) \pm \calO\paren{\delta_p^3}& \text{ if } i = j\\
        \end{cases}\\
        \lambda_{i,j,5}(t) & = \begin{cases}
            \pm \calO\paren{\delta_r^2} & \text{ if } i \neq j\\
            2\sum_{k=0}^{\infty}\frac{c_k^2}{k!} \pm \calO\paren{\delta_p^2} & \text{ if } i = j
        \end{cases}
    \end{align*}
    Here $C_{S,2} = \sum_{k=0}^{\infty}\frac{c_{k+1}^2 + c_kc_{k+2}}{k!}$.
\end{lemma}
\begin{proof}
    We are going to use Lemma~\ref{lem:cs_hermite} with $\bfv_1 = \bbfv_i(t),\bfv_2 = \bbfv_j(t)$ and $\bfw_1 = \bbfw_i(t),\bfw_2 = \bbfw_j(t)$. In this case, we have that $\bfv_1^\top\bfv_2 = I_{i,j}^{(1)}(t)$, $\bfw_1^\top\bfw_2= I_{i,j}^{(2}(t)$, and $\bfv_1^\top\bfw_2 = I_{i,j}^{(3)}(t), \bfv_2^\top\bfw_1 = I_{j,i}^{(3)}(t)$. Now, for $\lambda_{i,j,1}$, by definition, we have that
    \[
        \lambda_{i,j,1}(t) = \sum_{k,\ell=0}^{\infty}\frac{c_{k+1}c_{\ell+1}}{k!\ell!}\EXP{He_k\paren{\bbfv_i(t)^\top\bfx}He_{\ell}\paren{\bbfv_j(t)^\top\bfx}He_3\paren{\bbfw_i(t)^\top\bfx}He_3\paren{\bbfw_j(t)^\top\bfx}}
    \]
    Invoking Lemma~\ref{lem:cs_hermite} with $h_k = c_{k+1}, h_\ell' = c_{\ell+1}$ gives
    \[
        \lambda_{i,j,1}(t) = 6\sum_{k=0}^{\infty}\frac{c_{k+1}^2}{k!}I_{i,j}^{(1)}(t)^kI_{i,j}^{(2)}(t)^3 \pm \calO\paren{\delta_r}I_{i,j}^{(2)}(t)^2 \pm \calO\paren{\delta_r^4}
    \]
    Since $\left|I_{i,j}^{(2)}\right|\leq \delta_r$, we have that
    \[
        \lambda_{i,j,1}(t) = \pm \calO\paren{\delta_r^3}
    \]
    In the special case where $i = j$, we have that $I_{i,j}^{(1)} = I_{i,j}^{(2)} = 1$, and $\bfv_1^\top\bfw_1 = \bfv_1^\top\bfw_2 = \bfv_2^\top\bfw_1 = \bfv_2^\top\bfw_2 = I_{i,i}^{(3)}$. Therefore
    \[
        \lambda_{i,j,1}(t) = 6\sum_{k=0}^{\infty}\frac{c_{k+1}^2}{k!} \pm \calO\paren{\delta_p^2}
    \]
    For $\lambda_{i,j,5}$, by definition, we have that
    \[
        \lambda_{i,j,5}(t) = \sum_{k,\ell=0}^{\infty}\frac{c_{k}c_{\ell}}{k!\ell!}\EXP{He_k\paren{\bbfv_i(t)^\top\bfx}He_{\ell}\paren{\bbfv_j(t)^\top\bfx}He_2\paren{\bbfw_i(t)^\top\bfx}He_2\paren{\bbfw_j(t)^\top\bfx}}
    \]
    Invoking Lemma~\ref{lem:cs_hermite} with $h_k = c_k, h_{\ell}' = c_k$ gives
    \[
        \lambda_{i,j,5}(t) = 2\sum_{k=0}^{\infty}\frac{c_k^2}{k!}I_{i,j}^{(1)}(t)^kI_{i,j}^{(2)}(t)^2 \pm \calO\paren{\delta_r^2} = \pm \calO\paren{\delta_r^2}
    \]
    In the case where $i=j$, we have that
    \[
        \lambda_{i,j,5}(t) = 2\sum_{k=0}^{\infty}\frac{c_k^2}{k!} \pm \calO\paren{\delta_p^2}
    \]
    For $\lambda_{i,j,2}$, by definition, we have that
    \[
        \lambda_{i,j,2}(t) = \sum_{k,\ell=0}^{\infty}\frac{c_{k+1}c_{\ell}}{k!\ell!}\EXP{He_k\paren{\bbfv_i(t)^\top\bfx}He_{\ell}\paren{\bbfv_j(t)^\top\bfx}He_2\paren{\bbfw_i(t)^\top\bfx}He_3\paren{\bbfw_j(t)^\top\bfx}}
    \]
    Invoking Lemma~\ref{lem:cs_hermite} with $h_k = c_{k+1}, h_\ell' = c_\ell$ gives
    \[
        \lambda_{i,j,2}(t) = 6\sum_{k=0}^{\infty}\frac{c_{k+2}c_k}{k!}I_{i,j}^{(1)}(t)^kI_{i,j}^{(2)}(t)^2I_{i,j}^{(3)}(t) + 6\sum_{k=0}^{\infty}\frac{c_{k+1}^2}{k!}I_{i,j}^{(1)}(t)^kI_{i,j}^{(2)}(t)^2I_{j,j}^{(3)}(t) \pm \calO\paren{\delta_r^3} = \pm \calO\paren{\delta_r^3}
    \]
    In the case where $i = j$, we have that $I_{i,j}^{(1)}(t) = I_{i,j}^{(2)}(t) = 1$. Therefore
    \[
        \lambda_{i,j,2}(t) = 6I_{i,i}^{(3)}(t) \sum_{k=0}^{\infty}\frac{c_kc_{k+2} +c_{k+1}^2}{k!} \pm \calO\paren{\delta_p^3}
    \]
    For $\lambda_{i,j,3}$, by definition, we have that
    \[
        \lambda_{i,j,3}(t) = \sum_{k,\ell=0}^{\infty}\frac{c_{k}c_{\ell+1}}{k!\ell!}\EXP{He_k\paren{\bbfv_i(t)^\top\bfx}He_{\ell}\paren{\bbfv_j(t)^\top\bfx}He_2\paren{\bbfw_i(t)^\top\bfx}He_3\paren{\bbfw_j(t)^\top\bfx}}
    \]
    Invoking Lemma~\ref{lem:cs_hermite} with $h_k = c_k, h_{\ell}' = c_{\ell+1}$ gives
    \[
        \lambda_{i,j,3}(t) = 6\sum_{k=0}^{\infty}\frac{c_{k+1}^2}{k!}I_{i,j}^{(1)}(t)^kI_{i,j}^{(2)}(t)^2I_{i,j}^{(3)}(t) + 6\sum_{k=0}^{\infty}\frac{c_{k+2}c_k}{k!}I_{i,j}^{(1)}(t)^kI_{i,j}^{(2)}(t)^2I_{j,j}^{(3)}(t) \pm \calO\paren{\delta_r^3} = \pm \calO\paren{\delta_r^3}
    \]
    In the case where $i = j$, we have that $I_{i,j}^{(1)}(t) = I_{i,j}^{(2)}(t) = 1$. Therefore
    \[
        \lambda_{i,j,3}(t) = 6I_{i,i}^{(3)}(t) \sum_{k=0}^{\infty}\frac{c_kc_{k+2} +c_{k+1}^2}{k!} \pm \calO\paren{\delta_p^3}
    \]
    Lastly, for $\lambda_{i,j,4}$, we have that
    \[
        \lambda_{i,j,4}(t) = \sum_{k,\ell=0}^{\infty}\frac{c_{k+1}c_{\ell}}{k!\ell!}\EXP{He_k\paren{\bbfv_i(t)^\top\bfx}He_{\ell}\paren{\bbfv_j(t)^\top\bfx}He_3\paren{\bbfw_i(t)^\top\bfx}He_2\paren{\bbfw_j(t)^\top\bfx}}
    \]
    Here we need to apply Lemma~\ref{lem:cs_hermite} with $\bfv_1 = \bbfv_i(t),\bfv_2 = \bbfv_j(t), \bfw_1 = \bbfw_j(t), \bfw_2 = \bbfw_i(t)$ and $h_k = c_{k+1}, h_{\ell}' = c_{\ell}$. This gives that
    \[
        \lambda_{i,j,4}(t) = 6\sum_{k=0}^{\infty}\frac{c_{k+2}c_k}{k!}I_{i,j}^{(1)}(t)^kI_{i,j}^{(2)}(t)^2I_{i,i}^{(1)}(t) + 6\sum_{k=0}^{\infty}\frac{c_{k+1}^2}{k!}I_{i,j}^{(1)}(t)^kI_{i,j}^{(2)}(t)^2I_{j,i}^{(3)}(t) \pm \calO\paren{\delta_r^3} = \pm \calO\paren{\delta_r^3}
    \]
    In the case where $i = j$, we have that $I_{i,j}^{(1)}(t) = I_{i,j}^{(2)}(t) = 1$. Therefore
    \[
        \lambda_{i,j,4}(t) = 6I_{i,i}^{(3)}(t) \sum_{k=0}^{\infty}\frac{c_kc_{k+2} +c_{k+1}^2}{k!} \pm \calO\paren{\delta_p^3}
    \]
\end{proof}
With the above lemmas that studies $\lambda_{i,j,\ell}$s and $\hat{\lambda}_{i,j,\ell}$s, we are ready to analyze the dynamics of $\gamma_{i,j}^{(1)},\gamma_{i,j}^{(2)},\zeta_{i,j}^{(1)},\zeta_{i,j}^{(2)}$, and $I_{i,j}^{(1)},I_{i,j}^{(2)}, I_{i,j}^{(3)}$. In particular, we will fix any $\ell\in[m^\star]$, and assumes the inductive hypothesis.

\begin{lemma}
    \label{lem:dynamic_approx}
    Let $\aggepst{\ell}(t), \backepsi{\ell}(t)$, and $\varepsilon_{5,\ell}(t)$ be defined in Definition~\ref{def:error_bound} with $\varepsilon_{5,\ell}(t) \leq \calO\paren{\aggepst{\ell}(t)}$. Then the gradient alignment $\nabla_{\bfv_i}\calL\paren{\bm{\theta}(t)}\bbfv_j^\star$ and $\nabla_{\bfw_i}\calL\paren{\bm{\theta}(t)}\bbfw_j^\star$ satisfies
    \begin{align*}
        \nabla_{\bfv_i}\calL\paren{\bm{\theta}(t)}\bbfv_j^\star  =  -\frac{1}{a_i} \cdot \begin{cases}
            \pm \calO\paren{m\aggepst{\ell}(t)^3 + \aggepst{\ell}(t)\backepsi{\ell}(t)} & \text{ if } i  \in[m]\setminus \mathcal{R}_{\ell}\\
            \begin{aligned}
            & - \hat{\lambda}_{i_\ell^\star,j_\ell^\star,1}(t)\gamma_{i_\ell^\star,j_\ell^\star}^{(1)}(t)\gamma_{i,j}^{(1)}(t) \pm \calO\paren{\aggepso{\ell}(t)^2}\gamma_{i_\ell^\star,j_\ell^\star}^{(2)}(t)^2\\
            & \qqquad \pm \calO\paren{m\aggepso{\ell}(t)^3 + \aggepso{\ell}(t)\varepsilon_{5,\ell}(t)}
            \end{aligned} & \text{ if } i = i_\ell^\star, j\in[m^\star]\setminus \mathcal{C}_{\ell}\\
            \begin{aligned}
            & \paren{1 - \gamma_{i_\ell^\star,j_\ell^\star}^{(1)}(t)^2)}\hat{\lambda}_{i_\ell^\star,j_\ell^\star,1}(t) \pm \calO\paren{\aggepso{\ell}(t)^2}\gamma_{i_\ell^\star,j_\ell^\star}^{(2)}(t)^2\\
            & \qqquad \pm \calO\paren{m\aggepso{\ell}(t)^3 + \aggepso{\ell}(t)\varepsilon_{5,\ell}(t)}
            \end{aligned}& \text{ if }i = i_\ell^\star, j = j_\ell^\star
        \end{cases}\\
        \nabla_{\bfw_i}\calL\paren{\bm{\theta}(t)}\bbfw_{j}^\star = - \frac{9}{b_i}\cdot \begin{cases}
            \begin{aligned}
                & \hat{\lambda}_{i,j,5}(t) \pm \calO\paren{m\aggepst{\ell}(t)^3 + \aggepst{\ell}(t)\backepsi{\ell}(t)}\\
                & \qqquad \pm \calO\paren{\aggepst{\ell}(t)^2}\gamma_{i_\ell^\star,j_\ell^\star}^{(2)}(t)
            \end{aligned} & \text{ if } i \in[m]\setminus \mathcal{R}_{\ell}\\
            \begin{aligned}
            & \hat{\lambda}_{i,j,5}(t)- \lambda_{i_\ell^\star,j_\ell^\star,5}(t)\gamma_{i_\ell^\star,j_\ell^\star}^{(2)}(t)\gamma_{i,j}^{(2)}(t)\\
            & \qqquad \pm \calO\paren{\aggepso{\ell}(t)^2}\gamma_{i_\ell^\star,j_\ell^\star}^{(2)}(t)\\
            & \qqquad \pm \calO\paren{m\aggepso{\ell}(t)^2\aggepst{\ell}(t) + \aggepso{\ell}(t)\varepsilon_{5,\ell}(t)}
            \end{aligned} & \text{ if } i = i_\ell^\star, j \in[m^\star]\setminus \mathcal{C}_{\ell}\\
            \begin{aligned}
                & \paren{1 - \gamma_{i_\ell^\star,j_\ell^\star}^{(2)}(t)^2}\hat{\lambda}_{i_\ell^\star,j_\ell^\star,5}(t)\pm \calO\paren{\aggepso{\ell}(t)^2}\gamma_{i_\ell^\star,j_\ell^\star}^{(2)}(t)\\
                & \qqquad \pm \calO\paren{m\aggepso{\ell}(t)^2\aggepst{\ell}(t) + \aggepso{\ell}(t)\varepsilon_{5,\ell}(t)}
            \end{aligned}& \text{ if } i = i_\ell^\star, j = j_\ell^\star
        \end{cases}
    \end{align*}
    and in particular, for the case $\nabla_{\bfw_i}\calL\paren{\bm{\theta}(t)}\bbfw_{j_{\ell'}^\star}^\star$, we have that
    \begin{align*}
        \nabla_{\bfw_i}\calL\paren{\bm{\theta}(t)}\bbfw_{j_{\ell'}^\star}^\star & = \frac{9}{b_i}\lambda_{i,i_{\ell'}^\star,5}(t)\gamma_{i_{\ell'}^\star,j_{\ell'}^\star}^{(2)}(t)\\
        & \qqquad - \frac{9}{b_i}\cdot\begin{cases}
            \hat{\lambda}_{i,j,5}(t) \pm \calO\paren{m\aggepst{\ell}(t)^3 + \aggepst{\ell}(t)\backepsi{\ell}(t)}& \text{ if } i\in[m]\setminus \mathcal{R}_{\ell}\\
            \begin{aligned}
            & \hat{\lambda}_{i,j,5}(t)- \lambda_{i_\ell^\star,j_\ell^\star,5}(t)\gamma_{i_\ell^\star,j_\ell^\star}^{(2)}(t)\gamma_{i,j}^{(2)}(t)\\
            & \qqquad \pm \calO\paren{\aggepso{\ell}(t)^2}\gamma_{i_\ell^\star,j_\ell^\star}^{(2)}(t)^2\\
            & \qqquad \pm \calO\paren{m\aggepso{\ell}(t)^3 + \aggepst{\ell}(t)\varepsilon_{5,\ell}(t)}
            \end{aligned} & \text{ if } i =i_\ell^\star
        \end{cases}
    \end{align*}
    Further more, the mis-alignment terms $\nabla_{\bfv_i}\calL\paren{\bm{\theta}(t)}\bbfw_j^\star$ and $\nabla_{\bfw_i}\calL\paren{\bm{\theta}(t)}\bbfv_j^\star$ satisfies
    \begin{align*}
        \nabla_{\bfv_i}\calL\paren{\bm{\theta}(t)}\bbfw_j^\star = - \frac{1}{a_i}\cdot \begin{cases}
            \pm \calO\paren{m\aggepst{\ell}(t)^3 +\backepsi{\ell}(t)^2} & \text{ if } i\in[m]\setminus \mathcal{R}_{\ell}\\
            \begin{aligned}
                & \hat{\lambda}_{i_\ell^\star,j_\ell^\star,2}(t)\gamma_{i,j}^{(2)}(t) - \hat{\lambda}_{i_\ell^\star,j_\ell^\star,1}(t)\gamma_{i_\ell^\star,j_\ell^\star}^{(1)}(t)\zeta_{i,j}^{(1)}(t)\\
                & \qqquad - 36C_{S,2}\gamma_{i,j}^{(2)}(t)I_{i,i}^{(3)}(t)\\
                & \qqquad \pm \calO\paren{m\aggepso{\ell}(t)^2\aggepst{\ell}(t) +\backepsi{\ell}(t)^2} 
            \end{aligned} &\text{ if } i= i_\ell^\star, j\neq j_\ell^\star\\
            \begin{aligned}
                & 3\hat{\lambda}_{i_\ell^\star,j_\ell^\star,4}(t) + \hat{\lambda}_{i_\ell^\star,j_\ell^\star,2}(t)\gamma_{i_\ell^\star,j_\ell^\star}^{(2)}(t) - \hat{\lambda}_{i_\ell^\star,j_\ell^\star,1}(t)\gamma_{i_\ell^\star,j_\ell^\star}^{(1)}(t)\zeta_{i_\ell^\star,j_\ell^\star}^{(1)}(t)\\
                & \qqquad - 36C_{S,2}\gamma_{i,j}^{(2)}(t)I_{i,i}^{(3)}(t)\\
                & \qqquad \pm \calO\paren{m\aggepso{\ell}(t)^2\aggepst{\ell}(t) +\backepsi{\ell}(t)^2} 
            \end{aligned} &\text{ if } i= i_\ell^\star, j= j_\ell^\star
        \end{cases}\\
        \nabla_{\bfw_i}\calL\paren{\bm{\theta}(t)}\bbfv_j^\star = \frac{3}{b_i}\cdot \begin{cases}
            \pm \calO\paren{m\aggepst{\ell}(t)^3 +\backepsi{\ell}(t)^2} & \text{ if } i\in[m]\setminus \mathcal{R}_{\ell}\\
            \begin{aligned}
                & \hat{\lambda}_{i_\ell^\star,j_\ell^\star,2}(t)\gamma_{i,j}^{(1)}(t) - 3\hat{\lambda}_{i_\ell^\star,j_\ell^\star,5}(t)\gamma_{i_\ell^\star,j_\ell^\star}^{(2)}(t)\zeta_{i,j}^{(2)}(t)\\
                & \qqquad- 12C_{S,2}\gamma_{i,j}^{(2)}(t)I_{i,i}^{(3)}(t)\\
                & \qqquad \pm \calO\paren{m\aggepso{\ell}(t)^2\aggepst{\ell}(t) +\backepsi{\ell}(t)^2}
            \end{aligned} &\text{ if } i = i_\ell^\star, j\neq j_\ell^\star\\
            \begin{aligned}
                & \hat{\lambda}_{i_\ell^\star,j_\ell^\star,3}(t) + \hat{\lambda}_{i_\ell^\star,j_\ell^\star,2}(t)\gamma_{i_\ell^\star,j_\ell^\star}^{(1)}(t) - 3\hat{\lambda}_{i_\ell^\star,j_\ell^\star,5}(t)\gamma_{i_\ell^\star,j_\ell^\star}^{(2)}(t)\zeta_{i_\ell^\star,j_\ell^\star}^{(2)}(t)\\
                & \qqquad - 12C_{S,2}\gamma_{i,j}^{(2)}(t)I_{i,i}^{(3)}(t)\\
                & \qqquad \pm \calO\paren{m\aggepso{\ell}(t)^2\aggepst{\ell}(t) +\backepsi{\ell}(t)^2}
            \end{aligned} &\text{ if } i = i_\ell^\star, j = j_\ell^\star
        \end{cases}
    \end{align*}
    The self-alignments $\nabla_{\bfv_i}\calL\paren{\bm{\theta}(t)}\bbfv_j, \nabla_{\bfw_i}\calL\paren{\bm{\theta}(t)}\bbfw_j$, in the case of $i\neq j$, are given by
    \begin{align*}
        \nabla_{\bfv_i}\calL\paren{\bm{\theta}(t)}^\top\bbfv_j = - \frac{1}{a_i}\cdot \begin{cases}
            \pm \calO\paren{m\aggepst{\ell}(t)^3 + \aggepst{\ell}(t)\backepsi{\ell}(t)}& \text{ if } i\in[m]\setminus \mathcal{R}_{\ell}\\
            \begin{aligned}
                & \hat{\lambda}_{i_\ell^\star,j_\ell^\star,1}(t)\gamma_{j,j_\ell^\star}^{(1)}(t) - \hat{\lambda}_{i_\ell^\star,j_\ell^\star,1}(t)\gamma_{i_\ell^\star,j_\ell^\star}^{(1)}(t)I_{i,j}^{(1)}(t)\\
                & \qqquad \pm \calO\paren{\aggepst{\ell}(t)^2}\gamma_{i_\ell^\star,j_\ell^\star}^{(2)}(t)^2\\
                & \qqquad \pm \calO\paren{m\aggepst{\ell}(t)^3 + \aggepst{\ell}(t)\varepsilon_{5,\ell}(t)}
            \end{aligned} & \text{ if } i = i_\ell^\star
        \end{cases}\\
        \nabla_{\bfw_i}\calL\paren{\bm{\theta}(t)}^\top\bbfw_j = - \frac{9}{b_i}\cdot \begin{cases}
            \pm \calO\paren{m\aggepst{\ell}(t)^3 + \aggepst{\ell}(t)\backepsi{\ell}(t)}&\text{ if } i,j\in[m]\setminus \mathcal{R}_{\ell}\\
            \begin{aligned}
                & \hat{\lambda}_{i_\ell^\star,j_\ell^\star,5}(t)\gamma_{j,j_\ell^\star}^{(2)}(t) - \hat{\lambda}_{i_\ell^\star,j_\ell^\star,5}(t)\gamma_{i_\ell^\star,j_\ell^\star}^{(2)}(t)I_{i,j}^{(2)}(t)\\
                & \qqquad \pm \calO\paren{\aggepst{\ell}(t)^2}\gamma_{i_\ell^\star,j_\ell^\star}^{(2)}(t)^2\\
                & \qqquad \pm \calO\paren{m\aggepst{\ell}(t)^3 + \aggepst{\ell}(t)\varepsilon_{5,\ell}(t)}
            \end{aligned} & \text{ if } i = i_\ell^\star\\
            \begin{aligned}
                 & \pm \calO\paren{\aggepst{\ell}(t)^2}\gamma_{i_\ell^\star,j_\ell^\star}^{(2)}(t)\\
                 & \qqquad \pm \calO\paren{m\aggepst{\ell}(t)^3 + \aggepst{\ell}(t)\backepsi{\ell}(t)}
            \end{aligned}&\text{ if } j=i_\ell^\star
        \end{cases}
    \end{align*}
    The self-alignments $\nabla_{\bfv_i}\calL\paren{\bm{\theta}(t)}\bbfw_j, \nabla_{\bfw_i}\calL\paren{\bm{\theta}(t)}\bbfv_j$, in the case of $i\neq j$, are given by
    \begin{align*}
        \nabla_{\bfv_i}\calL\paren{\bm{\theta}(t)}^\top\bbfw_j & = -\frac{1}{a_i}\cdot \begin{cases}
            \begin{aligned}
                & \hat{\lambda}_{i_\ell^\star,j_\ell^\star,1}(t)\paren{\zeta_{j,j_\ell^\star}^{(2)}(t) - \gamma_{i_\ell^\star,j_\ell^\star}^{(1)}(t)I_{i,j}^{(3)}(t)}\\
                & \qqquad \pm \calO\paren{\aggepst{\ell}(t)^2}\gamma_{i_\ell^\star,j_\ell^\star}^{(2)}(t)^2\\
                & \qqquad \pm \calO\paren{m\aggepst{\ell}(t)^3 + \aggepst{\ell}(t)\varepsilon_{5,\ell}(t)}
            \end{aligned} & \text{ if } i = i_\ell^\star\\
            \pm \calO\paren{m\aggepst{\ell}(t)^3 + \aggepst{\ell}(t)\backepsi{\ell}(t)} & \text{ if } i \in [m]\setminus \mathcal{R}_{\ell}
        \end{cases}\\
        \nabla_{\bfw_i}\calL\paren{\bm{\theta}(t)}^\top\bbfv_j & = -\frac{9}{b_i}\cdot \begin{cases}
            \begin{aligned}
                & \hat{\lambda}_{i_\ell^\star,j_\ell^\star,5}(t)\paren{\zeta_{j,j_\ell^\star}^{(1)}(t) - \gamma_{i_\ell^\star,j_\ell^\star}^{(2)}(t)I_{i,j}^{(3)}(t)}\\
                & \qqquad \pm \calO\paren{\aggepst{\ell}(t)^2}\gamma_{i_\ell^\star,j_\ell^\star}^{(2)}(t)^2\\
                & \qqquad \pm \calO\paren{m\aggepst{\ell}(t)^3 + \aggepst{\ell}(t)\varepsilon_{5,\ell}(t)}
            \end{aligned} & \text{ if } i = i_\ell^\star\\
            \pm \calO\paren{m\aggepst{\ell}(t)^3 + \aggepst{\ell}(t)\backepsi{\ell}(t)} & \text{ if } i \in [m]\setminus \mathcal{R}_{\ell}
        \end{cases}
    \end{align*}
    Lastly, the self-alignments $\nabla_{\bfv_i}\calL\paren{\bm{\theta}(t)}\bbfw_j, \nabla_{\bfw_i}\calL\paren{\bm{\theta}(t)}\bbfv_j$, in the case of $i = j \in [m]\setminus\mathcal{R}_{\ell-1}$, are given by
    \begin{align*}
        \nabla_{\bfv_i}\calL\paren{\bm{\theta}(t)}^\top\bbfw_j & = -\frac{1}{a_i} \paren{\hat{\lambda}_{i_\ell^\star,j_\ell^\star,1}(t)\zeta_{j,j_\ell^\star}^{(2)}(t) + 3\hat{\lambda}_{i_\ell^\star,j_\ell^\star,2}(t) + 3\hat{\lambda}_{i_\ell^\star,j_\ell^\star,4}(t)\gamma_{i_\ell^\star,j_\ell^\star}^{(2)}(t)}\indy{i=i_\ell^\star}\\
        & \qqquad  -\frac{1}{a_i}\hat{\lambda}_{i_\ell^\star,j_\ell^\star,1}(t)\gamma_{i_\ell^\star,j_\ell^\star}^{(1)}(t)\indy{i=i_\ell^\star} - \frac{36}{a_i}C_{S,2}I_{i,i}^{(3)}(t) \pm \calO\paren{m\aggepst{\ell}(t)^3}\\
        \nabla_{\bfw_i}\calL\paren{\bm{\theta}(t)}^\top\bbfv_j & = -\frac{3}{b_i}\cdot \paren{3\hat{\lambda}_{i_\ell^\star,j_\ell^\star,5}(t)\zeta_{j,j_\ell^\star}^{(1)}(t) + \hat{\lambda}_{i_\ell^\star,j_\ell^\star,2}(t) + \hat{\lambda}_{i_\ell^\star,j_\ell^\star,3}(t)\gamma_{i_\ell^\star,j_\ell^\star}^{(1)}(t)}\indy{i=i_\ell^\star}\\
        & \qqquad - \frac{9}{b_i}\cdot \hat{\lambda}_{i_\ell^\star,j_\ell^\star,5}(t)\gamma_{i_\ell^\star,j_\ell^\star}^{(2)}(t)\indy{i=i_\ell^\star} - 36C_{S,2}I_{i,i}^{(3)}(t) \pm \calO\paren{m\aggepst{\ell}(t)^2}
    \end{align*}
\end{lemma}
\begin{proof}
    In the following of the proof we will assume that $i\in[m]\setminus \mathcal{R}_{\ell-1}$ and $j\in[m^\star]$. 
    For $\hat{\lambda}_{i,j,1},\dots\hat{\lambda}_{i,j,5}$, we apply Lemma~\ref{lem:approx_lambda_hat} with
    \[
        \delta_r = \backepso{\ell}(t);\; \text{ for } i \in [m]\setminus \mathcal{R}_{\ell-1}, j\in[m^\star];
    \]
    For $\lambda_{i,j,1},\dots,\lambda_{i,j,5}$, we apply
    \[
        \delta_r = \begin{cases}
            \backepso{\ell}(t);\; \text{ if } i,j\in [m]\setminus \mathcal{R}_{\ell-1}, \; i\neq j\\
            \forweps{\ell}(t);\; \text{ if } i\in \mathcal{R}_{\ell-1} \vee j\in \mathcal{R}_{\ell-1}, \; i\neq j
        \end{cases} \leq \aggepso{\ell}(t);\quad \delta_p = \begin{cases}
            \backepsi{\ell}(t) & \text{ if } i \in [m] \setminus \mathcal{R}_{\ell}\\
            \varepsilon_{5,\ell}(t) & \text{ if } i = i_\ell^\star
        \end{cases}
    \]
    For $\hat{\lambda}_{i_{\ell'}^\star,j,1},\dots, \hat{\lambda}_{i_{\ell'}^\star,j,5}$ for $\ell' < \ell$, we have that
    \[
        \delta_r = \forweps{\ell}(t)\leq \aggepso{\ell}(t);\quad 
    \]
    For $\lambda_{i_{\ell'}^\star,i_{\ell'}^\star,1},\dots,\lambda_{i_{\ell'}^\star,i_{\ell'}^\star,5}$, we have that
    \[
        \delta_r = \delta_p = \forweps{\ell}(t) \leq \aggepso{\ell}(t)
    \]
    Moreover, we also have that
    \begin{gather*}
        \left|\gamma_{i,j}^{(1)}(t)\right|, \left|\gamma_{i,j}^{(2)}(t)\right| \leq \begin{cases}
            \aggepst{\ell}(t) & \text{ if } i\neq i_\ell^\star\\
            \aggepso{\ell}(t) & \text{ if } i = i_\ell^\star
        \end{cases}\; \forall j\in[m^\star], (i,j)\neq (i_{\ell'}^\star,j_{\ell'}^\star)\; \forall \ell'\leq \ell\\
        \left|\zeta_{i,j}^{(1)}(t)\right|, \left|\zeta_{i,j}^{(2)}(t)\right|\leq \aggepso{\ell}(t);\forall i\in[m],j\in[m^\star]\\
        \left|I_{i,j}^{(1)}(t)\right|, \left|I_{i,j}^{(2)}(t)\right|, \left|I_{i,j}^{(3)}(t)\right|\leq \aggepso{\ell}(t);\forall i,j\in[m], i\neq j\\
        \left|I_{i,i}^{(3)}(t)\right| \leq \begin{cases}
            \forweps{\ell}(t) & \text{ if } i \in \mathcal{R}_{\ell-1}\\
            \varepsilon_{5,\ell}(t) & \text{ if } i = i_\ell^\star\\
            \backepsi{\ell}(t) & \text{ if } i \in[m]\setminus \mathcal{R}_{\ell}
        \end{cases}
    \end{gather*}
    This gives that for all $j\in[m^\star]$ and such that $(i,j)\neq (i_\ell^\star,j_\ell^\star)$
    \[
        \hat{\lambda}_{i,j,1}(t) \leq \begin{cases}
            \calO\paren{\aggepst{\ell}(t)^3} & \text{ if } i \neq i_\ell^\star\\
            \calO\paren{\aggepso{\ell}(t)^3} & \text{ if } i = i_\ell^\star
        \end{cases};\; \hat{\lambda}_{i,j,5}(t) \leq \begin{cases}
            \calO\paren{\aggepst{\ell}(t)^2} & \text{ if } i \neq i_\ell^\star\\
            \calO\paren{\aggepso{\ell}(t)^2} & \text{ if } i = i_\ell^\star
        \end{cases}
    \]
    We will analyze each dynamic separately. However, we should notice some common terms that appears in the dynamics.
    
    \textbf{Common terms.} To start, let's tackle some common terms in the dynamics we are interest in. In particular, we have that for $i\in [m]\setminus \mathcal{R}_{\ell-1}$,
    \begin{align*}
        \sum_{r=1}^m\lambda_{i,r,1}(t)I_{i,r}^{(1)}(t) & = \lambda_{i,i,1}(t) \pm \calO\paren{m\aggepso{\ell}(t)^3}\\
        \sum_{r=1}^m\lambda_{i,r,2}(t)I_{i,i}^{(3)}(t) & = \pm \begin{cases}
            \calO\paren{\backepsi{\ell}(t)^2 + m\aggepso{\ell}(t)^3} & \text{ if } i \in [m]\setminus \mathcal{R}_{\ell}\\
            \calO\paren{\varepsilon_{5,\ell}(t)^2 + m\aggepso{\ell}(t)^3} & \text{ if }i = i_\ell^\star
        \end{cases}\\
        \sum_{r=1}^m\lambda_{i,r,4}(t)I_{i,r}^{(3)}(t) & = \pm\begin{cases}
            \calO\paren{\backepsi{\ell}(t)^2 + m\aggepso{\ell}(t)^3} & \text{ if } i \in [m]\setminus \mathcal{R}_{\ell}\\
            \calO\paren{\varepsilon_{5,\ell}(t)^2 + m\aggepso{\ell}(t)^3} & \text{ if }i = i_\ell^\star
        \end{cases}\\
        \sum_{r=1}^{m^\star}\hat{\lambda}_{i,r,1}(t)\gamma_{i,r}^{(t)}(t) & = \begin{cases}
            \lambda_{i_\ell^\star,j_\ell^\star,1}(t)\gamma_{i_\ell^\star,j_\ell^\star}^{(1)}(t) \pm \calO\paren{m\aggepso{\ell}(t)^3} & \text{ if } i = i_\ell^\star\\
            \pm \calO\paren{m\aggepst{\ell}(t)^3} & \text{ if } i \in [m]\setminus \mathcal{R}_{\ell}
        \end{cases} \\
        \sum_{r=1}^{m^\star}\hat{\lambda}_{i,r,2}(t)I_{i,i}^{(3)}(t) & = \begin{cases}
            \pm \calO\paren{\aggepso{\ell}(t)^2}\gamma_{i_\ell^\star,j_\ell^\star}^{(2)}(t)^2 \pm \calO\paren{m\aggepso{\ell}(t)^3} & \text{ if } i = i_\ell^\star\\
            \pm \calO\paren{m\aggepso{\ell}(t)^2\aggepst{\ell}(t)}& \text{ if } i \in [m]\setminus \mathcal{R}_{\ell}
        \end{cases}\\
        \sum_{r=1}^{m^\star}\hat{\lambda}_{i,r,4}(t)\zeta_{i,r,}^{(1)}(t) & = \begin{cases}
            \pm \calO\paren{\aggepso{\ell}(t)^2}\gamma_{i_\ell^\star,j_\ell^\star}^{(2)}(t)^2 \pm \calO\paren{m\aggepso{\ell}(t)^3} & \text{ if } i = i_\ell^\star\\
            \pm \calO\paren{m\aggepso{\ell}(t)^2\aggepst{\ell}(t)}& \text{ if } i \in [m]\setminus \mathcal{R}_{\ell}
        \end{cases}
    \end{align*}
    Therefore, we have that
    \begin{equation}
        \label{eq:dynamic_approx_1}
        \begin{aligned}
            & \sum_{r=1}^m\lambda_{i,r,1}(t)I_{i,r}^{(1)}(t) + 3\sum_{r=1}^m\lambda_{i,r,2}(t)I_{i,i}^{(3)}(t) + 3\sum_{r=1}^m\lambda_{i,r,4}(t)I_{i,r}^{(3)}(t)\\
            & \qqquad = \lambda_{i,i,1} \pm \begin{cases}
                \calO\paren{\varepsilon_{5,\ell}(t)^2+ m\aggepso{\ell}(t)^3} & \text{ if } i = i_\ell^\star\\
                \calO\paren{\backepsi{\ell}(t)^2 + m\aggepst{\ell}(t)^3} & \text{ if } i \in [m]\setminus \mathcal{R}_{\ell}
            \end{cases}\\
            & \sum_{r=1}^{m^\star}\hat{\lambda}_{i,r,1}(t)\gamma_{i,r}^{(t)}(t) + 3\sum_{r=1}^{m^\star}\hat{\lambda}_{i,r,2}(t)I_{i,i}^{(3)}(t) + \sum_{r=1}^{m^\star}\hat{\lambda}_{i,r,4}(t)\zeta_{i,r,}^{(1)}(t)\\
            & \qqquad = \begin{cases}
                \hat{\lambda}_{i_\ell^\star,j_\ell^\star,1}(t)\gamma_{i_\ell^\star,j_\ell^\star}^{(1)}(t) \pm \calO\paren{\aggepso{\ell}(t)^2}\gamma_{i_\ell^\star,j_\ell^\star}^{(2)}(t)^2 \pm \calO\paren{m\aggepso{\ell}(t)^3} & \text{ if } i = i_\ell^\star\\
                \pm \calO\paren{m\aggepst{\ell}(t)^3} & \text{ if } i \in [m]\setminus \mathcal{R}_{\ell}
            \end{cases}
        \end{aligned}
    \end{equation}
    Moreover, we can also compute that 
    \begin{align*}
        \sum_{r=1}^{m}\lambda_{i,r,2}(t)I_{i,i}^{(3)}(t) & = \pm\begin{cases}
            \calO\paren{\backepsi{\ell}(t)^2 + m\aggepso{\ell}(t)^3} & \text{ if } i \in [m]\setminus \mathcal{R}_{\ell}\\
            \calO\paren{\varepsilon_{5,\ell}(t)^2 + m\aggepso{\ell}(t)^3} & \text{ if }i = i_\ell^\star
        \end{cases}\\
        \sum_{r=1}^m\lambda_{i,r,3}(t)I_{r,i}^{(3)}(t) & = \begin{cases}
            \calO\paren{\backepsi{\ell}(t)^2 + m\aggepso{\ell}(t)^3} & \text{ if } i \in [m]\setminus \mathcal{R}_{\ell}\\
            \calO\paren{\varepsilon_{5,\ell}(t)^2 + m\aggepso{\ell}(t)^3} & \text{ if }i = i_\ell^\star
        \end{cases}\\
        \sum_{r=1}^m\lambda_{i,r,5}(t)I_{r,i}^{(2)}(t) & = \lambda_{i,i,5}(t) \pm \calO\paren{m\aggepso{\ell}(t)^3}\\
        \sum_{r=1}^{m^\star}\hat{\lambda}_{i,r,2}(t)I_{i,i}^{(3)}(t) & = \begin{cases}
            \pm \calO\paren{\aggepso{\ell}(t)^2}\gamma_{i_\ell^\star,j_\ell^\star}^{(2)}(t)^2 \pm \calO\paren{m\aggepso{\ell}(t)^3} & \text{ if }i = i_\ell^\star\\
            \pm \calO\paren{m\aggepst{\ell}(t)^3}& \text{ if } i \in [m]\setminus \mathcal{R}_{\ell}
        \end{cases}\\
        \sum_{r=1}^{m^\star}\hat{\lambda}_{i,r,3}(t)\zeta_{i,r}^{(2)}(t) & = \begin{cases}
            \pm \calO\paren{\aggepso{\ell}(t)^2}\gamma_{i_\ell^\star,j_\ell^\star}^{(2)}(t)^2 \pm \calO\paren{m\aggepso{\ell}(t)^3} & \text{ if }i = i_\ell^\star\\
            \pm \calO\paren{m\aggepst{\ell}(t)^3}& \text{ if } i \in [m]\setminus \mathcal{R}_{\ell}
        \end{cases}\\
        \sum_{r=1}^{m^\star}\hat{\lambda}_{i,r,5}(t)\gamma_{i,r}^{(2)}(t) & = \begin{cases}
            \lambda_{i_\ell^\star,j_\ell^\star,5}(t)\gamma_{i_\ell^\star,j_\ell^\star}^{(2)}(t) \pm \calO\paren{m\aggepso{\ell}(t)^3} & \text{ if }i = i_\ell^\star\\
            \pm \calO\paren{m\aggepst{\ell}(t)^3}& \text{ if } i \in [m]\setminus \mathcal{R}_{\ell}
        \end{cases} 
    \end{align*}
    This gives that
    \begin{equation}
        \label{eq:dynamic_approx_2}
        \begin{aligned}
            & \sum_{r=1}^{m}\lambda_{i,r,2}(t)I_{i,i}^{(3)}(t) + \sum_{r=1}^m\lambda_{i,r,3}(t)I_{r,i}^{(3)}(t) + 3\sum_{r=1}^m\lambda_{i,r,5}(t)I_{r,i}^{(2)}(t)\\
            & \qqquad = \lambda_{i,i,5}(t)\pm \begin{cases}
                 \calO\paren{m\aggepst{\ell}(t)^3 + \backepsi{\ell}(t)^2} & \text{ if } i \in [m]\setminus \mathcal{R}_{\ell}\\
                 \calO\paren{m\aggepso{\ell}(t)^3 + \varepsilon_{5,\ell}(t)^2} & \text{ if }i = i_\ell^\star
            \end{cases}\\
            & \sum_{r=1}^{m^\star}\hat{\lambda}_{i,r,2}(t)I_{i,i}^{(3)}(t) + \sum_{r=1}^{m^\star}\hat{\lambda}_{i,r,3}(t)\zeta_{i,r}^{(2)}(t) + \sum_{r=1}^{m^\star}\hat{\lambda}_{i,r,5}(t)\gamma_{i,r}^{(2)}(t)\\
            & \qqquad = \begin{cases}
                \hat{\lambda}_{i_\ell^\star,j_\ell^\star,5}(t)\gamma_{i_\ell^\star,j_\ell^\star}^{(2)}(t) \pm \calO\paren{\aggepst{\ell}(t)^2}\gamma_{i_\ell^\star,j_\ell^\star}^{(2)}(t)^2 \pm \calO\paren{m\aggepso{\ell}(t)^3}\text{ if }i = i_\ell^\star\\
                 \calO\paren{m\aggepst{\ell}(t)^3}\text{ if } i \in [m]\setminus \mathcal{R}_{\ell}
            \end{cases}
        \end{aligned}
    \end{equation}
    Now we are ready to analyze the dynamics. 
    
    \textbf{Analysis of $\nabla_{\bfv_i}\calL\paren{\bm{\theta}(t)}\bbfv_j^\star$.} To analyze $\nabla_{\bfv_i}\calL\paren{\bm{\theta}(t)}\bbfv_j^\star$, we first compute the following quantities for $i\in [m]\setminus \mathcal{R}_{\ell-1}$
    \begin{align*}
        \hat{\lambda}_{i,j,1}(t) & = \begin{cases}
            \pm \calO\paren{\aggepst{\ell}(t)^3} \text{ if } i\in [m]\setminus \mathcal{R}_{\ell}\\
            \pm \calO\paren{\aggepso{\ell}(t)^3} & \text{ if } i = i_\ell^\star, j\neq j_\ell^\star
        \end{cases};\; \\
        \sum_{r=1}^m\lambda_{i,r,1}(t)\gamma_{r,j}^{(1)}(t) & = \lambda_{i,i,1}(t)\gamma_{i,j}^{(1)}(t)\pm \calO\paren{m\aggepso{\ell}(t)^3}\\
        \sum_{r=1}^m\lambda_{i,r,2}(t)\zeta_{i,j}^{(2)}(t) & = \pm \calO\paren{m\aggepst{\ell}(t)^3} \pm \begin{cases}
            \calO\paren{\aggepso{\ell}(t)\backepsi{\ell}(t)} & \text{ if } i \in[m]\setminus \mathcal{R}_{\ell}\\
            \calO\paren{\aggepso{\ell}(t)\varepsilon_{5,\ell}(t)} & \text{ if } i = i_\ell^\star
        \end{cases}\\
        \sum_{r=1}^m\lambda_{i,r,4}(t)\zeta_{r,j}^{(2)}(t) & = \pm \calO\paren{m\aggepst{\ell}(t)^3} \pm \begin{cases}
            \calO\paren{\aggepso{\ell}(t)\backepsi{\ell}(t)} & \text{ if } i \in[m]\setminus \mathcal{R}_{\ell}\\
            \calO\paren{\aggepso{\ell}(t)\varepsilon_{5,\ell}(t)} & \text{ if } i = i_\ell^\star
        \end{cases}\\
        \sum_{r=1}^{m^\star}\hat{\lambda}_{i,r,2}(t)\zeta_{i,j}^{(2)}(t) & = \begin{cases}
             \calO\paren{\aggepso{\ell}(t)^2}\gamma_{i_\ell^\star,j_\ell^\star}^{(2)}(t)^2 \pm \calO\paren{m\aggepso{\ell}(t)^3} & \text{ if } i = i_\ell^\star\\
              \pm \calO\paren{m\aggepst{\ell}(t)^3}& \text{ if } i \in[m]\setminus \mathcal{R}_{\ell}
        \end{cases}
    \end{align*}
    Therefore, combining with (\ref{eq:dynamic_approx_1}), and noticing that the term $\lambda_{i,i,1}(t)\gamma_{i,j}^{(1)}(t)$ cancels out, we have that
    \begin{align*}
        \nabla_{\bfv_i}\calL\paren{\bm{\theta}(t)}\bbfv_j^\star  =  -\frac{1}{a_i} \cdot \begin{cases}
            \pm \calO\paren{m\aggepst{\ell}(t)^3 + \aggepst{\ell}(t)\backepsi{\ell}(t)} & \text{ if } i  \in[m]\setminus \mathcal{R}_{\ell}\\
            \begin{aligned}
            & - \hat{\lambda}_{i_\ell^\star,j_\ell^\star,1}(t)\gamma_{i_\ell^\star,j_\ell^\star}^{(1)}(t)\gamma_{i,j}^{(1)}(t) \pm \calO\paren{\aggepso{\ell}(t)^2}\gamma_{i_\ell^\star,j_\ell^\star}^{(2)}(t)^2\\
            & \qqquad \pm \calO\paren{m\aggepso{\ell}(t)^3 + \aggepso{\ell}(t)\varepsilon_{5,\ell}(t)}
            \end{aligned} & \text{ if } i = i_\ell^\star, j\in[m^\star]\setminus \mathcal{C}_{\ell}\\
            \begin{aligned}
            & \paren{1 - \gamma_{i_\ell^\star,j_\ell^\star}^{(1)}(t)^2)}\hat{\lambda}_{i_\ell^\star,j_\ell^\star,1}(t) \pm \calO\paren{\aggepso{\ell}(t)^2}\gamma_{i_\ell^\star,j_\ell^\star}^{(2)}(t)^2\\
            & \qqquad \pm \calO\paren{m\aggepso{\ell}(t)^3 + \aggepso{\ell}(t)\varepsilon_{5,\ell}(t)}
            \end{aligned}& \text{ if }i = i_\ell^\star, j = j_\ell^\star
        \end{cases}
    \end{align*}
    \textbf{Analysis of $\nabla_{\bfw_i}\calL\paren{\bm{\theta}(t)}\bbfw_j^\star$.} To analyze $\nabla_{\bfw_i}\calL\paren{\bm{\theta}(t)}\bbfw_j^\star$, we first compute the following quantities
    \begin{align*}
        \sum_{r=1}^m\lambda_{i,r,5}(t)\gamma_{r,j}^{(2)}(t) & = \lambda_{i,i,5}(t)\gamma_{i,j}^{(5)}(t) + \lambda_{i,i_{\ell'}^\star,5}(t)\gamma_{i_{\ell'}^\star,j_{\ell'}^\star}^{(2)}(t)\indy{j=j_{\ell'}^\star,\ell' < \ell}\\
        & \qqquad \pm \begin{cases}
            \calO\paren{m\aggepst{\ell}(t)^3}\pm \calO\paren{\aggepst{\ell}(t)^2}\gamma_{i_\ell^\star,j_\ell^\star}^{(2)}(t)\indy{j=j_\ell^\star} & \text{ if } i \in [m]\setminus \mathcal{R}_{\ell}\\
            \calO\paren{m\aggepso{\ell}(t)^2\aggepst{\ell}(t)} & \text{ if } i = i_\ell^\star
        \end{cases}\\
        \sum_{r=1}^m\lambda_{i,r,2}(t)\zeta_{i,j}^{(1)}(t) & = \pm \calO\paren{m\aggepso{\ell}(t)^3} \pm \begin{cases}
            \calO\paren{\aggepst{\ell}(t)\backepsi{\ell}(t)} & \text{ if } i \in[m]\setminus \mathcal{R}_{\ell}\\
            \calO\paren{\aggepst{\ell}(t)\varepsilon_{5,\ell}(t)} & \text{ if } i = i_\ell^\star
        \end{cases}\\
        \sum_{r=1}^m\lambda_{i,r,3}(t)\zeta_{r,j}^{(1)}(t) & = \pm \calO\paren{m\aggepso{\ell}(t)^3} \pm \begin{cases}
            \calO\paren{\aggepst{\ell}(t)\backepsi{\ell}(t)} & \text{ if } i \in[m]\setminus \mathcal{R}_{\ell}\\
            \calO\paren{\aggepst{\ell}(t)\varepsilon_{5,\ell}(t)} & \text{ if } i = i_\ell^\star
        \end{cases}\\
        \sum_{r=1}^{m^\star}\hat{\lambda}_{i,r,2}(t)\zeta_{i,j}^{(1)}(t) & = \begin{cases}
            \pm \calO\paren{\aggepst{\ell}(t)^3} & \text{ if } i = i_\ell^\star\\
            \pm \calO\paren{\aggepso{\ell}(t)^2}\gamma_{i_\ell^\star,j_\ell^\star}^{(2)}(t)^2 \pm \calO\paren{m\aggepst{\ell}(t)^3} & \text{ if } i \in[m]\setminus \mathcal{R}_{\ell}
        \end{cases}
    \end{align*}
    Therefore, combining with (\ref{eq:dynamic_approx_2}), and noticing that the term $\lambda_{i,i,5}(t)\gamma_{i,j}^{(2)}(t)$ cancels out, we have that
    \begin{align*}
        \nabla_{\bfw_i}\calL\paren{\bm{\theta}(t)}\bbfw_{j_{\ell'}^\star}^\star & = \frac{9}{b_i}\lambda_{i,i_{\ell'}^\star,5}(t)\gamma_{i_{\ell'}^\star,j_{\ell'}^\star}^{(2)}(t)\\
        & \qqquad - \frac{9}{b_i}\cdot\begin{cases}
            \hat{\lambda}_{i,j,5}(t) \pm \calO\paren{m\aggepst{\ell}(t)^3 + \aggepst{\ell}(t)\backepsi{\ell}(t)}& \text{ if } i\in[m]\setminus \mathcal{R}_{\ell}\\
            \begin{aligned}
            & \hat{\lambda}_{i,j,5}(t)- \lambda_{i_\ell^\star,j_\ell^\star,5}(t)\gamma_{i_\ell^\star,j_\ell^\star}^{(2)}(t)\gamma_{i,j}^{(2)}(t)\\
            & \qqquad \pm \calO\paren{\aggepso{\ell}(t)^2}\gamma_{i_\ell^\star,j_\ell^\star}^{(2)}(t)^2\\
            & \qqquad \pm \calO\paren{m\aggepso{\ell}(t)^2\aggepst{\ell}(t) + \aggepst{\ell}(t)\varepsilon_{5,\ell}(t)}
            \end{aligned} & \text{ if } i =i_\ell^\star
        \end{cases}
    \end{align*}
    Moreover, we have that
    \begin{align*}
        \nabla_{\bfw_i}\calL\paren{\bm{\theta}(t)}\bbfw_{j}^\star = - \frac{9}{b_i}\cdot \begin{cases}
            \begin{aligned}
                & \hat{\lambda}_{i,j,5}(t) \pm \calO\paren{m\aggepst{\ell}(t)^3 + \aggepst{\ell}(t)\backepsi{\ell}(t)}\\
                & \qqquad \pm \calO\paren{\aggepst{\ell}(t)^2}\gamma_{i_\ell^\star,j_\ell^\star}^{(2)}(t)
            \end{aligned} & \text{ if } i \in[m]\setminus \mathcal{R}_{\ell}\\
            \begin{aligned}
            & \hat{\lambda}_{i,j,5}(t)- \lambda_{i_\ell^\star,j_\ell^\star,5}(t)\gamma_{i_\ell^\star,j_\ell^\star}^{(2)}(t)\gamma_{i,j}^{(2)}(t)\\
            & \qqquad \pm \calO\paren{\aggepso{\ell}(t)^2}\gamma_{i_\ell^\star,j_\ell^\star}^{(2)}(t)\\
            & \qqquad \pm \calO\paren{m\aggepso{\ell}(t)^2\aggepst{\ell}(t) + \aggepso{\ell}(t)\varepsilon_{5,\ell}(t)}
            \end{aligned} & \text{ if } i = i_\ell^\star, j \in[m^\star]\setminus \mathcal{C}_{\ell}\\
            \begin{aligned}
                & \paren{1 - \gamma_{i_\ell^\star,j_\ell^\star}^{(2)}(t)^2}\hat{\lambda}_{i_\ell^\star,j_\ell^\star,5}(t)\pm \calO\paren{\aggepso{\ell}(t)^2}\gamma_{i_\ell^\star,j_\ell^\star}^{(2)}(t)\\
                & \qqquad \pm \calO\paren{m\aggepso{\ell}(t)^2\aggepst{\ell}(t) + \aggepso{\ell}(t)\varepsilon_{5,\ell}(t)}
            \end{aligned}& \text{ if } i = i_\ell^\star, j = j_\ell^\star
        \end{cases}
    \end{align*}
    \textbf{Analysis of $\nabla_{\bfv_i}\calL\paren{\bm{\theta}(t)}\bbfw_j^\star$.} To analyze $\nabla_{\bfv_i}\calL\paren{\bm{\theta}(t)}\bbfw_j^\star$, we first compute the following quantities for $i\in[m]\setminus \mathcal{R}_{\ell-1}$:
    \begin{align*}
        \hat{\lambda}_{i,j,4}(t) & = \begin{cases}
            \pm \calO\paren{\aggepst{\ell}(t)^3} & \text{ if } i\in[m]\setminus \mathcal{R}_{\ell}\\
            \pm \calO\paren{\aggepso{\ell}(t)^3}& \text{ if } i = i_\ell^\star, j\neq j_\ell^\star
        \end{cases}\\
        \sum_{r=1}^m\lambda_{i,r,4}(t)\gamma_{r,j}^{(2)}(t) & = 6C_{S,2}\gamma_{i,j}^{(2)}(t)I_{i,i}^{(3)}(t) \pm \calO\paren{m\aggepso{\ell}(t)^2\aggepst{\ell}(t)}\\
        \sum_{r=1}^m\lambda_{i,r,1}(t)\zeta_{r,j}^{(1)}(t) & = \lambda_{i,i,1}(t)\zeta_{i,j}^{(1)}(t) \pm \calO\paren{m\aggepso{\ell}(t)^3}\\
        \sum_{r=1}^m\lambda_{i,r,2}(t)\gamma_{i,j}^{(2)}(t) & = 6C_{S,2}\gamma_{i,j}^{(2)}(t)I_{i,i}^{(3)}(t) \pm \calO\paren{m\aggepso{\ell}(t)^2\aggepst{\ell}(t)}\\
        \sum_{r=1}^{m^\star}\hat{\lambda}_{i,r,2}(t)\gamma_{i,j}^{(2)}(t) & = \begin{cases}
            \pm \calO\paren{m\aggepst{\ell}(t)^3} & \text{ if } i\in[m]\setminus \mathcal{R}_{\ell}\\
            \hat{\lambda}_{i_\ell^\star,j_\ell^\star,2}(t)\gamma_{i,j}^{(2)}(t)\pm \calO\paren{m\aggepso{\ell}(t)^3}& \text{ if } i = i_\ell^\star
        \end{cases}
    \end{align*}
    Combining with (\ref{eq:dynamic_approx_1}) and noticing that the term $\lambda_{i,i,1}(t)\zeta_{i,j}^{(1)}(t)$ cancels out, we have that
    \begin{align*}
        \nabla_{\bfv_i}\calL\paren{\bm{\theta}(t)}\bbfw_j^\star = - \frac{1}{a_i}\cdot \begin{cases}
            \pm \calO\paren{m\aggepst{\ell}(t)^3 +\backepsi{\ell}(t)^2} & \text{ if } i\in[m]\setminus \mathcal{R}_{\ell}\\
            \begin{aligned}
                & \hat{\lambda}_{i_\ell^\star,j_\ell^\star,2}(t)\gamma_{i,j}^{(2)}(t) - \hat{\lambda}_{i_\ell^\star,j_\ell^\star,1}(t)\gamma_{i_\ell^\star,j_\ell^\star}^{(1)}(t)\zeta_{i,j}^{(1)}(t)\\
                & \qqquad - 36C_{S,2}\gamma_{i,j}^{(2)}(t)I_{i,i}^{(3)}(t)\\
                & \qqquad \pm \calO\paren{m\aggepso{\ell}(t)^2\aggepst{\ell}(t) +\backepsi{\ell}(t)^2} 
            \end{aligned} &\text{ if } i= i_\ell^\star, j\neq j_\ell^\star\\
            \begin{aligned}
                & 3\hat{\lambda}_{i_\ell^\star,j_\ell^\star,4}(t) + \hat{\lambda}_{i_\ell^\star,j_\ell^\star,2}(t)\gamma_{i_\ell^\star,j_\ell^\star}^{(2)}(t) - \hat{\lambda}_{i_\ell^\star,j_\ell^\star,1}(t)\gamma_{i_\ell^\star,j_\ell^\star}^{(1)}(t)\zeta_{i_\ell^\star,j_\ell^\star}^{(1)}(t)\\
                & \qqquad - 36C_{S,2}\gamma_{i,j}^{(2)}(t)I_{i,i}^{(3)}(t)\\
                & \qqquad \pm \calO\paren{m\aggepso{\ell}(t)^2\aggepst{\ell}(t) +\backepsi{\ell}(t)^2} 
            \end{aligned} &\text{ if } i= i_\ell^\star, j= j_\ell^\star
        \end{cases}
    \end{align*}
    \textbf{Analysis of $\nabla_{\bfw_i}\calL\paren{\bm{\theta}(t)}\bbfv_j^\star$.} To analyze $\nabla_{\bfw_i}\calL\paren{\bm{\theta}(t)}\bbfv_j^\star$, we first compute the following for $i\in[m]\setminus \mathcal{R}_{\ell-1}$:
    \begin{align*}
        \hat{\lambda}_{i,j,3}(t) & = \begin{cases}
            \pm \calO\paren{\aggepst{\ell}(t)^3} & \text{ if } i\in[m]\setminus \mathcal{R}_{\ell}\\
            \pm \calO\paren{\aggepso{\ell}(t)^3}& \text{ if } i = i_\ell^\star, j\neq j_\ell^\star
        \end{cases}\\
        \sum_{r=1}^m\lambda_{i,r,3}(t)\gamma_{r,j}^{(1)}(t) & = 6C_{S,2}\gamma_{i,j}^{(1)}(t)I_{i,i}^{(3)}(t)\pm \calO\paren{m\aggepso{\ell}(t)^2\aggepst{\ell}(t)}\\
        \sum_{r=1}^m\lambda_{i,r,2}(t)\gamma_{i,j}^{(1)}(t) & =  6C_{S,2}\gamma_{i,j}^{(1)}(t)I_{i,i}^{(3)}(t)\pm \calO\paren{m\aggepso{\ell}(t)^2\aggepst{\ell}(t)}\\
        \sum_{r=1}^m\lambda_{i,r,5}(t)\zeta_{r,j}^{(2)}(t) &  = \lambda_{i,i,5}(t)\zeta_{i,j}^{(2)}(t) \pm \calO\paren{m\aggepst{\ell}(t)^3}\\
        \sum_{r=1}^{m^\star}\hat{\lambda}_{i,r,2}(t)\gamma_{i,j}^{(1)}(t) & =\begin{cases}
            \pm \calO\paren{m\aggepst{\ell}(t)^3} & \text{ if } i\in[m]\setminus \mathcal{R}_{\ell}\\
            \hat{\lambda}_{i_\ell^\star,j_\ell^\star,2}(t)\gamma_{i,j}^{(1)}(t)\pm \calO\paren{m\aggepso{\ell}(t)^3}& \text{ if } i = i_\ell^\star
        \end{cases}
    \end{align*}
    Combining with (\ref{eq:dynamic_approx_2}) and noticing that the term $\lambda_{i,i,5}\zeta_{i,j}^{(2)}(t)$ cancels out, we have that
    \begin{align*}
        \nabla_{\bfw_i}\calL\paren{\bm{\theta}(t)}\bbfv_j^\star = \frac{3}{b_i}\cdot \begin{cases}
            \pm \calO\paren{m\aggepst{\ell}(t)^3 +\backepsi{\ell}(t)^2} & \text{ if } i\in[m]\setminus \mathcal{R}_{\ell}\\
            \begin{aligned}
                & \hat{\lambda}_{i_\ell^\star,j_\ell^\star,2}(t)\gamma_{i,j}^{(1)}(t) - 3\hat{\lambda}_{i_\ell^\star,j_\ell^\star,5}(t)\gamma_{i_\ell^\star,j_\ell^\star}^{(2)}(t)\zeta_{i,j}^{(2)}(t)\\
                & \qqquad- 12C_{S,2}\gamma_{i,j}^{(2)}(t)I_{i,i}^{(3)}(t)\\
                & \qqquad \pm \calO\paren{m\aggepso{\ell}(t)^2\aggepst{\ell}(t) +\backepsi{\ell}(t)^2}
            \end{aligned} &\text{ if } i = i_\ell^\star, j\neq j_\ell^\star\\
            \begin{aligned}
                & \hat{\lambda}_{i_\ell^\star,j_\ell^\star,3}(t) + \hat{\lambda}_{i_\ell^\star,j_\ell^\star,2}(t)\gamma_{i_\ell^\star,j_\ell^\star}^{(1)}(t) - 3\hat{\lambda}_{i_\ell^\star,j_\ell^\star,5}(t)\gamma_{i_\ell^\star,j_\ell^\star}^{(2)}(t)\zeta_{i_\ell^\star,j_\ell^\star}^{(2)}(t)\\
                & \qqquad - 12C_{S,2}\gamma_{i,j}^{(2)}(t)I_{i,i}^{(3)}(t)\\
                & \qqquad \pm \calO\paren{m\aggepso{\ell}(t)^2\aggepst{\ell}(t) +\backepsi{\ell}(t)^2}
            \end{aligned} &\text{ if } i = i_\ell^\star, j = j_\ell^\star
        \end{cases}
    \end{align*}

    \textbf{Analysis of $\nabla_{\bfv_i}\calL\paren{\bm{\theta}(t)}^\top\bbfv_j$.} To analyze $\nabla_{\bfv_i}\calL\paren{\bm{\theta}(t)}^\top\bbfv_j$, we first compute that for $i,j\in[m]\setminus \mathcal{R}_{\ell-1}, i\neq j$
    \begin{align*}
        \sum_{r=1}^{m^\star}\hat{\lambda}_{i,r,1}(t)\gamma_{j,r}^{(1)}(t) &  = \hat{\lambda}_{i_\ell^\star,j_\ell^\star,1}(t)\gamma_{j,j_\ell^\star}^{(1)}(t)\indy{i=i_\ell^\star}\pm \calO\paren{m\aggepst{\ell}(t)^3}\\
        \sum_{r=1}^{m^\star}\hat{\lambda}_{i,r,2}(t)I_{j,i}^{(3)}(t) & = \pm\calO\paren{\aggepst{\ell}(t)^2}\gamma_{i_\ell^\star,j_\ell^\star}^{(2)}(t)^2\indy{i=i_\ell^\star} \pm \calO\paren{m\aggepso{\ell}(t)^3}\\
        \sum_{r=1}^{m^\star}\hat{\lambda}_{i,r,4}(t)\zeta_{j,r}^{(1)}(t) & =  \pm\calO\paren{\aggepst{\ell}(t)^2}\gamma_{i_\ell^\star,j_\ell^\star}^{(2)}(t)^2\indy{i=i_\ell^\star} \pm \calO\paren{m\aggepso{\ell}(t)^3}\\
        \sum_{r=1}^m\lambda_{i,r,2}(t)I_{j,i}^{(3)}(t) & = \pm \calO\paren{m\aggepst{\ell}(t)^3} \pm \begin{cases}
            \calO\paren{\aggepst{\ell}(t)\backepsi{\ell}(t)} & \text{ if } i\neq i_\ell^\star\\
            \calO\paren{\aggepst{\ell}(t)\varepsilon_{5,\ell}(t)} & \text{ if } i = i_\ell^\star
        \end{cases}\\
        \sum_{r=1}^m\lambda_{i,r,4}(t)I_{j,r}^{(3)}(t) & = \pm \calO\paren{m\aggepst{\ell}(t)^3} \pm \begin{cases}
            \calO\paren{\aggepst{\ell}(t)\backepsi{\ell}(t)} & \text{ if } i\neq i_\ell^\star\\
            \calO\paren{\aggepst{\ell}(t)\varepsilon_{5,\ell}(t)} & \text{ if } i = i_\ell^\star
        \end{cases}
    \end{align*}
    Combining with (\ref{eq:dynamic_approx_1}) and noticing that the term $\lambda_{i,i,1}(t)I_{i,j}^{(1)}(t)$ and the term $\lambda_{j,j,1}(t)I_{i,j}^{(1)}(t)$ cancels out, we have
    \begin{align*}
        \nabla_{\bfv_i}\calL\paren{\bm{\theta}(t)}^\top\bbfv_j = - \frac{1}{a_i}\cdot \begin{cases}
            \pm \calO\paren{m\aggepst{\ell}(t)^3 + \aggepst{\ell}(t)\backepsi{\ell}(t)}& \text{ if } i\in[m]\setminus \mathcal{R}_{\ell}\\
            \begin{aligned}
                & \hat{\lambda}_{i_\ell^\star,j_\ell^\star,1}(t)\gamma_{j,j_\ell^\star}^{(1)}(t) - \hat{\lambda}_{i_\ell^\star,j_\ell^\star,1}(t)\gamma_{i_\ell^\star,j_\ell^\star}^{(1)}(t)I_{i,j}^{(1)}(t)\\
                & \qqquad \pm \calO\paren{\aggepst{\ell}(t)^2}\gamma_{i_\ell^\star,j_\ell^\star}^{(2)}(t)^2\\
                & \qqquad \pm \calO\paren{m\aggepst{\ell}(t)^3 + \aggepst{\ell}(t)\varepsilon_{5,\ell}(t)}
            \end{aligned} & \text{ if } i = i_\ell^\star
        \end{cases}
    \end{align*}
    
    \textbf{Analysis of $\nabla_{\bfw_i}\calL\paren{\bm{\theta}(t)}^\top\bbfw_j$.} To analyze $\nabla_{\bfw_i}\calL\paren{\bm{\theta}(t)}^\top\bbfw_j$, we first compute that for $i,j\in[m]\setminus \mathcal{R}_{\ell-1}, i\neq j$
    \begin{align*}
        \sum_{r=1}^{m^\star}\hat{\lambda}_{i,r,2}(t)I_{i,j}^{(3)}(t) & = \pm \calO\paren{\aggepst{\ell}(t)^2}\gamma_{i_\ell^\star,j_\ell^\star}^{(2)}(t)^2\indy{i=i_\ell^\star} \pm \calO\paren{m\aggepst{\ell}(t)^3}\\
        \sum_{r=1}^{m^\star}\hat{\lambda}_{i,r,3}(t)\zeta_{j,r}^{(2)}(t) &= \pm \calO\paren{\aggepst{\ell}(t)^2}\gamma_{i_\ell^\star,j_\ell^\star}^{(2)}(t)^2\indy{i=i_\ell^\star} \pm \calO\paren{m\aggepst{\ell}(t)^3}\\
        \sum_{r=1}^{m^\star}\hat{\lambda}_{i,r,5}(t)\gamma_{j,r}^{(2)}(t) & = \hat{\lambda}_{i_\ell^\star,j_\ell^\star,5}\gamma_{j,j_\ell^\star}^{(2)}(t)\indy{i=i_\ell^\star} \pm \calO\paren{\aggepst{\ell}(t)^2}\gamma_{i_\ell^\star,j_\ell^\star}^{(2)}(t)\indy{j=i_\ell^\star} \pm \calO\paren{m\aggepst{\ell}(t)^3}\\
        \sum_{r=1}^m\lambda_{i,r,2}(t)I_{i,j}^{(3)}(t) & = \pm \calO\paren{m\aggepst{\ell}(t)^3} \pm \begin{cases}
            \calO\paren{\aggepst{\ell}(t)\backepsi{\ell}(t)} & \text{ if } i\neq i_\ell^\star\\
            \calO\paren{\aggepst{\ell}(t)\varepsilon_{5,\ell}(t)} & \text{ if } i = i_\ell^\star
        \end{cases}\\
        \sum_{r=1}^m\lambda_{i,r,3}(t)I_{r,j}^{(3)}(t) & = \pm \calO\paren{m\aggepst{\ell}(t)^3} \pm \begin{cases}
            \calO\paren{\aggepst{\ell}(t)\backepsi{\ell}(t)} & \text{ if } i\neq i_\ell^\star\\
            \calO\paren{\aggepst{\ell}(t)\varepsilon_{5,\ell}(t)} & \text{ if } i = i_\ell^\star
        \end{cases}\\
        \sum_{r=1}^m\lambda_{i,r,5}(t)I_{r,j}^{(2)}(t) & = \lambda_{i,i,5}(t)I_{i,j}^{(2)}(t) \pm \calO\paren{m\aggepst{\ell}(t)^3}
    \end{align*}
    Combining with (\ref{eq:dynamic_approx_2}) and noticing that the term $\lambda_{i,i,5}(t)I_{i,j}^{(2)}(t)$ and the term $\lambda_{j,j,5}(t)I_{i,j}^{(2)}(t)$ cancels out, we have
    \begin{align*}
        \nabla_{\bfw_i}\calL\paren{\bm{\theta}(t)}^\top\bbfw_j = - \frac{9}{b_i}\cdot \begin{cases}
            \pm \calO\paren{m\aggepst{\ell}(t)^3 + \aggepst{\ell}(t)\backepsi{\ell}(t)}&\text{ if } i,j\in[m]\setminus \mathcal{R}_{\ell}\\
            \begin{aligned}
                & \hat{\lambda}_{i_\ell^\star,j_\ell^\star,5}(t)\gamma_{j,j_\ell^\star}^{(2)}(t) - \hat{\lambda}_{i_\ell^\star,j_\ell^\star,5}(t)\gamma_{i_\ell^\star,j_\ell^\star}^{(2)}(t)I_{i,j}^{(2)}(t)\\
                & \qqquad \pm \calO\paren{\aggepst{\ell}(t)^2}\gamma_{i_\ell^\star,j_\ell^\star}^{(2)}(t)^2\\
                & \qqquad \pm \calO\paren{m\aggepst{\ell}(t)^3 + \aggepst{\ell}(t)\varepsilon_{5,\ell}(t)}
            \end{aligned} & \text{ if } i = i_\ell^\star\\
            \begin{aligned}
                 & \pm \calO\paren{\aggepst{\ell}(t)^2}\gamma_{i_\ell^\star,j_\ell^\star}^{(2)}(t)\\
                 & \qqquad \pm \calO\paren{m\aggepst{\ell}(t)^3 + \aggepst{\ell}(t)\backepsi{\ell}(t)}
            \end{aligned}&\text{ if } j=i_\ell^\star
        \end{cases}
    \end{align*}
    
    \textbf{Analysis of $\nabla_{\bfv_i}\calL\paren{\bm{\theta}(t)}^\top\bbfw_j$.} The analysis of $\nabla_{\bfv_i}\calL\paren{\bm{\theta}(t)}^\top\bbfw_j$ will be separated into two cases. First, regardless of the cases, we have that
    \begin{align*}
        \sum_{r=1}^{m^\star}\hat{\lambda}_{i,r,1}(t)\zeta_{j,r}(t) & = \hat{\lambda}_{i_\ell^\star,j_\ell^\star,1}(t)\zeta_{j,j_\ell^\star}^{(2)}(t)\indy{i=i_\ell^\star} \pm \calO\paren{m\aggepst{\ell}(t)^3}\\
        \sum_{r=1}^{m}\lambda_{i,r,1}(t)I_{r,j}^{(3)}(t) & = \lambda_{i,i,1}(t)I_{i,j}^{(3)}(t)\pm \calO\paren{m\aggepst{\ell}(t)^3}
    \end{align*}
    We first analyze the case $i = j$, and then we dive into $i\neq j$. In the case where $i=j$, we have that
    \begin{align*}
        \sum_{r=1}^{m^\star}\hat{\lambda}_{i,r,2}(t)I_{i,j}^{(2)}(t) & = \hat{\lambda}_{i_\ell^\star,j_\ell^\star,2}(t)\indy{i=i_\ell^\star} \pm \calO\paren{m\aggepst{\ell}(t)^3}\\
        \sum_{r=1}^{m^\star}\hat{\lambda}_{i,r,4}(t)\gamma_{j,r}^{(2)}(t) & = \hat{\lambda}_{i_\ell^\star,j_\ell^\star,4}(t)\gamma_{i_\ell^\star,j_\ell^\star}^{(2)}(t)\indy{i=i_\ell^\star} \pm \calO\paren{m\aggepst{\ell}(t)^3}\\
        \sum_{r=1}^m\lambda_{i,r,2}(t)I_{i,j}^{(2)}(t) & = 6C_{S,2}I_{i,i}^{(3)}(t) \pm \calO\paren{m\aggepst{\ell}(t)^3}\\
        \sum_{r=1}^m\lambda_{i,r,4}(t)I_{r,j}^{(2)}(t) & = 6C_{S,2}I_{i,i}^{(3)}(t) \pm \calO\paren{m\aggepst{\ell}(t)^3}
    \end{align*}
    Thus, for the case $i = j$, we have that
    \begin{align*}
        \nabla_{\bfv_i}\calL\paren{\bm{\theta}(t)}^\top\bbfw_j & = -\frac{1}{a_i} \paren{\hat{\lambda}_{i_\ell^\star,j_\ell^\star,1}(t)\zeta_{j,j_\ell^\star}^{(2)}(t) + 3\hat{\lambda}_{i_\ell^\star,j_\ell^\star,2}(t) + 3\hat{\lambda}_{i_\ell^\star,j_\ell^\star,4}(t)\gamma_{i_\ell^\star,j_\ell^\star}^{(2)}(t)}\indy{i=i_\ell^\star}\\
        & \qqquad  -\frac{1}{a_i}\hat{\lambda}_{i_\ell^\star,j_\ell^\star,1}(t)\gamma_{i_\ell^\star,j_\ell^\star}^{(1)}(t)\indy{i=i_\ell^\star} - \frac{36}{a_i}C_{S,2}I_{i,i}^{(3)}(t) \pm \calO\paren{m\aggepst{\ell}(t)^3}
    \end{align*}
    Now, for the case $i\neq j$, we can compute that
    \begin{align*}
        \sum_{r=1}^{m^\star}\hat{\lambda}_{i,r,2}(t)I_{i,j}^{(2)}(t) & = \pm \calO\paren{
        \aggepst{\ell}(t)^2}\gamma_{i_\ell^\star,j_\ell^\star}^{(2)}(t)^2\indy{i=i_\ell^\star} \pm \calO\paren{m\aggepst{\ell}(t)^3}\\
        \sum_{r=1}^{m^\star}\hat{\lambda}_{i,r,4}(t)\gamma_{j,r}^{(2)}(t) & = \pm \calO\paren{\aggepst{\ell}(t)^2}\gamma_{i_\ell^\star,j_\ell^\star}^{(2)}(t)^2\indy{i=i_\ell^\star} \pm \calO\paren{m\aggepst{\ell}(t)^3}\\
        \sum_{r=1}^m\lambda_{i,r,2}(t)I_{i,j}^{(2)}(t) & =  \pm \calO\paren{m\aggepst{\ell}(t)^3} \pm \begin{cases}
            \calO\paren{\aggepst{\ell}(t)\backepsi{\ell}(t)} & \text{ if } i \in [m]\setminus \mathcal{R}_{\ell}\\
            \calO\paren{\aggepst{\ell}(t)\varepsilon_{5,\ell}(t)} & \text{ if } i = i_\ell^\star
        \end{cases}\\
        \sum_{r=1}^m\lambda_{i,r,4}(t)I_{r,j}^{(2)}(t) & =   \pm \calO\paren{m\aggepst{\ell}(t)^3}\pm \begin{cases}
            \calO\paren{\aggepst{\ell}(t)\backepsi{\ell}(t)} & \text{ if } i \in [m]\setminus \mathcal{R}_{\ell}\\
            \calO\paren{\aggepst{\ell}(t)\varepsilon_{5,\ell}(t)} & \text{ if } i = i_\ell^\star
        \end{cases}
    \end{align*}
    Therefore, combining with (\ref{eq:dynamic_approx_1}) gives
    \begin{align*}
        \nabla_{\bfv_i}\calL\paren{\bm{\theta}(t)}^\top\bbfw_j & = -\frac{1}{a_i}\cdot \begin{cases}
            \begin{aligned}
                & \hat{\lambda}_{i_\ell^\star,j_\ell^\star,1}(t)\paren{\zeta_{j,j_\ell^\star}^{(2)}(t) - \gamma_{i_\ell^\star,j_\ell^\star}^{(1)}(t)I_{i,j}^{(3)}(t)}\\
                & \qqquad \pm \calO\paren{\aggepst{\ell}(t)^2}\gamma_{i_\ell^\star,j_\ell^\star}^{(2)}(t)^2\\
                & \qqquad \pm \calO\paren{m\aggepst{\ell}(t)^3 + \aggepst{\ell}(t)\varepsilon_{5,\ell}(t)}
            \end{aligned} & \text{ if } i = i_\ell^\star\\
            \pm \calO\paren{m\aggepst{\ell}(t)^3 + \aggepst{\ell}(t)\backepsi{\ell}(t)} & \text{ if } i \in [m]\setminus \mathcal{R}_{\ell}
        \end{cases}
    \end{align*}

    \textbf{Analysis of $\nabla_{\bfw_i}\calL\paren{\bm{\theta}(t)}^\top\bbfv_j$.} The analysis of $\nabla_{\bfw_i}\calL\paren{\bm{\theta}(t)}^\top\bbfv_j$ will also be separated into two cases. First, regardless of the cases, we have that
    \begin{align*}
        \sum_{r=1}^{m^\star}\hat{\lambda}_{i,r,5}(t)\zeta_{j,r}^{(1)}(t) & = \hat{\lambda}_{i_\ell^\star,j_\ell^\star,5}(t)\zeta_{j,j_\ell^\star}^{(1)}(t)\indy{i = i_\ell^\star} \pm \calO\paren{m\aggepst{\ell}(t)^3}\\
        \sum_{r=1}^m\lambda_{i,r,5}(t)I_{j,r}^{(3)}(t) & = \lambda_{i,i,5}(t)I_{j,i}^{(3)}(t) \pm \calO\paren{m\aggepst{\ell}(t)^2}
    \end{align*}
    We first analyze the case $i = j$, and then we dive into $i\neq j$. In the case where $i=j$, we have that
    \begin{align*}
        \sum_{r=1}^{m^\star}\hat{\lambda}_{i,r,2}(t)I_{j,i}^{(1)}(t) & = \hat{\lambda}_{i_\ell^\star,j_\ell^\star,2}(t)\indy{i=i_\ell^\star} \pm \calO\paren{m\aggepst{\ell}(t)^3}\\
        \sum_{r=1}^{m^\star}\hat{\lambda}_{i,r,3}(t)\gamma_{j,r}^{(1)}(t) & = \hat{\lambda}_{i_\ell^\star,j_\ell^\star,3}(t)\gamma_{i_\ell^\star,j_\ell^\star}^{(1)}(t)\indy{i=i_\ell^\star} \pm \calO\paren{m\aggepst{\ell}(t)^3}\\
        \sum_{r=1}^m\lambda_{j,r,2}(t)I_{i,j}^{(1)}(t) & = 6C_{S,2}I_{i,i}^{(3)}(t) \pm \calO\paren{m\aggepst{\ell}(t)^3}\\
        \sum_{r=1}^m\lambda_{j,r,3}(t)I_{i,r}^{(1)}(t) & = 6C_{S,2}I_{i,i}^{(3)}(t) \pm \calO\paren{m\aggepst{\ell}(t)^3}
    \end{align*}
    Thus, for the case $i = j$, we have that
    \begin{align*}
        \nabla_{\bfw_i}\calL\paren{\bm{\theta}(t)}^\top\bbfv_j & = -\frac{3}{b_i}\cdot \paren{3\hat{\lambda}_{i_\ell^\star,j_\ell^\star,5}(t)\zeta_{j,j_\ell^\star}^{(1)}(t) + \hat{\lambda}_{i_\ell^\star,j_\ell^\star,2}(t) + \hat{\lambda}_{i_\ell^\star,j_\ell^\star,3}(t)\gamma_{i_\ell^\star,j_\ell^\star}^{(1)}(t)}\indy{i=i_\ell^\star}\\
        & \qqquad - \frac{9}{b_i}\cdot \hat{\lambda}_{i_\ell^\star,j_\ell^\star,5}(t)\gamma_{i_\ell^\star,j_\ell^\star}^{(2)}(t)\indy{i=i_\ell^\star} - 36C_{S,2}I_{i,i}^{(3)}(t) \pm \calO\paren{m\aggepst{\ell}(t)^2}
    \end{align*}
    Now, for the case $i\neq j$, we can compute that
    \begin{align*}
        \sum_{r=1}^{m^\star}\hat{\lambda}_{i,r,2}(t)I_{j,i}^{(1)}(t) & = \pm \calO\paren{
        \aggepst{\ell}(t)^2}\gamma_{i_\ell^\star,j_\ell^\star}^{(2)}(t)^2\indy{i=i_\ell^\star} \pm \calO\paren{m\aggepst{\ell}(t)^3}\\
        \sum_{r=1}^{m^\star}\hat{\lambda}_{i,r,3}(t)\gamma_{j,r}^{(1)}(t) & = \pm \calO\paren{
        \aggepst{\ell}(t)^2}\gamma_{i_\ell^\star,j_\ell^\star}^{(2)}(t)^2\indy{i=i_\ell^\star} \pm \calO\paren{m\aggepst{\ell}(t)^3}\\
        \sum_{r=1}^m\lambda_{i,r,2}(t)I_{j,i}^{(1)}(t) & = \pm \calO\paren{m\aggepst{\ell}(t)^3} \pm \begin{cases}
            \calO\paren{\aggepst{\ell}(t)\backepsi{\ell}(t)} & \text{ if } i \in [m]\setminus \mathcal{R}_{\ell}\\
            \calO\paren{\aggepst{\ell}(t)\varepsilon_{5,\ell}(t)} & \text{ if } i = i_\ell^\star
        \end{cases}\\
        \sum_{r=1}^m\lambda_{i,r,3}(t)I_{j,r}^{(1)}(t) & = \pm \calO\paren{m\aggepst{\ell}(t)^3} \pm \begin{cases}
            \calO\paren{\aggepst{\ell}(t)\backepsi{\ell}(t)} & \text{ if } i \in [m]\setminus \mathcal{R}_{\ell}\\
            \calO\paren{\aggepst{\ell}(t)\varepsilon_{5,\ell}(t)} & \text{ if } i = i_\ell^\star
        \end{cases}
    \end{align*}
    Therefore, combining with (\ref{eq:dynamic_approx_1}) and (\ref{eq:dynamic_approx_2}) gives
    \begin{align*}
        \nabla_{\bfw_i}\calL\paren{\bm{\theta}(t)}^\top\bbfv_j & = -\frac{9}{b_i}\cdot \begin{cases}
            \begin{aligned}
                & \hat{\lambda}_{i_\ell^\star,j_\ell^\star,5}(t)\paren{\zeta_{j,j_\ell^\star}^{(1)}(t) - \gamma_{i_\ell^\star,j_\ell^\star}^{(2)}(t)I_{i,j}^{(3)}(t)}\\
                & \qqquad \pm \calO\paren{\aggepst{\ell}(t)^2}\gamma_{i_\ell^\star,j_\ell^\star}^{(2)}(t)^2\\
                & \qqquad \pm \calO\paren{m\aggepst{\ell}(t)^3 + \aggepst{\ell}(t)\varepsilon_{5,\ell}(t)}
            \end{aligned} & \text{ if } i = i_\ell^\star\\
            \pm \calO\paren{m\aggepst{\ell}(t)^3 + \aggepst{\ell}(t)\backepsi{\ell}(t)} & \text{ if } i \in [m]\setminus \mathcal{R}_{\ell}
        \end{cases}
    \end{align*}
\end{proof}

\subsection{Establishing the Inductive Hypothesis: Phase 1}
Starting from this section, we assume that for a fixed $\ell\in [m^\star]$, the inductive hypothesis holds. That is, the condition of Lemma~\ref{lem:dynamic_approx} holds. Then we shall analyze the convergence for that $\ell$ to prove the inductive hypothesis and establish convergence. Notice that, by the statement of the inductive hypothesis, the case $\ell=1$ naturally satisfies it, thus requiring no additional proof. Therefore, we focus on the case of a general fixed $\ell$, and the proof will be divided into two phase. Phase 1 (this section) show that there exist some $T^\star$ such that $\gamma_{i_\ell^\star,j_\ell^\star}^{(1)}\paren{T^\star}$ and $\gamma_{i_\ell^\star,j_\ell^\star}^{(2)}\paren{T^\star}$ are some constant close to 1, while $\aggepst{\ell}(t), \backepsi{\ell}(t)$, and $\varepsilon_{5,\ell}(t)$ are $\calO\paren{\frac{m^2}{\delta_{\mathbb{P}}\sqrt{d}}}$.

We denote the following time
\begin{equation}
    \label{eq: T_c_def}
    \begin{gathered}
        T_c := \min \left\{t\geq 0: A_1(t)\wedge A_2(t)\wedge A_3(t)\right\}\\
        A_1(t) = \left\{\min\left\{\backepst{\ell}(t),\varepsilon_{3,\ell}(t),\varepsilon_{4,\ell}(t)\right\} > \frac{\beta_6m^2}{\delta_{\mathbb{P}}\sqrt{d}}\right\};\quad A_2(t) = \left\{\backepsi{\ell}(t) \geq \beta_6 m\aggepst{\ell}(t)^3\right\}\\
        A_3(t) = \left\{\varepsilon_{5,\ell}(t) \geq \beta_6 \paren{m\aggepst{\ell}(t)^3 + \backepso{\ell}(t)\gamma_{i_\ell^\star,j_\ell^\star}^{(2)}(t)^2}\right\}
    \end{gathered}
\end{equation}
By the initialization property, we have that
\[
    \backepso{\ell}(t),\varepsilon_{3,\ell}(t),\varepsilon_{4,\ell}(t) \leq \frac{\beta_3}{\sqrt{d}}\paren{\log \frac{m}{\delta_{\mathbb{P}}}}^{\frac{1}{2}};\quad \backepsi{\ell}(0) = \varepsilon_{5,\ell}(0) = 0
\]
Therefore, $T_c > 0$. Moreover, for all $t\leq T_c$, by the inductive hypothesis, we have that 
\[
    \backepso{\ell}(t),\varepsilon_{3,\ell}(t),\varepsilon_{4,\ell}(t) \leq \frac{\beta_6m^2}{\delta_{\mathbb{P}}\sqrt{d}}\;\Rightarrow \forweps{\ell}(t) \leq \calO\paren{\frac{m^2}{\delta_{\mathbb{P}}\sqrt{d}}}
\]
which implies that $\aggepst{\ell}(t) \leq \calO\paren{\frac{m^2}{\delta_{\mathbb{P}}\sqrt{d}}}$. Also, by definition we have that $\aggepso{\ell}(t) \leq \aggepst{\ell}(t)$.

\textbf{Growth of $\gamma_{i_\ell^\star,j_\ell^\star}^{(2)}(t)$.} We will start with analyzing $\gamma_{i_\ell^\star,j_\ell^\star}^{(2)}(t)$. For any $\xi\in (0,1)$, recall the definition of $T_{\ell}\paren{\xi}$ in Definition~\ref{def:recovery_time}.
    We should notice that, although by definition we have $\gamma_{i_\ell^\star,j_\ell^\star}^{(2)}(0) > 0$, it is not always the case that $\gamma_{i_\ell^\star,j_\ell^\star}^{(1)}(0)$ will also be positive. Therefore, we also need to control the "negativeness" of $\gamma_{i_\ell^\star,j_\ell^\star}^{(1)}(0)$ when analyzing the growth of $\gamma_{i_\ell^\star,j_\ell^\star}^{(2)}(t)$. To do this, we present the following lemma.

\begin{lemma}
    \label{lem:gamma_2_growth}
    Suppose that the inductive hypothesis in Condition~\ref{cond:inductive_hypo} and the initialization condition in Condition~\ref{cond:init} holds. Let $T_c$ be defined in (\ref{eq: T_c_def}) and $T(\xi)$ in Definition~\ref{def:recovery_time}.  If $T_c \geq T_{\ell}\paren{\frac{1}{2}}$, then we have that
    \begin{gather*}
        \derivt \gamma_{i_\ell^\star,j_\ell^\star}^{(2)}(t) \geq 0;\; \lambda_{i_\ell^\star,j_\ell^\star,5}(t)\geq 0;\; \gamma_{i_\ell^\star,j_\ell^\star}^{(1)}(t) \geq -\frac{\beta_6}{\sqrt{d}}\paren{\log\frac{m}{\delta_{\mathbb{P}}}}^{\frac{1}{2}};\; \forall t \leq \min\{T_c, T(\xi)\}\\
        \gamma_{i_\ell^\star,j_\ell^\star}^{(2)}(t+T_0) \geq \paren{\gamma_{i_\ell^\star,j_\ell^\star}^{(2)}(T_0)^{-1} - \frac{18}{b_{i_\ell^\star}^2}\paren{c_0^2\paren{1-\xi^2} - \calO\paren{\frac{m^7}{\delta_{\mathbb{P}}^3\sqrt{d}}}}t}^{-1}\\
         T_{\ell}\paren{\xi} \leq \frac{b_{i_\ell^\star}^2}{18}\paren{c_0^2\paren{1-\xi} - \calO\paren{\frac{m^7}{\delta_{\mathbb{P}}^3\sqrt{d}}}}^{-1}\gamma_{i_\ell^\star,j_\ell^\star}^{(2)}(0)^{-1}
    \end{gather*}
    for all $\xi$ such that $(1-\xi^2)^{-1} \leq \calO(1)$. 
\end{lemma}
\begin{proof}
    To start, we lower bound the time-derivative of $\gamma_{i_\ell^\star,j_\ell^\star}^{(2)}(t)$. Let $T_c'$ be defined as
    \[
        T_c' = \min\left\{t\geq 0: \gamma_{i_\ell^\star,j_\ell^\star}^{(1)}(t) < -\frac{\beta_6}{\sqrt{d}}\paren{\log\frac{m}{\delta_{\mathbb{P}}}}^{\frac{1}{2}} \right\}
    \]
    Since $\paren{\bfv_{i_\ell^\star}^\top\bbfv_j^\star}\leq \frac{\beta_3}{d}\log \frac{m}{\delta_{\mathbb{P}}}$ and $\norm{\bfv_{i_\ell^\star}}_2\geq 1 - \beta_2\delta_s$ for some $\beta_3> 0, \beta_2\leq o(1)$ and $\delta_s \leq \calO\paren{\frac{1}{m^2}}$, we have that $\gamma_{i_\ell^\star,j_\ell^\star}^{(1)}(0)^2\leq \frac{2\beta_3}{d}\log \frac{m}{\delta_{\mathbb{P}}}$. For $\beta_6 \geq 4\sqrt{\beta_3}$, we must have that $T_c' > 0$.
    By Lemma~\ref{lem:sum_bound}, we have that for all $t\leq T_c'$, it holds that
    \[
        \sum_{k=0}^{\infty}\frac{c_{k}^2}{k!}\gamma_{i_\ell^\star,j_\ell^\star}^{(1)}(t)^k \in c_0^2 \pm \calO\paren{\frac{1}{\sqrt{d}}\paren{\log\frac{m}{\delta_{\mathbb{P}}}}^{\frac{1}{2}}};\quad  \sum_{k=0}^{\infty}\frac{c_{k+1}^2}{k!}\gamma_{i_\ell^\star,j_\ell^\star}^{(1)}(t)^k \in c_1^2 \pm \calO\paren{\frac{1}{\sqrt{d}}\paren{\log\frac{m}{\delta_{\mathbb{P}}}}^{\frac{1}{2}}}
    \]
    By Lemma~\ref{lem:dynamic_approx}, we have that
    \begin{align*}
        \derivt \gamma_{i_\ell^\star,j_\ell^\star}^{(2)}(t) & = - \frac{1}{b_{i_\ell^\star}}\nabla_{\bfw_{i_\ell}^\star}\calL\paren{\bm{\theta}(t)}^\top\bbfw_{j_\ell^\star}^\star\\
        & = \frac{9}{b_{i_\ell^\star}^2}\paren{1 - \gamma_{i_\ell^\star,j_\ell^\star}^{(2)}(t)^2}\gamma_{i_\ell^\star,j_\ell^\star}^{(2)}\hat{\lambda}_{i_\ell^\star,j_\ell^\star,5}(t) \pm \calO\paren{\aggepst{\ell}(t)^2}\gamma_{i_\ell^\star,j_\ell^\star}^{(2)}(t) \pm \calO\paren{\aggepst{\ell}(t)^3 + \aggepst{\ell}(t)\varepsilon_{5,\ell}(t)}
    \end{align*}
    Therefore, with the form of $\hat{\lambda}_{i_\ell^\star,j_\ell^\star,5}$ from Lemma~\ref{lem:approx_lambda_hat}, we have that
    \begin{equation}
        \label{eq:gamma_2_growth_1}
        \begin{aligned}
            \derivt \gamma_{i_\ell^\star,j_\ell^\star}^{(2)}(t) & = \frac{18}{b_{i_\ell^\star}^2}\paren{1-\gamma_{i_\ell^\star,j_\ell^\star}^{(2)}(t)^2}\gamma_{i_\ell^\star,j_\ell^\star}^{(2)}(t)^2\sum_{k=0}^{\infty}\frac{c_{k}^2}{k!}\gamma_{i_\ell^\star,j_\ell^\star}^{(1)}(t)^k \pm \calO\paren{\aggepst{\ell}(t)^3 + \aggepst{\ell}(t)\varepsilon_{5,\ell}(t)}\\
            & \qqquad \pm \calO\paren{\aggepst{\ell}(t)^2}\gamma_{i_\ell^\star,j_\ell^\star}^{(2)}(t)\\
            & \geq \frac{18}{b_{i_\ell^\star}^2}\paren{1-\gamma_{i_\ell^\star,j_\ell^\star}^{(2)}(t)^2 - \calO\paren{\frac{m^4}{\delta_{\mathbb{P}}^2\sqrt{d}}}}\gamma_{i_\ell^\star,j_\ell^\star}^{(2)}(t)^2\paren{c_0^2 - \calO\paren{\frac{1}{\sqrt{d}}\paren{\log\frac{m}{\delta_{\mathbb{P}}}}^{\frac{1}{2}}}}\\
            & \qqquad - \calO\paren{\aggepst{\ell}(t)^3 + \aggepst{\ell}(t)\varepsilon_{5,\ell}(t)}\\
            & \geq \frac{18}{b_{i_\ell^\star}^2}\paren{1-\xi^2 - \calO\paren{\frac{m^4}{\delta_{\mathbb{P}}^2\sqrt{d}}}}\paren{c_0^2 - \calO\paren{\frac{1}{\sqrt{d}}\paren{\log\frac{m}{\delta_{\mathbb{P}}}}^{\frac{1}{2}}}}\gamma_{i_\ell^\star,j_\ell^\star}^{(2)}(t)^2 - \calO\paren{\frac{m^{7}}{\delta_{\mathbb{P}}^3d^{\frac{3}{2}}}}\\
            & \geq \frac{18}{b_{i_\ell^\star}^2}\paren{c_0^2\paren{1-\xi^2} - \calO\paren{\frac{m^4}{\delta_{\mathbb{P}}^2\sqrt{d}}}}\gamma_{i_\ell^\star,j_\ell^\star}^{(2)}(t)^2 - \calO\paren{\frac{m^{7}}{\delta_{\mathbb{P}}^3d^{\frac{3}{2}}}}
        \end{aligned}
    \end{equation}
    Let $\xi$ be given such that $\paren{1-\xi^2}^{-1}\leq \calO\paren{1}$ By the lower bound that $\gamma_{i_\ell^\star,j_\ell^\star}^{(2)}(0)^2 \geq \frac{\log m^\star}{d}$, we have that
    \[
        \derivt \gamma_{i_\ell^\star,j_\ell^\star}^{(2)}(0) \geq \beta_{\text{temp}}\gamma_{i_\ell^\star,j_\ell^\star}^{(2)}(0)^2 - \calO\paren{\frac{m^{7}}{\delta_{\mathbb{P}}^3d^{\frac{3}{2}}}}\geq 0
    \]
    for $d\geq \beta_5 \delta_{\mathbb{P}}^6m^{16}$ for some $\beta_{\text{temp}} > 0$. This shows that $\gamma_{i_\ell^\star,j_\ell^\star}^{(2)}(t)^2 \geq \frac{\log m^\star}{d}$ for all $t\leq \min\{T_c, T_c', T_{\ell}(\xi)\}$ for any $\paren{1-\xi^2}^{-1}\leq \calO\paren{1}$, and thus $\derivt \gamma_{i_\ell^\star,j_\ell^\star}^{(2)}(t) \geq 0$ for the sane choice of $t$. This shows the first property. Moreover, by Lemma~\ref{lem:approx_lambda_hat},
    \begin{align*}
        \hat{\lambda}_{i_\ell^\star,j_\ell^\star,5}(t) & = 2\gamma_{i_\ell^\star,j_\ell^\star}^{(2)}(t)^2\sum_{k=0}^{\infty}\frac{c_k^2}{k!}\gamma_{i_\ell^\star,j_\ell^\star}^{(1)}(t)^k \pm \calO\paren{\aggepst{\ell}(t)^2}\gamma_{i_\ell^\star,j_\ell^\star}^{(2)}(t)^2 \pm \calO\paren{\aggepst{\ell}(t)^4}\\
        & \geq 2\gamma_{i_\ell^\star,j_\ell^\star}^{(2)}(t)^2\paren{c_0^2 - \calO\paren{\frac{m^4}{\delta_{\mathbb{P}}^2\sqrt{d}}}} - \calO\paren{\frac{m^4}{\delta_{\mathbb{P}}^2d}}\\
        & \geq 0
    \end{align*}
    as $d\geq \beta_5\delta_{\mathbb{P}}^6m^{16}$. This shows the second property. To lower bound $\gamma_{i_\ell^\star,j_\ell^\star}^{(1)}(t)$, we first write that, by Lemma~\ref{lem:dynamic_approx}
    \begin{align*}
        \derivt \gamma_{i_\ell^\star,j_\ell^\star}^{(1)}(t) & = -\frac{1}{a_i}\nabla_{\bfv_{i_\ell^\star}}\calL\paren{\bm{\theta}(t)}^\top\bbfv_{j_\ell^\star}^\star\\
        & = \frac{1}{a_i^2}\paren{1 - \gamma_{i_\ell^\star,j_\ell^\star}^{(1)}(t)^2}\hat{\lambda}_{i_\ell^\star,j_\ell^\star,1}(t)\pm \calO\paren{\aggepst{\ell}(t)^2}\gamma_{i_\ell^\star,j_\ell^\star}^{(2)}(t)^2 \pm \calO\paren{m\aggepst{\ell}(t)^3 + \aggepst{\ell}(t)\varepsilon_{5,\ell}(t)}
    \end{align*}
    Therefore, we could notice that, since
    \[
        \sum_{k=0}^{\infty}\frac{c_{k+1}^2}{k!}\gamma_{i_\ell^\star,j_\ell^\star}^{(1)}(t)^k \geq c_1^2 - \calO\paren{\frac{1}{\sqrt{d}}\paren{\log\frac{m}{\delta_{\mathbb{P}}}}^{\frac{1}{2}}} \geq 0
    \]
    for $d\geq \beta_5m^{14} \geq \beta_5\log \frac{m}{\delta_{\mathbb{P}}}$\footnote{since we require $m \geq \frac{\beta_4\log m^\star}{\delta_{\mathbb{P}}}$}, we must have that for all $t\leq \min\{T_c, T_c', T_{\ell}(\xi)\}$ for any $\paren{1-\xi^2}^{-1}\leq \calO\paren{1}$, either $\gamma_{i_\ell^\star,j_\ell^\star}^{(1)}(t) \geq \frac{1}{2}$, or
    \begin{equation}
        \label{eq:gamma_2_growth_2}
        \begin{aligned}
            \derivt \gamma_{i_\ell^\star,j_\ell^\star}^{(1)}(t) & \geq 6\paren{1 - \gamma_{i_\ell^\star,j_\ell^\star}^{(1)}(t)^2}\gamma_{i_\ell^\star,j_\ell^\star}^{(2)}(t)^3\sum_{k=0}^{\infty}\frac{c_{k+1}^2}{k!}\gamma_{i_\ell^\star,j_\ell^\star}^{(1)}(t)^k - \calO\paren{m\aggepst{\ell}(t)^3 + \aggepst{\ell}(t)\varepsilon_{5,\ell}(t)}\\
            & \qqquad - \calO\paren{\aggepst{\ell}(t)^2}\gamma_{i_\ell^\star,j_\ell^\star}^{(2)}(t)^2\\
            & \geq c_1^2\paren{\gamma_{i_\ell^\star,j_\ell^\star}^{(2)}(t) - \calO\paren{\aggepst{\ell}(t)^2}}\gamma_{i_\ell^\star,j_\ell^\star}^{(2)}(t)^2 - \calO\paren{\frac{m^7}{\delta_{\mathbb{P}}^3d^{\frac{3}{2}}}}\\
            & \geq - \calO\paren{\frac{m^7}{\delta_{\mathbb{P}}^3d^{\frac{3}{2}}}}
        \end{aligned}
    \end{equation}
    where the last inequality follows from the fact that
    \[
        \aggepst{\ell}(t)^2 \leq \calO\paren{\frac{m^4}{\delta_{\mathbb{P}}^2d}} \leq \paren{\frac{\log m^\star}{d}}^{\frac{1}{2}} \leq \gamma_{i_\ell^\star,j_\ell^\star}^{(2)}(0) \leq \gamma_{i_\ell^\star,j_\ell^\star}^{(2)}(t)
    \]
    Therefore, for all $t\leq \min\{T_c, T_c', T_{\ell}(\xi)\}$ for any $\paren{1-\xi^2}^{-1}\leq \calO\paren{1}$ we must have that
    \[
        \gamma_{i_\ell^\star,j_\ell^\star}^{(1)}(t) \geq \gamma_{i_\ell^\star,j_\ell^\star}^{(1)}(0) - \calO\paren{\frac{m^7}{\delta_{\mathbb{P}}^3d^{\frac{3}{2}}}}\cdot t \geq - \frac{\sqrt{2\beta_4}}{\sqrt{d}}\paren{\log \frac{m}{\delta_{\mathbb{P}}}}^{\frac{1}{2}}
    \]
    Choosing $\beta_6 \geq 4\sqrt{\beta_4}$ gives that
    \[
        T_{c'} \geq \min\left\{T_c, T_{\ell}(\xi), \Omega\paren{\frac{d}{m^7}}\right\} 
    \]
    Now we are going to lower bound $T_{\ell}(\xi)$. When $d\geq \beta_5 m^{15}$, 
    \[
        \derivt \gamma_{i_\ell^\star,j_\ell^\star}^{(2)}(t) \geq \frac{18}{b_{i_\ell^\star}^2}\paren{c_0^2\paren{1-\xi^2} - \calO\paren{\frac{m^7}{\delta_{\mathbb{P}}^3\sqrt{d}}}}\gamma_{i_\ell^\star,j_\ell^\star}^{(2)}(t)^2
    \]
    Solving the differential equation gives that
    \[
        \gamma_{i_\ell^\star,j_\ell^\star}^{(2)}(t+T_0) \geq \paren{\gamma_{i_\ell^\star,j_\ell^\star}^{(2)}(T_0)^{-1} - \frac{18}{b_{i_\ell^\star}^2}\paren{c_0^2\paren{1-\xi^2} - \calO\paren{\frac{m^7}{\delta_{\mathbb{P}}^3\sqrt{d}}}}t}^{-1}
    \]
    for any $T_0 \geq 0$ and $t +T_0 \leq T_c$.
    This gives that, if $T(\xi) \leq T_c$, then
    \begin{align*}
        T_{\ell}(\xi) & \leq \frac{b_{i_\ell^\star}^2}{18}\paren{c_0^2\paren{1-\xi} - \calO\paren{\frac{m^7}{\delta_{\mathbb{P}}^3\sqrt{d}}}}^{-1}\paren{\gamma_{i_\ell^\star,j_\ell^\star}^{(2)}(0)^{-1} - \zeta^{-1}}\\
        & \leq \frac{b_{i_\ell^\star}^2}{18}\paren{c_0^2\paren{1-\xi} - \calO\paren{\frac{m^7}{\delta_{\mathbb{P}}^3\sqrt{d}}}}^{-1}\gamma_{i_\ell^\star,j_\ell^\star}^{(2)}(0)^{-1}
    \end{align*}
    This shows the last two properties. Moreover, we have that $T_{\ell}(\frac{1}{2}) \leq \calO\paren{\paren{\frac{d}{\log m^\star}}^{\frac{1}{2}}}$. However, at $T_{\ell}\paren{\frac{1}{2}}$, by (\ref{eq:gamma_2_growth_2}), we have that
    \begin{align*}
        \derivt \gamma_{i_\ell^\star,j_\ell^\star}^{(1)}(t) \geq c_1^2\paren{\gamma_{i_\ell^\star,j_\ell^\star}^{(2)}(t) - \calO\paren{\aggepst{\ell}(t)^2}}\gamma_{i_\ell^\star,j_\ell^\star}^{(2)}(t)^2 - \calO\paren{\frac{m^7}{\delta_{\mathbb{P}}^3d^{\frac{3}{2}}}} \geq 0
    \end{align*}
    which implies that $\gamma_{i_\ell^\star,j_\ell^\star}^{(1)}(t) \geq -\frac{\beta_6}{\sqrt{d}}\paren{\log\frac{m}{\delta_{\mathbb{P}}}}^{\frac{1}{2}}$ for all $t\leq \min\left\{T_c, T_{\ell}(\xi)\right\}$ for all $\xi$ such that $(1-\xi^2)^{-1}\leq \calO\paren{1}$, as long as $T\paren{\frac{1}{2}} \leq T_c$.
\end{proof}

\textbf{Upper bounding $\gamma_{i,j}^{(2)}(t)$ for $i\in[m]\setminus \mathcal{R}_{\ell-1}$ and $j\in[m^\star]\setminus\mathcal{C}_{\ell-1}$ with $(i,j)\neq (i_\ell^\star,j_\ell^\star)$.} Here we are going to show that the growth of $\gamma_{i,j}^{(2)}(t)$ with $i\in[m]\setminus \mathcal{R}_{\ell-1}, j\in[m^\star]\setminus \mathcal{C}_{\ell-1}$ and $(i,j)\neq (i_\ell^\star,j_\ell^\star)$ is slow in terms of when $\gamma_{i_\ell^\star,j_\ell^\star}^{(2)}(t)$ reaches $\xi$, $\gamma_{i,j}^{(2)}(t)$ is no bigger than $\calO\paren{\frac{m^2}{\delta_{\mathbb{P}}\sqrt{d}}}$.
\begin{lemma}
    \label{lem:gamma_2_ub}
    Suppose that the inductive hypothesis in Condition~\ref{cond:inductive_hypo}, and the initialization condition in Condition~\ref{cond:init}. Let $T_c$ be defined in (\ref{eq: T_c_def}) and $T_{\ell}(\xi)$ in Definition~\ref{def:recovery_time}. Then there exists some constant $\beta_7 > 0$ such that for all $t\leq \min\left\{T_c, T\paren{\frac{\delta_{\mathbb{P}}^2}{d^{\frac{2}{5}}}} + \frac{\beta_7\delta_{\mathbb{P}}\sqrt{d}}{m^2}\right\}$, we have that $\left|\gamma_{i,j}^{(2)}(t)\right| \leq O\paren{\frac{m^2}{\delta_{\mathbb{P}}\sqrt{d}}}$ for all $(i,j)\neq (i_\ell^\star,j_\ell^\star)$.
\end{lemma}
\begin{proof}
For $t \leq T_c$, by Lemma~\ref{lem:dynamic_approx}, we write out the dynamic of $\gamma_{i,j}^{(2)}(t)$ as
    \begin{align*}
        \derivt \gamma_{i,j}^{(2)}(t) & = - \nabla_{\bfv_i}\calL\paren{\bm{\theta}(t)}^\top\bbfv_j^\star\\
        & = \frac{18}{b_i^2}\sum_{k=0}^{\infty}\frac{c_k^2}{k!}\gamma_{i,j}^{(1)}(t)^k\gamma_{i,j}^{(2)}(t)^2 - \frac{9}{b_i^2}\cdot \hat{\lambda}_{i_\ell^\star,j_\ell^\star,5}\gamma_{i_\ell^\star,j_\ell^\star}^{(2)}(t)\gamma_{i,j}^{(2)}(t)\indy{i=i_\ell^\star}\\
        & \qqquad \pm \calO\paren{\frac{m^7}{\delta_{\mathbb{P}}^3d^{\frac{3}{2}}}} \pm \calO\paren{\frac{m^4}{\delta_{\mathbb{P}}^2d}}\gamma_{i_\ell^\star,j_\ell^\star}^{(2)}(t)\indy{i=i_\ell^\star}
    \end{align*}
    As shown in Lemma~\ref{lem:gamma_2_growth}, $\gamma_{i_\ell^\star,j_\ell^\star}^{(2)}(t) \geq 0, \hat{\lambda}_{i_\ell^\star,j_\ell^\star,5}(t)\geq 0$ for all $t\leq T_c$.  Therefore, we have that 
    \begin{align*}
        \derivt \left|\gamma_{i,j}^{(2)}(t)\right| & \leq \frac{18}{b_i^2}\left|\sum_{k=0}^{\infty}\frac{c_k^2}{k!}\gamma_{i,j}^{(1)}(t)^k\right|\cdot \left|\gamma_{i,j}^{(2)}(t)\right|^2 - \frac{9}{b_i^2}\cdot \hat{\lambda}_{i_\ell^\star,j_\ell^\star,5}\gamma_{i_\ell^\star,j_\ell^\star}^{(2)}(t)\left|\gamma_{i,j}^{(2)}(t)\right|\indy{i=i_\ell^\star}\\
        & \qqquad + \calO\paren{\frac{m^7}{\delta_{\mathbb{P}}^3d^{\frac{3}{2}}}} + \calO\paren{\frac{m^4}{\delta_{\mathbb{P}}^2d}}\gamma_{i_\ell^\star,j_\ell^\star}^{(2)}(t)\\
        & \leq \frac{18}{b_i^2}\sum_{k=0}^{\infty}\frac{c_k^2}{k!}\left|\gamma_{i,j}^{(1)}(t)\right|^k\left|\gamma_{i,j}^{(2)}(t)\right|^2+ \calO\paren{\frac{m^7}{\delta_{\mathbb{P}}^3d^{\frac{3}{2}}}} + \calO\paren{\frac{m^4}{\delta_{\mathbb{P}}^2d}}\gamma_{i_\ell^\star,j_\ell^\star}^{(2)}(t)\\
        & \leq \frac{18}{b_i^2}\paren{c_0^2 + \calO\paren{\frac{m^2}{\delta_{\mathbb{P}}\sqrt{d}}}}\left|\gamma_{i,j}^{(2)}(t)\right|^2+ \calO\paren{\frac{m^7}{\delta_{\mathbb{P}}^3d^{\frac{3}{2}}}} + \calO\paren{\frac{m^4}{\delta_{\mathbb{P}}^2d}}\gamma_{i_\ell^\star,j_\ell^\star}^{(2)}(t)
    \end{align*}
    where the last inequality is because $\left|\gamma_{i,j}^{(1)}(t)\right| \leq \varepsilon_1(t) \leq \calO\paren{\frac{m^2}{\delta_{\mathbb{P}}\sqrt{d}}}$ for $t\leq T_c$. For any $t\leq T\paren{\frac{\delta_{\mathbb{P}}^2}{d^{\frac{2}{5}}}}$, we must have that
    \[
        \calO\paren{\frac{m^4}{\delta_{\mathbb{P}}^2d}}\gamma^{(2)}_{i_\ell^\star,j_\ell^\star}(t) \leq \calO\paren{\frac{m^4}{d^{\frac{7}{5}}}}
    \]
    Therefore, for any $t\leq T\paren{\frac{\delta_{\mathbb{P}}^2}{d^{\frac{2}{5}}}}$, we have
    \[
        \derivt \left|\gamma_{i,j}^{(2)}(t)\right| \leq \frac{18}{b_i^2}\paren{c_0^2 + \calO\paren{\frac{m^2}{\delta_{\mathbb{P}}\sqrt{d}}}}\left|\gamma_{i,j}^{(2)}(t)\right|^2 + \calO\paren{\frac{m^4}{d^{\frac{7}{5}}} + \frac{m^7}{\delta_{\mathbb{P}}^3d^{\frac{3}{2}}}}
    \]
    
    Notice that $\left|\gamma_{i,j}^{(2)}(t)\right|$ must be upper bounded by $\hat{\gamma}_{i,j}^{(2)}(t)$ given by
    \[
        \derivt\hat{\gamma}_{i,j}^{(2)}(t) = \frac{18}{b_i^2}\paren{c_0^2 + \calO\paren{\frac{m^2}{\sqrt{d}}}}\hat{\gamma}_{i,j}^{(2)}(t)^2 + \calO\paren{\frac{m^4}{d^{\frac{7}{5}}} + \frac{m^7}{\delta_{\mathbb{P}}^3d^{\frac{3}{2}}}};\quad \hat{\gamma}_{i,j}^{(2)}(0) = \left|\gamma_{i,j}^{(2)}(0)\right|
    \]
    Observe that for any $t\leq T_c$, $\hat{\gamma}_{i,j}^{(2)}(t)$ grows monotonically as $\hat{\gamma}_{i,j}^{(2)}(0)$ grows. By the initialization property, we have that $\hat{\gamma}_{i,j}^{(2)}(0) = \left|\gamma_{i,j}^{(2)}(0)\right| \leq \frac{\sqrt{\beta_3}}{\sqrt{d}}\paren{\log\frac{m}{\delta_{\mathbb{P}}}}^{\frac{1}{2}}$. 
    We can observe that $\hat{\gamma}_{i,j}^{(2)}(t)$ increases as $\hat{\gamma}_{i,j}^{(2)}(t)$ becomes larger. Thus, it suffice to consider $\hat{\gamma}_{i,j}^{(2)}(0) = \frac{\sqrt{\beta_3}}{\sqrt{d}}\paren{\log\frac{m}{\delta_{\mathbb{P}}}}^{\frac{1}{2}}$. In this case, 
    \[
        \derivt \hat{\gamma}_{i,j}^{(2)}(t) \leq \frac{18}{b_i^2}\paren{c_0^2 + \calO\paren{\frac{m^4}{d^{\frac{2}{5}}} + \frac{m^7}{\delta_{\mathbb{P}}^3\sqrt{d}}}}\hat{\gamma}_{i,j}^{(2)}(t)^2
    \]
    Solving the differential equation gives that
    \[
        \left|\gamma_{i,j}^{(2)}(t)\right| \leq \hat{\gamma}_{i,j}^{(2)}(t) \leq \paren{\left|\gamma_{i,j}^{(2)}(0)\right|^{-1} - \frac{18}{b_i^2}\paren{c_0^2 + \calO\paren{\frac{m^4}{d^{\frac{2}{5}}} + \frac{m^7}{\delta_{\mathbb{P}}^3\sqrt{d}}}}t}^{-1}
    \]
    Recall that $(1+\delta_s)^2b_i^2\gamma_{i,j}^{(2)}(0)^2 \leq b_{i_\ell^\star}\gamma_{i_\ell^\star,j_\ell^\star}^{(2)}(0)^2$ by the initialization property. Thus, $t\leq T\paren{\frac{\delta_{\mathbb{P}}^2}{d^{\frac{2}{5}}}}$, it holds that
    \begin{align*}
        \left|\gamma_{i,j}^{(2)}(t)\right| & \leq \paren{\left|\gamma_{i,j}^{(2)}(0)\right|^{-1} - \frac{18}{b_i^2}\paren{c_0^2 + \calO\paren{\frac{m^4}{d^{\frac{2}{5}}} + \frac{m^7}{\delta_{\mathbb{P}}^3\sqrt{d}}}}T\paren{\frac{m^{\frac{3}{2}}}{d^{\frac{1}{4}}}}}^{-1}\\
        & \leq \paren{\left|\gamma_{i,j}^{(2)}(0)\right|^{-1} - \frac{b_{i_\ell^\star}^2}{b_i^2}\paren{c_0^2 + \calO\paren{\frac{m^4}{d^{\frac{2}{5}}} + \frac{m^7}{\delta_{\mathbb{P}}^3\sqrt{d}}}}\paren{c_0^2\paren{1-\frac{\delta_{\mathbb{P}}^2}{d^{\frac{2}{5}}}} - \calO\paren{\frac{m^7}{\delta_{\mathbb{P}}\sqrt{d}}}}^{-1}\gamma_{i_\ell^\star,j_\ell^\star}^{(2)}(0)^{-1}}^{-1}\\
        & \leq \paren{\left|\gamma_{i,j}^{(2)}(0)\right|^{-1} - \frac{b_{i_\ell^\star}^2}{b_i^2}\cdot \frac{c_0^2 + \calO\paren{\frac{m^4}{d^{\frac{2}{5}}} + \frac{m^7}{\delta_{\mathbb{P}}^3\sqrt{d}}}}{c_0^2 - \calO\paren{\frac{\delta_{\mathbb{P}^2}}{d^{\frac{2}{5}}} + \frac{m^7}{\delta_{\mathbb{P}}^3\sqrt{d}}}}\gamma_{i_\ell^\star,j_\ell^\star}^{(2)}(0)^{-1}}^{-1}\\
        & \leq \paren{\frac{b_i}{b_{i_\ell^\star}}(1+\delta_s)\gamma_{i_\ell^\star,j_\ell^\star}^{(2)}(0)^{-1} - \frac{b_{i_\ell^\star}^2}{b_i^2}\cdot \frac{c_0^2 + \calO\paren{\frac{m^4}{d^{\frac{2}{5}}} + \frac{m^7}{\delta_{\mathbb{P}}^3\sqrt{d}}}}{c_0^2 - \calO\paren{\frac{\delta_{\mathbb{P}^2}}{d^{\frac{2}{5}}} + \frac{m^7}{\delta_{\mathbb{P}}^3\sqrt{d}}}}\gamma_{i_\ell^\star,j_\ell^\star}^{(2)}(0)^{-1}}^{-1}\\
        & \leq \frac{b_{i_\ell^\star}}{b_i}\gamma_{i_\ell^\star,j_\ell^\star}^{(2)}(0)\paren{1+\delta_s - \frac{b_{i_\ell^\star}^3}{b_i^3}\cdot \frac{c_0^2 + \calO\paren{\frac{m^4}{d^{\frac{2}{5}}} + \frac{m^7}{\delta_{\mathbb{P}}^3\sqrt{d}}}}{c_0^2 - \calO\paren{\frac{\delta_{\mathbb{P}^2}}{d^{\frac{2}{5}}} + \frac{m^7}{\delta_{\mathbb{P}}^3\sqrt{d}}}}}^{-1}
    \end{align*}
    By the initialization property, we have that $b_i,b_{i_\ell^\star} \in [1-\beta_2\delta_s, 1+\beta_2\delta_s]$ for $\beta_2\leq o(1)$. Therefore, we have that
    \begin{align*}
         \frac{b_{i_\ell^\star}^3}{b_i^3}\cdot \frac{c_0^2 + \calO\paren{\frac{m^4}{d^{\frac{2}{5}}} + \frac{m^7}{\delta_{\mathbb{P}}^3\sqrt{d}}}}{c_0^2 - \calO\paren{\frac{\delta_{\mathbb{P}^2}}{d^{\frac{2}{5}}} + \frac{m^7}{\delta_{\mathbb{P}}^3\sqrt{d}}}} & \leq \paren{\frac{1+\beta_2\delta_s}{1 - \beta_2\delta_s}}^3\frac{c_0^2 + \calO\paren{\frac{m^4}{d^{\frac{2}{5}}} + \frac{m^7}{\delta_{\mathbb{P}}^3\sqrt{d}}}}{c_0^2 - \calO\paren{\frac{\delta_{\mathbb{P}^2}}{d^{\frac{2}{5}}} + \frac{m^7}{\delta_{\mathbb{P}}^3\sqrt{d}}}}\\
         & \leq \paren{\frac{1 + \beta_2\delta_s}{1 - \beta_2\delta_s}}^4\\
         & \leq 1 + \frac{1}{2}\delta_s
    \end{align*}
    where the second inequality is due to $d\geq \beta_5m^{16}$ and the last inequality due to $\beta_2\leq o(1)$. Therefore
    \begin{align*}
        \left|\gamma_{i,j}^{(2)}(t)\right| & \leq \frac{b_{i_\ell^\star}}{b_i}\gamma_{i_\ell^\star,j_\ell^\star}^{(2)}(0)\cdot \frac{2}{\delta_s} \leq \calO\paren{\frac{m^2}{\delta_{\mathbb{P}}\sqrt{d}}}
    \end{align*}
    Let $T_1 = T\paren{\frac{m^{\frac{3}{2}}}{d^{\frac{1}{4}}}}$. Then for $t\geq T_1$, the dynamic of $\gamma_{i,j}^{(2)}(t)$ is upper bounded by
    \begin{align*}
        \derivt\hat{\gamma}_{i,j}^{(2)}(t) \leq \frac{18}{b_i^2}\paren{c_0^2 + \calO\paren{\frac{m^2}{\sqrt{d}}}}\left|\gamma_{i,j}^{(2)}(t)\right|^2 + \calO\paren{\frac{m^4}{\delta_{\mathbb{P}}^2d}}
    \end{align*}
    Let $T_2$ be the smallest $t \geq 0$ such that $\hat{\gamma}_{i,j}^{(2)}(t) \geq2\hat{\gamma}_{i,j}^{(2)}(T_1)$. Then for all $t\leq T_2$, we have $\derivt \hat{\gamma}_{i,j}^{(2)}(t) \leq \calO\paren{\frac{m^4}{d}}$. Therefore
    \begin{align*}
        \hat{\gamma}_{i,j}^{(2)}(t+T_1) \leq \hat{\gamma}_{i,j}^{(2)}(T_1) + t \cdot \max_{t\leq T_2}\derivt \hat{\gamma}_{i,j}^{(2)}(t) \leq \hat{\gamma}_{i,j}^{(2)}(T_1) + \calO\paren{\frac{m^4}{\delta_{\mathbb{P}}^2d}}t 
    \end{align*}
    Thus, we must have that $T_2 \geq \calO\paren{\frac{\delta_{\mathbb{P}}\sqrt{d}}{m^2}} + T_1$. Therefore, we can conclude that there exists some constant $\beta_7>0$ such that for all $t\leq T\paren{\frac{\delta_{\mathbb{P}}^2}{d^{\frac{2}{5}}}} + \frac{\beta_7\delta_{\mathbb{P}}\sqrt{d}}{m^2}$ and $t\leq T_c$, we have that
    \[
        \left|\gamma_{i,j}^{(2)}(t)\right| \leq O\paren{\frac{m^2}{\delta_{\mathbb{P}}\sqrt{d}}}
    \]
    for all $(i,j)\neq (i_\ell^\star,j_\ell^\star)$.
\end{proof}

\textbf{Upper bounding $\gamma_{i,j}^{(2)}(t)$ for $i\in[m]\setminus \mathcal{R}_{\ell-1},j\in\mathcal{C}_{\ell-1}$.} In this section, we show that the alignment of $\bbfw_i$ with previously recovered components $\bbfw_j^\star$ must be small.
\begin{lemma}
    \label{lem:gamma_2_prev}
    Suppose that the inductive hypothesis in Condition~\ref{cond:inductive_hypo} and the initialization condition in Condition~\ref{cond:init} holds. Let $T_c$ be defined in (\ref{eq: T_c_def}) and $T_{\ell}(\xi)$ in Definition~\ref{def:recovery_time}. Then we have that for all $t\leq \min\left\{T_c, T_{\ell}\paren{\xi} + \calO\paren{\frac{\delta_{\mathbb{P}}\sqrt{d}}{m^2}}, \calO\paren{\frac{\delta_{\mathbb{P}}^2d^{\frac{3}{4}}}{m^5}}\right\}$ for any $\xi$ such that $(1 - \xi)^{-1}\leq \calO\paren{1}$,  it holds that 
    \[
        \left|\gamma_{i,j_{\ell'}^\star}^{(2)}(t)\right|\leq \calO\paren{\frac{m^2}{\delta_{\mathbb{P}}\sqrt{d}}};\;\forall i\in[m]\setminus\mathcal{R}_{\ell-1},\ell'< \ell
    \]
\end{lemma}
\begin{proof}
    By the inductive hypothesis, we have that $\gamma^{(2)}_{i_{\ell'}^\star,j_{\ell'}^\star}(t)\geq 1 - \calO\paren{\frac{m^7}{\delta_{\mathbb{P}}^3d^{\frac{3}{2}}}}$.
    Fix any $i\in[m]\setminus \mathcal{R}_{\ell-1}$. By Lemma~\ref{lem:dynamic_approx}, we have that
    \begin{align*}
        \derivt \gamma_{i,j_{\ell}^\star}^{(2)}(t) & = -\frac{1}{b_i}\nabla_{\bfw_i}\calL\paren{\bm{\theta}(t)}^\top\bbfw_{j_{\ell'}^\star}^\star\\
        & = \frac{9}{b_i^2}\paren{\hat{\lambda}_{i,j_{\ell'}^\star,5}(t) - \lambda_{i,i_{\ell'}^\star,5}(t)\gamma_{i_{\ell'}^\star,j_{\ell'}^\star}^{(2)}(t)} - \frac{9}{b_i^2}\hat{\lambda}_{i_{\ell}^\star,j_{\ell}^\star,5}(t)\gamma_{i_{\ell}^\star,j_{\ell}^\star}^{(2)}(t)\gamma_{i,j_{\ell'}^\star}^{(2)}(t)\indy{i=i_{\ell}^\star}\\
        & \qqquad \pm \calO\paren{\aggepst{\ell}(t)^2}\gamma_{i_{\ell}^\star,j_{\ell}^\star}^{(2)}(t)^2 \pm \calO\paren{m\aggepst{\ell}(t)^3 + \aggepst{\ell}(t)\varepsilon_{5,\ell}(t) + \aggepst{\ell}(t)\backepsi{\ell}(t)}
    \end{align*}
    Due to the same reasoning as in the previous lemma, we have that $\hat{\lambda}_{i_{\ell}^\star,j_{\ell}^\star,5}(t)\gamma_{i_{\ell}^\star,j_{\ell}^\star}^{(2)}(t) \geq 0$. Therefore, for all $t\leq T_c$, we have that
    \begin{align*}
        \derivt \left|\gamma_{i,j_{\ell'}^\star}^{(2)}(t)\right| & \leq \frac{9}{b_i^2}\left|\hat{\lambda}_{i,j_{\ell'}^\star,5}(t) - \lambda_{i,i_{\ell}^\star,5}(t)\gamma_{i_{\ell'}^\star,j_{\ell'}^\star}^{(2)}(t)\right| + \calO\paren{\frac{m^4}{\delta_{\mathbb{P}}^2d}}\gamma_{i_{\ell}^\star,j_{\ell}^\star}^{(2)}(t)^2 + \calO\paren{\frac{m^7}{\delta_{\mathbb{P}}^3d^{\frac{3}{2}}}}\\
        & \leq \frac{9c_0^2}{b_i^2}\left|\gamma_{i,j_{\ell'}}^{(2)}(t)^2 - I_{i,i_{\ell'}^\star}^{(2)}(t)^2\gamma_{i_{\ell'}^\star,j_{\ell'}^\star}^{(2)}(t)\right| + \calO\paren{\frac{m^4}{\delta_{\mathbb{P}}^2d}}\gamma_{i_{\ell}^\star,j_{\ell}^\star}^{(2)}(t)^2 + \calO\paren{\frac{m^7}{\delta_{\mathbb{P}}^3d^{\frac{3}{2}}}}\\
        & \leq \frac{9c_0^2}{b_i^2}\paren{1 - \gamma_{i_{\ell'}^\star,j_{\ell'}^\star}^{(2)}(t)}\gamma_{i,j_{\ell'}}^{(2)}(t)^2 + \frac{9c_0^2}{b_i^2}\gamma_{i_{\ell'}^\star,j_{\ell'}^\star}^{(2)}(t)\left|\gamma_{i,j_{\ell'}}^{(2)}(t)^2 - I_{i,i_{\ell'}^\star}^{(2)}(t)^2\right|\\
        & \qqquad  + \calO\paren{\frac{m^4}{\delta_{\mathbb{P}}^2d}}\gamma_{i_{\ell}^\star,j_{\ell}^\star}^{(2)}(t)^2 + \calO\paren{\frac{m^7}{\delta_{\mathbb{P}}^3d^{\frac{3}{2}}}}
    \end{align*}
    Recall that $\gamma_{i,j_{\ell'}^\star}^{(2)}(t) = \bbfv_i^\top\bbfv_{j_{\ell'}^\star}^\star$ and $I_{i,i_{\ell'}^\star} = \bbfv_i^\top\bbfv_{i_{\ell'}^\star}$. Therefore,
    \begin{align*}
        \left|\gamma_{i,j_{\ell'}^\star}^{(2)}(t)^2 - I_{i,i_{\ell'}^\star}^{(2)}(t)^2\right| & = \left|\gamma_{i,j_{\ell'}^\star}^{(2)}(t) +  I_{i,i_{\ell'}^\star}^{(2)}(t)\right|\cdot \left|\gamma_{i,j_{\ell'}^\star}^{(2)}(t) - I_{i,i_{\ell'}^\star}^{(2)}(t)\right| \\
        & \leq 2\left|\gamma_{i,j_{\ell'}^\star}^{(2)}(t)\right|\cdot \left|\gamma_{i,j_{\ell'}}^{(2)}(t) - I_{i,i_{\ell'}^\star}^{(2)}(t)\right| + \paren{\gamma_{i,j_{\ell'}}^{(2)}(t) - I_{i,i_{\ell'}^\star}^{(2)}(t)}^2\\
        & \leq 2\left|\gamma_{i,j_{\ell'}^\star}^{(2)}(t)\right|\cdot \left|\bbfv_i^\top\paren{\bbfv_{j_{\ell'}^\star}^\star - \bbfv_{i_{\ell'}^\star}}\right| + \paren{\bbfv_i^\top\paren{\bbfv_{j_{\ell'}^\star}^\star - \bbfv_{i_{\ell'}^\star}}}^2\\
        & \leq 2\left|\gamma_{i,j_{\ell'}^\star}^{(2)}(t)\right| \cdot \norm{\bbfv_{j_{\ell'}^\star}^\star - \bbfv_{i_{\ell'}^\star}}_2 + \norm{\bbfv_{j_{\ell'}^\star}^\star - \bbfv_{i_{\ell'}^\star}}_2^2\\
        & \leq \calO\paren{\frac{m^4}{\delta_{\mathbb{P}}^2d^{\frac{3}{4}}}}\left|\gamma_{i,j_{\ell'}^\star}^{(2)}(t)\right| + \calO\paren{\frac{m^7}{\delta_{\mathbb{P}}^3d^{\frac{3}{2}}}}
    \end{align*}
    Moreover, since $\gamma_{i_{\ell'}^\star,j_{\ell'}^\star}^{(2)}(t)\geq 1 - \calO\paren{\frac{m^7}{\delta_{\mathbb{P}}^3d^{\frac{3}{2}}}}$, for $T_{p,\ell-1} \leq t \leq T_c$, we must have that
    \begin{align*}
        \derivt \left|\gamma_{i,j_{\ell'}^\star}^{(2)}(t)\right| & \leq \calO\paren{\frac{m^4}{\delta_{\mathbb{P}}^2d^{\frac{3}{4}}}}\left|\gamma_{i,j_{\ell'}^\star}^{(2)}(t)\right| + \calO\paren{\frac{m^4}{\delta_{\mathbb{P}}^2d}}\gamma_{i_{\ell}^\star,j_{\ell}^\star}^{(2)}(t)^2 + \calO\paren{\frac{m^7}{\delta_{\mathbb{P}}^3d^{\frac{3}{2}}}}\\
        & \leq \calO\paren{\frac{m^4}{\delta_{\mathbb{P}}^2d^{\frac{3}{4}}}}\left|\gamma_{i,j_{\ell'}^\star}^{(2)}(t)\right| + \calO\paren{\frac{m^4}{\delta_{\mathbb{P}}^2d}}\gamma_{i_{\ell}^\star,j_{\ell}^\star}^{(2)}(t)^2 + \calO\paren{\frac{m^7}{\delta_{\mathbb{P}}^3d^{\frac{3}{2}}}}\\
        & \leq \calO\paren{\frac{m^4}{\delta_{\mathbb{P}}^2d}}\frac{d}{dt}\gamma_{i_{\ell}^\star,j_{\ell}^\star}^{(2)}(t) + \calO\paren{\frac{m^7}{\delta_{\mathbb{P}}^3d^{\frac{5}{4}}}}
    \end{align*}
    where the last inequality follows from Lemma~\ref{lem:gamma_2_growth}. This gives that
    \begin{align*}
        \left|\gamma_{i,j_{\ell'}^\star}^{(2)}(t+T_{p,\ell-1})\right| & \leq \left|\gamma_{i,j_{\ell'}^\star}^{(2)}(T_{p,\ell-1})\right| + \calO\paren{\frac{m^4}{\delta_{\mathbb{P}}^2d}}\gamma_{i_{\ell}^\star,j_{\ell}^\star}^{(2)}(t) + \calO\paren{\frac{m^7t}{\delta_{\mathbb{P}}^3d^{\frac{3}{2}}}}\\
        & \leq \left|\gamma_{i,j_{\ell'}^\star}^{(2)}(T_{p,\ell-1})\right| + \calO\paren{\frac{m^4}{\delta_{\mathbb{P}}^2d} + \frac{m^7t}{\delta_{\mathbb{P}}^3d^{\frac{5}{4}}}}\\
        & \leq \calO\paren{\frac{m^2}{\delta_{\mathbb{P}}\sqrt{d}} + \frac{m^7t}{\delta_{\mathbb{P}}^3d^{\frac{5}{4}}}}
    \end{align*}
    where the last step follows from the inductive hypothesis. Further requiring that $t\leq \calO\paren{\frac{\delta_{\mathbb{P}}^2d^{\frac{3}{4}}}{m^5}}$ keeps $\left|\gamma_{i,j_{\ell'}^\star}^{(2)}(t+T_{p,\ell-1})\right| \leq \calO\paren{\frac{m^2}{\delta_{\mathbb{P}}\sqrt{d}}}$. When $\gamma_{i_{\ell}^\star,j_{\ell}^\star}^{(2)}(t) \geq \xi$, we can still obtain that
    \[
        \derivt \left|\gamma_{i,j_{\ell'}^\star}^{(2)}(t)\right| \leq \calO\paren{\frac{m^4}{\delta_{\mathbb{P}}^2d}}
    \]
    Thus, for all $t\leq \min\left\{T_{\ell}\paren{\xi} + \calO\paren{\frac{\delta_{\mathbb{P}}\sqrt{d}}{m^2}}, \calO\paren{\frac{\delta_{\mathbb{P}}^2d^{\frac{3}{4}}}{m^5}}\right\}$ we can guarantee that $\left|\gamma_{i,j_{\ell'}^\star}^{(2)}(t+T_{p,\ell-1})\right|\leq \calO\paren{\frac{m^2}{\delta_{\mathbb{P}}\sqrt{d}}}$.
\end{proof}

\textbf{Bounding $\varepsilon_{1,\ell}(t), \varepsilon_{2,\ell}(t)$, and $\backepso{\ell}(t)$.} Here we are going to upper bound $\varepsilon_{1,\ell}(t), \varepsilon_{2,\ell}(t)$ and $\backepso{\ell}(t)$. In particular, we are going to analyze $\gamma_{i,j}^{(1)}(t)$ for $i\in[m]\setminus \mathcal{R}_{\ell-1}$ and $j\in[m^\star]$ where $(i,j)\neq i_{\ell}^\star, j_{\ell}^\star$,  $\zeta_{i,j}^{(1)}, \zeta_{i,j}^{(2)}$ for $i\in[m]\setminus \mathcal{R}_{\ell-1}$ and $j\in[m^\star]$, and also $I_{i,j}^{(1)}, I_{i,j}^{(2)}, I_{i,j}^{(3)}$ for $i,j\in[m]$ and $i\neq j$.

\begin{lemma}
    \label{lem:eps_1_ub}
    Suppose that the inductive hypothesis in Condition~\ref{cond:inductive_hypo} and the initialization condition in Condition~\ref{cond:init} holds. Let $T_c$ be defined in (\ref{eq: T_c_def}) and $T(\xi)$ in Definition~\ref{def:recovery_time}. Then there exists some constant $\beta_7 > 0$ such that for all $\beta_8 \leq \calO(1)$, for all $t\leq \min\left\{T_c, T\paren{\frac{\delta_{\mathbb{P}}^2}{d^{\frac{2}{5}}}} + \frac{\beta_7\delta_{\mathbb{P}}\sqrt{d}}{m}, T\paren{\xi}\right\} + \beta_8$, we have that $\aggepst{\ell} \leq \calO\paren{\frac{m^2}{\delta_{\mathbb{P}}\sqrt{d}}}$. Moreover, for all $t\leq \min\left\{\calO\paren{\frac{\delta_{\mathbb{P}}^2d}{m^5}}, T_c\right\}$ we shall have that
    \begin{align*}
        \max\left\{\left|\gamma_{i,j}^{(1)}(t)\right|,\left|\zeta_{i,j}^{(1)}(t)\right|,\left|\zeta_{i,j}^{(2)}(t)\right|\right\} & \leq \calO\paren{\frac{m^2}{\delta_{\mathbb{P}}\sqrt{d}}};\;\forall i\in[m]\setminus \mathcal{R}_{\ell},j\in[m^\star]\\
        \max\left\{\left|I_{i,j}^{(1)}(t)\right|,\left|I_{i,j}^{(2)}(t)\right|,\left|I_{i,j}^{(3)}(t)\right|\right\} & \leq \calO\paren{\frac{m^2}{\delta_{\mathbb{P}}\sqrt{d}}};\;\forall i,j\in[m]\setminus \mathcal{R}_{\ell}, i\neq j
    \end{align*}
\end{lemma}
\begin{proof}
    Throughout the proof, we will relax the upper bound in terms of $\aggepso{\ell}(t)$ into the upper bound in terms of $\aggepst{\ell}(t)$.
    
    \textbf{Bounding $\zeta_{i,j}^{(1)}(t),\zeta_{i,j}^{(2)}(t)$.} First, we are going to derive some rough estimation for $\zeta_{i,j}^{(1)}(t),\zeta_{i,j}^{(1)}(t)$ for $i\in[m]\setminus \mathcal{R}_{\ell-1}, j\in[m^\star]$. By Lemma~\ref{lem:dynamic_approx}, we have that
    \begin{align*}
        \derivt\left|\zeta_{i,j}^{(1)}(t)\right| & = -\frac{1}{a_i}\nabla_{\bfv_i}\calL\paren{\bm{\theta}(t)}^\top\bbfw_j^\star\cdot \text{sign}\paren{\zeta_{i,j}^{(1)}(t)}\\
        & \leq \frac{3}{a_i^2}\left|\hat{\lambda}_{i_\ell^\star,j_\ell^\star,4}\right| + \frac{1}{a_i^2}\left|\hat{\lambda}_{i_{\ell}^\star,j_{\ell}^\star,2}(t)\gamma_{i_{\ell}^\star,j_{\ell}^\star}^{(2)}(t)\right| - \hat{\lambda}_{i_{\ell}^\star,j_\ell^\star,1}(t)\gamma_{i_{\ell}^\star,j_{\ell}^\star}^{(1)}(t)\left|\zeta_{i_{\ell}^\star,j_{\ell}^\star}^{(1)}(t)\right|\indy{i=i_\ell^\star}\\
        & \qqquad + \frac{36}{a_i^2}\left|C_{S,2}\gamma_{i,j}^{(2)}(t)I_{i,i}^{(3)}\right| + \calO\paren{m\aggepst{\ell}(t)^3 + \backepsi{\ell}(t)^2 + \varepsilon_{5,\ell}(t)^2}
    \end{align*}
    Diving into the details of the first three terms, we have that
    \begin{align*}
        \left|\hat{\lambda}_{i_\ell^\star,j_\ell^\star,4}(t)\right| & \leq 6\left|I_{i_\ell^\star,i_\ell^\star}^{(3)}(t)\right|\gamma_{i_\ell^\star,j_\ell^\star}^{(2)}(t)^2\sum_{k=0}^{\infty}\frac{c_{k+2}c_k}{k!}\left|\gamma_{i_\ell^\star,j_\ell^\star}^{(1)}(t)\right|^k\\
        & \qqquad + 6\left|\zeta_{i,j}^{(2)}(t)\right|\gamma_{i_\ell^\star,j_\ell^\star}^{(2)}(t)^2\sum_{k=0}^{\infty}\frac{c_{k+1}^2}{k!}\left|\gamma_{i_\ell^\star,j_\ell^\star}^{(1)}(t)\right|^k + \calO\paren{\aggepst{\ell}(t)^3}\\
        & \leq \calO\paren{\backepso{\ell}(t) + \varepsilon_{5,\ell}(t)}\gamma_{i_\ell^\star,j_\ell^\star}^{(2)}(t)^2+ \calO\paren{\aggepst{\ell}(t)^3}\\
        \left|\hat{\lambda}_{i_\ell^\star,j_\ell^\star,2}(t)\gamma_{i_\ell^\star,j_\ell^\star}^{(2)}(t)\right| & \leq 6\left|\zeta_{i,j}^{(1)}(t)\right|\gamma_{i_\ell^\star,j_\ell^\star}^{(2)}(t)^3\sum_{k=0}^{\infty}\frac{c_kc_{k+2}}{k!}\left|\gamma_{i_\ell^\star,j_{\ell^\star}}^{(1)}(t)\right|^k + \calO\paren{\aggepst{\ell}(t)^3}\\
        & \leq \calO\paren{\backepso{\ell}(t)}\gamma_{i_\ell^\star,j_\ell^\star}^{(2)}(t)^2 + \calO\paren{\aggepst{\ell}(t)^3}\\
        - \hat{\lambda}_{i_\ell^\star,j_\ell^\star,1}(t)\gamma_{i_\ell^\star,j_\ell^\star}^{(1)}(t)\left|\zeta_{i,j}^{(1)}(t)\right| & \leq 6\left|\zeta_{i,j}^{(1)}(t)\right|\gamma_{i_\ell^\star,j_\ell^\star}^{(2)}(t)^3\sum_{k=0}^{\infty}\frac{c_{k+1}^2}{k!}\left|\gamma_{i_\ell^\star,j_\ell^\star}^{(1)}(t)\right|^k\\
        & \qqquad + \calO\paren{\aggepst{\ell}(t)^2}\gamma_{i_\ell^\star,j_\ell^\star}^{(2)}(t)^2 + \calO\paren{\aggepst{\ell}(t)^3}\\
        & \leq \calO\paren{\aggepst{\ell}(t)^2+ \backepso{\ell}(t)}\gamma_{i_\ell^\star,j_\ell^\star}^{(2)}(t)^2 + \calO\paren{\aggepst{\ell}(t)^3}
    \end{align*}
    Also, we notice that $\left|\gamma_{i,j}^{(2)}(t)\right| \leq \calO\paren{\varepsilon_{5,\ell}(t) + \backepsi{\ell}(t)}$. Applying the fact that $a_i^{-1}\leq \calO\paren{1}$, we have that
    \begin{align*}
        \derivt \left|\zeta_{i,j}^{(1)}(t)\right| & \leq \calO\paren{\aggepst{\ell}(t)^2+ \backepso{\ell}(t)}\gamma_{i_\ell^\star,j_\ell^\star}^{(2)}(t)^2 + \calO\paren{m\aggepst{\ell}(t)^3 + \backepsi{\ell}(t) + \varepsilon_{5,\ell}(t)}\\
        & \leq \calO\paren{\aggepst{\ell}(t)^2+ \backepso{\ell}(t)}\gamma_{i_\ell^\star,j_\ell^\star}^{(2)}(t)^2 + \calO\paren{m\aggepst{\ell}(t)^3}
    \end{align*}
    Similarly, for $\zeta_{i,j}^{(2)}(t)$, by Lemma~\ref{lem:dynamic_approx}, we have that
    \begin{align*}
        \derivt\left|\zeta_{i,j}^{(2)}(t)\right| & = -\frac{1}{b_i}\nabla_{\bfw_i}\calL\paren{\bm{\theta}(t)}^\top\bbfv_j^\star\cdot \text{sign}\paren{\zeta_{i,j}^{(2)}(t)}\\
        & \leq \frac{3}{b_i^2}\left|\hat{\lambda}_{i_\ell^\star,j_\ell^\star,3}(t)\right| + \frac{3}{b_i^2}\left|\hat{\lambda}_{i_\ell^\star,j_\ell^\star,2}(t)\gamma_{i_\ell^\star,j_\ell^\star}^{(1)}(t)\right| - 3\hat{\lambda}_{i_\ell^\star,j_\ell^\star,5}(t)\gamma_{i_\ell^\star,j_\ell^\star}^{(2)}(t)\left|\zeta_{i,j}^{(2)}(t)\right|\indy{i=i_\ell^\star}\\
        & \qqquad + \frac{36}{b_i^2}\left|C_{S,2}\gamma_{i,j}^{(1)}(t)I_{i,i}^{(3)}(t)\right| + \calO\paren{m\aggepst{\ell}(t)^3 + \backepsi{\ell}(t)^2 + \varepsilon_{5,\ell}(t)^2}
    \end{align*}
    The third term is apparently negative due to the fact that $\hat{\lambda}_{i_\ell^\star,j_\ell^\star,5}(t) \geq 0, \gamma_{i_\ell^\star,j_\ell^\star}^{(2)}(t)\geq 0$. For the first two, we have that
    \begin{align*}
        \left|\hat{\lambda}_{i_\ell^\star,j_\ell^\star,3}(t)\right| & \leq 6\left|\zeta_{i,j}^{(1)}(t)\right|\gamma_{i_\ell^\star,j_\ell^\star}^{(2)}(t)^2\sum_{k=0}^{\infty}\frac{c_{k+1}^2}{k!}\left|\gamma_{i_\ell^\star,j_\ell^\star}^{(1)}(t)\right|^k + \calO\paren{\aggepst{\ell}(t)^3}\\
        & \leq \calO\paren{\backepso{\ell}(t)}\gamma_{i_\ell^\star,j_\ell^\star}^{(2)}(t)^2+ \calO\paren{\aggepst{\ell}(t)^3}\\
        \left|\hat{\lambda}_{i_\ell^\star,j_\ell^\star,2}(t)\gamma_{i_\ell^\star,j_\ell^\star}^{(1)}(t)\right| & \leq 6\left|\zeta_{i,j}^{(1)}(t)\right|\gamma_{i_\ell^\star,j_\ell^\star}^{(2)}(t)^2\sum_{k=0}^{\infty}\frac{c_{k+2}c_k}{k!}\left|\gamma_{i_\ell^\star,j_\ell^\star}^{(1)}(t)\right|^{k+1} + \calO\paren{\aggepst{\ell}(t)^3}\\
        & \leq \calO\paren{\backepso{\ell}(t)}\gamma_{i_\ell^\star,j_\ell^\star}^{(2)}(t)^2+ \calO\paren{\aggepst{\ell}(t)^3}
    \end{align*}
    As before, we have that $\left|I_{i,i}^{(3)}(t)\right|\leq \calO\paren{\backepsi{\ell}(t) + \varepsilon_{5,\ell}(t)}$ Therefore, we have that for $t \leq T_c$,
    \begin{align*}
        \derivt \left|\zeta_{i,j}^{(2)}(t)\right| \leq \calO\paren{\backepso{\ell}(t)}\gamma_{i_\ell^\star,j_\ell^\star}^{(2)}(t)^2+ \calO\paren{\aggepst{\ell}(t)^3}
    \end{align*}
    \textbf{Bounding $I_{i,j}^{(1)}(t), I_{i,j}^{(2)}(t)$, and $I_{i,j}^{(3)}(t)$ for $i,j\in[m]\setminus \mathcal{R}_{\ell-1}$, and $i\neq j$.} Next, we will look at $I_{i,j}^{(1)}$. By Lemma~\ref{lem:dynamic_approx}, we have that
    \begin{align*}
        \derivt \left|I_{i,j}^{(1)}(t)\right| & = -\text{sign}\paren{I_{i,j}^{(1)}(t)}\paren{\frac{1}{a_i}\nabla_{\bfv_i}\calL\paren{\bm{\theta}(t)}^\top\bbfv_j + \frac{1}{a_j}\nabla_{\bfv_j}\calL\paren{\bm{\theta}(t)}^\top\bbfv_i}\\
        & \leq \frac{1}{a_i^2}\left|\hat{\lambda}_{i_\ell^\star,j_\ell^\star,1}(t)\right|\paren{\left|\gamma_{i,j_\ell^\star}^{(1)}(t)\right| + \left|\gamma_{j,j_\ell^\star}^{(1)}(t)\right|} - \frac{1}{a_i^2}\hat{\lambda}_{i_\ell^\star,j_\ell^\star,1}(t)\gamma_{i_\ell^\star,j_\ell^\star}^{(1)}(t)\left|I_{i,j}^{(1)}(t)\right|\indy{i=i_\ell^\star\vee j = i_\ell^\star}\\
        & \qqquad + \calO\paren{\aggepst{\ell}(t)^2}\gamma_{i_\ell^\star,j_\ell^\star}^{(2)}(t)^2 + \calO\paren{m\aggepst{\ell}(t)^3 + \aggepst{\ell}(t)\varepsilon_{5,\ell}(t) + \aggepst{\ell}(t)\backepsi{\ell}(t)}
    \end{align*}
    Since $i\neq j$, for the first term we have that
    \begin{align*}
        \left|\hat{\lambda}_{i_\ell^\star,j_\ell^\star,1}(t)\right|\paren{\left|\gamma_{i,j_\ell^\star}^{(1)}(t)\right| + \left|\gamma_{j,j_\ell^\star}^{(1)}(t)\right|} & \leq \calO\paren{\aggepst{\ell}(t)}\left|\hat{\lambda}_{i_\ell^\star,j_\ell^\star,1}(t)\right|\\
        & \leq \calO\paren{\aggepst{\ell}(t)}\gamma_{i_\ell^\star,j_\ell^\star}^{(2)}(t)^2 + \calO\paren{\aggepst{\ell}(t)^3}
    \end{align*}
    For the second term, we have that
    \begin{align*}
        - \hat{\lambda}_{i_\ell^\star,j_\ell^\star,1}\gamma_{i_\ell^\star,j_\ell^\star}^{(1)}(t)\left|I_{i,j}^{(1)}(t)\right| & \leq 6\left|I_{i,j}^{(1)}(t)\right|\gamma_{i_\ell^\star,j_\ell^\star}^{(2)}(t)^3\sum_{k=0}^{\infty}\frac{c_{k+1}^2}{k!}\left|\gamma_{i_\ell^\star,j_\ell^\star}^{(1)}(t)\right|^k\\
        & \qqquad + \calO\paren{\aggepst{\ell}(t)^2}\gamma_{i_\ell^\star,j_\ell^\star}^{(2)}(t)^2 + \calO\paren{\aggepst{\ell}(t)^3}\\
        & \leq \calO\paren{\aggepst{\ell}(t)}\gamma_{i_\ell^\star,j_\ell^\star}^{(2)}(t)^2 + \calO\paren{\aggepst{\ell}(t)^3}
    \end{align*}
    Therefore, we have that
    \[
        \derivt \left|I_{i,j}^{(1)}(t)\right| \leq \calO\paren{\aggepst{\ell}(t)}\gamma_{i_\ell^\star,j_\ell^\star}^{(2)}(t)^2 + \calO\paren{m\aggepst{\ell}(t)^3}
    \]
    For $I_{i,j}^{(2)}(t)$, we have that
    \begin{align*}
        \derivt \left|I_{i,j}^{(2)}(t)\right| & = -\text{sign}\paren{I_{i,j}^{(2)}(t)}\paren{\frac{1}{b_i}\nabla_{\bfw_i}\calL\paren{\bm{\theta}(t)}^\top\bbfw_j + \frac{1}{b_j}\nabla_{\bfw_j}\calL\paren{\bm{\theta}(t)}}\\
        & \leq \frac{9}{b_i^2}\left|\hat{\lambda}_{i_\ell^\star,j_\ell^\star,5}(t)\right|\paren{\left|\gamma_{i,j_\ell^\star}^{(2)}(t)\right| + \left|\gamma_{j,j_\ell^\star}^{(2)}(t)\right|} - \frac{9}{b_i^2}\hat{\lambda}_{i_\ell^\star,j_\ell^\star,5}(t)\gamma_{i_\ell^\star,j_\ell^\star}^{(2)}(t)\left|I_{i,j}^{(2)}(t)\right|\indy{i=i_\ell^\star\vee j = i_\ell^\star}\\
        & \qqquad + \calO\paren{\aggepst{\ell}(t)}\gamma_{i_\ell^\star,j_\ell^\star}^{(2)}(t)^2 + \calO\paren{m\aggepst{\ell}(t)^3}\\
        & \leq \frac{9}{b_i^2}\left|\hat{\lambda}_{i_\ell^\star,j_\ell^\star,5}(t)\right|\paren{\left|\gamma_{i,j_\ell^\star}^{(2)}(t)\right| + \left|\gamma_{j,j_\ell^\star}^{(2)}(t)\right|} + \calO\paren{\aggepst{\ell}(t)}\gamma_{i_\ell^\star,j_\ell^\star}^{(2)}(t)^2 + \calO\paren{m\aggepst{\ell}(t)^3}
    \end{align*}
    since $\lambda_{i_\ell^\star,j_\ell^\star,5}(t)\gamma_{i_\ell^\star,j_\ell^\star}^{(2)}(t)\geq 0$. For the first term, we have that
    \[
        \left|\hat{\lambda}_{i_\ell^\star,j_\ell^\star,5}(t)\right|\paren{\left|\gamma_{i,j_\ell^\star}^{(2)}(t)\right| + \left|\gamma_{j,j_\ell^\star}^{(2)}(t)\right|} \leq \calO\paren{\aggepst{\ell}(t)}\left|\hat{\lambda}_{i_\ell^\star,j_\ell^\star,5}(t)\right| \leq \calO\paren{\aggepst{\ell}(t)}\gamma_{i_\ell^\star,j_\ell^\star}^{(2)}(t)^2 + \calO\paren{\aggepst{\ell}(t)^3}
    \]
    Therefore, we have that
    \[
        \derivt \left|I_{i,j}^{(2)}(t)\right| \leq \calO\paren{\aggepst{\ell}(t)}\gamma_{i_\ell^\star,j_\ell^\star}^{(2)}(t)^2 + \calO\paren{m\aggepst{\ell}(t)^3}
    \]
    For $I_{i,j}^{(3)}$, by Lemma~\ref{lem:dynamic_approx}, we also have that
    \begin{align*}
        \derivt \left|I_{i,j}^{(3)}(t)\right| & = -\text{sign}\paren{\frac{1}{a_i}\nabla_{\bfv_i}\calL\paren{\bm{\theta}(t)}^\top\bbfw_j +\frac{1}{b_j}\nabla_{\bfw_j}\calL\paren{\bm{\theta}(t)}^\top\bbfv_i} \\
        & \leq \frac{1}{a_i^2}\left|\hat{\lambda}_{i_\ell^\star,j_\ell^\star,1}(t)\zeta_{j,j_\ell^\star}^{(2)}(t)\right| - \frac{1}{a_i^2}\hat{\lambda}_{i_\ell^\star,j_\ell^\star,1}(t)\gamma_{i_\ell^\star,j_\ell^\star}^{(1)}(t)\left|I_{i,j}^{(3)}(t)\right|\indy{i=i_\ell^\star}\\
        & \qqquad + \frac{9}{b_j^2}\left|\hat{\lambda}_{i_\ell^\star,j_\ell^\star,5}(t)\zeta_{i,j_\ell^\star}^{(1)}(t)\right| - \frac{9}{b_j^2}\hat{\lambda}_{i_\ell^\star,j_\ell^\star,5}(t)\gamma_{i_\ell^\star,j_\ell^\star}^{(2)}(t)\left|I_{i,j}^{(3)}(t)\right|\indy{j=i_\ell^\star}\\
        & \qqquad + \calO\paren{\aggepst{\ell}(t)^2}\gamma_{i_\ell^\star,j_\ell^\star}^{(2)}(t) + \calO\paren{m\aggepst{\ell}(t)^3}
    \end{align*}
    where for the first, second, and third term, we have
    \begin{align*}
        \left|\hat{\lambda}_{i_\ell^\star,j_\ell^\star,1}(t)\zeta_{j,j_\ell^\star}^{(2)}(t)\right| & \leq \calO\paren{\aggepst{\ell}(t)}\gamma_{i_\ell^\star,j_\ell^\star}^{(2)}(t)^2 +\calO\paren{\aggepst{\ell}(t)^3}\\
        \left|\hat{\lambda}_{i_\ell^\star,j_\ell^\star,1}(t)\zeta_{j,j_\ell^\star}^{(1)}(t)\right| & \leq \calO\paren{\aggepst{\ell}(t)}\gamma_{i_\ell^\star,j_\ell^\star}^{(2)}(t)^2 +\calO\paren{\aggepst{\ell}(t)^3}\\
        - \hat{\lambda}_{i_\ell^\star,j_\ell^\star,1}\gamma_{i_\ell^\star,j_\ell^\star}^{(1)}(t)\left|I_{i,j}^{(3)}(t)\right| & \leq 6\left|I_{i,j}^{(3)}(t)\right|\gamma_{i_\ell^\star,j_\ell^\star}^{(2)}(t)^3\sum_{k=0}^{\infty}\frac{c_{k+1}^2}{k!}\left|\gamma_{i_\ell^\star,j_\ell^\star}^{(1)}(t)\right|^k\\
        & \qqquad + \calO\paren{\aggepst{\ell}(t)^2}\gamma_{i_\ell^\star,j_\ell^\star}^{(2)}(t)^2 + \calO\paren{\aggepst{\ell}(t)^3}\\
        & \leq \calO\paren{\aggepst{\ell}(t)}\gamma_{i_\ell^\star,j_\ell^\star}^{(2)}(t)^2 + \calO\paren{\aggepst{\ell}(t)^3}
    \end{align*}
    Noticing that $\hat{\lambda}_{i_\ell^\star,j_\ell^\star,5}(t), \gamma_{i_\ell^\star,j_\ell^\star}^{(2)}(t)\geq 0$ for all $t\leq T_c$, we have that
    \[
        \derivt \left|I_{i,j}^{(3)}(t)\right| \leq  \calO\paren{\aggepst{\ell}(t)}\gamma_{i_\ell^\star,j_\ell^\star}^{(2)}(t)^2 + \calO\paren{m\aggepst{\ell}(t)}
    \]
    \textbf{Bounding $\gamma_{i,j}^{(1)}(t)$.} Finally, for $\gamma_{i,j}^{(1)}(t)$, in the case of $(i,j)\neq (i_\ell^\star,j_\ell^\star)$, we have that
    \begin{align*}
        \derivt \left|\gamma_{i,j}^{(1)}(t)\right| & = -\frac{1}{a_i}\text{sign}\paren{\nabla_{\bfv_i}\calL\paren{\bm{\theta}(t)}^\top\bbfv_j^\star}\\
        & \leq -\hat{\lambda}_{i_\ell^\star,j_\ell^\star,1}(t)\gamma_{i_\ell^\star,j_\ell^\star}^{(1)}(t)\left|\gamma_{i,j}^{(1)}(t)\right|\indy{i=i_\ell^\star} + \calO\paren{\aggepst{\ell}(t)}\gamma_{i_\ell^\star,j_\ell^\star}^{(2)}(t)^2 + \calO\paren{m\aggepst{\ell}(t)^3}
    \end{align*}
    where since $(i,j)\neq (i_\ell^\star,j_\ell^\star)$, we must have that
    \[
        -\hat{\lambda}_{i_\ell^\star,j_\ell^\star,1}(t)\gamma_{i_\ell^\star,j_\ell^\star}^{(1)}(t)\left|\gamma_{i,j}^{(1)}(t)\right| \leq \calO\paren{\aggepst{\ell}(t)}\gamma_{i_\ell^\star,j_\ell^\star}^{(2)}(t)^2 + \calO\paren{\aggepst{\ell}(t)^3}
    \]
    Therefore
    \[
        \derivt \left|\gamma_{i,j}^{(1)}(t)\right| \leq\calO\paren{\aggepst{\ell}(t)}\gamma_{i_\ell^\star,j_\ell^\star}^{(2)}(t)^2 + \calO\paren{m\aggepst{\ell}(t)^3}
    \]

    \textbf{Bounding $I_{i,j}^{(1)}(t), I_{i,j}^{(2)}(t)$, and $I_{i,j}^{(3)}(t)$ for $i,j\in[m]\setminus \mathcal{R}_{\ell-1}$ for $i\in\mathcal{R}_{\ell-1}$ or $j\in\mathcal{R}_{\ell-1}$.} Recall the definition of $I_{i,j}^{(1)}(t), I_{i,j}^{(2)}(t)$, and $I_{i,j}^{(3)}(t)$
    \[
        I_{i,j}^{(1)}(t) = \bbfv_i(t)^\top\bbfv_j(t);\; I_{i,j}^{(2)}(t) = \bbfw_i(t)^\top\bbfw_j(t);\; I_{i,j}^{(3)}(t) = \bbfv_i(t)^\top\bbfw_j(t)
    \]
    For the sake of convenience, we let $i\in\mathcal{R}_{\ell-1}$ and study $I_{i,j}^{(1)}(t), I_{i,j}^{(2)}(t)$, and $I_{i,j}^{(3)}(t), I_{j,i}^{(3)}(t)$. Since $i\in\mathcal{R}_{\ell-1}$, there exists $\ell' < \ell$ such that $i = i_{\ell'}^\star$. Therefore, for $I_{i,j}^{(1)}(t)$, when $t\geq T_{p,\ell-1}$, we have that
    \begin{equation}
        \label{eq:I1_star_bound}
        \begin{aligned}
            \left|I_{i,j}^{(1)}(t)\right| & = \left|\bbfv_j(t)^\top\bbfv_{j_{\ell'}^\star}^\star + \bbfv_j(t)^\top\paren{\bbfv_{i_{\ell'}^\star}(t)-\bbfv_{j_{\ell'}^\star}^\star}\right|\\
            & \leq \left|\gamma_{j,j_{\ell'}^\star}^{(1)}(t)\right| + \norm{\bbfv_{i_{\ell'}^\star}(t) - \bbfv_{j_{\ell'}^\star}(t)}_2\\
            & \leq \left|\gamma_{j,j_{\ell'}^\star}^{(1)}(t)\right| + \calO\paren{\frac{m^4}{\delta_{\mathbb{P}}^2d^{\frac{3}{4}}}}
        \end{aligned}
    \end{equation}
    for all $t \geq T_{p,\ell-1}$. Similarly, for $I_{i,j}^{(2)}(t)$, when $t\geq T_{p,\ell-1}$, we have that
    \begin{equation}
        \label{eq:I2_star_bound}
        \begin{aligned}
            \left|I_{i,j}^{(2)}(t)\right| & = \left|\bbfw_j(t)^\top\bbfw_{j_{\ell'}^\star}^\star + \bbfw_j(t)^\top\paren{\bbfw_{i_{\ell'}^\star}(t)-\bbfw_{j_{\ell'}^\star}^\star}\right|\\
            & \leq \left|\gamma_{j,j_{\ell'}^\star}^{(2)}(t)\right| + \norm{\bbfw_{i_{\ell'}^\star}(t) - \bbfw_{j_{\ell'}^\star}(t)}_2\\
            & \leq \left|\gamma_{j,j_{\ell'}^\star}^{(2)}(t)\right| + \calO\paren{\frac{m^4}{\delta_{\mathbb{P}}^2d^{\frac{3}{4}}}}
        \end{aligned}
    \end{equation}
    For $I_{i,j}^{(3)}(t)$ and $I_{j,i}^{(3)}(t)$, when $t\geq T_{p,\ell-1}$, we have that
    \begin{equation}
        \label{eq:I3_star_bound}
        \begin{aligned}
            \left|I_{i,j}^{(3)}(t)\right| & = \left|\bbfw_j(t)^\top\bbfv_{j_{\ell'}^\star}^\star + \bbfw_j(t)^\top\paren{\bbfv_{i_{\ell'}^\star}(t)-\bbfv_{j_{\ell'}^\star}^\star}\right|\\
            & \leq \left|\zeta_{j,j_{\ell'}^\star}^{(2)}(t)\right| + \norm{\bbfv_{i_{\ell'}^\star}(t) - \bbfv_{j_{\ell'}^\star}(t)}_2\\
            & \leq \left|\zeta_{j,j_{\ell'}^\star}^{(2)}(t)\right| + \calO\paren{\frac{m^4}{\delta_{\mathbb{P}}^2d^{\frac{3}{4}}}}\\
            \left|I_{j,i}^{(3)}(t)\right| & = \left|\bbfv_i(t)^\top\bbfw_{j_{\ell'}^\star}^\star + \bbfv_i(t)^\top\paren{\bbfw_{i_{\ell'}^\star}(t)-\bbfw_{j_{\ell'}^\star}^\star}\right|\\
            & \leq \left|\zeta_{i,j_{\ell'}^\star}^{(1)}(t)\right| + \norm{\bbfw_{i_{\ell'}^\star}(t) - \bbfw_{j_{\ell'}^\star}(t)}_2\\
            & \leq \left|\zeta_{i,j_{\ell'}^\star}^{(1)}(t)\right| + \calO\paren{\frac{m^4}{\delta_{\mathbb{P}}^2d^{\frac{3}{4}}}}
        \end{aligned}
    \end{equation}
    Moreover, by the inductive hypothesis, we have that for $t\leq T_{p,\ell-1}$
    \[
        \left|I_{i,j}^{(1)}(t)\right|, \left|I_{i,j}^{(2)}(t)\right|, \left|I_{i,j}^{(3)}(t)\right|, \left|I_{j,i}^{(3)}(t)\right| \leq \calO\paren{\frac{m^2}{\delta_{\mathbb{P}}\sqrt{d}}}
    \]
    \textbf{Gathering the results.}
    Defining $\hat{\varepsilon}_{\ell}(t)$ to be the maximum of $\left|\zeta_{i,j}^{(1)}(t)\right|, \left|\zeta_{i,j}^{(2)}(t)\right|, \left|\gamma_{i,j}^{(1)}(t)\right|$ for $i,j$ bounded above excluding $(i,j) = (i_{\ell}^\star,j_{\ell}^\star$ for $\gamma_{i,j}^{(1)}(t)$, and $\left|I_{i,j}^{(1)}(t)\right|, \left|I_{i,j}^{(1)}(t)\right|, \left|I_{i,j}^{(1)}(t)\right|$ for $i,j$ bounded above excluding $i\in\mathcal{R}_{\ell-1}$ or $j\in\mathcal{R}_{\ell-1}$.
    Gathering the results, we have that
    \[
        \derivt \hat{\varepsilon}_{\ell}(t)\leq \calO\paren{\aggepst{\ell}(t)}\gamma_{i_\ell^\star,j_\ell^\star}^{(2)}(t)^2 + \calO\paren{m\aggepst{\ell}(t)^3}
    \]
    for $t\leq T_c$. In the meantime, by Lemma~\ref{lem:gamma_2_ub} we have that
    \[
        \left|\gamma_{i,j}^{(2)}(t)\right| \leq \calO\paren{\frac{m^2}{\delta_{\mathbb{P}}\sqrt{d}}};\;\forall\;t\leq \min\left\{T_c, T\paren{\frac{\delta_{\mathbb{P}}^2}{d^{\frac{2}{5}}}} + \frac{\beta_7\delta_{\mathbb{P}}\sqrt{d}}{m^2}\right\}
    \]
    where $i\in[m]\setminus \mathcal{R}_{\ell-1}, j\in[m^\star]\setminus \mathcal{C}_{\ell-1}$ and $(i,j)\neq (i_\ell^\star,j_\ell^\star)$. 
    Therefore, we have that
    \begin{align*}
        \derivt \hat{\varepsilon}_{\ell}(t) & \leq \calO\paren{\hat{\varepsilon}_{\ell}(t) + \calO\paren{\frac{m^2}{\delta_{\mathbb{P}}\sqrt{d}}}}\gamma_{i_\ell^\star,j_\ell^\star}^{(2)}(t)^2 + \calO\paren{m\aggepst{\ell}(t)^3}\\
        & \leq \calO\paren{\hat{\varepsilon}_{\ell}(t) + \calO\paren{\frac{m^2}{\delta_{\mathbb{P}}\sqrt{d}}}}\derivt \gamma_{i_\ell^\star,j_\ell^\star}^{(2)}(t) + \calO\paren{m\aggepst{\ell}(t)^3}\\
        & \leq \calO\paren{\hat{\varepsilon}_{\ell}(t) + \calO\paren{\frac{m^2}{\delta_{\mathbb{P}}\sqrt{d}}}}\derivt \gamma_{i_\ell^\star,j_\ell^\star}^{(2)}(t) + \calO\paren{\frac{m^7}{\delta_{\mathbb{P}}^3d^{\frac{3}{2}}}}
    \end{align*}
    where the last inequality follows from Lemma~\ref{lem:gamma_2_growth}. Solving the differential equation gives
    \begin{align*}
         \hat{\varepsilon}_{\ell}(t) \leq \calO\paren{\frac{m^2}{\delta_{\mathbb{P}}\sqrt{d}}}\hat{\varepsilon}_{\ell}(0)\exp{\calO\paren{\gamma_{i_\ell^\star,j_\ell^\star}^{(2)}(t)}} + \calO\paren{\frac{m^7}{\delta_{\mathbb{P}}^3d^{\frac{3}{2}}}}\exp{\calO\paren{\gamma_{i_\ell^\star,j_\ell^\star}^{(2)}(t)}}\cdot t
    \end{align*}
    Imposing $\gamma_{i_\ell^\star,j_\ell^\star}^{(2)}(t) \leq 1$ gives the desired result.

    \textbf{Fine-Grained Result for $i\in[m]\setminus \mathcal{R}_{\ell}$.} By Lemma~\ref{lem:dynamic_approx}, we have that for all $i\in[m]\setminus \mathcal{R}_{\ell}$ and $j\in[m^\star]$, it holds that
    \[
        \max\left\{\left|\gamma_{i,j}^{(1)}(t)\right|,\left|\zeta_{i,j}^{(1)}(t)\right|,\left|\zeta_{i,j}^{(2)}(t)\right|\right\} \leq \calO\paren{m\aggepst{\ell}(t)^3 + \aggepst{\ell}(t)\backepsi{\ell}(t)} \leq \calO\paren{\frac{m^7}{\delta_{\mathbb{P}}^2d^{\frac{3}{2}}}}
    \]
    For $i,j\in[m]\setminus \mathcal{R}_{\ell}$ with $i\neq j$, we have that
    \[
        \max\left\{\left|I_{i,j}^{(1)}(t)\right|,\left|I_{i,j}^{(2)}(t)\right|,\left|I_{i,j}^{(3)}(t)\right|\right\}\leq \calO\paren{m\aggepst{\ell}(t)^3 + \aggepst{\ell}(t)\backepsi{\ell}(t)} \leq \calO\paren{\frac{m^7}{\delta_{\mathbb{P}}^3d^{\frac{3}{2}}}}
    \]
    Thus, for all $t\leq \calO\paren{\frac{\delta_{\mathbb{P}}^2d}{m^5}}$ we shall have that
    \begin{align*}
        \max\left\{\left|\gamma_{i,j}^{(1)}(t)\right|,\left|\zeta_{i,j}^{(1)}(t)\right|,\left|\zeta_{i,j}^{(2)}(t)\right|\right\} & \leq \calO\paren{\frac{m^2}{\delta_{\mathbb{P}}\sqrt{d}}};\;\forall i\in[m]\setminus \mathcal{R}_{\ell},j\in[m^\star]\\
        \max\left\{\left|I_{i,j}^{(1)}(t)\right|,\left|I_{i,j}^{(2)}(t)\right|,\left|I_{i,j}^{(3)}(t)\right|\right\} & \leq \calO\paren{\frac{m^2}{\delta_{\mathbb{P}}\sqrt{d}}};\;\forall i,j\in[m]\setminus \mathcal{R}_{\ell}, i\neq j
    \end{align*}
\end{proof}

\textbf{Bounding $I_{i,i}^{(3)}(t)$.} Here we are going to upper bound $I_{i,i}^{(3)}(t)$ for $i\in[m]\setminus \mathcal{R}_{\ell-1}$

\begin{lemma}
    \label{lem:eps_2_ub}
    Suppose that the inductive hypothesis in Condition~\ref{cond:inductive_hypo}, and the initialization condition in Condition~\ref{cond:init}. Let $T_c$ be defined in (\ref{eq: T_c_def}) and $T(\xi)$ in Definition~\ref{def:recovery_time}. Then for all $t\leq T_c$, we have that
    \[
        \left|I_{i,i}^{(3)}(t)\right| \leq \begin{cases}
            \calO\paren{m\aggepst{\ell}(t)^3} & \text{ if } i \neq i_\ell^\star\\
            \calO\paren{\aggepst{\ell}(t)}\gamma_{i_\ell^\star,j_\ell^\star}^{(2)}(t)^2+  \calO\paren{m\aggepst{\ell}(t)^3} & \text{ if } i = i_\ell^\star
        \end{cases}
    \]
\end{lemma}
\begin{proof}
    By Lemma~\ref{lem:dynamic_approx}, we have that, in the case where $i\neq i_\ell^\star$
    \begin{align*}
        \derivt \left|I_{i,i}^{(3)}(t)\right| & = -\text{sign}\paren{I_{i,i}^{(3)}(t)}\paren{\frac{1}{a_i}\nabla_{\bfv_i}\calL\paren{\bm{\theta}(t)}^\top\bbfw_j + \frac{1}{b_i}\nabla_{\bfw_i}\calL\paren{\bm{\theta}(t)}^\top\bbfv_j}\\
        & \leq -36C_{S,2}\paren{\frac{1}{a_i^2} + \frac{1}{b_i^2}}\left|I_{i,i}^{(3)}(t)\right| + \calO\paren{m\aggepst{\ell}(t)^3}
    \end{align*}
    Since $\aggepst{\ell}(t)$ is monotonic non-decreasin, we must have that $\left|I_{i,i}^{(3)}(t)\right| \leq \calO\paren{m\aggepst{\ell}(t)^3}$ for all $t \geq 0$. Similarly, in the case $i = i_\ell^\star$, we have that
    \begin{align*}
        \derivt\left|I_{i,i}^{(3)}(t)\right|  & = -\text{sign}\paren{I_{i,i}^{(3)}(t)}\paren{\frac{1}{a_i}\nabla_{\bfv_i}\calL\paren{\bm{\theta}(t)}^\top\bbfw_j + \frac{1}{b_i}\nabla_{\bfw_i}\calL\paren{\bm{\theta}(t)}^\top\bbfv_j}\\
        & = \frac{1}{a_i^2}\left|\hat{\lambda}_{i_\ell^\star,j_\ell^\star,1}(t)\zeta_{i_\ell^\star,j_\ell^\star}^{(2)}(t)\right| + \frac{9}{b_i^2}\left|\hat{\lambda}_{i_\ell^\star,j_\ell^\star,5}(t)\zeta_{i_\ell^\star,j_\ell^\star}^{(1)}(t)\right| + 3\left|\hat{\lambda}_{i_\ell^\star,j_\ell^\star,2}(t)\right|\paren{\frac{1}{a_i^2} + \frac{1}{b_i^2}}\\
        & \qqquad +\frac{3}{a_i^2}\text{sign}\paren{I_{i,i}^{(3)}(t)}\hat{\lambda}_{i_\ell^\star,j_\ell^\star,4}(t)\gamma_{i_\ell^\star,j_\ell^\star}^{(2)}(t) + \frac{3}{b_i^2}\left|\hat{\lambda}_{i_\ell^\star,j_\ell^\star,3}(t)\gamma_{i_\ell^\star,j_\ell^\star}^{(1)}(t)\right| \\
        & \qqquad -36C_{S,2}\paren{\frac{1}{a_i^2} + \frac{1}{b_i^2}}\left|I_{i,i}^{(3)}(t)\right| -  \frac{1}{a_i^2}\hat{\lambda}_{i_\ell^\star,j_\ell^\star,1}(t)\gamma_{i_\ell^\star,j_\ell^\star}^{(1)}\left|I_{i,i}^{(3)}(t)\right|\\
        & \qqquad - \frac{9}{b_i^2}\hat{\lambda}_{i_\ell^\star,j_\ell^\star,5}(t)\gamma_{i_\ell^\star,j_\ell^\star}^{(2)}\left|I_{i,i}^{(3)}(t)\right| + \calO\paren{m\aggepst{\ell}(t)^3}
    \end{align*}
    As in the previous proof, we have that
    \begin{align*}
        \left|\hat{\lambda}_{i_\ell^\star,j_\ell^\star,3}(t)\gamma_{i_\ell^\star,j_\ell^\star}^{(1)}(t)\right| & \leq 6\left|\zeta_{i,j}^{(1)}(t)\right|\gamma_{i_\ell^\star,j_\ell^\star}^{(2)}(t)^2\sum_{k=0}^{\infty}\frac{c_kc_{k+2}}{k!}\left|\gamma_{i_\ell^\star,j_{\ell^\star}}^{(1)}(t)\right|^{k+1} + \calO\paren{\varepsilon_1(t)^3}\\
        & \leq \calO\paren{\backepso{\ell}(t)}\gamma_{i_\ell^\star,j_\ell^\star}^{(2)}(t)^2 + \calO\paren{m\aggepst{\ell}(t)^3}
    \end{align*}
    The second and third term requires a more careful analysis. First, by the definition of $\hat{\lambda}_{i_\ell^\star,j_\ell^\star,4}$, we have that
    \begin{align*}
        & \hat{\lambda}_{i_\ell^\star,j_\ell^\star,4}(t)\gamma_{i_\ell^\star,j_\ell^\star}^{(2)}(t)\cdot \text{sign}\paren{I_{i,i}^{(3)}(t)}\\
        & \qqquad  \leq 18\left|I_{i,i}^{(3)}(t)\right|\gamma_{i_\ell^\star,j_\ell^\star}^{(2)}(t)^3\sum_{k=0}^{\infty}\frac{c_{k+2}c_k}{k!}\gamma_{i_\ell^\star,j_\ell^\star}^{(1)}(t)^k + 18\left|\zeta_{i,j}^{(2)}(t)\gamma_{i_\ell^\star,j_\ell^\star}^{(2)}(t)^3\sum_{k=0}^{\infty}\frac{c_{k+1}^2}{k!}\gamma_{i_\ell^\star,j_\ell^\star}^{(1)}(t)^k\right|\\
        & \qqquad \leq 18\left|I_{i,i}^{(3)}(t)\right|\gamma_{i_\ell^\star,j_\ell^\star}^{(2)}(t)^3\sum_{k=0}^{\infty}\frac{c_{k+2}c_k}{k!}\gamma_{i_\ell^\star,j_\ell^\star}^{(1)}(t)^k + \calO\paren{\aggepst{\ell}(t)}\gamma_{i_\ell^\star,j_\ell^\star}^{(2)}(t)^3\\
        & \qqquad \leq \calO\paren{\frac{1}{\sqrt{d}}\paren{\log \frac{m}{\delta_{\mathbb{P}}}}^{\frac{1}{2}}}\left|I_{i,i}^{(3)}(t)\right| + \calO\paren{\aggepst{\ell}(t)}\gamma_{i_\ell^\star,j_\ell^\star}^{(2)}(t)^3
    \end{align*}
    Moreover, by the definition of $\hat{\lambda}_{i_\ell^\star,j_\ell^\star,1}$, we have that
    \begin{align*}
        -\lambda_{i_\ell^\star,j_\ell^\star,1}(t)\gamma_{i_\ell^\star,j_\ell^\star}^{(1)} & =  - 6\sum_{k=0}^{\infty}\frac{c_{k+1}^2}{k!}\gamma_{i_\ell^\star,j_\ell^\star}^{(1)}(t)^{k+1}\gamma_{i_\ell^\star,j_\ell^\star}^{(2)}(t)^{3}\\
        & = -6\gamma_{i_\ell^\star,j_\ell^\star}^{(1)}(t)\gamma_{i_\ell^\star,j_\ell^\star}^{(2)}(t)^{3}\paren{c_1^2 \pm \calO\paren{\frac{1}{\sqrt{d}}\paren{\log \frac{m}{\delta_{\mathbb{P}}}}^{\frac{1}{2}}}}\\
        & \leq \calO\paren{\frac{1}{\sqrt{d}}\paren{\log \frac{m}{\delta_{\mathbb{P}}}}^{\frac{1}{2}}}
    \end{align*}
    Recalling that $\lambda_{i_\ell^\star,j_\ell^\star,5}(t)\gamma_{i_\ell^\star,j_\ell^\star}^{(2)} \geq 0$ for all $t$, we have that
    \begin{align*}
        \derivt\left|I_{i,i}^{(3)}(t)\right| & \leq -\paren{36C_{S,2}\paren{\frac{1}{a_i^2} + \frac{1}{b_i^2}} - \calO\paren{\frac{1}{\sqrt{d}}\paren{\log \frac{m}{\delta_{\mathbb{P}}}}^{\frac{1}{2}}}}\left|I_{i,i}^{(3)}(t)\right|\\
        & \qqquad + \calO\paren{\aggepst{\ell}(t)}\gamma_{i_\ell^\star,j_\ell^\star}^{(2)}(t)^2+  \calO\paren{m\aggepst{\ell}(t)^3}\\
        & \leq -35C_{S,2}\paren{\frac{1}{a_i^2} + \frac{1}{b_i^2}}\left|I_{i,i}^{(3)}(t)\right| + \calO\paren{\aggepst{\ell}(t)}\gamma_{i_\ell^\star,j_\ell^\star}^{(2)}(t)^2+  \calO\paren{m\aggepst{\ell}(t)^3}
    \end{align*}
    Since $\aggepst{\ell}(t)$ is monotonically non-decreasing, we must have that $I_{i,i}^{(3)} \leq \calO\paren{\aggepst{\ell}(t)}\gamma_{i_\ell^\star,j_\ell^\star}^{(2)}(t)^2+  \calO\paren{m\aggepst{\ell}(t)^3}$ for all $t\geq0$.
\end{proof}

With the above preparation work, we are ready to derive the result for phase 1 convergence.

\begin{lemma}
    \label{lem:phase_1_conv}
    Suppose that the inductive hypothesis in Condition~\ref{cond:inductive_hypo} and the initialization condition in Condition~\ref{cond:init} holds. Let $T_c$ be defined in (\ref{eq: T_c_def}) and $T(\xi)$ be defined in Definition~\ref{def:recovery_time}. Then there exists constant $\beta_7 \geq 0$ such that for all $\xi$ satisfying $(1-\xi^2)^{-1}\leq \calO\paren{1}$ and all constant $\beta_8\geq 0$, it holds that $T_c \geq \min\left\{T\paren{\frac{\delta_{\mathbb{P}}^2}{d^{\frac{2}{5}}}} + \frac{\beta_7\delta_{\mathbb{P}}\sqrt{d}}{m^2}, T\paren{\xi}\right\} + \beta_8$, and there exists $T_1 \leq \min\left\{T_c, \calO\paren{\sqrt{d}}\right\}$ such that for all $T_1 \leq t\leq T_c$ we have that $\gamma_{i_\ell^\star,j_\ell^\star}^{(1)}(t),\gamma_{i_\ell^\star,j_\ell^\star}^{(2)}(t) \geq 0.9$.
    \begin{proof}
        By Lemma~\ref{lem:gamma_2_prev}, Lemma~\ref{lem:eps_1_ub} and Lemma~\ref{lem:eps_2_ub}, we have that $T_c \geq \min\left\{T\paren{\frac{\delta_{\mathbb{P}}}{d^{\frac{2}{5}}}} + \frac{\beta_7\delta_{\mathbb{P}}\sqrt{d}}{m}, T\paren{\xi}\right\} + \beta_8$. Thus, it remains to show that there exists some $T_1 \leq T_c$ such that $\gamma_{i_\ell^\star,j_\ell^\star}^{(1)}(t),\gamma_{i_\ell^\star,j_\ell^\star}^{(2)}(t) \geq 0.9$. To start, we choose $\xi = \frac{1}{2}$. By Lemma~\ref{lem:gamma_2_growth} we have that
        \begin{align*}
            \frac{1}{2} & = \gamma_{i_\ell^\star,j_\ell^\star}^{(2)}\paren{T\paren{\frac{1}{2}}}\\
            & \geq \paren{\gamma_{i_\ell^\star,j_\ell^\star}^{(2)}\paren{T\paren{\frac{\delta_{\mathbb{P}}}{d^{\frac{2}{5}}}}}^{-1} - \frac{18}{b_{i_\ell^\star}}\paren{c_0^2\paren{1-\frac{1}{4}} - \calO\paren{\frac{m^7}{\delta_{\mathbb{P}}^3\sqrt{d}}}}\paren{T\paren{\frac{1}{2}} - T\paren{\frac{\delta_{\mathbb{P}}}{d^{\frac{2}{5}}}}}}^{-1}
        \end{align*}
        This gives that
        \[
            T\paren{\frac{1}{2}} \leq \calO\paren{\frac{d^\frac{2}{5}}{\delta_{\mathbb{P}}^2}} + T\paren{\frac{\delta_{\mathbb{P}}}{d^{\frac{2}{5}}}} \leq T\paren{\frac{\delta_{\mathbb{P}}}{d^{\frac{2}{5}}}} + \frac{\beta_7\delta_{\mathbb{P}}\sqrt{d}}{m}
        \]
        when $d\geq \delta_{\mathbb{P}}^{10}m^5$.
        Thus, it suffice to show that a $T_1 \leq T\paren{\frac{1}{2}} + \beta_8$ achieves $\gamma_{i_\ell^\star,j_\ell^\star}^{(1)}(t),\gamma_{i_\ell^\star,j_\ell^\star}^{(2)}(t) \geq 0.9$ for some constant $\beta_8$. Notice that this choice of $T_1$ satisfies $T_1 \leq \calO\paren{\sqrt{d}}$ That is, if we define
        \[
            T_{1,1} = \min\left\{t\geq 0:\gamma_{i_\ell^\star,j_\ell^\star}^{(1)}(t)\geq 0.9\right\};\; T_{1,2} = \min\left\{t\geq 0:\gamma_{i_\ell^\star,j_\ell^\star}^{(2)}(t)\geq 0.9\right\}
        \]
        Then it suffice to show that $T_{1,1},T_{1,2}\leq T\paren{\frac{1}{2}} + \beta_8$. By Lemma~\ref{lem:dynamic_approx}, we have that for $T\paren{\frac{1}{2}}\leq t\leq \min\left\{T_{1,2}, T\paren{\frac{1}{2}} +\beta_8\right\}$, it holds that
        \begin{align*}
            \derivt \gamma_{i_\ell^\star,j_\ell^\star}^{(2)}(t) & \geq \gamma_{i_\ell^\star,j_\ell^\star}^{(2)}(t)^2\paren{1 - \gamma_{i_\ell^\star,j_\ell^\star}^{(2)}(t)^2}\sum_{k=0}^\infty\frac{c_k^2}{k!}\gamma_{i_\ell^\star,j_\ell^\star}^{(1)}(t)^k - \calO\paren{\aggepst{\ell}(t)}\\
            & \geq \gamma_{i_\ell^\star,j_\ell^\star}^{(2)}(t)^2\paren{1 - \gamma_{i_\ell^\star,j_\ell^\star}^{(2)}(t)^2}\paren{c_0^2 - \calO\paren{\frac{1}{\sqrt{d}}\paren{\log \frac{m}{\delta_{\mathbb{P}}}}^{\frac{1}{2}}}}- \calO\paren{\frac{m^2}{\delta_{\mathbb{P}}\sqrt{d}}}\\
            & \geq \frac{c_0^2}{2}\gamma_{i_\ell^\star,j_\ell^\star}^{(2)}(t)^2\paren{1 - \gamma_{i_\ell^\star,j_\ell^\star}^{(2)}(t)^2}- \calO\paren{\frac{m^2}{\delta_{\mathbb{P}}\sqrt{d}}}
        \end{align*}
        This shows that $\derivt \gamma_{i_\ell^\star,j_\ell^\star}^{(2)}(t) \geq 0$ as long as $1 - \gamma_{i_\ell^\star,j_\ell^\star}^{(2)}(t)^2 \geq \calO\paren{\frac{m^2}{\delta_{\mathbb{P}}\sqrt{d}}}$. Thus, we can conclude that $\gamma_{i_\ell^\star,j_\ell^\star}^{(2)}(t)\geq \frac{1}{2}$ for all $T\paren{\frac{1}{2}}\leq t\leq \min\left\{T_{1,2}, T\paren{\frac{1}{2}} +\beta_8\right\}$. For the dynamic of $\gamma_{i_\ell^\star,j_\ell^\star}^{(2)}(t)$, this implies that for $T\paren{\frac{1}{2}}\leq t\leq \min\left\{T_{1,2}, T\paren{\frac{1}{2}} +\beta_8\right\}$
        \begin{align*}
            \derivt \gamma_{i_\ell^\star,j_\ell^\star}^{(2)}(t) \geq \frac{c_0^2}{42} - \calO\paren{\frac{m^4}{\delta_{\mathbb{P}}^2d}} \geq \frac{c_0^2}{50}
        \end{align*}
        Thus $\gamma_{i_\ell^\star,j_\ell^\star}^{(2)}\paren{t+T\paren{\frac{1}{2}}}\geq \gamma_{i_\ell^\star,j_\ell^\star}^{(2)}\paren{T\paren{\frac{1}{2}}} + \frac{c_0^2}{50}\cdot t$, which implies that $T_{1,2} \leq T\paren{\frac{1}{2}} + \frac{20}{c_0^2} \leq T\paren{\frac{1}{2}} + \beta_8$ for some $\beta_8 \geq 0$. Next, for the dynamic of $\gamma_{i_\ell^\star,j_\ell^\star}^{(2)}(t)$, we have that
        \begin{align*}
            \derivt \gamma_{i_\ell^\star,j_\ell^\star}^{(1)}(t) & \geq 6\paren{1 - \gamma_{i_\ell^\star,j_\ell^\star}^{(1)}(t)^2}\gamma_{i_\ell^\star,j_\ell^\star}^{(2)}(t)^3\sum_{k=0}^\infty\frac{c_k^2}{k!}\gamma_{i_\ell^\star,j_\ell^\star}^{(1)}(t)^k - \calO\paren{\frac{m^2}{\delta_{\mathbb{P}}\sqrt{d}}}\\
            & \geq \frac{3}{4}\paren{1 - \gamma_{i_\ell^\star,j_\ell^\star}^{(1)}(t)^2}\paren{c_1^2 - \calO\paren{\frac{1}{\sqrt{d}}\paren{\log \frac{m}{\delta_{\mathbb{P}}}}^{\frac{1}{2}}}}- \calO\paren{\frac{m^2}{\delta_{\mathbb{P}}\sqrt{d}}}\\
            & \geq \frac{c_1^2}{8}
        \end{align*}
        Thus $\gamma_{i_\ell^\star,j_\ell^\star}^{(1)}\paren{t+T\paren{\frac{1}{2}}}\geq \gamma_{i_\ell^\star,j_\ell^\star}^{(1)}\paren{T\paren{\frac{1}{2}}} + \frac{c_1^2}{8}\cdot t$, which implies that $T_{1,1}\leq T\paren{\frac{1}{2}} + \frac{16}{c_1^2}$ since $\gamma_{i_\ell^\star,j_\ell^\star}^{(1)}\paren{T\paren{\frac{1}{2}}} \geq - \calO\paren{\frac{1}{\sqrt{d}}\paren{\log \frac{m}{\delta_{\mathbb{P}}}}^{\frac{1}{2}}}$ by Lemma~\ref{lem:gamma_2_growth}. Therefore, we can conclude that $T_{1,1},T_{1,2} \leq T\paren{\frac{1}{2}} + \beta_8$ for some constant $\beta_8$. Now, we consider the dynamic of $\gamma_{i_\ell^\star,j_\ell^\star}^{(1)}(t)$ and $\gamma_{i_\ell^\star,j_\ell^\star}^{(2)}(t)$ for $t\ geq \max\left\{T_{1,1}, T_{1,2}\right\}$. As before, we have that
        \begin{align*}
            \derivt \gamma_{i_\ell^\star,j_\ell^\star}^{(2)}(t) \geq \frac{c_1^2}{2}\gamma_{i_\ell^\star,j_\ell^\star}^{(2)}(t)^3\paren{1-\gamma_{i_\ell^\star,j_\ell^\star}^{(1)}(t)^2} - \calO\paren{\frac{m^2}{\delta_{\delta_{\mathbb{P}}\sqrt{d}}}}\\
            \derivt \gamma_{i_\ell^\star,j_\ell^\star}^{(2)}(t)\geq \frac{c_0^2}{2}\gamma_{i_\ell^\star,j_\ell^\star}^{(2)}(t)^2\paren{1-\gamma_{i_\ell^\star,j_\ell^\star}^{(2)}(t)^2} - \calO\paren{\frac{m^2}{\delta_{\delta_{\mathbb{P}}\sqrt{d}}}}
        \end{align*}
        We can observe that $\derivt \gamma_{i_\ell^\star,j_\ell^\star}^{(2)}(t), \derivt \gamma_{i_\ell^\star,j_\ell^\star}^{(2)}(t)$ if 
        \[
            \gamma_{i_\ell^\star,j_\ell^\star}^{(1)}(t), \gamma_{i_\ell^\star,j_\ell^\star}^{(2)}(t) \leq 1 - \calO\paren{\frac{m^2}{\delta_{\delta_{\mathbb{P}}\sqrt{d}}}}
        \]
        Thus, for all $T_1 \leq t \leq T_c$ we have that $\gamma_{i_\ell^\star,j_\ell^\star}^{(1)}(t), \gamma_{i_\ell^\star,j_\ell^\star}^{(2)}(t) \geq 0.9$
    \end{proof}
\end{lemma}

\subsection{Establishing the Inductive Hypothesis: Phase 2}
\subsection{Phase 2: Growth to Near-Perfect Alignment.}
In this section, our goal is to show complete the inductive hypothesis.

\textbf{The decay of $\left|\zeta_{i_\ell^\star,j_\ell^\star}^{(1)}(t)\right|, \left|\zeta_{i_\ell^\star,j_\ell^\star}^{(1)}(t)\right|$ and $\left|I_{i_\ell^\star,j_\ell^\star}^{(3)}(t)\right|$.} 

\begin{lemma}
    \label{lem:exp_decay}
    Suppose that the inductive hypothesis in Condition~\ref{cond:inductive_hypo}, the initialization condition in Condition~\ref{cond:init}, and the condition on the sigmoid function in Assumption~\ref{asump:sigmoid} holds. Then for $t\geq T_1$, where $T_1$ is defined in Lemma~\ref{lem:phase_1_conv}, we have that
    \begin{gather*}
        \max\left\{\left|\zeta_{i_\ell^\star,j_\ell^\star}^{(1)}(t)\right|, \left|\zeta_{i_\ell^\star,j_\ell^\star}^{(1)}(t)\right|, \left|I_{i_\ell^\star,i_\ell^\star}^{(3)}(t)\right|\right\} \leq \begin{cases}\calO\paren{m\aggepso{\ell}(t)^2 + \frac{m^2}{\delta_{\mathbb{P}}\sqrt{d}}} & \text{ for } T_1 \leq t \leq T_c\\
        \calO\paren{m\aggepso{\ell}(t)^2 + \frac{m^7}{\delta_{\mathbb{P}}^3d^{\frac{3}{2}}}} & \text{ for } T_1+ \calO\paren{\log d} \leq t \leq T_c
        \end{cases}
    \end{gather*}
\end{lemma}
\begin{proof}
    To ease the analysis in this section, we are going to define
    \begin{align*}
        Q_0(t) & = \gamma_{i_\ell^\star,j_\ell^\star}^{(2)}(t)^2\sum_{k=0}^{\infty}\frac{c_{k}^2}{k!}\gamma_{i_\ell^\star,j_\ell^\star}^{(1)}(t)^k;\\
        Q_1(t) & = \gamma_{i_\ell^\star,j_\ell^\star}^{(2)}(t)^2\sum_{k=0}^{\infty}\frac{c_{k+1}^2}{k!}\gamma_{i_\ell^\star,j_\ell^\star}^{(1)}(t)^k;\\
        Q_2(t) & = \gamma_{i_\ell^\star,j_\ell^\star}^{(2)}(t)^2\sum_{k=0}^{\infty}\frac{c_kc_{k+2}}{k!}\gamma_{i_\ell^\star,j_\ell^\star}^{(1)}(t)^k
    \end{align*}
    By Lemma~\ref{lem:approx_lambda_hat}, we have that
    \begin{align*}
        \hat{\lambda}_{i_\ell^\star,j_\ell^\star,1}(t) & = 6\gamma_{i_\ell^\star,j_\ell^\star}^{(2)}(t)Q_1(t) \pm \calO\paren{\aggepso{\ell}(t)^2}\gamma_{i_\ell^\star,j_\ell^\star}^{(2)}(t)^2 \pm \calO\paren{\aggepso{\ell}(t)^3}\\
        \hat{\lambda}_{i_\ell^\star,j_\ell^\star,2}(t) & = 6\zeta_{i_\ell^\star,j_\ell^\star}^{(1)}(t)Q_2(t) \pm \calO\paren{\aggepso{\ell}(t)^3}\\
        \hat{\lambda}_{i_\ell^\star,j_\ell^\star,3}(t) & = 6\zeta_{i_\ell^\star,j_\ell^\star}^{(1)}(t)Q_1(t) \pm \calO\paren{\aggepso{\ell}(t)^3}\\
        \hat{\lambda}_{i_\ell^\star,j_\ell^\star,4}(t) & = 6I_{i_\ell^\star,j_\ell^\star}^{(3)}(t)Q_2(t) + 6\zeta_{i_\ell^\star,j_\ell^\star}^{(2)}(t)Q_1(t) \pm \calO\paren{\aggepso{\ell}(t)^3}\\
        \hat{\lambda}_{i_\ell^\star,j_\ell^\star,5}(t) & = 2Q_0(t) \pm \calO\paren{\aggepso{\ell}(t)^2}\gamma_{i_\ell^\star,j_\ell^\star}^{(2)}(t) \pm \calO\paren{\aggepso{\ell}(t)^3}
    \end{align*}
    We start with a detailed analysis of $I_{i_\ell^\star,j_\ell^\star}^{(3)}(t)$. Recall that
    \begin{align*}
        \derivt I_{i_\ell^\star,j_\ell^\star}^{(3)}(t) & = \frac{1}{a_i^2}\hat{\lambda}_{i_\ell^\star,j_\ell^\star,1}(t)\zeta_{i_\ell^\star,j_\ell^\star}^{(2)}(t) + \frac{9}{b_i^2}\hat{\lambda}_{i_\ell^\star,j_\ell^\star,5}(t)\zeta_{i_\ell^\star,j_\ell^\star}^{(1)}(t) + 3\hat{\lambda}_{i_\ell^\star,j_\ell^\star,2}(t)\paren{\frac{1}{a_i^2} + \frac{1}{b_i^2}}\\
        & \qqquad +\frac{3}{a_i^2}\hat{\lambda}_{i_\ell^\star,j_\ell^\star,4}(t)\gamma_{i_\ell^\star,j_\ell^\star}^{(2)}(t) + \frac{3}{b_i^2}\hat{\lambda}_{i_\ell^\star,j_\ell^\star,3}(t)\gamma_{i_\ell^\star,j_\ell^\star}^{(1)}(t) \\
        & \qqquad -36C_{S,2}\paren{\frac{1}{a_i^2} + \frac{1}{b_i^2}}I_{i_\ell^\star,i_\ell^\star}^{(3)}(t) -  \frac{1}{a_i^2}\hat{\lambda}_{i_\ell^\star,j_\ell^\star,1}(t)\gamma_{i_\ell^\star,j_\ell^\star}^{(1)}(t)I_{i_\ell^\star,i_\ell^\star}^{(3)}(t)\\
        & \qqquad - \frac{9}{b_i^2}\hat{\lambda}_{i_\ell^\star,j_\ell^\star,5}(t)\gamma_{i_\ell^\star,j_\ell^\star}^{(2)}(t)I_{i_\ell^\star,i_\ell^\star}^{(3)}(t) \pm \calO\paren{m\aggepso{\ell}(t)^2\aggepst{\ell}(t)}\\
        & = \frac{6}{a_i^2}\gamma_{i_\ell^\star,j_\ell^\star}^{(2)}(t)Q_1(t)\zeta_{i_\ell^\star,j_\ell^\star}^{(2)}(t) + \frac{18}{b_i^2}Q_0(t)\zeta_{i_\ell^\star,j_\ell^\star}^{(1)}(t) + 18Q_2(t)\zeta_{i_\ell^\star,j_\ell^\star}^{(1)}(t)\paren{\frac{1}{a_i^2} + \frac{1}{b_i^2}}\\
        & \qqquad + \frac{18}{a_i^2}Q_2(t)\gamma_{i_\ell^\star,j_\ell^\star}^{(2)}(t)I_{i_\ell^\star,j_\ell^\star}^{(3)}(t) + \frac{18}{a_i^2}Q_1(t)\gamma_{i_\ell^\star,j_\ell^\star}^{(2)}(t)\zeta_{i_\ell^\star,j_\ell^\star}^{(2)}(t) + \frac{18}{b_i^2}\gamma_{i_\ell^\star,j_\ell^\star}^{(1)}(t)Q_1(t)\zeta_{i_\ell^\star,j_\ell^\star}^{(1)}(t)\\
        & \qqquad - 36C_{S,2}\paren{\frac{1}{a_i^2} + \frac{1}{b_i^2}}I_{i_\ell^\star,i_\ell^\star}^{(3)}(t) - \frac{6}{a_i^2}\gamma_{i_\ell^\star,j_\ell^\star}^{(1)}(t)\gamma_{i_\ell^\star,j_\ell^\star}^{(2)}(t)Q_1(t)I_{i_\ell^\star,i_\ell^\star}^{(3)}(t)\\
        & \qqquad - \frac{18}{b_i^2}\gamma_{i_\ell^\star,j_\ell^\star}^{(2)}(t)Q_0(t)I_{i_\ell^\star,i_\ell^\star}^{(3)}(t) \pm \calO\paren{m\aggepso{\ell}(t)^2\aggepst{\ell}(t)}
    \end{align*}
    Next, we will look into $\zeta_{i_\ell^\star,j_\ell^\star}^{(1)}(t), \zeta_{i_\ell^\star,j_\ell^\star}^{(1)}(t)$. In particular, for $\zeta_{i_\ell^\star,j_\ell^\star}^{(1)}(t)$, we have that
    \begin{align*}
        \derivt \zeta_{i_\ell^\star,j_\ell^\star}^{(1)}(t) & = \frac{3}{a_i^2}\hat{\lambda}_{i_\ell^\star,j_\ell^\star,4} + \frac{1}{a_i^2}\hat{\lambda}_{i_\ell^\star,j_\ell^\star,2}(t)\gamma_{i_\ell^\star,j_\ell^\star}^{(2)}(t) - \frac{1}{a_i^2}\hat{\lambda}_{i_\ell^\star,j_\ell^\star,1}(t)\gamma_{i_\ell^\star,j_\ell^\star}^{(1)}(t)\zeta_{i_\ell^\star,j_\ell^\star}^{(1)}(t)\\
        & \qqquad -\frac{36}{a_i^2}C_{S,2}\gamma_{i_\ell^\star,j_\ell^\star}^{(2)}(t)I_{i_\ell^\star,j_\ell^\star}^{(3)}(t) \pm \calO\paren{m\aggepso{\ell}(t)^2\aggepst{\ell}(t)}\\
        & = \frac{18}{a_i^2}Q_2(t)I_{i_\ell^\star,j_\ell^\star}^{(3)}(t) + \frac{18}{a_i^2}Q_1(t)\zeta_{i_\ell^\star,j_\ell^\star}^{(2)}(t) + \frac{6}{a_i^2}\gamma_{i_\ell^\star,j_\ell^\star}^{(2)}(t)Q_2(t)\zeta_{i_\ell^\star,j_\ell^\star}^{(1)}(t)\\
        & \qqquad - \frac{6}{a_i^2}\gamma_{i_\ell^\star,j_\ell^\star}^{(1)}(t)\gamma_{i_\ell^\star,j_\ell^\star}^{(2)}(t)Q_1(t)\zeta_{i_\ell^\star,j_\ell^\star}^{(1)}(t) -\frac{36}{a_i^2}C_{S,2}\gamma_{i_\ell^\star,j_\ell^\star}^{(2)}(t)I_{i_\ell^\star,j_\ell^\star}^{(3)}(t) \pm \calO\paren{m\aggepso{\ell}(t)^2\aggepst{\ell}(t)}
    \end{align*}
    Similarly, for $\zeta_{i_\ell^\star,j_\ell^\star}^{(2)}(t)$, we have that
    \begin{align*}
        \derivt \zeta_{i_\ell^\star,j_\ell^\star}^{(2)}(t) & = \frac{3}{b_i^2}\hat{\lambda}_{i_\ell^\star,j_\ell^\star,3}(t) + \frac{3}{b_i^2}\hat{\lambda}_{i_\ell^\star,j_\ell^\star,2}(t)\gamma_{i_\ell^\star,j_\ell^\star}^{(1)}(t) - \frac{9}{b_i^2}\hat{\lambda}_{i_\ell^\star,j_\ell^\star,5}(t)\gamma_{i_\ell^\star,j_\ell^\star}^{(2)}(t)\zeta_{i_\ell^\star,j_\ell^\star}^{(2)}(t)\\
        & \qqquad -\frac{36}{b_i^2}C_{S,2}\gamma_{i_\ell^\star,j_\ell^\star}^{(2)}(t)I_{i_\ell^\star,j_\ell^\star}^{(3)}(t) \pm \calO\paren{m\aggepso{\ell}(t)^2\aggepst{\ell}(t)}\\
        & = \frac{18}{b_i^2}Q_1(t)\zeta_{i_\ell^\star,j_\ell^\star}^{(1)}(t) + \frac{18}{b_i^2}\gamma_{i_\ell^\star,j_\ell^\star}^{(2)}(t)Q_2(t)\zeta_{i_\ell^\star,j_\ell^\star}^{(1)}(t) - \frac{18}{b_i^2}Q_0(t)\gamma_{i_\ell^\star,j_\ell^\star}^{(2)}(t)\zeta_{i_\ell^\star,j_\ell^\star}^{(2)}(t)\\
        & \qqquad -\frac{36}{b_i^2}C_{S,2}\gamma_{i_\ell^\star,j_\ell^\star}^{(2)}(t)I_{i_\ell^\star,j_\ell^\star}^{(3)}(t) \pm \calO\paren{m\aggepso{\ell}(t)^2\aggepst{\ell}(t)}
    \end{align*}
    Then we can write $\zeta_{i_\ell^\star,j_\ell^\star}^{(1)}(t), \zeta_{i_\ell^\star,j_\ell^\star}^{(2)}(t)$, and $I_{i_\ell^\star,j_\ell^\star}^{(3)}(t)$ into a system given by
    \begin{align*}
        \derivt I_{i_\ell^\star,j_\ell^\star}^{(3)}(t) = -\alpha_1I_{i_\ell^\star,j_\ell^\star}^{(3)}(t) + \iota_1\zeta_{i_\ell^\star,j_\ell^\star}^{(1)}(t) + \rho_1\zeta_{i_\ell^\star,j_\ell^\star}^{(2)}(t)\pm \calO\paren{m\aggepso{\ell}(t)^2\aggepst{\ell}(t)}\\
        \derivt \zeta_{i_\ell^\star,j_\ell^\star}^{(1)}(t) = -\alpha_2I_{i_\ell^\star,j_\ell^\star}^{(3)}(t) - \iota_2\zeta_{i_\ell^\star,j_\ell^\star}^{(1)}(t) + \rho_2\zeta_{i_\ell^\star,j_\ell^\star}^{(2)}(t)\pm \calO\paren{m\aggepso{\ell}(t)^2\aggepst{\ell}(t)}\\
        \derivt \zeta_{i_\ell^\star,j_\ell^\star}^{(2)}(t) = -\alpha_3I_{i_\ell^\star,j_\ell^\star}^{(3)}(t) + \iota_3\zeta_{i_\ell^\star,j_\ell^\star}^{(1)}(t) + \rho_3\zeta_{i_\ell^\star,j_\ell^\star}^{(2)}(t)\pm \calO\paren{m\aggepso{\ell}(t)^2\aggepst{\ell}(t)}
    \end{align*}
    where
    \begin{gather*}
        \alpha_1 = 36C_{S,2}\paren{\frac{1}{a_i^2} + \frac{1}{b_i^2}} + \frac{6}{a_i^2}\gamma_{i_\ell^\star,j_\ell^\star}^{(1)}(t)\gamma_{i_\ell^\star,j_\ell^\star}^{(2)}(t)Q_1(t) + \frac{18}{b_i^2}\gamma_{i_\ell^\star,j_\ell^\star}^{(1)}(t)Q_1(t) -  \frac{18}{a_i^2}\gamma_{i_\ell^\star,j_\ell^\star}^{(2)}(t)Q_2(t)\\
        \alpha_2 =\frac{36}{a_i^2}C_{S,2}\gamma_{i_\ell^\star,j_\ell^\star}^{(2)}(t) - \frac{18}{a_i^2}Q_2(t);\quad \alpha_3 =\frac{36}{b_i^2}C_{S,2}\gamma_{i_\ell^\star,j_\ell^\star}^{(2)}(t)\\
        \iota_1 =\frac{18}{b_i^2}\paren{Q_0(t) + \gamma_{i_\ell^\star,j_\ell^\star}^{(1)}(t)Q_1(t)} + 18\paren{\frac{1}{a_i^2} + \frac{1}{b_i^2}}Q_2(t);\\ \iota_2 = \frac{6}{a_i^2}\paren{\gamma_{i_\ell^\star,j_\ell^\star}^{(1)}(t)\gamma_{i_\ell^\star,j_\ell^\star}^{(2)}(t)Q_1(t) - \gamma_{i_\ell^\star,j_\ell^\star}^{(2)}(t)Q_2(t)}\\
        \iota_3 = \frac{18}{b_i^2}\paren{Q_1(t) + \gamma_{i_\ell^\star,j_\ell^\star}^{(1)}(t)Q_2(t)};\quad \rho_1 = \frac{24}{a_i^2}\gamma_{i_\ell^\star,j_\ell^\star}^{(1)}(t)Q_1(t)\\
        \rho_2 = \frac{18}{a_i^2}Q_1(t);\quad \rho_3 = \frac{18}{b_i^2}Q_0(t)\gamma_{i_\ell^\star,j_\ell^\star}^{(2)}(t)
    \end{gather*}
\end{proof}

By Lemma~\ref{lem:ode_stability}, we first need to check that $\iota_2\rho_3 \geq \iota_3\rho_2$
\begin{align*}
    \iota_2\rho_3 - \iota_3\rho_2 & =  \frac{108}{a_i^2b_i^2}\paren{\gamma_{i_\ell^\star,j_\ell^\star}^{(1)}(t)\gamma_{i_\ell^\star,j_\ell^\star}^{(2)}(t)Q_1(t) - \gamma_{i_\ell^\star,j_\ell^\star}^{(2)}(t)Q_2(t)}\gamma_{i_\ell^\star,j_\ell^\star}^{(2)}(t)Q_0(t)\\
    & \qqquad - \frac{324}{a_i^2b_i^2}\paren{Q_1(t) + \gamma_{i_\ell^\star,j_\ell^\star}^{(1)}(t)Q_2(t)}Q_1(t)
\end{align*}
By Lemma~\ref{lem:Q2_bound}, we have that $Q_2(t) \leq 0$ for all $\gamma_{i_\ell^\star,j_\ell^\star}^{(1)}(t), \gamma_{i_\ell^\star,j_\ell^\star}^{(2)}(t) \geq 0$. Therefore
\[
    \iota_2\rho_3 - \iota_3\rho_2 \geq \frac{108}{a_i^2b_i^2}Q_1(t)\paren{\gamma_{i_\ell^\star,j_\ell^\star}^{(1)}(t)\gamma_{i_\ell^\star,j_\ell^\star}^{(2)}(t)^2Q_0(t) - 3Q_1(t)}
\]
Thus the condition holds if $Q_0(t) \geq 4.12Q_1(t)$ given that $\gamma_{i_\ell^\star,j_\ell^\star}^{(1)}(t), \gamma_{i_\ell^\star,j_\ell^\star}^{(2)}(t) \geq 0.9$, which translates to
\[
    \sum_{k=0}^{\infty}\frac{c_{k}^2}{k!}\gamma_{i_\ell^\star,j_\ell^\star}^{(1)}(t)^k \geq 4.12\sum_{k=0}^{\infty}\frac{c_{k+1}^2}{k!}\gamma_{i_\ell^\star,j_\ell^\star}^{(1)}(t)^k
\]
where we can use
\[
    \sum_{k=0}^{\infty}\frac{c_{k}^2}{k!}\gamma_{i_\ell^\star,j_\ell^\star}^{(1)}(t)^k \geq c_0^2 = 0.25
\]
and noticing that $\gamma_{i_\ell^\star,j_\ell^\star}^{(2)}(t)^{-2}Q_1(t) \leq \sum_{k=0}^{\infty}\frac{c_{k+1}^2}{k!} = \EXP[x\sim\mathcal{N}\paren{0,1}]{\pi'(x)^2} \leq 0.05$. Then, we need to check that
\[
    \alpha_1^2\iota_2 + \alpha_1^2\rho_3 + \alpha_1\alpha_2\iota_1 + \alpha_1\alpha_3\rho_1 + \alpha_1\iota_2^2 + \alpha_2\iota_1\iota_2  + \alpha_1\rho_3^2 + \alpha_3\rho_1\rho_3 > \alpha_2\iota_3\rho_1 + \alpha_3\iota_1\rho_2
\]
To start, we notice that
\[
    \alpha_1\alpha_2 - \alpha_3\rho_2 = \frac{36}{a_i^2}C_{S,2}\gamma_{i_\ell^\star,j_\ell^\star}^{(2)}(t)\paren{\alpha_1 - \frac{a_i^2}{b_i^2}\cdot \rho_2} = \frac{36}{a_i^2}C_{S,2}\gamma_{i_\ell^\star,j_\ell^\star}^{(2)}(t)\paren{\alpha_1 - \frac{18}{b_i^2}Q_1(t)} > 0
\]
when $\gamma_{i_\ell^\star,j_\ell^\star}^{(1)}(t), \gamma_{i_\ell^\star,j_\ell^\star}^{(2)}(t) \geq 0.9$. Thus, we have that $\alpha_1\alpha_2\iota_1 > \alpha_3\iota_1\rho_2$. Lastly, we also notice that
\begin{align*}
    \iota_1\iota_2 - \iota_3\rho_1 & \geq \frac{108}{a_i^2b_i^2}\gamma_{i_\ell^\star,j_\ell^\star}^{(1)}(t)\gamma_{i_\ell^\star,j_\ell^\star}^{(2)}(t)\paren{Q_0(t) + \gamma_{i_\ell^\star,j_\ell^\star}^{(1)}(t)Q_1(t)}Q_1(t)\\
    & \qqquad + \frac{108}{a_i^2}\paren{\frac{1}{a_i^2} + \frac{1}{b_i^2}}\gamma_{i_\ell^\star,j_\ell^\star}^{(1)}(t)\gamma_{i_\ell^\star,j_\ell^\star}^{(2)}(t)Q_1(t)Q_2(t) - \frac{432}{a_i^2b_i^2}\gamma_{i_\ell^\star,j_\ell^\star}^{(1)}(t)Q_1(t)^2\\
    & \geq \frac{108}{a_i^2b_i^2}\gamma_{i_\ell^\star,j_\ell^\star}^{(1)}(t)Q_1(t)\paren{\gamma_{i_\ell^\star,j_\ell^\star}^{(2)}(t)Q_0(t) + \gamma_{i_\ell^\star,j_\ell^\star}^{(1)}(t)\gamma_{i_\ell^\star,j_\ell^\star}^{(2)}(t)Q_1(t) - 4Q_1(t)}\\
    & \qqquad + \frac{108}{a_i^2}\paren{\frac{1}{a_i^2} + \frac{1}{b_i^2}}\gamma_{i_\ell^\star,j_\ell^\star}^{(1)}(t)\gamma_{i_\ell^\star,j_\ell^\star}^{(2)}(t)Q_1(t)Q_2(t)
\end{align*}
which can be numerically verified as $\gamma_{i_\ell^\star,j_\ell^\star}^{(2)}(t) \geq 0.9$. Therefore, by Lemma~\ref{lem:ode_stability} we have that
\begin{align*}
    & \max\left\{\left|\zeta_{i_\ell^\star,j_\ell^\star}^{(1)}(T_1+t)\right|, \left|\zeta_{i_\ell^\star,j_\ell^\star}^{(1)}(T_1+t)\right|, \left|I_{i_\ell^\star,i_\ell^\star}^{(3)}(T_1+t)\right|\right\}\\
    & \qqquad \leq e^{-\Omega\paren{t}}\paren{\left|\zeta_{i_\ell^\star,j_\ell^\star}^{(1)}(T_1)\right|+\left|\zeta_{i_\ell^\star,j_\ell^\star}^{(1)}(T_1)\right|+ \left|I_{i_\ell^\star,i_\ell^\star}^{(3)}(T_1)\right|} + \calO\paren{m\aggepst{\ell}(t)^3}\\
    & \qqquad \leq e^{-\Omega\paren{t}}\calO\paren{\frac{m^2}{\delta_{\mathbb{P}}\sqrt{d}}} + \calO\paren{m\aggepso{\ell}(t)^2\aggepst{\ell}(t)}
\end{align*}
Thus, for $t\geq T_1 + \calO\paren{\log d}$ we shall have that 
\[
    \max\left\{\left|\zeta_{i_\ell^\star,j_\ell^\star}^{(1)}(t)\right|, \left|\zeta_{i_\ell^\star,j_\ell^\star}^{(1)}(t)\right|, \left|I_{i_\ell^\star,i_\ell^\star}^{(3)}(t)\right|\right\} \leq \calO\paren{m\aggepso{\ell}(t)^2 + \frac{m^7}{\delta_{\mathbb{P}}^3d^{\frac{3}{2}}}}
\]
The same holds for $\zeta_{i_\ell^\star,j}^{(1)}(t), \zeta_{i_\ell^\star,j}^{(1)}(t), I_{i,i}^{(3)}(t)$ as their growth is upper bounded by the above.

\textbf{Bounding $\zeta_{i_\ell^\star,j}^{(1)}(t), \zeta_{i_\ell^\star,j}^{(2)}(t)$ and $\gamma_{i_{\ell}^\star,j}^{(1)}(t), \gamma_{i_{\ell}^\star,j}^{(1)}(t)$ for $j\in[m^\star]\setminus \left\{j_\ell^\star\right\}$}

\begin{lemma}
    \label{lem:i_star_ub}
    Suppose that the inductive hypothesis in Condition~\ref{cond:inductive_hypo} and the initialization condition in Condition~\ref{cond:init} hold. Then for all $t\geq 0$ we have that
    \begin{gather*}
        \left|\zeta_{i_\ell^\star,j}^{(1)}(t)\right|, \left|\zeta_{i_\ell^\star,j}^{(2)}(t)\right| \leq \calO\paren{\frac{m^2}{\delta_{\mathbb{P}}\sqrt{d}} + m\aggepso{\ell}(t)^3 + \aggepso{\ell}(t)^2};\;\forall\;j\in[m^\star]\\
        \left|\gamma_{i_\ell^\star,j}^{(1)}(t)\right|, \left|\gamma_{i_\ell^\star,j}^{(2)}(t)\right| \leq \calO\paren{\frac{m^2}{\delta_{\mathbb{P}}\sqrt{d}} + m\aggepso{\ell}(t)^3 + \aggepso{\ell}(t)^2};\;\forall\;j\neq j_\ell^\star
    \end{gather*}
\end{lemma}
\begin{proof}
    For $j\in[m^\star]\setminus \left\{j_\ell^\star\right\}$, we write out the dynamic of $\zeta_{i_\ell^\star,j}^{(1)}(t), \zeta_{i_\ell^\star,j}^{(2)}(t)$ from Lemma~\ref{lem:dynamic_approx}
    \begin{align*}
        \derivt \left|\zeta_{i_\ell^\star,j}^{(1)}(t)\right| & = - \frac{1}{a_{i_\ell^\star}^2}\hat{\lambda}_{i_\ell^\star,j_\ell^\star,1}(t)\gamma_{i_\ell^\star,j_\ell^\star}^{(1)}(t)\left|\zeta_{i_\ell^\star,j}^{(1)}(t)\right| + \calO\paren{m\aggepso{\ell}(t)^3 + \aggepso{\ell}(t)^2}\\
        \derivt \left|\zeta_{i_\ell^\star,j}^{(2)}(t)\right| & = - \frac{9}{b_{i_\ell^\star}^2}\hat{\lambda}_{i_\ell^\star,j_\ell^\star,5}(t)\gamma_{i_\ell^\star,j_\ell^\star}^{(2)}(t)\left|\zeta_{i_\ell^\star,j}^{(2)}(t)\right| + \calO\paren{m\aggepso{\ell}(t)^3 + \aggepso{\ell}(t)^2}
    \end{align*}
    where we applied the upper bound that $\left|\hat{\lambda}_{i_\ell^\star,j_{\ell}^\star,2}(t)\gamma_{i_\ell^\star,j}^{(1)}(t)\right|, \left|\hat{\lambda}_{i_\ell^\star,j_{\ell}^\star,2}(t)\gamma_{i_\ell^\star,j}^{(2)}(t)\right| \leq \calO\paren{\aggepso{\ell}(t)^2}$ and $\left|\gamma_{i_\ell^\star,j}^{(2)}(t)I_{i_\ell^\star,i_\ell^\star}^{(3)}(t)\right| \leq \calO\paren{\frac{m^2}{\delta_{\mathbb{P}}\sqrt{d}}}$ from Lemma~\ref{lem:exp_decay}.
    Since both $\hat{\lambda}_{i_\ell^\star,j_\ell^\star,1}(t)\gamma_{i_\ell^\star,j_\ell^\star}^{(1)}(t)$ and $\hat{\lambda}_{i_\ell^\star,j_\ell^\star,5}(t)\gamma_{i_\ell^\star,j_\ell^\star}^{(2)}(t)$ are positive, we have that $\left|\zeta_{i_\ell^\star,j}^{(1)}(t)\right|$ and $\left|\zeta_{i_\ell^\star,j}^{(2)}(t)\right|$ enjoys an exponential decay up to $\calO\paren{m\aggepso{\ell}(t)^3 + \aggepso{\ell}(t)^2}$. Recall that at the end of phase 1 we have $\left|\zeta_{i_\ell^\star,j}^{(1)}(t)\right|, \left|\zeta_{i_\ell^\star,j}^{(2)}(t)\right| \leq \calO\paren{\frac{m^2}{
    \delta_{\mathbb{P}}\sqrt{d}}}$. Therefore, we can conclude that
    \[
        \left|\zeta_{i_\ell^\star,j}^{(1)}(t)\right|, \left|\zeta_{i_\ell^\star,j}^{(2)}(t)\right| \leq \calO\paren{\frac{m^2}{\delta_{\mathbb{P}}\sqrt{d}} + m\aggepso{\ell}(t)^3 + \aggepso{\ell}(t)^2}
    \]
    Now, we focus on $\gamma_{i_\ell^\star,j}^{(1)}(t)$ and $\gamma_{i_\ell^\star,j}^{(2)}(t)$ for $j\neq j_\ell^\star$. In particular, by Lemma~\ref{lem:dynamic_approx}, we have that
    \begin{align*}
        \derivt \gamma_{i_\ell^\star,j}^{(1)}(t) & = -\frac{1}{a_{i_\ell^\star}^2}\hat{\lambda}_{i_\ell^\star,j_\ell^\star,1}(t)\gamma_{i_\ell^\star,j_\ell^\star}^{(1)}(t)\left|\gamma_{i_\ell^\star,j}^{(1)}(t)\right| + \calO\paren{m\aggepso{\ell}(t)^3 + \aggepso{\ell}(t)^2}\\
        \derivt \gamma_{i_\ell^\star,j}^{(2)}(t) & = -\frac{9}{b_{i_\ell^\star}^3}\hat{\lambda}_{i_\ell^\star,j_\ell^\star,5}(t)\gamma_{i_\ell^\star,j_\ell^\star}^{(2)}(t)\left|\gamma_{i_\ell^\star,j}^{(2)}(t)\right| + \calO\paren{m\aggepso{\ell}(t)^3 + \aggepso{\ell}(t)^2}
    \end{align*}
    Since both $\hat{\lambda}_{i_\ell^\star,j_\ell^\star,1}(t)\gamma_{i_\ell^\star,j_\ell^\star}^{(1)}(t)$ and $\hat{\lambda}_{i_\ell^\star,j_\ell^\star,5}(t)\gamma_{i_\ell^\star,j_\ell^\star}^{(2)}(t)$ are positive, we have that $\left|\gamma_{i_\ell^\star,j}^{(1)}(t)\right|$ and $\left|\gamma_{i_\ell^\star,j}^{(2)}(t)\right|$ for $j\neq j_\ell^\star$ enjoys an exponential decay up to $\calO\paren{m\aggepso{\ell}(t)^3 + \aggepso{\ell}(t)^2}$. Thus, we can conclude that
    \[
        \left|\gamma_{i_\ell^\star,j}^{(1)}(t)\right|, \left|\gamma_{i_\ell^\star,j}^{(2)}(t)\right| \leq \calO\paren{\frac{m^2}{\delta_{\mathbb{P}}\sqrt{d}}+m\aggepso{\ell}(t)^3 + \aggepso{\ell}(t)^2};\;\forall\;j\neq j_\ell^\star
    \]
\end{proof}

\textbf{Phase 2 Growth of $\gamma_{i_\ell^\star,j_\ell^\star}^{(1)}(t), \gamma_{i_\ell^\star,j_\ell^\star}^{(2)}(t)$}. In this section, we show that $\gamma_{i_\ell^\star,j_\ell^\star}^{(2)}(t)$ continues growing up to at least $1 - \calO\paren{m\aggepso{\ell}(t)^3 + \frac{m^7}{\delta_{\mathbb{P}}^3d^{\frac{3}{2}}}}$.

\begin{lemma}
    \label{lem:growth_to_perfect}
    Suppose that the inductive hypothesis in Condition~\ref{cond:inductive_hypo} and the initialization condition in Condition~\ref{cond:init} hold. Then for all $t\geq T_1 + \calO\paren{\log d}$, where $T_1$ is defined in Lemma~\ref{lem:phase_1_conv}, we have that
    \[
        \gamma_{i_\ell^\star,j_\ell^\star}^{(1)}(t),\gamma_{i_\ell^\star,j_\ell^\star}^{(2)}(t)\geq 1 - \calO\paren{m\aggepso{\ell}(t)^3 + \frac{m^7}{\delta_{\mathbb{P}}^3d^{\frac{3}{2}}}}
    \]
\end{lemma}
\begin{proof}
    To start, we need to perform a more fine-grained analysis of the dynamic of $\gamma_{i_\ell^\star,j_\ell^\star}^{(2)}(t)$. Recall from the proof of Lemma~\ref{lem:dynamic_approx} we have that for $(i,j) = (i_\ell^\star,j_\ell^\star)$
    \begin{align*}
        \sum_{r=1}^m\lambda_{i,r,5}(t)\gamma_{r,j}^{(2)}(t) & = \lambda_{i,i,5}(t)\gamma_{i,j}^{(5)}(t) \pm \calO\paren{m\aggepso{\ell}(t)^2\aggepst{\ell}(t)}\\
        \sum_{r=1}^m\lambda_{i,r,2}(t)\zeta_{i,j}^{(1)}(t) & = \pm \calO\paren{m\aggepso{\ell}(t)^3} \pm \calO\paren{\aggepso{\ell}(t)\varepsilon_{5,\ell}(t)}\\
        \sum_{r=1}^m\lambda_{i,r,3}(t)\zeta_{r,j}^{(1)}(t) & = \pm \calO\paren{m\aggepso{\ell}(t)^3} \pm \calO\paren{\aggepso{\ell}(t)\varepsilon_{5,\ell}(t)}\\
        \sum_{r=1}^{m^\star}\hat{\lambda}_{i,r,2}(t)\zeta_{i,j}^{(1)}(t) & = \pm \calO\paren{\aggepso{\ell}(t)^2}\gamma_{i_\ell^\star,j_\ell^\star}^{(2)}(t)^2 \pm \calO\paren{m\aggepso{\ell}(t)^3}
    \end{align*}
    and also
    \begin{align*}
        \sum_{r=1}^{m}\lambda_{i,r,2}(t)I_{i,i}^{(3)}(t) & = \pm\calO\paren{\varepsilon_{5,\ell}(t)^2 + m\aggepso{\ell}(t)^3}\\
        \sum_{r=1}^m\lambda_{i,r,3}(t)I_{r,i}^{(3)}(t) & = \pm \calO\paren{\varepsilon_{5,\ell}(t)^2 + m\aggepst{\ell}(t)^3}\\
        \sum_{r=1}^m\lambda_{i,r,5}(t)I_{r,i}^{(2)}(t) & = \lambda_{i,i,5}(t) \pm \calO\paren{m\aggepso{\ell}(t)^3}\\
        \sum_{r=1}^{m^\star}\hat{\lambda}_{i,r,2}(t)I_{i,i}^{(3)}(t) & = \pm \calO\paren{\aggepso{\ell}(t)^2}\gamma_{i_\ell^\star,j_\ell^\star}^{(2)}(t)^2 \pm \calO\paren{m\aggepso{\ell}(t)^3}\\
        \sum_{r=1}^{m^\star}\hat{\lambda}_{i,r,3}(t)\zeta_{i,r}^{(2)}(t) & = \pm \calO\paren{\aggepso{\ell}(t)^2}\gamma_{i_\ell^\star,j_\ell^\star}^{(2)}(t)^2 \pm \calO\paren{m\aggepso{\ell}(t)^3}\\
        \sum_{r=1}^{m^\star}\hat{\lambda}_{i,r,5}(t)\gamma_{i,r}^{(2)}(t) & = \lambda_{i_\ell^\star,j_\ell^\star,5}(t)\gamma_{i_\ell^\star,j_\ell^\star}^{(2)}(t) \pm \calO\paren{m\aggepso{\ell}(t)^3}
    \end{align*}
    In this case, we need to refine the bound that
    \begin{align*}
        \sum_{r=1}^{m^\star}\hat{\lambda}_{i_\ell^\star,r,2}(t)\zeta_{i_\ell^\star,j_\ell^\star}^{(1)}(t) & = \hat{\lambda}_{i_\ell^\star,j_\ell^\star,2}(t)\zeta_{i_\ell^\star,j_\ell^\star}^{(1)}(t) \pm \calO\paren{m\aggepso{\ell}(t)^3}\\
        \sum_{r=1}^{m^\star}\hat{\lambda}_{i_\ell^\star,r,2}(t)I_{i_\ell^\star,i_\ell^\star}^{(3)}(t) & = \hat{\lambda}_{i_\ell^\star,j_\ell^\star,2}(t)I_{i_\ell^\star,i_\ell^\star}^{(3)}(t) \pm \calO\paren{m\aggepso{\ell}(t)^3}\\
        \sum_{r=1}^{m^\star}\hat{\lambda}_{i_\ell^\star,r,3}(t)\zeta_{i_\ell^\star,r}^{(2)}(t) & = \hat{\lambda}_{i_\ell^\star,j_\ell^\star,3}(t)\zeta_{i_\ell^\star,j_\ell^\star}^{(2)}(t) \pm \calO\paren{m\aggepso{\ell}(t)^3}
    \end{align*}
    Denote $\hat{\varepsilon}_{\ell}(t) = \max\left\{\left|\zeta_{i_\ell^\star,j_\ell^\star}^{(1)}(t)\right|, \left|\zeta_{i_\ell^\star,j_\ell^\star}^{(2)}(t)\right|, \left|I_{i_\ell^\star,i_\ell^\star}^{(3)}(t)\right|\right\}$. Then we have that
    \begin{align*}
        \left|\sum_{r=1}^{m^\star}\hat{\lambda}_{i_\ell^\star,r,2}(t)\zeta_{i_\ell^\star,j_\ell^\star}^{(1)}(t)\right| & \leq \calO\paren{\hat{\varepsilon}_{\ell}(t)^2 + m\aggepso{\ell}(t)^3}\\
        \left|\sum_{r=1}^{m^\star}\hat{\lambda}_{i_\ell^\star,r,2}(t)I_{i_\ell^\star,i_\ell^\star}^{(3)}(t)\right| & \leq \calO\paren{\hat{\varepsilon}_{\ell}(t)^2 + m\aggepso{\ell}(t)^3}\\
        \left|\sum_{r=1}^{m^\star}\hat{\lambda}_{i_\ell^\star,r,3}(t)\zeta_{i_\ell^\star,r}^{(2)}(t)\right| & \leq \calO\paren{\hat{\varepsilon}_{\ell}(t)^2 + m\aggepso{\ell}(t)^3}
    \end{align*}
    Noticing that $\varepsilon_{5,\ell}(t) \leq \hat{\varepsilon}_{\ell}(t)$, we have that 
    \[
        \derivt \gamma_{i_\ell^\star,j_\ell^\star}^{(2)}(t) \geq \paren{1-\gamma_{i_\ell^\star,j_\ell^\star}^{(2)}(t)^2}\hat{\lambda}_{i_\ell^\star,j_\ell^\star,5}(t) - \calO\paren{\hat{\varepsilon}_{\ell}(t)^2 + m\aggepso{\ell}(t)^3}
    \]
    By Lemma~\ref{lem:cs_hermite} we have that
    \[
        \hat{\lambda}_{i_\ell^\star,j_\ell^\star,5}(t) \geq  2\gamma_{i_\ell^\star,j_\ell^\star}^{(2)}(t)\sum_{k=0}^{\infty}\frac{c_k^2}{k!}\gamma_{i_\ell^\star,j_\ell^\star}^{(1)}(t)^k - \calO\paren{\hat{\varepsilon}_{\ell}(t)^2}
    \]
    This gives that
    \begin{align*}
        \derivt \gamma_{i_\ell^\star,j_\ell^\star}^{(2)}(t) & \geq 2\paren{1-\gamma_{i_\ell^\star,j_\ell^\star}^{(2)}(t)^2}\gamma_{i_\ell^\star,j_\ell^\star}^{(2)}(t)^2\sum_{k=0}^{\infty}\frac{c_k^2}{k!}\gamma_{i_\ell^\star,j_\ell^\star}^{(1)}(t)^k - \calO\paren{\hat{\varepsilon}_{\ell}(t)^2 + m\aggepso{\ell}(t)^3}\\
        & \geq c_0^2\paren{1-\gamma_{i_\ell^\star,j_\ell^\star}^{(2)}(t)^2}- \calO\paren{\hat{\varepsilon}_{\ell}(t)^2 + m\aggepso{\ell}(t)^3}
    \end{align*}
    Let $\hat{T}$ be the first time when $\hat{\varepsilon}_{\ell}(t) \leq \calO\paren{m\aggepso{\ell}(t)^3 + \frac{m^7}{\delta_{\mathbb{P}}^3d^{\frac{3}{2}}}}$. Recall that for all $t\leq \hat{T}$ we have $\hat{\varepsilon}_{\ell}(t) \leq \calO\paren{\frac{m^2}{\delta_{\mathbb{P}}\sqrt{d}}}$. Therefore, we have that either $\gamma_{i_\ell^\star,j_\ell^\star}^{(2)}(t) \geq 1 - \calO\paren{\frac{m^4}{\delta_{\mathbb{P}}^2d}}$ or $\derivt \gamma_{i_\ell^\star,j_\ell^\star}^{(2)}(t) \geq 0$. Thus, we must have that at $\hat{T}$
    \[
        \gamma_{i_\ell^\star,j_\ell^\star}^{(2)}(\hat{T})\geq \gamma_{i_\ell^\star,j_\ell^\star}^{(2)}(T_1) \geq 0.9
    \]
    For all $t\geq \hat{T}$, we have that
    \[
        \derivt \gamma_{i_\ell^\star,j_\ell^\star}^{(2)}(t) \geq c_0^2\paren{1-\gamma_{i_\ell^\star,j_\ell^\star}^{(2)}(t)^2} - \calO\paren{m\aggepso{\ell}(t)^3 + \frac{m^7}{\delta_{\mathbb{P}}^3d^{\frac{3}{2}}}}
    \]
    Before $\gamma_{i_\ell^\star,j_\ell^\star}^{(2)}(t)$ first reaches $\calO\paren{m\aggepso{\ell}(t)^3 + \frac{m^7}{\delta_{\mathbb{P}}^3d^{\frac{3}{2}}}}$, we have that
    \[
        \derivt \gamma_{i_\ell^\star,j_\ell^\star}^{(2)}(t) \geq \frac{c_0^2}{2}\paren{1-\gamma_{i_\ell^\star,j_\ell^\star}^{(2)}(t)^2}
    \]
    This gives that
    \[
        \gamma_{i_\ell^\star,j_\ell^\star}^{(2)}(t+\hat{T}) \geq \frac{\paren{1 + \gamma_{i_\ell^\star,j_\ell^\star}^{(2)}(\hat{T})}\exp{\frac{c_0^2t}{2}} - 1 + \gamma_{i_\ell^\star,j_\ell^\star}^{(2)}(\hat{T})}{\paren{1 + \gamma_{i_\ell^\star,j_\ell^\star}^{(2)}(\hat{T})}\exp{\frac{c_0^2t}{2}} + 1 - \gamma_{i_\ell^\star,j_\ell^\star}^{(2)}(\hat{T})} \geq \frac{\exp{\frac{c_0^2t}{2}} - 0.1}{\exp{\frac{c_0^2t}{2}} + 0.1}
    \]
    Thus, for $t\geq \calO\paren{\log d}$ we have that $\gamma_{i_\ell^\star,j_\ell^\star}^{(2)}(t+\hat{T})\geq 1 - \calO\paren{\frac{m^7}{\delta_{\mathbb{P}}^3d^{\frac{3}{2}}} + m\aggepso{\ell}(t)^3}$ and stays at that magnitude. We conclude the proof by noticing that $\hat{T}\leq T_1 + \calO\paren{\log d}$ by Lemma~\ref{lem:exp_decay}, and the same analysis holds for $\gamma_{i_\ell^\star,j_\ell^\star}^{(1)}(t)$.
\end{proof}

With all the preparation work, we are ready to state the lemma for phase 2 convergence.

\begin{lemma}
    \label{lem:phase2}
    Suppose that the inductive hypothesis in Condition~\ref{cond:inductive_hypo}, the initialization condition in Condition~\ref{cond:init}, and the condition on the sigmoid function in Assumption~\ref{asump:sigmoid} holds. Let $T_c$ be defined in (\ref{eq: T_c_def}). Then for all $t\geq T_1$, where $T_1$ is defined in Lemma~\ref{lem:phase_1_conv}, we have that
    \[
        \gamma_{i_\ell^\star,j_\ell^\star}^{(1)}(t),\gamma_{i_\ell^\star,j_\ell^\star}^{(2)}(t)\geq \begin{cases}
            0.9 & \text{ for } t \geq T_1\\
            1 - \calO\paren{m\aggepso{\ell}(t)^3 + \frac{m^7}{\delta_{\mathbb{P}}^3d^{\frac{3}{2}}}} & \text{ for } t \geq T_1 + \calO\paren{\log d}
        \end{cases}
    \]
    Moreover, for all $t \geq T_1$, we can upper bound $\left|\zeta_{i_\ell^\star,j}^{(1)}(t)\right|, \left|\zeta_{i_\ell^\star,j}^{(2)}(t)\right|$ for $j\in[m^\star]$ and $\left|\gamma_{i_\ell^\star,j}^{(1)}(t)\right|, \left|\gamma_{i_\ell^\star,j}^{(2)}(t)\right|$ for $j\in[m^\star]\setminus \{j_{\ell}^\star\}$ by
    \begin{gather*}
        \left|\zeta_{i_\ell^\star,j}^{(1)}(t)\right|, \left|\zeta_{i_\ell^\star,j}^{(2)}(t)\right| \leq \calO\paren{\frac{m^2}{\delta_{\mathbb{P}}\sqrt{d}} + m\aggepso{\ell}(t)^3 + \aggepso{\ell}(t)^2};\;\forall\;j\in[m^\star]\\
        \left|\gamma_{i_\ell^\star,j}^{(1)}(t)\right|, \left|\gamma_{i_\ell^\star,j}^{(2)}(t)\right| \leq \calO\paren{\frac{m^2}{\delta_{\mathbb{P}}\sqrt{d}} + m\aggepso{\ell}(t)^3 + \aggepso{\ell}(t)^2};\;\forall\;j\neq j_\ell^\star
    \end{gather*}
\end{lemma}
\begin{proof}
    The first part of the proof simply follows from a combination of Lemma~\ref{lem:exp_decay}, Lemma~\ref{lem:i_star_ub}, and Lemma~\ref{lem:growth_to_perfect}.
\end{proof}

\subsection{Formalizing the Proof of Theorem~\ref{thm:gf_conv}}
Now we are ready to prove the theorem for gradient flow.
\begin{proof}[Proof of Theorem~\ref{thm:gf_conv}]
    By Lemma~\ref{lem:phase2}, we have that
    \begin{gather*}
        \left|\zeta_{i_\ell^\star,j}^{(1)}(t)\right|, \left|\zeta_{i_\ell^\star,j}^{(2)}(t)\right| \leq \calO\paren{\frac{m^2}{\delta_{\mathbb{P}}\sqrt{d}} + m\aggepso{\ell}(t)^3 + \aggepso{\ell}(t)^2};\;\forall\;j\in[m^\star]\\
        \left|\gamma_{i_\ell^\star,j}^{(1)}(t)\right|, \left|\gamma_{i_\ell^\star,j}^{(2)}(t)\right| \leq \calO\paren{\frac{m^2}{\delta_{\mathbb{P}}\sqrt{d}} + m\aggepso{\ell}(t)^3 + \aggepso{\ell}(t)^2};\;\forall\;j\neq j_\ell^\star
    \end{gather*}
    By the definition of $\aggepso{\ell}(t)$, if $\aggepso{\ell}(t)\geq \aggepst{\ell+1}(t)$, then we must have that $\left|\zeta_{i_\ell^\star,j}^{(1)}(t)\right|, \left|\zeta_{i_\ell^\star,j}^{(2)}(t)\right|$ and $\left|\gamma_{i_\ell^\star,j}^{(1)}(t)\right|, \left|\gamma_{i_\ell^\star,j}^{(2)}(t)\right|$ is greater than $\aggepst{\ell+1}(t)$. In this case, we must have that $\aggepso{\ell}(t)$ is dominated by $\left|\zeta_{i_\ell^\star,j}^{(1)}(t)\right|, \left|\zeta_{i_\ell^\star,j}^{(2)}(t)\right|$ and $\left|\gamma_{i_\ell^\star,j}^{(1)}(t)\right|, \left|\gamma_{i_\ell^\star,j}^{(2)}(t)\right|$. For $t\leq T_c$, we can thus conclude that $\left|\zeta_{i_\ell^\star,j}^{(1)}(t)\right|, \left|\zeta_{i_\ell^\star,j}^{(2)}(t)\right|$ and $\left|\gamma_{i_\ell^\star,j}^{(1)}(t)\right|, \left|\gamma_{i_\ell^\star,j}^{(2)}(t)\right|$ must stay below $\calO\paren{\frac{m^2}{\delta_{\mathbb{P}}\sqrt{d}}}$. Otherwise, we will have that $\aggepso{\ell}(t)\leq \aggepst{\ell+1}(t)$ and thus
    \begin{gather*}
        \left|\zeta_{i_\ell^\star,j}^{(1)}(t)\right|, \left|\zeta_{i_\ell^\star,j}^{(2)}(t)\right| \leq \calO\paren{\frac{m^2}{\delta_{\mathbb{P}}\sqrt{d}} + m\aggepso{\ell+1}(t)^3 + \aggepso{\ell+1}(t)^2};\;\forall\;j\in[m^\star]\\
        \left|\gamma_{i_\ell^\star,j}^{(1)}(t)\right|, \left|\gamma_{i_\ell^\star,j}^{(2)}(t)\right| \leq \calO\paren{\frac{m^2}{\delta_{\mathbb{P}}\sqrt{d}} + m\aggepso{\ell+1}(t)^3 + \aggepso{\ell+1}(t)^2};\;\forall\;j\neq j_\ell^\star
    \end{gather*}
    which implies that for all $t$ such that $\aggepso{\ell}(t)\leq \calO\paren{\frac{m^2}{\delta_{\mathbb{P}}\sqrt{d}}}$, it holds that
    \begin{align*}
        \left|\zeta_{i_\ell^\star,j}^{(1)}(t)\right|, \left|\zeta_{i_\ell^\star,j}^{(2)}(t)\right| \leq \calO\paren{\frac{m^2}{\delta_{\mathbb{P}}\sqrt{d}}};\;\forall\;j\in[m^\star]\\
        \left|\gamma_{i_\ell^\star,j}^{(1)}(t)\right|, \left|\gamma_{i_\ell^\star,j}^{(2)}(t)\right| \leq \calO\paren{\frac{m^2}{\delta_{\mathbb{P}}\sqrt{d}}};\;\forall\;j\neq j_\ell^\star
    \end{align*}
    Combining with the inductive hypothesis that $\forweps{\ell}(t)\leq \calO\paren{\frac{m^2}{\delta_{\mathbb{P}}\sqrt{d}}}$ gives that $\forweps{\ell+1}(t) \leq \calO\paren{\frac{m^2}{\delta_{\mathbb{P}}\sqrt{d}}}$ when $\aggepso{\ell+1}(t)\leq \calO\paren{\frac{m^2}{\delta_{\mathbb{P}}\sqrt{d}}}$. This gives the third inductive hypothesis. For the same reasoning, by Lemma~\ref{lem:phase2}, we can also have the first inductive hypothesis. Lastly, the second inductive hypothesis follows from Lemma~\ref{lem:gamma_2_ub}.
    
    Based on the first and second inductive hypothesis, we can conclude bullet point 1-3 in Theorem~\ref{thm:gf_conv}. To see the last statement, we recall from Lemma~\ref{lem:phase2} that
    \[
        \gamma_{i_{\ell}^\star,j_{\ell}^\star}^{(1)}(t), \gamma_{i_{\ell}^\star,j_{\ell}^\star}^{(2)}(t) \geq 1 - \calO\paren{m\aggepso{\ell}(t)^3 + \frac{m^7}{\delta_{\mathbb{P}}^3d^{\frac{3}{2}}}} \text{ for } t \geq T_1 + \calO\paren{\log d}
    \]
    Let $T^\star \geq T_1 + \calO\paren{\log d}$ to be as small as possible for all $\ell$. Then $T^\star \leq \calO\paren{\sqrt{d} + \log d} \leq \calO\paren{\sqrt{d}}$ by Lemma~\ref{lem:phase_1_conv}.  By definition of $\aggepso{\ell}(t)$, we have that $\aggepso{\ell}(t) \leq \aggepso{\ell+1}(t)$ for all $\ell\in[m^\star -1]$. Thus, $\aggepso{\ell}(t) \leq \aggepso{m^\star}(t)$ for all $\ell\in[m^\star]$. This gives that
    \[
        \gamma_{i_{\ell}^\star,j_{\ell}^\star}^{(1)}(t), \gamma_{i_{\ell}^\star,j_{\ell}^\star}^{(2)}(t) \geq 1 - \calO\paren{m\aggepso{m^\star}(t)^3 + \frac{m^7}{\delta_{\mathbb{P}}^3d^{\frac{3}{2}}}} \text{ for } t \geq T^\star
    \]
    Thus, it remains to bound $\aggepso{m^\star}(t)$. By definition, $\aggepso{m^\star}(t)$ depends on $\gamma_{i_\ell^\star,j}^{(1)}(t), \gamma_{i_\ell^\star,j}^{(1)}(t)$ for all $\ell\in[m^\star]$ and $j\in[m^\star]\setminus \{j_{\ell}^\star\}$, $\zeta_{i,j}^{(1)}(t), \zeta_{i,j}^{(2)}(t)$ for all $i\in[m]$ and $j\in[m^\star]$, and $I_{i,j}^{(1)}(t), I_{i,j}^{(2)}(t), I_{i,j}^{(3)}(t)$ for all $i,j\in[m]$ with $i\neq j$. Recall that by Lemma~\ref{lem:phase2}, we have that for all $t\geq T_1$, it holds that
    \begin{gather*}
        \left|\zeta_{i_\ell^\star,j}^{(1)}(t)\right|, \left|\zeta_{i_\ell^\star,j}^{(2)}(t)\right| \leq \calO\paren{\frac{m^2}{\delta_{\mathbb{P}}\sqrt{d}} + m\aggepso{\ell}(t)^3 + \aggepso{\ell}(t)^2};\;\forall\;j\in[m^\star]\\
        \left|\gamma_{i_\ell^\star,j}^{(1)}(t)\right|, \left|\gamma_{i_\ell^\star,j}^{(2)}(t)\right| \leq \calO\paren{\frac{m^2}{\delta_{\mathbb{P}}\sqrt{d}} + m\aggepso{\ell}(t)^3 + \aggepso{\ell}(t)^2};\;\forall\;j\neq j_\ell^\star
    \end{gather*}
    For $\zeta_{i,j}^{(1)}(t), \zeta_{i,j}^{(2)}(t)$ and $I_{i,j}^{(1)}(t), I_{i,j}^{(2)}(t), I_{i,j}^{(3)}(t)$ with $i\in[m]\setminus \mathcal{R}_{m^\star}$, we obtain from Lemma~\ref{lem:eps_1_ub} that the above are bounded by $\calO\paren{\frac{m^2}{\delta_{\mathbb{P}}\sqrt{d}}}$ for all $t \leq \calO\paren{\frac{\delta_{\mathbb{P}}^2d}{m^5}}$. Lastly, for $I_{i,j}^{(1)}(t), I_{i,j}^{(2)}(t), I_{i,j}^{(3)}(t)$ with $i\in\mathcal{R}_{m^\star}$ or $j\in\mathcal{R}_{m^\star}$, we can obtain from (\ref{eq:I1_star_bound}), (\ref{eq:I2_star_bound}), and (\ref{eq:I3_star_bound}) that
    \[
        \max\left\{\left|I_{i,j}^{(1)}(t)\right|, \left|I_{i,j}^{(2)}(t)\right|, \left|I_{i,j}^{(3)}(t)\right|\right\} \leq \max\left\{\left|\gamma_{j,j_{\ell'}^\star}^{(1)}(t)\right|, \left|\gamma_{j,j_{\ell'}^\star}^{(2)}(t)\right|, \left|\zeta_{j,j_{\ell'}^\star}^{(1)}(t)\right|, \left|\zeta_{j,j_{\ell'}^\star}^{(1)}(t)\right|\right\} + \calO\paren{\frac{m^4}{\delta_{\mathbb{P}}^2d^{\frac{3}{4}}}}
    \]
    As $\left|\zeta_{j,j_{\ell'}^\star}^{(1)}(t)\right|, \left|\zeta_{j,j_{\ell'}^\star}^{(1)}(t)\right|$ are bounded above, we just need to look into $\left|\gamma_{j,j_{\ell'}^\star}^{(1)}(t)\right|, \left|\gamma_{j,j_{\ell'}^\star}^{(2)}(t)\right|$. which are bounded by $\calO\paren{\frac{m^2}{\delta_{\mathbb{P}}\sqrt{d}}}$ as shown in Lemma~\ref{lem:gamma_2_prev} when $t\leq \left\{T_{m^\star}\paren{\xi} + \calO\paren{\frac{m^2}{\delta_{\mathbb{P}}\sqrt{d}}}, \calO\paren{\frac{\delta_{\mathbb{P}}^2d^{\frac{3}{4}}}{m^5}}\right\}$. Combining all the bounds above we can conclude that $\aggepso{m^\star}(t) \leq \calO\paren{\frac{m^2}{\delta_{\mathbb{P}}\sqrt{d}}}$ for all $t\leq \min\left\{T_{m^\star}\paren{\xi} + \calO\paren{\frac{\delta_{\mathbb{P}}\sqrt{d}}{m^2}}\right\}$. Thus, we can obtain that
    \[
        \gamma_{i_{\ell}^\star,j_{\ell}^\star}^{(1)}(t), \gamma_{i_{\ell}^\star,j_{\ell}^\star}^{(2)}(t) \geq 1 - \calO\paren{\frac{m^7}{\delta_{\mathbb{P}}^3d^{\frac{3}{2}}}};\;\forall T^\star \leq t \leq T^\star + \calO\paren{\frac{\delta_{\mathbb{P}}\sqrt{d}}{m^2}}
    \]
\end{proof}

\section{Proof of Theorem~\ref{thm:pruning_guarantee}}
\label{sec:pruning_guarantee}
\begin{proof}[Proof of Theorem~\ref{thm:pruning_guarantee}]
    We simply need to show that under the stopping criteria (\ref{eq:stopping_criteria}) the procedure in (\ref{eq:pruning_procedure}) satisfies that $r_\tau\in[m]\setminus [m^\star]$ for all $\tau \in [\tau^\star]$ and that $\tau^\star = m - m^\star$. This is done in \textbf{Part 1} and \textbf{Part 2} below. Before we start, we define the loss over the pruned model as
    \[
        \calL_{\mathcal{S}}\paren{\bm{\theta}} = \EXP[\bfx]{\paren{f_{\mathcal{S}}\paren{\bm{\theta},\bfx} - f^\star\paren{\bfx}}^2}
    \]
    and we write the target pruned model as
    \[
        \hat{f}\paren{\bm{\theta},\bfx} = \sum_{i=1}^{m^\star}\pi\paren{\bbfv_i^\top\bfx}\sigma\paren{\bbfw_i^\top\bfx}        
    \]
    For the convenience, we also denote $h_i\paren{\bfx} := \pi\paren{\bbfv_i^\top\bfx}\sigma\paren{\bbfw_i^\top\bfx}$ and $h_i^\star\paren{\bfx} := \pi\paren{\bbfv_i^{\star\top}\bfx}\sigma\paren{\bbfw_i^{\star\top}\bfx}$.
    
    \textbf{Part 1.} Assume that $\mathcal{S}_{\tau-1}\subseteq [m]\setminus [m^\star]$ and $|\mathcal{S}_{\tau-1}|\leq m - m^\star$. Let $r_{\tau}\in[m]\setminus [m^\star]$ and $r_{\tau}' \in [m^\star]$. Let $\mathcal{S}_{\tau} = \mathcal{S}_{\tau-1}\cup\{r_{\tau}\}$ and $\mathcal{S}_{\tau}' = \mathcal{S}_{\tau-1}\cup\{r_{\tau}'\}$. Moreover, let $\mathcal{S}^\perp = [m]\setminus [m^\star]\setminus \mathcal{S}_{\tau}$. Then we have that
    \begin{align*}
        \calL_{\mathcal{S}_{\tau}}\paren{\bm{\theta}} - \calL_{\mathcal{S}_{\tau}'}\paren{\bm{\theta}} & = \EXP[\bfx]{\paren{f_{\mathcal{S}_{\tau}}\paren{\bm{\theta},\bfx} - f^\star\paren{\bfx}}^2 - \paren{f_{\mathcal{S}_{\tau}'}\paren{\bm{\theta},\bfx} - f^\star\paren{\bfx}}^2}\\
        & = \EXP[\bfx]{\paren{f_{\mathcal{S}_{\tau}}\paren{\bm{\theta},\bfx} + f_{\mathcal{S}_{\tau}'}\paren{\bm{\theta},\bfx} - 2f^\star\paren{\bfx}}\paren{f_{\mathcal{S}_{\tau}}\paren{\bm{\theta},\bfx} - f_{\mathcal{S}_{\tau}'}\paren{\bm{\theta},\bfx}}}\\
        & = \EXP[\bfx]{\paren{f_{\mathcal{S}_{\tau}}\paren{\bm{\theta},\bfx} + f_{\mathcal{S}_{\tau}'}\paren{\bm{\theta},\bfx} - 2\hat{f}\paren{\bm{\theta},\bfx}}\paren{f_{\mathcal{S}_{\tau}}\paren{\bm{\theta},\bfx} - f_{\mathcal{S}_{\tau}'}\paren{\bm{\theta},\bfx}}}\\
        & \qqquad - 2\underbrace{\EXP[\bfx]{\paren{\hat{f}\paren{\bm{\theta},\bfx} - f^\star\paren{\bfx}}\paren{f_{\mathcal{S}_{\tau}}\paren{\bm{\theta},\bfx} - f_{\mathcal{S}_{\tau}'}\paren{\bm{\theta},\bfx}}}}_{\mathcal{T}_1}\\
        & = \EXP[\bfx]{\paren{2\sum_{i\in\mathcal{S}^\perp}h_i\paren{\bfx} + h_{r_\tau}\paren{\bfx} - h_{r_{\tau}'}\paren{\bfx}}\paren{h_{r_{\tau}'}\paren{\bfx} - h_{r_{\tau}}\paren{\bfx}}} - 2\mathcal{T}_1\\
        & = -\EXP[\bfx]{\paren{h_{r_{\tau}'}\paren{\bfx} - h_{r_{\tau}}\paren{\bfx}}^2} - 2\mathcal{T}_1  + 2\underbrace{\EXP[\bfx]{\sum_{i\in\mathcal{S}^\perp}h_i\paren{\bfx}\paren{h_{r_{\tau}'}\paren{\bfx} - h_{r_{\tau}}\paren{\bfx}}}}_{\mathcal{T}_2}\\
        & \leq -\EXP[\bfx]{\paren{h_{r_{\tau}'}\paren{\bfx} - h_{r_{\tau}}\paren{\bfx}}^2} + 2\left|\mathcal{T}_1\right| + 2\left|\mathcal{T}_2\right|
    \end{align*}
    It suffice to upper bound $\left|\mathcal{T}_1\right|$ and $\left|\mathcal{T}_2\right|$, and lower bound $\EXP[\bfx]{\paren{h_{r_{\tau}'}\paren{\bfx} - h_{r_{\tau}}\paren{\bfx}}^2}$. To start, the lower bound can be derived as
    \begin{equation}
        \label{eq:thm_prune_lb}
        \begin{aligned}
            \EXP[\bfx]{\paren{h_{r_{\tau}'}\paren{\bfx} - h_{r_{\tau}}\paren{\bfx}}^2} & = \EXP[\bfx]{h_{r_{\tau}'}\paren{\bfx}^2} + \EXP[\bfx]{h_{r_{\tau}}\paren{\bfx}^2} - 2\EXP[\bfx]{h_{r_{\tau}}\paren{\bfx}h_{r_{\tau}'}\paren{\bfx}} \geq 12\sum_{k=0}^{\infty}\frac{c_k^2}{k!} - \calO\paren{\varepsilon^2}
        \end{aligned}
    \end{equation}
    where the last inequality follows from Lemma~\ref{lem:ub_prod}.
    Next, for $\left|\mathcal{T}_1\right|$, we have that
    \begin{equation}
        \begin{aligned}
            \mathcal{T}_1 & = \EXP[\bfx]{\paren{\hat{f}\paren{\bm{\theta},\bfx} - f^\star\paren{\bfx}}\paren{h_{r_\tau}\paren{\bfx} - h_{r_\tau'}\paren{\bfx}}}\\
            & = \EXP[\bfx]{h_{r_\tau}\paren{\bfx}\sum_{i=1}^{m^\star}h_i\paren{\bfx}} - \EXP[\bfx]{h_{r_\tau}\paren{\bfx}\sum_{i=1}^{m^\star}h_i^\star\paren{\bfx}} - \EXP[\bfx]{h_{r_\tau}'\paren{\bfx}\sum_{i=1,i\neq r_{\tau}'}^{m^\star}h_i\paren{\bfx}}\\
            & \qqquad+  \EXP[\bfx]{h_{r_\tau}'\paren{\bfx}\sum_{i=1,i\neq r_{\tau}'}^{m^\star}h_i^\star\paren{\bfx}} + \EXP[\bfx]{h_{r_\tau'}\paren{\bfx}^2 - h_{r_{\tau}'}\paren{\bfx}h_{r_\tau'}^\star\paren{\bfx}}\\
            & = \EXP[\bfx]{h_{r_\tau'}\paren{\bfx}^2 - h_{r_{\tau}'}\paren{\bfx}h_{r_\tau'}^\star\paren{\bfx}} \pm  \calO\paren{m^\star\varepsilon^4}
        \end{aligned}
    \end{equation}
    By Lemma~\ref{lem:ub_prod}, since $r_\tau'\in[m^\star]$, we have that
    \begin{equation}
        \EXP[\bfx]{h_{r_\tau'}\paren{\bfx}h_{r_\tau'}^\star\paren{\bfx}} = 6\EXP[\bfx]{h_{r_\tau'}(\bfx)^2} \pm \calO\paren{\varepsilon}
    \end{equation}
    Thus, we have that
    \begin{equation}
        \label{eq:thm_prune_ub1}
        \left|\mathcal{T}_1\right| \leq \calO\paren{\varepsilon + m^\star\varepsilon^4}
    \end{equation}
    For $\left|\mathcal{T}_2\right|$, we notice that $r_\tau\in\mathcal{S}_{\tau}$ and $r_{\tau}'\in[m^\star]$. Therefore, $r_\tau, r_\tau'\notin \mathcal{S}^\perp$. Thus
    \begin{equation}
        \label{eq:thm_prune_ub2}
        \begin{aligned}
            \left|\mathcal{T}_2\right| \leq \sum_{i\in\mathcal{S}^\perp}\left|\EXP[\bfx]{h_i\paren{\bfx}h_{r_\tau}\paren{\bfx}}\right| + \sum_{i\in\mathcal{S}^\perp}\left|\EXP[\bfx]{h_i\paren{\bfx}h_{r_\tau'}\paren{\bfx}}\right| \leq \calO\paren{m\varepsilon^4}
        \end{aligned}
    \end{equation}
    Combining (\ref{eq:thm_prune_lb}), (\ref{eq:thm_prune_ub1}), and (\ref{eq:thm_prune_ub2}) gives that
    \[
        \calL_{\mathcal{S}_{\tau}}\paren{\bm{\theta}} - \calL_{\mathcal{S}_{\tau}'}\paren{\bm{\theta}} \leq -12\sum_{k=0}^{\infty}\frac{c_k^2}{k!} + \calO\paren{\varepsilon + m\varepsilon^4} \leq 0
    \]
    when $\varepsilon\leq o\paren{\frac{1}{\sqrt{m}}}$.

    \textbf{Part 2.} Assume that $\tau^\star < m-m^\star$. We show that $\calL_{\mathcal{S}_{\tau^\star}}\paren{\bm{\theta}} \geq \calL_{\mathcal{S}_{\tau^\star}}\paren{\bm{\theta}+1}$ by letting $\mathcal{S}_{\tau^\star+1} = \mathcal{S}_{\tau^\star} \cup\{r^\star\}$ where $r^\star \in[m]\setminus [m^\star]$. Similar to before, let $\mathcal{S}^\perp = [m]\setminus[m^\star]\setminus \mathcal{S}_{\tau^\star+1}$, we have that
    \begin{align*}
        \calL_{\mathcal{S}_{\tau^\star}}\paren{\bm{\theta}} - \calL_{\mathcal{S}_{\tau^\star}}\paren{\bm{\theta}+1} & = \EXP[\bfx]{\paren{f_{\mathcal{S}_{\tau^\star}}\paren{\bm{\theta},\bfx} - f^\star\paren{\bfx}}^2 - \paren{f_{\mathcal{S}_{\tau^\star +1}}\paren{\bm{\theta},\bfx} - f^\star\paren{\bfx}}^2}\\
        & = \EXP[\bfx]{\paren{f_{\mathcal{S}_{\tau^\star}}\paren{\bm{\theta},\bfx} + f_{\mathcal{S}_{\tau^\star+1}}\paren{\bm{\theta},\bfx} - 2f^\star\paren{\bfx}}\paren{f_{\mathcal{S}_{\tau^\star}}\paren{\bm{\theta},\bfx} - f_{\mathcal{S}_{\tau^\star+1}}\paren{\bm{\theta},\bfx}}}\\
        & = \EXP[\bfx]{h_{r^\star}\paren{\bfx}\paren{f_{\mathcal{S}_{\tau^\star}}\paren{\bm{\theta},\bfx} + f_{\mathcal{S}_{\tau^\star+1}}\paren{\bm{\theta},\bfx} - 2f^\star\paren{\bfx}}}\\
        & = \EXP[\bfx]{h_{r^\star}\paren{\bfx}\paren{f_{\mathcal{S}_{\tau^\star}}\paren{\bm{\theta},\bfx} + f_{\mathcal{S}_{\tau^\star+1}}\paren{\bm{\theta},\bfx} - 2\hat{f}\paren{\bm{\theta},\bfx}}}\\
        & \qqquad - 2\underbrace{\EXP[\bfx]{h_{r^\star}\paren{\bfx}\paren{\hat{f}\paren{\bm{\theta},\bfx} - f^\star\paren{\bfx}}}}_{\mathcal{T}_1}\\
        & = \EXP[\bfx]{h_{r^\star}\paren{\bfx}^2} + 2\underbrace{\EXP[\bfx]{h_{r^\star}\paren{\bfx}\sum_{i\in\mathcal{S}^\perp}h_i\paren{\bfx}}}_{\mathcal{T}_2} - 2\mathcal{T}_1\\
        & \geq \EXP[\bfx]{h_{r^\star}\paren{\bfx}^2} - 2\left|\mathcal{T}_1\right| - 2\left|\mathcal{T}_2\right|
    \end{align*}
    As before, we have that $\EXP[\bfx]{h_{r^\star}\paren{\bfx}^2} \geq 6\sum_{k=0}^{\infty}\frac{c_k^2}{k!} - \calO\paren{\varepsilon^2}$ and $\left|\mathcal{T}_2\right| \leq \calO\paren{m\varepsilon^4}$. For $\mathcal{T}_1$, we have that
    \begin{align*}
        \left|\mathcal{T}_1\right| & = \left|\EXP[\bfx]{h_{r^\star}\paren{\bfx}\paren{\hat{f}\paren{\bm{\theta},\bfx} - f^\star\paren{\bfx}}}\right|\\
        & \leq \sum_{i=1}^{m^\star}\left|\EXP[\bfx]{h_{r^\star}\paren{\bfx}h_i\paren{\bfx}}\right| + \sum_{i=1}^{m^\star}\left|\EXP[\bfx]{h_{r^\star}\paren{\bfx}h_i^\star\paren{\bfx}}\right|\\
        & \leq \calO\paren{m\varepsilon^4}
    \end{align*}
    Thus, we have that
    \[
        \calL_{\mathcal{S}_{\tau^\star}}\paren{\bm{\theta}} - \calL_{\mathcal{S}_{\tau^\star}}\paren{\bm{\theta}+1} \geq 6\sum_{k=0}^{\infty}\frac{c_k^2}{k!} - \calO\paren{\varepsilon^2 + m\varepsilon^4} \geq 0
    \]
    when $\varepsilon\leq \calO\paren{\frac{1}{\sqrt{m}}}$. This shows that $\tau^\star\geq m-m^\star$. Next, we assume that $\tau^\star > m-m^\star$. Then $r_{m-m^\star+1}\in[m^\star]$. We show that $\calL_{\mathcal{S}_{m-m^\star}}\paren{\bm{\theta}} \leq \calL_{\mathcal{S}_{m-m^\star+1}}\paren{\bm{\theta}}$. As before, let . Notice that by \textbf{Part 1}, $f_{\mathcal{S}_{m-m^\star}}\paren{\bm{\theta},\bfx} = \hat{f}\paren{\bm{\theta},\bfx}$. Then we have that
    \begin{align*}
        \calL_{\mathcal{S}_{m-m^\star}}\paren{\bm{\theta}} - \calL_{\mathcal{S}_{m-m^\star+1}}\paren{\bm{\theta}} & = \EXP[\bfx]{h_{r_{m-m^\star+1}}\paren{\bfx}\paren{f_{\mathcal{S}_{m-m^\star}}\paren{\bm{\theta},\bfx} + f_{\mathcal{S}_{m-m^\star+1}}\paren{\bm{\theta},\bfx} - 2f^\star\paren{\bfx}}}\\
        & = 2\EXP[\bfx]{h_{r_{m-m^\star+1}}\paren{\bfx}\paren{\hat{f}\paren{\bm{\theta},\bfx} - f^\star\paren{\bfx}}} - \EXP[\bfx]{h_{r_{m-m^\star+1}}\paren{\bfx}^2}
    \end{align*}
    As before, we have that $\EXP[\bfx]{h_{r_{m-m^\star+1}}\paren{\bfx}^2} \geq 6\sum_{k=0}^{\infty}\frac{c_k^2}{k!} - \calO\paren{\varepsilon^2}$. It remains to upper bound the first term. In particular, we have that
    \begin{align*}
        \EXP[\bfx]{h_{r_{m-m^\star+1}}\paren{\bfx}\paren{\hat{f}\paren{\bm{\theta},\bfx} - f^\star\paren{\bfx}}} & \leq \EXP[\bfx]{h_{r_{m-m^\star+1}}\paren{\bfx}^2 - h_{r_{m-m^\star+1}}\paren{\bfx}h_{r_{m-m^\star+1}}^\star\paren{\bfx}}\\
        & \qqquad + \sum_{i\neq r_{m-m^\star+1}}\EXP[\bfx]{h_{r_{m-m^\star+1}}\paren{\bfx}\paren{h_i\paren{\bfx} - h_i^\star\paren{\bfx}}}\\
        & \leq \calO\paren{\varepsilon+m^\star\varepsilon^4}
    \end{align*}
    Thus, we can conclude that
    \[
        \calL_{\mathcal{S}_{m-m^\star}}\paren{\bm{\theta}} - \calL_{\mathcal{S}_{m-m^\star+1}}\paren{\bm{\theta}} \leq -6\sum_{k=0}^{\infty}\frac{c_k^2}{k!} + \calO\paren{\varepsilon+m\varepsilon^4} \leq 0
    \]
    when $\varepsilon\leq \calO\paren{\frac{1}{\sqrt{m}}}$. This shows that $\tau^\star = m - m^\star$, which finishes the proof.
\end{proof}

\section{Proof of Theorem~\ref{thm:fine-tuning}}
\label{sec:proof_fine_tuning}
We will analyze the Hessian in a small region near the global minima $\bm{\theta}^\star = \left\{\paren{\bbfv_i^\star,\bbfw_i^\star}\right\}_{i=1}^{m^\star}$. To do this, we utilize the following second-order Stein's lemma.
\begin{lemma}
    \label{lem:sec_stein}
    Let $\bfv,\bfw\in\R^d$ and $g,h:\R\rightarrow \R$. Then we have that
    \begin{align*}
        \EXP[\bfx\sim\mathcal{N}\paren{0,\bfI_d}]{g\paren{\bfv^\top\bfx}h\paren{\bfw^\top\bfx}\bfx\bfx^\top} & = \EXP[\bfx\sim\mathcal{N}\paren{0,\bfI_d}]{g\paren{\bfv^\top\bfx}h\paren{\bfw^\top\bfx}}\bfI_d\\
        & \qqquad + \EXP[\bfx\sim\mathcal{N}\paren{0,\bfI_d}]{g'\paren{\bfv^\top\bfx}h'\paren{\bfw^\top\bfx}}\paren{\bfv\bfw^\top + \bfv\bfw^\top}\\
        & \qqquad + \EXP[\bfx\sim\mathcal{N}\paren{0,\bfI_d}]{g''\paren{\bfv^\top\bfx}h\paren{\bfw^\top\bfx}}\bfv\bfv^\top\\
        & \qqquad + \EXP[\bfx\sim\mathcal{N}\paren{0,\bfI_d}]{g\paren{\bfv^\top\bfx}h''\paren{\bfw^\top\bfx}}\bfw\bfw^\top
    \end{align*}
\end{lemma}
The proof of Lemma~\ref{lem:sec_stein} follows by applying Stein's lemma twice. In particular, we shall prove the following result
\begin{theorem}
    \label{thm:hessian_pd}
    Let $\bm{\theta} = \left\{\paren{\bfv_i,\bfw_i}\right\}_{i=1}^{m^\star}$ be the parameter of the MoE, let $\alpha_1,\dots,\alpha_{m^\star},\beta_1,\dots\beta_{m^\star} \geq \Omega\paren{1}$ , and let $\bfu_1,\dots,\bfu_{m^\star},\bfq_1,\dots,\bfq_{m^\star}\in\R^d$ be any set of vectors such that $\norm{\paren{\bfI - \bbfv_i\bbfv_i}\bfu_i}_2^2 \geq \mathcal{Q}^\star\norm{\bfu_i}_2^2$ and $\norm{\paren{\bfI - \bbfw_i\bbfw_i}\bfq_i}_2^2 \geq \mathcal{Q}^\star\norm{\bfq_i}_2^2$ for some $\mathcal{Q}^\star > 0$ for all $i\in[m^\star]$. If $\bm{\theta}$ also satisfies that $\norm{\bbfv_i - \bbfv_i^\star}_2, \norm{\bbfw_i - \bbfw_i^\star}_2\leq \frac{\varepsilon}{2}$ for some $\varepsilon \leq o\paren{\frac{N_{\min}\mathcal{Q}^\star}{m^{\star 2}}}$, and $\frac{C_{S,0}}{C_{S,1}}\geq \frac{N_{\max}^2(1+\beta_i)^2}{N_{\min}^2\mathcal{Q}^{\star2}\alpha_i^2}$, then we have that
    \[
        \begin{bmatrix}
        \alpha_1\bfu_1\\ \vdots\\ \alpha_{m^\star}\bfu_{m^\star} \\ \beta_1\bfq_1 \\ \vdots \\\beta_{m^\star}\bfq_{m^\star}
        \end{bmatrix}^\top\nabla^2\calL\paren{\bm{\theta}}\begin{bmatrix}
            \bfu_1\\ \vdots\\ \bfu_{m^\star} \\ \bfq_1 \\ \vdots \\\bfq_{m^\star}
        \end{bmatrix} \geq N_{\min}\mathcal{Q}^\star\kappa\sum_{i=1}^{m^\star}\paren{\norm{\bfu_i}_2^2 + \norm{\bfq_i}_2^2}
    \]
    for some constant $\kappa > 0$.
\end{theorem}
\begin{proof}
    \textbf{Form of Hessian.} Here we are going to compute $\frac{\partial^2}{\partial\bfv_i\partial\bfv_j}\calL\paren{\bm{\theta}}, \frac{\partial^2}{\partial\bfw_i\partial\bfw_j}\calL\paren{\bm{\theta}}$, and $\frac{\partial^2}{\partial\bfv_i\partial\bfw_j}\calL\paren{\bm{\theta}}$. Recall that the gradient takes the form
    \begin{align*}
        \frac{\partial}{\partial\bfv_i}\calL\paren{\bm{\theta}} & = \frac{1}{\norm{\bfv_i}_2}\paren{\bfI_d - \bbfv_i\bbfv_i^\top}\EXP[\bfx]{\paren{f\paren{\bm{\theta},\bfx} - f^\star\paren{\bfx}}\pi'\paren{\bbfv_i^\top\bfx}\sigma\paren{\bbfw_i^\top\bfx}\bfx}\\
        \frac{\partial}{\partial\bfw_i}\calL\paren{\bm{\theta}} & = \frac{1}{\norm{\bfw_i}_2}\paren{\bfI_d - \bbfw_i\bbfw_i^\top}\EXP[\bfx]{\paren{f\paren{\bm{\theta},\bfx} - f^\star\paren{\bfx}}\pi\paren{\bbfv_i^\top\bfx}\sigma'\paren{\bbfw_i^\top\bfx}\bfx}
    \end{align*}
    Therefore
    \begin{align*}
        \frac{\partial^2}{\partial\bfv_i\partial\bfv_j}\calL\paren{\bm{\theta}} & = \frac{1}{\norm{\bfv_i}_2\norm{\bfv_j}_2}\paren{\bfI-\bbfv_i\bbfv_i^\top}\underbrace{\EXP[\bfx]{\pi'\paren{\bbfv_i^\top\bfx}\sigma\paren{\bbfw_i^\top\bfx}\pi'\paren{\bbfv_j^\top\bfx}\sigma\paren{\bbfw_j^\top\bfx}\bfx\bfx^\top}}_{\calT_{i,j,1}}\paren{\bfI-\bbfv_j\bbfv_j^\top}\\
        & \qqquad + \frac{\indy{i=j}}{\norm{\bfv_i}_2^2}\paren{\bfI-\bbfv_i\bbfv_i^\top}\underbrace{\EXP[\bfx]{\paren{f\paren{\bm{\theta},\bfx} - f^\star\paren{\bfx}}\pi''\paren{\bbfv_i^\top\bfx}\sigma\paren{\bbfw_i^\top\bfx}\bfx\bfx^\top}}_{\calT_{i,j,2}}\paren{\bfI-\bbfv_i\bbfv_i^\top}\\
        & \qqquad - \frac{\indy{i=j}}{\norm{\bfv_i}_2^2}\paren{\bfI-\bbfv_i\bbfv_i^\top}\underbrace{\EXP[\bfx]{\paren{f\paren{\bm{\theta},\bfx} - f^\star\paren{\bfx}}\pi'\paren{\bbfv_i^\top\bfx}\sigma\paren{\bbfw_i^\top\bfx}\bbfv_i^\top\bfx}}_{\calT_{i,j,3}}\\
        & \qqquad - \frac{\indy{i=j}}{\norm{\bfv_i}_2^2}\underbrace{\EXP[\bfx]{\paren{f\paren{\bm{\theta},\bfx} - f^\star\paren{\bfx}}\pi'\paren{\bbfv_i^\top\bfx}\sigma\paren{\bbfw_i^\top\bfx}\paren{\bbfv_i\bfx^\top + \bfx\bbfv_i^\top}}}_{\calT_{i,j,4}}\\
        \frac{\partial^2}{\partial\bfw_i\partial\bfw_j}\calL\paren{\bm{\theta}} & = \frac{1}{\norm{\bfw_i}_2\norm{\bfw_j}_2}\paren{\bfI-\bbfw_i\bbfw_i^\top}\underbrace{\EXP[\bfx]{\pi\paren{\bbfv_i^\top\bfx}\sigma'\paren{\bbfw_i^\top\bfx}\pi\paren{\bbfv_j^\top\bfx}\sigma'\paren{\bbfw_j^\top\bfx}\bfx\bfx^\top}\paren{\bfI-\bbfw_j\bbfw_j^\top}}_{\calT_{i,j,5}}\\
        & \qqquad + \frac{\indy{i=j}}{\norm{\bfw_i}_2^2}\paren{\bfI-\bbfw_i\bbfw_i^\top}\underbrace{\EXP[\bfx]{\paren{f\paren{\bm{\theta},\bfx} - f^\star\paren{\bfx}}\pi\paren{\bbfv_i^\top\bfx}\sigma''\paren{\bbfw_i^\top\bfx}\bfx\bfx^\top}}_{\calT_{i,j,6}}\paren{\bfI-\bbfw_i\bbfw_i^\top}\\
        & \qqquad - \frac{\indy{i=j}}{\norm{\bfw_i}_2^2}\paren{\bfI-\bbfw_i\bbfw_i^\top}\underbrace{\EXP[\bfx]{\paren{f\paren{\bm{\theta},\bfx} - f^\star\paren{\bfx}}\pi\paren{\bbfv_i^\top\bfx}\sigma'\paren{\bbfw_i^\top\bfx}\bbfw_i^\top\bfx}}_{\calT_{i,j,7}}\\
        & \qqquad - \frac{\indy{i=j}}{\norm{\bfw_i}_2^2}\underbrace{\EXP[\bfx]{\paren{f\paren{\bm{\theta},\bfx} - f^\star\paren{\bfx}}\pi\paren{\bbfv_i^\top\bfx}\sigma'\paren{\bbfw_i^\top\bfx}\paren{\bbfw_i\bfx^\top + \bfx\bbfw_i^\top}}}_{\calT_{i,j,8}}\\
        \frac{\partial^2}{\partial\bfv_i\partial\bfw_j}\calL\paren{\bm{\theta}} & = \frac{1}{\norm{\bfv_i}_2\norm{\bfw_j}_2}\paren{\bfI-\bbfv_i\bbfv_i^\top}\underbrace{\EXP[\bfx]{\pi'\paren{\bbfv_i^\top\bfx}\sigma\paren{\bbfw_i^\top\bfx}\pi\paren{\bbfv_j^\top\bfx}\sigma'\paren{\bbfw_j^\top\bfx}\bfx\bfx^\top}}_{\calT_{i,j,9}}\paren{\bfI-\bbfw_j\bbfw_j^\top}\\
        & \qqquad + \frac{\indy{i=j}}{\norm{\bfv_i}_2\norm{\bfw_i}_2}\paren{\bfI-\bbfv_i\bbfv_i^\top}\underbrace{\EXP[\bfx]{\paren{f\paren{\bm{\theta},\bfx} - f^\star\paren{\bfx}}\pi'\paren{\bbfv_i^\top\bfx}\sigma'\paren{\bbfw_i^\top\bfx}\bfx\bfx^\top}}_{\calT_{i,j,10}}\paren{\bfI-\bbfw_i\bbfw_i^\top}
    \end{align*}
    For the convenience of the analysis, we define $C_{S,0} = 2\sum_{k=0}^{\infty}\frac{c_k^2}{k!}$ and $C_{S,1} = 6\sum_{k=0}^{\infty}\frac{c_{k+1}^2}{k!}$. Our next lemma controls the magnitudes of these blocks.
    
    \textbf{Bounding $\calT_{i,j,2},\calT_{i,j,6}$ and $\calT_{i,j,10}$.} By Lemma~\ref{lem:sec_stein}, we have that
    \begin{align*}
        \calT_{i,j,2} & = \EXP[\bfx]{\paren{f\paren{\bm{\theta},\bfx} - f^\star\paren{\bfx}}\pi''\paren{\bbfv_i^\top\bfx}\sigma\paren{\bbfw_i^\top\bfx}}\bfI_d\\
        & \qqquad + \EXP[\bfx]{\paren{f\paren{\bm{\theta},\bfx} - f^\star\paren{\bfx}}\nabla^2_{\bfx}\paren{\pi''\paren{\bbfv_i^\top\bfx}\sigma\paren{\bbfw_i^\top\bfx}}}\\
        & \qqquad + \EXP[\bfx]{\nabla^2_{\bfx}\paren{f\paren{\bm{\theta},\bfx} - f^\star\paren{\bfx}}\pi''\paren{\bbfv_i^\top\bfx}\sigma\paren{\bbfw_i^\top\bfx}}\\
        & \qqquad + \EXP[\bfx]{\nabla_{\bfx}\paren{f\paren{\bm{\theta},\bfx} - f^\star\paren{\bfx}}\nabla_{\bfx}\paren{\pi''\paren{\bbfv_i^\top\bfx}\sigma\paren{\bbfw_i^\top\bfx}}^\top}
    \end{align*}
    Similarly, we have that
    \begin{align*}
        \calT_{i,j,6} & = \EXP[\bfx]{\paren{f\paren{\bm{\theta},\bfx} - f^\star\paren{\bfx}}\pi\paren{\bbfv_i^\top\bfx}\sigma''\paren{\bbfw_i^\top\bfx}}\bfI_d\\
        & \qqquad + \EXP[\bfx]{\paren{f\paren{\bm{\theta},\bfx} - f^\star\paren{\bfx}}\nabla^2_{\bfx}\paren{\pi\paren{\bbfv_i^\top\bfx}\sigma''\paren{\bbfw_i^\top\bfx}}}\\
        & \qqquad + \EXP[\bfx]{\nabla^2_{\bfx}\paren{f\paren{\bm{\theta},\bfx} - f^\star\paren{\bfx}}\pi\paren{\bbfv_i^\top\bfx}\sigma''\paren{\bbfw_i^\top\bfx}}\\
        & \qqquad + \EXP[\bfx]{\nabla_{\bfx}\paren{f\paren{\bm{\theta},\bfx} - f^\star\paren{\bfx}}\nabla_{\bfx}\paren{\pi\paren{\bbfv_i^\top\bfx}\sigma''\paren{\bbfw_i^\top\bfx}}^\top}
    \end{align*}
    Also, we have that
    \begin{align*}
        \calT_{i,j,10} & = \EXP[\bfx]{\paren{f\paren{\bm{\theta},\bfx} - f^\star\paren{\bfx}}\pi'\paren{\bbfv_i^\top\bfx}\sigma'\paren{\bbfw_i^\top\bfx}}\bfI_d\\
        & \qqquad + \EXP[\bfx]{\paren{f\paren{\bm{\theta},\bfx} - f^\star\paren{\bfx}}\nabla^2_{\bfx}\paren{\pi'\paren{\bbfv_i^\top\bfx}\sigma'\paren{\bbfw_i^\top\bfx}}}\\
        & \qqquad + \EXP[\bfx]{\nabla^2_{\bfx}\paren{f\paren{\bm{\theta},\bfx} - f^\star\paren{\bfx}}\pi'\paren{\bbfv_i^\top\bfx}\sigma'\paren{\bbfw_i^\top\bfx}}\\
        & \qqquad + \EXP[\bfx]{\nabla_{\bfx}\paren{f\paren{\bm{\theta},\bfx} - f^\star\paren{\bfx}}\nabla_{\bfx}\paren{\pi'\paren{\bbfv_i^\top\bfx}\sigma'\paren{\bbfw_i^\top\bfx}}^\top}
    \end{align*}
    Thus, we can apply Lemma~\ref{lem:hermite_residue_bound} to obtain that
    \[
        \norm{\calT_{i,j,2}}_2, \norm{\calT_{i,j,6}}_2, \norm{\calT_{i,j,10}}_2 \leq \calO\paren{m^\star\varepsilon}
    \]
    
    \textbf{Bounding $\mathcal{T}_{i,j,3}$ and $\mathcal{T}_{i,j,7}$}
    By Stein's Lemma, we have that
    \begin{align*}
        \mathcal{T}_{i,j,3} & = \EXP[\bfx]{\bbfv_i^\top\nabla_{\bfx}\paren{f\paren{\bm{\theta},\bfx} - f^\star\paren{\bfx}}\pi'\paren{\bbfv_i^\top\bfx}\sigma\paren{\bbfw_i^\top\bfx}}\\
        & \qqquad + \EXP[\bfx]{\paren{f\paren{\bm{\theta},\bfx} - f^\star\paren{\bfx}}\bbfv_i^\top\nabla_{\bfx}\paren{\pi'\paren{\bbfv_i^\top\bfx}\sigma\paren{\bbfw_i^\top\bfx}}}\\
        \mathcal{T}_{i,j,7} & = \EXP[\bfx]{\bbfw_i^\top\nabla_{\bfx}\paren{f\paren{\bm{\theta},\bfx} - f^\star\paren{\bfx}}\pi\paren{\bbfv_i^\top\bfx}\sigma'\paren{\bbfw_i^\top\bfx}}\\
        & \qqquad + \EXP[\bfx]{\paren{f\paren{\bm{\theta},\bfx} - f^\star\paren{\bfx}}\bbfw_i^\top\nabla_{\bfx}\paren{\pi\paren{\bbfv_i^\top\bfx}\sigma'\paren{\bbfw_i^\top\bfx}}}
    \end{align*}
    Therefore, we have that
    \begin{align*}
        \left|\mathcal{T}_{i,j,3}\right| & \leq \norm{\EXP[\bfx]{\nabla_{\bfx}\paren{f\paren{\bm{\theta},\bfx} - f^\star\paren{\bfx}}\pi'\paren{\bbfv_i^\top\bfx}\sigma\paren{\bbfw_i^\top\bfx}}}_2\\
        & \qqquad + \norm{\EXP[\bfx]{\paren{f\paren{\bm{\theta},\bfx} - f^\star\paren{\bfx}}\nabla_{\bfx}\paren{\pi'\paren{\bbfv_i^\top\bfx}\sigma\paren{\bbfw_i^\top\bfx}}}}_2\\
        \left|\mathcal{T}_{i,j,7}\right| & \leq \norm{\EXP[\bfx]{\nabla_{\bfx}\paren{f\paren{\bm{\theta},\bfx} - f^\star\paren{\bfx}}\pi\paren{\bbfv_i^\top\bfx}\sigma'\paren{\bbfw_i^\top\bfx}}}_2\\
        & \qqquad + \norm{\EXP[\bfx]{\paren{f\paren{\bm{\theta},\bfx} - f^\star\paren{\bfx}}\nabla_{\bfx}\paren{\pi\paren{\bbfv_i^\top\bfx}\sigma'\paren{\bbfw_i^\top\bfx}}}}_2
    \end{align*}
    Applying Lemma~\ref{lem:hermite_residue_bound} gives that
    \[
        \left|\mathcal{T}_{i,j,3}\right|, \left|\mathcal{T}_{i,j,7}\right| \leq \calO\paren{m^\star\varepsilon}
    \]
    
    \textbf{Bounding $\mathcal{T}_{i,j,4}$ and $\mathcal{T}_{i,j,8}$.} By the structure of $\mathcal{T}_{i,j,4}$ and $\mathcal{T}_{i,j,8}$, we have that
    \begin{align*}
        \norm{\mathcal{T}_{i,j,4}}_2 & \leq 2\norm{\EXP[\bfx]{\paren{f\paren{\bm{\theta},\bfx} - f^\star\paren{\bfx}}\pi'\paren{\bbfv_i^\top\bfx}\sigma\paren{\bbfw_i^\top\bfx}}\bfx}_2\\
        & \leq 2\norm{\EXP[\bfx]{\nabla_{\bfx}\paren{f\paren{\bm{\theta},\bfx} - f^\star\paren{\bfx}}\pi'\paren{\bbfv_i^\top\bfx}\sigma\paren{\bbfw_i^\top\bfx}}}_2\\
        & \qqquad + 2\norm{\EXP[\bfx]{\paren{f\paren{\bm{\theta},\bfx} - f^\star\paren{\bfx}}\nabla_{\bfx}\paren{\pi'\paren{\bbfv_i^\top\bfx}\sigma\paren{\bbfw_i^\top\bfx}}}}_2
    \end{align*}
    Similarly, for $\mathcal{T}_{i,j,8}$, we have that
    \begin{align*}
        \norm{\mathcal{T}_{i,j,8}}_2 & \leq 2\norm{\EXP[\bfx]{\paren{f\paren{\bm{\theta},\bfx} - f^\star\paren{\bfx}}\pi\paren{\bbfv_i^\top\bfx}\sigma'\paren{\bbfw_i^\top\bfx}}\bfx}_2\\
        & \leq 2\norm{\EXP[\bfx]{\nabla_{\bfx}\paren{f\paren{\bm{\theta},\bfx} - f^\star\paren{\bfx}}\pi\paren{\bbfv_i^\top\bfx}\sigma'\paren{\bbfw_i^\top\bfx}}}_2\\
        & \qqquad + 2\norm{\EXP[\bfx]{\paren{f\paren{\bm{\theta},\bfx} - f^\star\paren{\bfx}}\nabla_{\bfx}\paren{\pi\paren{\bbfv_i^\top\bfx}\sigma'\paren{\bbfw_i^\top\bfx}}}}_2
    \end{align*}
    Applying Lemma~\ref{lem:hermite_residue_bound} gives that
    \[
        \norm{\mathcal{T}_{i,j,4}}_2, \norm{\mathcal{T}_{i,j,8}}_2 \leq \calO\paren{m^\star\varepsilon}
    \]
    
    \textbf{Bounding $\calT_{i,j,1}, \calT_{i,j,5}, \calT_{i,j,10}$ for $i\neq j$.} By Lemma~\ref{lem:sec_stein}, we have that
    \begin{align*}
        \mathcal{T}_{i,j,1} & = \EXP[\bfx]{\pi'\paren{\bbfv_i^\top\bfx}\sigma\paren{\bbfw_i^\top\bfx}\pi'\paren{\bbfv_j^\top\bfx}\sigma\paren{\bbfw_j^\top\bfx}}\bfI_d\\
        & \qqquad + \EXP[\bfx]{\nabla_{\bfx}^2\paren{\pi'\paren{\bbfv_i^\top\bfx}\sigma\paren{\bbfw_i^\top\bfx}}\pi'\paren{\bbfv_j^\top\bfx}\sigma\paren{\bbfw_j^\top\bfx}}\\
        & \qqquad + \EXP[\bfx]{\pi'\paren{\bbfv_i^\top\bfx}\sigma\paren{\bbfw_i^\top\bfx}\nabla_{\bfx}^2\paren{\pi'\paren{\bbfv_j^\top\bfx}\sigma\paren{\bbfw_j^\top\bfx}}}\\
        & \qqquad + \EXP[\bfx]{\nabla_{\bfx}\paren{\pi'\paren{\bbfv_i^\top\bfx}\sigma\paren{\bbfw_i^\top\bfx}}\nabla_{\bfx}\paren{\pi'\paren{\bbfv_j^\top\bfx}\sigma\paren{\bbfw_j^\top\bfx}}^\top}\\
        \mathcal{T}_{i,j,5} & = \EXP[\bfx]{\pi\paren{\bbfv_i^\top\bfx}\sigma'\paren{\bbfw_i^\top\bfx}\pi\paren{\bbfv_j^\top\bfx}\sigma'\paren{\bbfw_j^\top\bfx}}\bfI_d\\
        & \qqquad + \EXP[\bfx]{\nabla_{\bfx}^2\paren{\pi\paren{\bbfv_i^\top\bfx}\sigma'\paren{\bbfw_i^\top\bfx}}\pi\paren{\bbfv_j^\top\bfx}\sigma'\paren{\bbfw_j^\top\bfx}}\\
        & \qqquad + \EXP[\bfx]{\pi\paren{\bbfv_i^\top\bfx}\sigma'\paren{\bbfw_i^\top\bfx}\nabla_{\bfx}^2\paren{\pi\paren{\bbfv_j^\top\bfx}\sigma'\paren{\bbfw_j^\top\bfx}}}\\
        & \qqquad + \EXP[\bfx]{\nabla_{\bfx}\paren{\pi\paren{\bbfv_i^\top\bfx}\sigma'\paren{\bbfw_i^\top\bfx}}\nabla_{\bfx}\paren{\pi\paren{\bbfv_j^\top\bfx}\sigma'\paren{\bbfw_j^\top\bfx}}^\top}\\
        \mathcal{T}_{i,j,10} & = \EXP[\bfx]{\pi'\paren{\bbfv_i^\top\bfx}\sigma\paren{\bbfw_i^\top\bfx}\pi\paren{\bbfv_j^\top\bfx}\sigma'\paren{\bbfw_j^\top\bfx}}\bfI_d\\
        & \qqquad + \EXP[\bfx]{\nabla_{\bfx}^2\paren{\pi'\paren{\bbfv_i^\top\bfx}\sigma\paren{\bbfw_i^\top\bfx}}\pi\paren{\bbfv_j^\top\bfx}\sigma'\paren{\bbfw_j^\top\bfx}}\\
        & \qqquad + \EXP[\bfx]{\pi'\paren{\bbfv_i^\top\bfx}\sigma\paren{\bbfw_i^\top\bfx}\nabla_{\bfx}^2\paren{\pi\paren{\bbfv_j^\top\bfx}\sigma'\paren{\bbfw_j^\top\bfx}}}\\
        & \qqquad + \EXP[\bfx]{\nabla_{\bfx}\paren{\pi'\paren{\bbfv_i^\top\bfx}\sigma\paren{\bbfw_i^\top\bfx}}\nabla_{\bfx}\paren{\pi\paren{\bbfv_j^\top\bfx}\sigma'\paren{\bbfw_j^\top\bfx}}^\top}
    \end{align*}
    Applying Lemma~\ref{lem:hermite_hess_bound} gives that
    \[
        \norm{\mathcal{T}_{i,j,1}}_2, \norm{\mathcal{T}_{i,j,5}}_2, \norm{\mathcal{T}_{i,j,10}}_2 \leq \calO\paren{\varepsilon}
    \]
    
    \textbf{Bounding $\calT_{i,i,1}, \calT_{i,i,5}, \calT_{i,i,10}$.} By Lemma~\ref{lem:sec_stein}, we have that
    \begin{align*}
        \calT_{i,i,1} & = \EXP[\bfx]{\pi'\paren{\bbfv_i^\top\bfx}^2\sigma\paren{\bbfw_i^\top\bfx}^2}\bfI_d + 4\EXP[\bfx]{\pi''\paren{\bbfv_i^\top\bfx}^2\sigma'\paren{\bbfw_i^\top\bfx}^2}\paren{\bbfv_i\bbfw_i^\top + \bbfw_i\bbfv_i^\top}\\
        & \qqquad + 2\EXP[\bfx]{\paren{\pi''\paren{\bbfv_i^\top\bfx}^2 + \pi'''\paren{\bbfv_i^\top\bfx}\pi'\paren{\bbfv_i^\top\bfx}}\sigma\paren{\bbfw_i^\top\bfx}^2}\bbfv_i\bbfv_i^\top\\
        & \qqquad + 2\EXP[\bfx]{\paren{\sigma'\paren{\bbfw_i^\top\bfx}^2 + \sigma''\paren{\bbfw_i^\top\bfx}\sigma\paren{\bbfw_i^\top\bfx}}\pi'\paren{\bbfv_i^\top\bfx}^2}\bbfw_i\bbfw_i^\top\\
        \calT_{i,i,5} & = \EXP[\bfx]{\pi\paren{\bbfv_i^\top\bfx}^2\sigma'\paren{\bbfw_i^\top\bfx}^2}\bfI_d + 4\EXP[\bfx]{\pi'\paren{\bbfv_i^\top\bfx}^2\sigma''\paren{\bbfw_i^\top\bfx}^2}\paren{\bbfv_i\bbfw_i^\top + \bbfw_i\bbfv_i^\top}\\
        & \qqquad + 2\EXP[\bfx]{\paren{\pi'\paren{\bbfv_i^\top\bfx}^2 + \pi''\paren{\bbfv_i^\top\bfx}\pi\paren{\bbfv_i^\top\bfx}}\sigma'\paren{\bbfw_i^\top\bfx}^2}\bbfv_i\bbfv_i^\top\\
        & \qqquad + 2\EXP[\bfx]{\paren{\sigma''\paren{\bbfw_i^\top\bfx}^2 + \sigma'''\paren{\bbfw_i^\top\bfx}\sigma'\paren{\bbfw_i^\top\bfx}}\pi\paren{\bbfv_i^\top\bfx}^2}\bbfw_i\bbfw_i^\top\\
        \calT_{i,i,10} & = \EXP[\bfx]{\pi'\paren{\bbfv_i^\top\bfx}\pi\paren{\bbfv_i^\top\bfx}\sigma'\paren{\bbfw_i^\top\bfx}\sigma\paren{\bbfw_i^\top\bfx}}\bfI_d\\
        & \qqquad + \EXP[\bfx]{\pi'\paren{\bbfv_i^\top\bfx}^2\sigma'\paren{\bbfw_i^\top\bfx}^2}\paren{\bbfv_i\bbfw_i^\top + \bbfw_i\bbfv_i^\top}\\
        & \qqquad + \EXP[\bfx]{\pi''\paren{\bbfv_i^\top\bfx}\pi\paren{\bbfv_i^\top\bfx}\sigma''\paren{\bbfw_i^\top\bfx}\sigma\paren{\bbfw_i^\top\bfx}}\paren{\bbfv_i\bbfw_i^\top + \bbfw_i\bbfv_i^\top}\\
        & \qqquad + \EXP[\bfx]{\pi'\paren{\bbfv_i^\top\bfx}^2\sigma''\paren{\bbfw_i^\top\bfx}\sigma\paren{\bbfw_i^\top\bfx}}\paren{\bbfv_i\bbfw_i^\top + \bbfw_i\bbfv_i^\top}\\
        & \qqquad + \EXP[\bfx]{\pi''\paren{\bbfv_i^\top\bfx}\pi\paren{\bbfv_i^\top\bfx}\sigma'\paren{\bbfw_i^\top\bfx}^2}\paren{\bbfv_i\bbfw_i^\top + \bbfw_i\bbfv_i^\top}\\
        & \qqquad + \EXP[\bfx]{\paren{\pi'''\paren{\bbfv_i^\top\bfx}\pi\paren{\bbfv_i^\top\bfx}^2 + 2\pi''\paren{\bbfv_i^\top\bfx}\pi'\paren{\bbfv_i^\top\bfx}}\sigma'\paren{\bbfw_i^\top\bfx}\sigma\paren{\bbfw_i^\top\bfx}}\bbfv_i\bbfv_i^\top\\
        & \qqquad + \EXP[\bfx]{\paren{2\sigma''\paren{\bbfw_i^\top\bfx}\sigma'\paren{\bbfw_i^\top\bfx}^2 + \sigma'''\paren{\bbfw_i^\top\bfx}\sigma\paren{\bbfw_i^\top\bfx}}\pi'\paren{\bbfv_i^\top\bfx}\pi\paren{\bbfv_i^\top\bfx}}\bbfw_i\bbfw_i^\top
    \end{align*}
    Thus, we have that
    \begin{align*}
        & \paren{\bfI - \bbfv_i\bbfv_i^\top}\calT_{i,i,1}\paren{\bfI - \bbfv_i\bbfv_i^\top}\\
        &\qqquad  = \EXP[\bfx]{\pi'\paren{\bbfv_i^\top\bfx}^2\sigma\paren{\bbfw_i^\top\bfx}^2}\paren{\bfI - \bbfv_i\bbfv_i^\top}\\
        & \qqqqquad + 2\EXP[\bfx]{\paren{\sigma'\paren{\bbfw_i^\top\bfx}^2 + \sigma''\paren{\bbfw_i^\top\bfx}\sigma\paren{\bbfw_i^\top\bfx}}\pi'\paren{\bbfv_i^\top\bfx}^2}\paren{\bfI - \bbfv_i\bbfv_i^\top}\bbfw_i\bbfw_i^\top\paren{\bfI - \bbfv_i\bbfv_i^\top}\\
        & \qqquad = \EXP[\bfx]{\pi'\paren{\bbfv_i^\top\bfx}^2\sigma\paren{\bbfw_i^\top\bfx}^2}\paren{\bfI - \bbfv_i\bbfv_i^\top}\\ & \qqqqquad + 2\EXP[\bfx]{\sigma'\paren{\bbfw_i^\top\bfx}^2\pi'\paren{\bbfv_i^\top\bfx}^2}\paren{\bfI - \bbfv_i\bbfv_i^\top}\bbfw_i\bbfw_i^\top\paren{\bfI - \bbfv_i\bbfv_i^\top} + \hat{\mathcal{T}}_{i,i,1}\\
        & \paren{\bfI - \bbfw_i\bbfw_i^\top}\calT_{i,i,5}\paren{\bfI - \bbfw_i\bbfw_i^\top}\\
        & \qqquad = \EXP[\bfx]{\pi\paren{\bbfv_i^\top\bfx}^2\sigma'\paren{\bbfw_i^\top\bfx}^2}\paren{\bfI - \bbfw_i\bbfw_i^\top}\\
        & \qqqqquad + 2\EXP[\bfx]{\paren{\pi'\paren{\bbfv_i^\top\bfx}^2 + \pi''\paren{\bbfv_i^\top\bfx}\pi\paren{\bbfv_i^\top\bfx}}\sigma'\paren{\bbfw_i^\top\bfx}^2}\paren{\bfI - \bbfw_i\bbfw_i^\top}\bbfv_i\bbfv_i^\top\paren{\bfI - \bbfw_i\bbfw_i^\top}
    \end{align*}
    with $\norm{\hat{\mathcal{T}}_{i,i,1}}_2\leq \calO\paren{\varepsilon}$. This gives that
    \begin{gather*}
        \bfu^\top\paren{\bfI - \bbfv_i\bbfv_i^\top}\calT_{i,i,1}\paren{\bfI - \bbfv_i\bbfv_i^\top}\bfu \geq \EXP[\bfx]{\pi'\paren{\bbfv_i^\top\bfx}^2\sigma\paren{\bbfw_i^\top\bfx}^2}\norm{\paren{\bfI - \bbfv_i\bbfv_i^\top}\bfu}_2^2 - \calO\paren{\varepsilon}\\
        \bfu^\top\paren{\bfI - \bbfw_i\bbfw_i^\top}\calT_{i,i,5}\paren{\bfI - \bbfw_i\bbfw_i^\top}\bfu \geq \EXP[\bfx]{\pi\paren{\bbfv_i^\top\bfx}^2\sigma'\paren{\bbfw_i^\top\bfx}^2} \norm{\paren{\bfI - \bbfw_i\bbfw_i^\top}\bfu}_2^2
    \end{gather*}
    Imposing the condition that $\norm{\paren{\bfI - \bbfv_i\bbfv_i^\top}\bfu}_2^2 \geq \mathcal{Q}^\star\norm{\bfu}_2^2$ gives that
    \begin{gather*}
        \bfu^\top\paren{\bfI - \bbfv_i\bbfv_i^\top}\calT_{i,i,1}\paren{\bfI - \bbfv_i\bbfv_i^\top}\bfu \geq \mathcal{Q}^\star\EXP[\bfx]{\pi'\paren{\bbfv_i^\top\bfx}^2\sigma\paren{\bbfw_i^\top\bfx}^2}\norm{\bfu}_2^2 - \calO\paren{\varepsilon}\\
        \bfu^\top\paren{\bfI - \bbfw_i\bbfw_i^\top}\calT_{i,i,5}\paren{\bfI - \bbfw_i\bbfw_i^\top}\bfu\geq \mathcal{Q}^\star\EXP[\bfx]{\pi\paren{\bbfv_i^\top\bfx}^2\sigma'\paren{\bbfw_i^\top\bfx}^2}\norm{\bfu}_2^2
    \end{gather*}
    For $\mathcal{T}_{i,i,10}$, we first notice that
    \begin{gather*}
        \left|\EXP[\bfx]{\pi'\paren{\bbfv_i^\top\bfx}\pi\paren{\bbfv_i^\top\bfx}\sigma'\paren{\bbfw_i^\top\bfx}\sigma\paren{\bbfw_i^\top\bfx}}\right|\leq \calO\paren{\varepsilon}\\
        \left|\EXP[\bfx]{\pi''\paren{\bbfv_i^\top\bfx}\pi\paren{\bbfv_i^\top\bfx}\sigma''\paren{\bbfw_i^\top\bfx}\sigma\paren{\bbfw_i^\top\bfx}}\right| \leq \calO\paren{\varepsilon} \\
        \left|\EXP[\bfx]{\pi'\paren{\bbfv_i^\top\bfx}^2\sigma''\paren{\bbfw_i^\top\bfx}\sigma\paren{\bbfw_i^\top\bfx}}\right|\leq \calO\paren{\varepsilon} \\
        \left|\EXP[\bfx]{\paren{\pi'''\paren{\bbfv_i^\top\bfx}\pi\paren{\bbfv_i^\top\bfx}^2 + 2\pi''\paren{\bbfv_i^\top\bfx}\pi'\paren{\bbfv_i^\top\bfx}}\sigma'\paren{\bbfw_i^\top\bfx}\sigma\paren{\bbfw_i^\top\bfx}}\right| \leq \calO\paren{\varepsilon}\\
        \left|\EXP[\bfx]{\paren{2\sigma''\paren{\bbfw_i^\top\bfx}\sigma'\paren{\bbfw_i^\top\bfx}^2 + \sigma'''\paren{\bbfw_i^\top\bfx}\sigma\paren{\bbfw_i^\top\bfx}}\pi'\paren{\bbfv_i^\top\bfx}\pi\paren{\bbfv_i^\top\bfx}}\right| \leq \calO\paren{\varepsilon}
    \end{gather*}
    Therefore, we have that
    \begin{align*}
        & \paren{\bfI - \bbfw_i\bbfw_i^\top}\calT_{i,i,10}\paren{\bfI - \bbfv_i\bbfv_i^\top}\\
        & \qqquad = \EXP[\bfx]{\paren{\pi'\paren{\bbfv_i^\top\bfx}^2 + \pi''\paren{\bbfv_i^\top\bfx}\pi\paren{\bbfv_i^\top\bfx}}\sigma'\paren{\bbfw_i^\top\bfx}^2}\paren{\bfI - \bbfw_i\bbfw_i^\top}\bbfv_i\bbfw_i^\top\paren{\bfI - \bbfv_i\bbfv_i^\top}
    \end{align*}
    This implies that
    \begin{align*}
        & \bfu_1^\top\paren{\bfI - \bbfw_i\bbfw_i^\top}\calT_{i,i,10}\paren{\bfI - \bbfv_i\bbfv_i^\top}\bfu_2\\
        & \qqquad \leq \EXP[\bfx]{\paren{\pi'\paren{\bbfv_i^\top\bfx}^2 + \pi''\paren{\bbfv_i^\top\bfx}\pi\paren{\bbfv_i^\top\bfx}}\sigma'\paren{\bbfw_i^\top\bfx}^2}\norm{\bfu_1}_2\norm{\bfu_2}_2
    \end{align*}
\end{proof}

Now, we are ready to prove the fine-tuninig convergence.

\begin{proof}[Proof of Theorem~\ref{thm:fine-tuning}]
    By the mean-value theorem, we have that
    \[
        \nabla\calL\paren{\bm{\theta}} = \nabla\calL\paren{\bm{\theta}^\star} + \nabla^2\calL\paren{\hat{\bm{\theta}}}\paren{\bm{\theta} - \bm{\theta}^\star} = \nabla^2\calL\paren{\hat{\bm{\theta}}}\paren{\bm{\theta} - \bm{\theta}^\star}
    \]
    for some $\hat{\bm{\theta}} \in [\bm{\theta},\bm{\theta}^\star]$. The gradient flow dynamic implies that
    \begin{align*}
        \derivt \norm{\bar{\bm{\theta}}(t) - \bm{\theta}^\star}_2^2 & = \inner{\bar{\bm{\theta}}(t) - \bm{\theta}^\star}{\derivt \bar{\bm{\theta}}(t)}\\
        & = \inner{\bar{\bm{\theta}}(t) - \bm{\theta}^\star}{\bfN_{\bm{\theta}(t)}\derivt \bm{\theta}(t)}\\
        & = -\inner{\bfN_{\bm{\theta}(t)}\paren{\bar{\bm{\theta}}(t) - \bm{\theta}^\star}}{\nabla\calL\paren{\bm{\theta}}}\\
        & = -\inner{\bfN_{\bm{\theta}(t)}\paren{\bar{\bm{\theta}}(t) - \bm{\theta}^\star}}{\nabla^2\calL\paren{\hat{\bm{\theta}}}\paren{\bar{\bm{\theta}} - \bm{\theta}^\star}}
    \end{align*}
    Notice that $\bfN_{\bm{\theta}(t)}\paren{\bar{\bm{\theta}}(t) - \bm{\theta}^\star}$ takes the form
    \[
        \bfN_{\bm{\theta}(t)}\paren{\bar{\bm{\theta}}(t) - \bm{\theta}^\star} = \begin{bmatrix}
            \norm{\bfv_1}_2^{-1}\bbfv_1\\
            \vdots\\
            \norm{\bfv_{m^\star}}_2^{-1}\bbfv_{m^\star}\\
            \norm{\bfw_1}_2^{-1}\bbfw_1\\
            \vdots\\
            \norm{\bbfw_{m^\star}}_2^{-1}\bbfw_{m^\star}
        \end{bmatrix}
    \]
    Thus, we are going to apply Theorem~\ref{thm:hessian_pd} with $N_{\min} = 1 - o(1)\frac{\delta_{\mathbb{P}}}{m^2}, N_{\max} = 1 + o(1)\frac{\delta_{\mathbb{P}}}{m^2}$, and $\alpha_i = \norm{\bfv_i}_2^{-1} \geq \frac{m^2}{m^2 + o(1)\delta_{\mathbb{P}}}, \beta_i = \norm{\bfw_i}_2^{-1} \leq \frac{m^2}{m^2 - o(1)\delta_{\mathbb{P}}}$. This leads to the condition that $\varepsilon\leq o\paren{\frac{\mathcal{Q}^{\star 2}}{m^{\star 2}}}$ and $\frac{C_{S,0}}{C_{S,1}}\geq \frac{1.05}{\mathcal{Q}^{\star 2}}$. Under such condition, by Theorem~\ref{thm:hessian_pd}, we have that
    \[
        \derivt \norm{\bar{\bm{\theta}}(t) - \bm{\theta}^\star}_2^2 \leq  - \paren{1 - o(1)\frac{\delta_{\mathbb{P}}}{m^2}}\mathcal{Q}^\star\kappa\norm{\bar{\bm{\theta}}(t) - \bm{\theta}^\star}_2^2
    \]
    This shows that $\norm{\bar{\bm{\theta}}(t) - \bm{\theta}^\star}_2^2$ decreases monotonically.
    To find $\mathcal{Q}^\star$, we notice that
    \begin{align*}
        \norm{\paren{\bfI-\bbfv_i\bbfv_i^\top}\paren{\bbfv_i(t) - \bbfv_i^\star}}_2^2 & = \inner{\bbfv_i(t) - \bbfv_i^\star}{\paren{\bfI-\bbfv_i\bbfv_i^\top}\paren{\bbfv_i(t) - \bbfv_i^\star}}\\
        & = \norm{\bbfv_i(t) - \bbfv_i^\star}_2^2 - \paren{\bbfv_i^\top\paren{\bbfv_i(t) - \bbfv_i^\star}}^2\\
        & = \norm{\bbfv_i(t) - \bbfv_i^\star}_2^2 - \paren{1 - \bbfv_i^\top\bbfv_i^\star}^2\\
        & = \norm{\bbfv_i(t) - \bbfv_i^\star}_2^2 - \frac{1}{4}\norm{\bbfv_i(t) - \bbfv_i^\star}_2^4\\
        & = \paren{1 - \frac{1}{4}\norm{\bbfv_i(t) - \bbfv_i^\star}_2^2}\norm{\bbfv_i(t) - \bbfv_i^\star}_2^2
    \end{align*}
    Similarly, we can obtain that
    \[
        \norm{\paren{\bfI-\bbfw_i\bbfw_i^\top}\paren{\bbfw_i(t) - \bbfw_i^\star}}_2^2 = \paren{1 - \frac{1}{4}\norm{\bbfw_i(t) - \bbfw_i^\star}_2^2}\norm{\bbfw_i(t) - \bbfw_i^\star}_2^2
    \]
    This gives that $\mathcal{Q}^\star = 1 - \frac{1}{4}\max_{i\in[m^\star]}\max\left\{\norm{\bbfv_i(0) - \bbfv_i^\star}_2^2, \paren{\bbfw_i(0) - \bbfw_i^\star}\right\} = 1 - \calO\paren{\varepsilon}$. Thus, the condition that $\varepsilon\leq o\paren{\frac{1}{m^{\star 2}}}$ and $C_{S,0} \geq 1.1C_{S,1}$ suffice. This gives us that
    \[
        \derivt \norm{\bar{\bm{\theta}}(t) - \bm{\theta}^\star}_2^2 \leq  - \paren{1 - o(1)\frac{\delta_{\mathbb{P}}}{m^2}}\paren{1 - \calO\paren{\varepsilon}}\kappa\norm{\bar{\bm{\theta}}(t) - \bm{\theta}^\star}_2^2 \leq \frac{\kappa}{2}\norm{\bar{\bm{\theta}}(t) - \bm{\theta}^\star}_2^2
    \]
    Solving the ODE gives that
    \[
        \norm{\bar{\bm{\theta}}(t) - \bm{\theta}^\star}_2^2 \leq \exp{-\frac{\kappa t}{2}}\norm{\bar{\bm{\theta}}(0) - \bm{\theta}^\star}_2^2
    \]
\end{proof}

\section{Auxiliary Results}

\subsection{Hermite Polynomials}
\begin{lemma}[Restatement of Lemma~\ref{lem:prod_hermite}]
    \label{lem:prod_hermite_re}
    Let $\bfx\sim\mathcal{N}\paren{\bm{0},\bm{\Sigma}}$. For some multi-index $\bfk\in\N^n$, we define the multi-variate Hermite polynomial as
    \[
        He_{\bfk}\paren{\bfx} = \prod_{i=1}^nHe_{\bfk[i]}\paren{\bfx[i]}
    \]
    Then we have that
    \[
        \EXP[\bfx\sim\mathcal{N}\paren{\bm{0},\bm{\Sigma}}]{He_{\bfk}\paren{\bfx}} = \paren{\prod_{i=1}^n\bfk[i]!}\sum_{\bfM\in\mathcal{S}}\prod_{i,j=1}^n\frac{\bm{\Sigma}[i,j]^{\bfM[i,j]}}{\bfM[i,j]!}
    \]
    where the set $\mathcal{S}$ is defined by
    \[
        \mathcal{S} = \left\{\bfM\in\N^{n\times n}: \bfM = \bfM^\top, \sum_{j=1}^n\bfM[i,j] = \bfk[i], \bfM[i,i] = 0;\;\forall i\in[n],\;\right\}
    \]
\end{lemma}
\begin{proof}
    Consider the generating function of Hermite polynomials
    \[
        \exp{xt - \frac{t^2}{2}} = \sum_{k=0}^\infty \frac{He_k(x)}{k!}\cdot t^k
    \]
    Let $x_1,\dots,x_n\sim\mathcal{N}\paren{0,1}$, then for all $\left\{t_i\right\}_{i=1}^n$ we have
    \begin{align*}
    \exp{\sum_{i=1}^n\paren{x_it_i-\frac{t_i^2}{2}}} & = \prod_{i=1}^n\exp{x_it_i-\frac{t_i^2}{2}}\\
    & = \prod_{i=1}^n\paren{\sum_{k=0}^\infty \frac{He_k(x_i)}{k!}\cdot t_i^k}\\
    & = \sum_{\bfk\in\N^n}\paren{\prod_{i=1}^n\frac{He_{k[i]}\paren{x_i}}{k[i]!}}\cdot \bft^\bfk
    \end{align*}
    where $\bfk\in\N^n$ is the multi-index. 
    On the other hand, if we write $x_i = \bfu_i^\top\bfx$ for some $\bfx\sim\mathcal{N}\paren{\bm{0}, \bfI}$ and $\bfu_i$ satisfying $\norm{\bfu_i}_2 = 1$, then we have
    \begin{align*}
        \EXP[x_i\sim\mathcal{N}\paren{0,1}]{\exp{\sum_{i=1}^n\paren{x_it_i-\frac{t_i^2}{2}}}} & = \EXP[x_i\sim\mathcal{N}\paren{0,1}]{\exp{\sum_{i=1}^n\paren{\bfu_i^\top\bfx \cdot t_i - \frac{t_i^2}{2}}}}\\
        & = \EXP[x_i\sim\mathcal{N}\paren{0,1}]{\exp{\bfx^\top\paren{\sum_{i=1}^nt_i\bfu_i}}}\exp{-\frac{1}{2}\sum_{i=1}^nt_i^2}
    \end{align*}
    Since $\bfx\sim\mathcal{N}\paren{\bm{0},\bfI}$, we must have that $\bfx^\top\paren{\sum_{i=1}^nt_i\bfu_i}\sim\mathcal{N}\paren{\bm{0},\norm{\sum_{i=1}^nt_i\bfu_i}_2^2}$. By the moment-generating function of Gaussian random variable we have that
    \[
        \EXP[x_i\sim\mathcal{N}\paren{0,1}]{\exp{\bfx^\top\paren{\sum_{i=1}^nt_i\bfu_i}}} = \exp{\frac{1}{2}\norm{\sum_{i=1}^nt_i\bfu_i}_2^2}
    \]
    Thus, we have that
    \[
        \EXP[x_i\sim\mathcal{N}\paren{0,1}]{\exp{\sum_{i=1}^n\paren{x_it_i-\frac{t_i^2}{2}}}} = \exp{\frac{1}{2}\paren{\norm{\sum_{i=1}^nt_i\bfu_i}_2^2 - \sum_{i=1}^nt_i^2}} = \exp{\frac{1}{2}\sum_{i\neq j}\bfu_i^\top\bfu_j t_it_j}
    \]
    Applying Taylor's expansion gives
    \[
        \exp{\frac{1}{2}\sum_{i\neq j}\bfu_i^\top\bfu_j t_it_j} = \sum_{\ell=0}^\infty\frac{1}{\ell !}\paren{\sum_{i < j}\bfu_i^\top\bfu_j t_it_j}^\ell
    \]
    Combining the results gives
    \[
        \sum_{\bfk\in\N^d}\EXP[x_i\sim\mathcal{N}\paren{0,1}]{\prod_{i=1}^nHe_{\bfk[i]}\paren{x_i}}\prod_{i=1}^n\frac{t_i^{\bfk[i]}}{\bfk[i]!} = \sum_{\ell=0}^\infty\frac{1}{\ell !}\paren{\sum_{i < j}\bfu_i^\top\bfu_j t_it_j}^\ell
    \]
    We intend to find out the cofficients of term $\prod_{i=1}^nt_i^{\bfk[i]}$ on the right-hand side. Notice that such term must only appears for term with $\ell$ satisfying $2\ell = \norm{\bfk}_1$. By the multinomial theorem we have that
    \[
        \paren{\sum_{i\neq j}\bfu_i^\top\bfu_j t_it_j}^\ell = \sum_{\bfk':\norm{\bfk'}_1 = \ell}\ell!\cdot \prod_{i < j}\frac{\paren{\bfu_i^\top\bfu_j}^{\bfk'[i,j]}}{\bfk'[i,j]!}t_i^{\bfk'[i,j]}t_j^{\bfk'[i,j]}
    \]
    Therefore, we must have that
    \[
        \EXP[x_i\sim\mathcal{N}\paren{0,1}]{\prod_{i=1}^nHe_{\bfk[i]}\paren{x_i}} = \paren{\prod_{i=1}^n\bfk[i]!}\sum_{\bfM\in\mathcal{S}}\prod_{i < j}\frac{\EXP{x_ix_j}^{\bfM[i,j]}}{\bfM[i,j]!}
    \]
    where $\mathcal{S}$ is given by
    \[
        \mathcal{S} = \left\{\bfM\in\N^{n\times n}:\Tr{\bfM} = 0;\;\forall i\in[n],\;\sum_{j=1}^n\bfM[i,j] = \bfk[i]\right\}
    \]
    Using vector notations, we have that
    \[
        \EXP[\bfx\sim\mathcal{N}\paren{\bm{0},\bm{\Sigma}}]{He_{\bfk}\paren{\bfx}} = \paren{\prod_{i=1}^n\bfk[i]!}\sum_{\bfM\in\mathcal{S}}\prod_{i<j}\frac{\bm{\Sigma}[i,j]^{\bfM[i,j]}}{\bfM[i,j]!}
    \]
\end{proof}

\begin{lemma}[Parseval's Identity]
    \label{lem:Parseval}
    Let $f(x)$ be given such that $\EXP[x\sim\mathcal{N}\paren{0,1}]{f^{(\ell)}(x)^2} \leq B$, and let $c_k$ be the $k$th Hermite coefficient of $f\paren{x}$. Then we have that
    \[
        c_{k+\ell}^2\leq B\cdot k!;\quad \forall k\geq 0
    \]
\end{lemma}
\begin{proof}
    Taking the Hermite expansion of $f^{(\ell)}$ gives
    \[
        f^{(\ell)} = \sum_{k=0}^{\infty}\frac{c'_k}{k!}He_k\paren{x};\quad c'_k = \EXP[x\sim\mathcal{N}\paren{0,1}]{f^{(\ell+k)}(x)} = c_{k+\ell}
    \]
    Therefore, we have that
    \[
        \EXP[x\sim\mathcal{N}\paren{0,1}]{f^{(\ell)}\paren{x}^2} = \sum_{k=0}^{\infty}\frac{c_k'{^2}}{k!} = \sum_{k=0}^{\infty}\frac{c_{k+\ell}}{k!} \leq B
    \]
    This implies that $c_{k+\ell}^2\leq B\cdot k!$ since $c_{k+\ell}^2\geq 0$ for all $k$.
\end{proof}

\begin{lemma}
    \label{lem:cs_hermite}
    Let $\bfv_1, \bfv_2, \bfw_1,\bfw_2\in\R^d$ be vectors of unit norm such that
    \begin{align*}
        \max\left\{\left|\bfv_1^\top\bfw_1\right|,\left|\bfv_1^\top\bfw_2\right|, \left|\bfv_2^\top\bfw_1\right|,\left|\bfv_2^\top\bfw_2\right|\right\} \leq \delta_{r}
    \end{align*}
    with some $\delta_r\in (0, 1)$. 
    Let $\left\{h_k\right\}_{k=0}^{\infty}, \left\{h_k'\right\}_{k=0}^{\infty}$ be two sequences of real numbers such that
    \begin{equation}
        \label{eq:sum_coef_converge_33}
        \sum_{k=0}^{\infty}\frac{h_{k+a}h_{k+b}'}{k!}\leq \calO\paren{1};\;\;\forall a+b\leq 6,\; a,b\in\mathbb{N}\cup\{0\}
    \end{equation}
    Then we have that
    \begin{align}
        \label{eq:lem_hermite_form_1}
        \begin{aligned}
            & \sum_{k,\ell=0}^{\infty}\frac{h_kh'_{\ell}}{k!\ell!}\EXP[\bfx\sim\mathcal{N}\paren{\bm{0},\bfI}]{He_k\paren{\bfv_1^\top\bfx}He_{\ell}\paren{\bfv_2^\top\bfx}He_3\paren{\bfw_1^\top\bfx}He_3\paren{\bfw_2^\top\bfx}}\\
            & \qqquad = 6\sum_{k=0}^{\infty}\frac{h_kh_{k}'}{k!}\paren{\bfv_1^\top\bfv_2}^k\paren{\bfw_1^\top\bfw_2}^3 \pm \calO\paren{\delta_r^2\paren{\bfw_1^\top\bfw_2}^2 + \delta_r^4}
        \end{aligned}\\
        \label{eq:lem_hermite_form_2}
        \begin{aligned}
            & \sum_{k,\ell=0}^{\infty}\frac{h_kh'_{\ell}}{k!\ell!}\EXP[\bfx\sim\mathcal{N}\paren{\bm{0},\bfI}]{He_k\paren{\bfv_1^\top\bfx}He_{\ell}\paren{\bfv_2^\top\bfx}He_2\paren{\bfw_1^\top\bfx}He_2\paren{\bfw_2^\top\bfx}}\\
            & \qqquad = 2\sum_{k=0}^{\infty}\frac{h_kh_{k'}}{k!}\paren{\bfv_1^\top\bfv_2}^k\paren{\bfw_1^\top\bfw_2}^2 \pm \calO\paren{\delta_r^2\cdot\left|\bfw_1^\top\bfw_2\right| + \delta_r^4}
        \end{aligned}\\
        \label{eq:lem_hermite_form_3}
        \begin{aligned}
            & \sum_{k,\ell=0}^{\infty}\frac{h_kh'_{\ell}}{k!\ell!}\EXP[\bfx\sim\mathcal{N}\paren{\bm{0},\bfI}]{He_k\paren{\bfv_1^\top\bfx}He_{\ell}\paren{\bfv_2^\top\bfx}He_2\paren{\bfw_1^\top\bfx}He_3\paren{\bfw_2^\top\bfx}}\\
            & \qqquad = 6\sum_{k=0}^{\infty}\frac{h_{k+1}h_{k}'}{k!}\paren{\bfv_1^\top\bfv_2}^k\paren{\bfw_1^\top\bfw_2}^2\bfv_1^\top\bfw_2\\
            & \qqqqquad+ 6\sum_{k=0}^{\infty}\frac{h_{k}h_{k+1}'}{k!}\paren{\bfv_1^\top\bfv_2}^k\paren{\bfw_1^\top\bfw_2}^2\bfv_2^\top\bfw_2 \pm \calO\paren{\delta_r^3}.
        \end{aligned}
    \end{align}
\end{lemma}

\begin{proof}
    The general idea of proving this lemma is to use Lemma~\ref{lem:prod_hermite}. In particular, in our case we have that
    \[
        \bm{\Sigma}[i,j] = [\bfv_1, \bfv_2, \bfw_1, \bfw_2]^\top [\bfv_1, \bfv_2, \bfw_1, \bfw_2]\in \R^{4\times 4}
    \]
    For the convenience of the analysis, we denote
    \[
        \gamma_1 = \bfv_1^\top\bfv_2;\;\gamma_2 = \bfw_1^\top\bfw_2;\;\zeta_{ij} = \bfv_i^\top\bfw_j
    \]
    By the assumption, we have that $|\zeta_{ij}|\leq \delta_r$.
    
    \textbf{Proof of (\ref{eq:lem_hermite_form_1}).} We start from the first equation. In particular we need to study
    \[
        \EXP[\bfx\sim\mathcal{N}\paren{\bm{0},\bfI}]{He_k\paren{\bfv_1^\top\bfx}He_{\ell}\paren{\bfv_2^\top\bfx}He_3\paren{\bfw_1^\top\bfx}He_3\paren{\bfw_2^\top\bfx}}
    \]
    By Lemma~\ref{lem:prod_hermite}, we need to enumerate all $\mathcal{S}$, which is essentially the symmetric matrix $\bfM$ with zero diagonal and non-negative entries whose row-sum equal to the vector $[k, \ell , 3, 3]$. This is equivalent to construct a weighted graph with four nodes and node degree $[k, \ell , 3, 3]$. Thus, it suffice to consider cases $k = \ell, |k-\ell|=2, |k-\ell| = 4$, and $|k-\ell| = 6$. Due to symmetry between $k$ and $\ell$, we will first study the case $k\geq \ell$ and switch the indices to obtain all cases. 
    
    \textit{Case $k = \ell$.} The node $\bfw_1$ and $\bfw_2$ has a total degree of 6, therefore, the pair of node $\bfv_1$ and $\bfv_2$ can have outgoing degree at most 6. Thus, the condition can be broken down into $\bfM[1,2] \in \{k, k-1, k-2, k-3\}$. When $\bfM[1,2] = k$, we have that $\bfM[3,4] = 3$, and all other edges 0. When $\bfM[1,2] = k-1$, we have that $\bfM[3,4] = 2$ and either $\paren{\bfM[1,3],\bfM[2,4]} = (1,1)$ or $\paren{\bfM[1,4],\bfM[2,3]} = (1,1)$. When $\bfM[1,2] = k-2$, we have that $\bfM[3,4] = 1$. Here we can have $\paren{\bfM[1,3],\bfM[2,4]} = (2,2), \paren{\bfM[1,4],\bfM[2,3]} = (2,2)$, or $\paren{\bfM[1,3],\bfM[2,4], \bfM[1,4],\bfM[2,3]} = (1,1, 1, 1)$. When $\bfM[1,2] = k-3$, then $\bfM[3,4] = 0$. Thus we have $\paren{\bfM[1,3],\bfM[2,4]} = (3,3), \paren{\bfM[1,4],\bfM[2,3]} = (3,3)$ or $\paren{\bfM[1,3],\bfM[2,4], \bfM[1,4],\bfM[2,3]} = (1,1, 2, 2)$, or $\paren{\bfM[1,3],\bfM[2,4], \bfM[1,4],\bfM[2,3]} = (2,2,1,1)$. Plugging the possibilities into Lemma~\ref{lem:prod_hermite} gives that, under the case $k=\ell$, we have
    \begin{align*}
        & \EXP[\bfx\sim\mathcal{N}\paren{\bm{0},\bfI}]{He_k\paren{\bfv_1^\top\bfx}He_{\ell}\paren{\bfv_2^\top\bfx}He_3\paren{\bfw_1^\top\bfx}He_3\paren{\bfw_2^\top\bfx}}\\
        & \qqquad = 6k!\gamma_1^k\gamma_2^3 + 18P(k,1)k!\gamma_1^{k-1}\gamma_2^2\paren{\zeta_{11}\zeta_{22} + \zeta_{12}\zeta_{21}}\\
        & \qqqqquad + 9P(k,2)k!\gamma_1^{k-2}\gamma_2\paren{\zeta_{11}^2\zeta_{22}^2 + 4\zeta_{11}\zeta_{12}\zeta_{21}\zeta_{22} + \zeta_{12}^2\zeta_{21}^2}\\
        & \qqqqquad + P(k,3)k!\gamma_1^{k-3}\paren{\zeta_{11}^3\zeta_{22}^3 + \zeta_{12}^3\zeta_{21}^3 + 9\zeta_{11}^2\zeta_{12}\zeta_{21}\zeta_{22}^2 + 9\zeta_{11}\zeta_{12}^2\zeta_{21}^2\zeta_{22}}\\
        & \qqquad = 6k!\gamma_1^k\gamma_2^3\pm \calO\paren{\delta_r^2\gamma_2^2}\cdot k!P(\ell,1)\pm \calO\paren{\delta_r^4}\cdot k!\paren{P(\ell,2) + P(\ell,3)}
    \end{align*}
    where $P(k, a) = \frac{k!}{(k-a)!}$ if $k\geq a$ and $P(k,a) = 0$ if $k < a$ represents the permutation number.

    \textit{Case $k = \ell+2$.} In this case we have that $\bfM[1,2]\in\{\ell, \ell-1, \ell-2\}$. If $\bfM[1,2] = \ell$, then $\bfM[3,4] = 2$. Here we have that $\paren{\bfM[1,3],\bfM[1,4]} = (1,1)$. If $\bfM[1,2]=\ell - 1$, then $\bfM[3,4] = 1$. Here we have $\paren{\bfM[1,3],\bfM[1,4],\bfM[2,4]} = (2, 1, 1)$ or $\paren{\bfM[1,3],\bfM[1,4],\bfM[2,3]} = (1, 2, 1)$. If $\bfM[1,2] = \ell-2$, then $\bfM[3,4] = 0$. Here we have $\paren{\bfM[1,3], \bfM[1,4], \bfM[2,4]} = (3,1,2)$ or $\paren{\bfM[1,3],\bfM[1,4],\bfM[2,3]} = (1, 3, 2)$ or $\paren{\bfM[1,3],\bfM[1,4],\bfM[2,3],\bfM[2,4]} = (2,2,1,1)$. Gathering all possibilities gives that, under the case $k = \ell+2$, we have
    \begin{align*}
        & \EXP[\bfx\sim\mathcal{N}\paren{\bm{0},\bfI}]{He_k\paren{\bfv_1^\top\bfx}He_{\ell}\paren{\bfv_2^\top\bfx}He_3\paren{\bfw_1^\top\bfx}He_3\paren{\bfw_2^\top\bfx}}\\
        & \qqquad = 18k!\gamma_1^{\ell}\gamma_2^2\zeta_{11}\zeta_{12} + 18P(\ell,1)k!\gamma_1^{\ell-1}\gamma_2\zeta_{11}\zeta_{12}\paren{\zeta_{11}\zeta_{22} + \zeta_{12}\zeta_{21}}\\
        & \qqqqquad + 3P(\ell,2)k!\gamma_1^{\ell-2}\zeta_{11}\zeta_{12}\paren{\zeta_{11}^2\zeta_{22}^2 + \zeta_{12}^2\zeta_{21}^2}\\
        & \qqqqquad + 9P(\ell,2)k!\gamma_1^{\ell-2}\zeta_{11}^2\zeta_{12}^2\zeta_{21}\zeta_{22}\\
        & \qqquad = \pm \calO\paren{\delta_r^2\gamma_2^2}\cdot k!\pm \calO\paren{\delta_r^4}\cdot k!\paren{P(\ell,1) + P(\ell,2)}
    \end{align*}
    \textit{Case $k = \ell + 4$.} In this case we have that $\bfM[1,2]\in\{\ell, \ell-1\}$. If $\bfM[1,2] = \ell$, then $\bfM[3,4] = 1$. Here we have that $\paren{\bfM[1,3],\bfM[1,4]} = (2,2)$. If $\bfM[1,2] = \ell-1$, then $\bfM[3,4] = 0$. Here we have that $\paren{\bfM[1,3],\bfM[1,4],\bfM[2,4]} = \paren{3,2,1}$ or $\paren{\bfM[1,3],\bfM[1,4],\bfM[2,3]} = \paren{2,3,1}$. Gathering all possibilities gives that, under the case $k = \ell+4$, we have
    \begin{align*}
        & \EXP[\bfx\sim\mathcal{N}\paren{\bm{0},\bfI}]{He_k\paren{\bfv_1^\top\bfx}He_{\ell}\paren{\bfv_2^\top\bfx}He_3\paren{\bfw_1^\top\bfx}He_3\paren{\bfw_2^\top\bfx}}\\
        & \qqquad = 9k!\gamma_1^{\ell}\gamma_2\zeta_{11}^2\zeta_{12}^2 + 3P(\ell,1)k!\gamma_1^{\ell-1}\zeta_{11}^2\zeta_{12}^2\paren{\zeta_{11}\zeta_{22} + \zeta_{21}\zeta_{12}}\\
        & \qqquad = \pm \calO\paren{\delta_r^4}\cdot k!\paren{1 + P(\ell,1)}
    \end{align*}
    \textit{Case $k = \ell+6$.} In this case we have that $\bfM[1,2] = \ell, \bfM[3,4] = 0$ and $\bfM[1,3] = \bfM[1,4] = 3$. Thus, if $k = \ell + 6$, we have that
    \begin{align*}
        & \EXP[\bfx\sim\mathcal{N}\paren{\bm{0},\bfI}]{He_k\paren{\bfv_1^\top\bfx}He_{\ell}\paren{\bfv_2^\top\bfx}He_3\paren{\bfw_1^\top\bfx}He_3\paren{\bfw_2^\top\bfx}} = k!\gamma_1^{\ell}\gamma_2\zeta_{11}^3\zeta_{12}^3 = \pm \calO\paren{\delta_r^4}\cdot k!
    \end{align*}
    Putting things together, we have that
    \begin{align*}
        & \sum_{k,\ell=0}^{\infty}\frac{h_kh'_{\ell}}{k!\ell!}\EXP[\bfx\sim\mathcal{N}\paren{\bm{0},\bfI}]{He_k\paren{\bfv_1^\top\bfx}He_{\ell}\paren{\bfv_2^\top\bfx}He_3\paren{\bfw_1^\top\bfx}He_3\paren{\bfw_2^\top\bfx}}\\
        & \qqquad = 6\sum_{k=0}^{\infty}\frac{h_kh_{k}'}{k!}\gamma_1^k\gamma_2^3 \pm \calO\paren{\delta_r^2}\sum_{k=0}^\infty\frac{h_kh_k'}{k!} \pm \calO\paren{\delta_r^2\gamma_2^2}\sum_{k=0}^\infty\frac{h_{k+1}h_{k+1}'}{k!} \pm \calO\paren{\delta_r^4}\sum_{k=0}^\infty\frac{c_{k+2}c_{k+2}'}{k!}\\
        & \qqqqquad \pm \calO\paren{\delta_r^4}\sum_{k=0}^\infty\frac{h_{k+3}h_{k+3}'}{k!} \pm \calO\paren{\delta_r^4}\sum_{k=0}^\infty\frac{h_kh_{k+2}' + h_k'h_{k+2}}{k!} \pm \calO\paren{\delta_r^4}\sum_{k=0}^\infty\frac{h_{k+1}h_{k+3}' + h_{k+1}'h_{k+3}}{k!}\\
        & \qqqqquad \pm \calO\paren{\delta_r^4}\sum_{k=0}^\infty\frac{h_{k+2}h_{k+4}' + h_{k+2}'h_{k+4}}{k!}  \pm \calO\paren{\delta_r^4}\sum_{k=0}^\infty\frac{h_{k}h_{k+4}' + c_{k}'c_{k+4}}{k!}\\
        & \qqqqquad \pm \calO\paren{\delta_r^4}\sum_{k=0}^\infty\frac{h_{k+1}h_{k+5}' + h_{k+1}'h_{k+5}}{k!} \pm \calO\paren{\delta_r^4}\sum_{k=0}^\infty\frac{h_{k}h_{k+6}' + h_{k}'h_{k+6}}{k!}\\
        & \qqquad = 6\sum_{k=0}^{\infty}\frac{h_kh_{k}'}{k!}\gamma_1^k\gamma_2^3 \pm \calO\paren{\delta_r^2\gamma_2^2 + \delta_r^4}
    \end{align*}
    \textbf{Proof of (\ref{eq:lem_hermite_form_2}).} Similar to before, the combination of $k,\ell$ can be $k = \ell$, $|k-\ell|=2$, or $|k-\ell|=4$ due to the total degree of $\bfw_1$ and $\bfw_2$ is 4. We study the case $k \geq \ell$. 

    \textit{Case $k = \ell$. } In this case, we have $\bfM[1,2] \in\{\ell, \ell-1, \ell-2\}$. If $\bfM[1,2] = \ell$, then $\bfM[3,4] = 2$ and all other edges are 0. If $\bfM[1,2] = \ell-1$, then $\bfM[3,4] = 1$, and either $\paren{\bfM[1,3],\bfM[2,4]} = \paren{1,1}$ or $\paren{\bfM[1,4],\bfM[2,3]} = \paren{1,1}$. If $\bfM[1,2] = \ell-2$, then $\bfM[3,4] = 0$. Here we can have $\paren{\bfM[1,3],\bfM[2,4]} = \paren{2,2}$ or $\paren{\bfM[1,4],\bfM[2,3]} = \paren{2,2}$ or $\paren{\bfM[1,3],\bfM[1,4],\bfM[2,3],\bfM[2,4]} = \paren{1,1,1,1}$. Gathering all possibilities gives that, under the case $k = \ell$, we have
    \begin{align*}
        & \EXP[\bfx\sim\mathcal{N}\paren{\bm{0},\bfI}]{He_k\paren{\bfv_1^\top\bfx}He_{\ell}\paren{\bfv_2^\top\bfx}He_2\paren{\bfw_1^\top\bfx}He_2\paren{\bfw_2^\top\bfx}}\\
        & \qqquad = 2k!\gamma_1^k\gamma_2^2 + 4 P(\ell,1) k!\gamma_1^{k-1}\gamma_2\paren{\zeta_{11}\zeta_{22} + \zeta_{12}\zeta_{21}}\\
        & \quad\quad\quad\quad\quad + P(\ell,2)k!\gamma_{1}^{k-2}\paren{\zeta_{11}^2\zeta_{22}^2 + \zeta_{12}^2\zeta_{21}^2 + 4\zeta_{11}\zeta_{12}\zeta_{21}\zeta_{22}}\\
        & \qqquad = 2k!\gamma_1^k\gamma_2^2 \pm \calO\paren{\delta_r^2\gamma_2}\cdot k!P(\ell,1) \pm \calO\paren{\delta_r^4}\cdot k!P(\ell,2)
    \end{align*}
    In the case $\zeta_{22} = 0$, we denote $\hat{\zeta} = \max\left\{\left|\zeta_{21}\right|, \left|\zeta_{12}\right|\right\}$ we have that
    \begin{align*}
        & \EXP[\bfx\sim\mathcal{N}\paren{\bm{0},\bfI}]{He_k\paren{\bfv_1^\top\bfx}He_{\ell}\paren{\bfv_2^\top\bfx}He_2\paren{\bfw_1^\top\bfx}He_2\paren{\bfw_2^\top\bfx}}\\
        & \qqquad = 2k!\gamma_1^k\gamma_2^2 + 4 P(\ell,1) k!\gamma_1^{k-1}\gamma_2\zeta_{12}\zeta_{21} + P(\ell,2)k!\gamma_{1}^{k-2}\zeta_{12}^2\zeta_{21}^2\\
        & \qqquad = 2k!\gamma_1^k\gamma_2^2 \pm \calO\paren{\hat{\zeta}^2\gamma_2}\cdot k!P(\ell,1) \pm \calO\paren{\hat{\zeta}^4}\cdot k!P(\ell,2)
    \end{align*}
    \textit{Case $k = \ell+2$.} In this case we have $\bfM[1,2]\in\{\ell, \ell-1\}$. If $\bfM[1,2]= \ell$, then $\bfM[3,4] = 1$ and $\paren{\bfM[1,3],\bfM[1,4]} = \paren{1,1}$. If $\bfM[1,2] = \ell-1$, then $\bfM[3,4] =0$ and either $\paren{\bfM[1,3],\bfM[1,4],\paren{\bfM}[2,3]} = \paren{1,2,1}$ or $\paren{\bfM[1,3],\bfM[1,4],\paren{\bfM}[2,4]} = \paren{2, 1,1}$. Gathering all possibilities gives that, under the case $k = \ell+2$, we have
    \begin{align*}
        & \EXP[\bfx\sim\mathcal{N}\paren{\bm{0},\bfI}]{He_k\paren{\bfv_1^\top\bfx}He_{\ell}\paren{\bfv_2^\top\bfx}He_2\paren{\bfw_1^\top\bfx}He_2\paren{\bfw_2^\top\bfx}}\\
        & \qqquad = 4k!\gamma_1^{\ell}\gamma_2\zeta_{11}\zeta_{12} + 2P(\ell,1) k!\gamma_1^{\ell-1}\zeta_{11}\zeta_{12}\paren{\zeta_{11}\zeta_{22} + \zeta_{12}\zeta_{21}}\\
        & \qqquad = \pm \calO\paren{\delta_r^2\gamma_2}\cdot k! \pm \calO\paren{\delta_r^4}\cdot k!P(\ell,1)
    \end{align*}
    In the case where $\zeta_{22} = 0$, we have that
    \begin{align*}
        & \EXP[\bfx\sim\mathcal{N}\paren{\bm{0},\bfI}]{He_k\paren{\bfv_1^\top\bfx}He_{\ell}\paren{\bfv_2^\top\bfx}He_2\paren{\bfw_1^\top\bfx}He_2\paren{\bfw_2^\top\bfx}}\\
        & \qqquad = 4k!\gamma_1^{\ell}\gamma_2\zeta_{11}\zeta_{12} + 2P(\ell,1) k!\gamma_1^{\ell-1}\zeta_{11}\zeta_{12}^2\zeta_{21}\\
        & \qqquad = \pm \calO\paren{\delta_r\hat{\zeta}\gamma_2}\cdot k! \pm \calO\paren{\hat{\zeta}^3}\cdot k!P(\ell,1)
    \end{align*}
    \textit{Case $k = \ell+4$.} In this case we can only have $\bfM[1,2] = \ell$, $\bfM[3,4] =0$, and $\bfM[1,3] = \bfM[1,4] = 2$. Thus if $k = \ell + 4$, we have
    \[
        \EXP[\bfx\sim\mathcal{N}\paren{\bm{0},\bfI}]{He_k\paren{\bfv_1^\top\bfx}He_{\ell}\paren{\bfv_2^\top\bfx}He_2\paren{\bfw_1^\top\bfx}He_2\paren{\bfw_2^\top\bfx}} = k!\gamma_1^{\ell}\zeta_{11}^2\zeta_{12}^2 = \pm \calO\paren{\delta_r^4}\cdot k!
    \]
    In the case where $\zeta_{22} = 0$, we have that
    \[
        \EXP[\bfx\sim\mathcal{N}\paren{\bm{0},\bfI}]{He_k\paren{\bfv_1^\top\bfx}He_{\ell}\paren{\bfv_2^\top\bfx}He_2\paren{\bfw_1^\top\bfx}He_2\paren{\bfw_2^\top\bfx}} = k!\gamma_1^{\ell}\zeta_{11}^2\zeta_{12}^2 = \pm \calO\paren{\delta_r^2\hat{\zeta}^2}\cdot k!
    \]
    Putting things together, we have that
    \begin{align*}
        & \sum_{k,\ell=0}^{\infty}\frac{h_kh'_{\ell}}{k!\ell!}\EXP[\bfx\sim\mathcal{N}\paren{\bm{0},\bfI}]{He_k\paren{\bfv_1^\top\bfx}He_{\ell}\paren{\bfv_2^\top\bfx}He_2\paren{\bfw_1^\top\bfx}He_2\paren{\bfw_2^\top\bfx}}\\
        & \qqquad = 2\sum_{k=0}^{\infty}\frac{h_kh_{k'}}{k!}\gamma_1^k\gamma_2^2 \pm \calO\paren{\delta_r^2\gamma_2}\sum_{k=0}^{\infty}\frac{h_{k+1}h_{k+1}'}{k!} \pm \calO\paren{\delta_r^4}\sum_{k=0}^{\infty}\frac{h_{k+2}h_{k+2}^2}{k!}\\
        & \qqqqquad \pm \calO\paren{\delta_r^2\gamma_2}\sum_{k=0}^{\infty}\frac{h_kh_{k+2} + h_k'h_{k+2}}{k!} \pm \calO\paren{\delta_r^4}\sum_{k=0}^{\infty}\frac{h_{k+1}h_{k+3} + h_{k+1}'h_{k+3}'}{k!}\\
        & \qqqqquad \pm \calO\paren{\delta_r^4}\sum_{k=0}^{\infty}\frac{h_kh_{k+4} + h_k'h_{k+4}}{k!}\\
        & \qqquad = 2\sum_{k=0}^{\infty}\frac{h_kh_{k'}}{k!}\gamma_1^k\gamma_2^2 \pm \calO\paren{\delta_r^2\gamma_2 + \delta_r^4}
    \end{align*}
    In the case where $\zeta_{22} = 0$, we have that
    \begin{align*}
        & \sum_{k,\ell=0}^{\infty}\frac{h_kh'_{\ell}}{k!\ell!}\EXP[\bfx\sim\mathcal{N}\paren{\bm{0},\bfI}]{He_k\paren{\bfv_1^\top\bfx}He_{\ell}\paren{\bfv_2^\top\bfx}He_2\paren{\bfw_1^\top\bfx}He_2\paren{\bfw_2^\top\bfx}}\\
        & \qqquad = 2\sum_{k=0}^{\infty}\frac{h_kh_{k'}}{k!}\gamma_1^k\gamma_2^2 \pm \calO\paren{\hat{\zeta}^2\gamma_2}\sum_{k=0}^{\infty}\frac{h_{k+1}h_{k+1}'}{k!} \pm \calO\paren{\hat{\zeta}^4}\sum_{k=0}^{\infty}\frac{h_{k+2}h_{k+2}^2}{k!}\\
        & \qqqqquad \pm \calO\paren{\delta_r\hat{\zeta}\gamma_2}\sum_{k=0}^{\infty}\frac{h_kh_{k+2}'}{k!} \pm \calO\paren{\hat{\zeta}^3}\sum_{k=0}^{\infty}\frac{h_{k+1}h_{k+3}'}{k!}\\
        & \qqqqquad \pm \calO\paren{\delta_r^2\hat{\zeta}^2}\sum_{k=0}^{\infty}\frac{h_kh_{k+4}'}{k!}\\
        & \qqquad = 2\sum_{k=0}^{\infty}\frac{h_kh_{k'}}{k!}\gamma_1^k\gamma_2^2 \pm \calO\paren{\delta_r\hat{\zeta}\gamma_2 + \delta_r\hat{\zeta}_r^2}
    \end{align*}
    \textbf{Proof of (\ref{eq:lem_hermite_form_3}). } We notice that in this case $k,\ell$ must satisfy $|k-\ell|\in\{1, 3, 5\}$. Similar to before, we assume that $k\geq \ell$.

    \textit{Case $k = \ell+1$.} In this case $\bfM[1,2] \in\{\ell, \ell-1, \ell-2\}$. If $\bfM[1,2] = \ell$, then $\bfM[3,4] = 2$ and $\bfM[1,4] = 1$. If $\bfM[1,2] = \ell-1$, then $\bfM[3,4] = 1$, and either $\paren{\bfM[1,3],\bfM[1,4],\bfM[2,4]} = \paren{1, 1, 1}$ or $\paren{\bfM[1,4],\bfM[2,3]} = \paren{2,1}$. If $\bfM[1,2] = \ell-2$, then $\bfM[3,4] = 0$. Here we have $\paren{\bfM[1,4],\bfM[2,3]} = \paren{3,2}$ or $\paren{\bfM[1,3],\bfM[1,4],\bfM[2,3],\bfM[2,4]} = \paren{1,2,1,1}$ or $\paren{\bfM[1,3],\bfM[1,4],\bfM[2,4]} = \paren{2,1,2}$. Gathering all possibilities gives that, under the case $k = \ell+1$, we have
    \begin{align*}
        & \EXP[\bfx\sim\mathcal{N}\paren{\bm{0},\bfI}]{He_k\paren{\bfv_1^\top\bfx}He_{\ell}\paren{\bfv_2^\top\bfx}He_2\paren{\bfw_1^\top\bfx}He_3\paren{\bfw_2^\top\bfx}}\\
        & \qqquad = 6k!\gamma_1^{\ell}\gamma_2^2\zeta_{12} + 12P(\ell,1)k!\gamma_1^{\ell-1}\gamma_2\zeta_{12}\paren{2\zeta_{11}\zeta_{22} + \zeta_{12}\zeta_{21}}\\
        & \qqqqquad + P(\ell,2)k!\gamma_1^{\ell-2}\zeta_{12}\paren{\zeta_{12}^2\zeta_{21}^2 + 3\zeta_{11}^2\zeta_{22}^2 + 6\zeta_{11}\zeta_{12}\zeta_{21}\zeta_{22}}\\
        & \qqquad = 6k!\gamma_1^{\ell}\gamma_2^2\zeta_{12}\pm \calO\paren{\delta_r^3}\paren{P(\ell,1) + P(\ell,2)}
    \end{align*}
    \textit{Case $k = \ell+3$.} In this case we have $\bfM[1,2]\in\{\ell,\ell-1\}$. If $\bfM[1,2] = \ell$, then $\bfM[3,4] = 1$, and $\paren{\bfM[1,3], \bfM[1,4]} = (1,2)$. If $\bfM[1,2] = \ell-1$, then $\bfM[3,4] = 0$, and we have $\paren{\bfM[1,3],\bfM[1,4],\bfM[2,3]} = \paren{1,3,1}$ or $\paren{\bfM[1,3],\bfM[1,4],\bfM[2,4]} = \paren{2,2,1}$. Gathering all possibilities gives that, under the case $k = \ell+3$, we have
    \begin{align*}
        & \EXP[\bfx\sim\mathcal{N}\paren{\bm{0},\bfI}]{He_k\paren{\bfv_1^\top\bfx}He_{\ell}\paren{\bfv_2^\top\bfx}He_2\paren{\bfw_1^\top\bfx}He_3\paren{\bfw_2^\top\bfx}}\\
        & \qqquad = 6k!\gamma_1^{\ell}\gamma_2\zeta_{11}\zeta_{12}^2 + 6P(\ell,1)k!\gamma_1^{\ell-1}\zeta_{11}\zeta_{12}^2\paren{3\zeta_{21} + 2\zeta_{22}}\\
        & \qqquad = \pm \calO\paren{\delta_r^3}\paren{1 + P(\ell,1)}
    \end{align*}
    \textit{Case $k = \ell+5$.} In this case we must have that $\bfM[1,2] = \ell,\bfM[3,4] =0$, and $\bfM[1,3] = 2,\bfM[1,4] = 3$. Thus
    \[
        \EXP[\bfx\sim\mathcal{N}\paren{\bm{0},\bfI}]{He_k\paren{\bfv_1^\top\bfx}He_{\ell}\paren{\bfv_2^\top\bfx}He_2\paren{\bfw_1^\top\bfx}He_3\paren{\bfw_2^\top\bfx}} = k!\gamma_1^{\ell}\zeta_{11}^2\zeta_{12}^3 = \pm \calO\paren{\delta_r^5}
    \]
    Putting things together gives
    \begin{align*}
        & \sum_{k,\ell=0}^{\infty}\frac{h_kh'_{\ell}}{k!\ell!}\EXP[\bfx\sim\mathcal{N}\paren{\bm{0},\bfI}]{He_k\paren{\bfv_1^\top\bfx}He_{\ell}\paren{\bfv_2^\top\bfx}He_2\paren{\bfw_1^\top\bfx}He_3\paren{\bfw_2^\top\bfx}}\\
        & \qqquad = 6\sum_{k=0}^{\infty}\frac{h_{k+1}h_{k}'}{k!}\gamma_1^k\gamma_2^2\zeta_{12} + 6\sum_{k=0}^{\infty}\frac{h_{k}h_{k+1}'}{k!}\gamma_1^k\gamma_2^2\zeta_{22} \pm \calO\paren{\delta_r^3}\sum_{k=0}^{\infty}\frac{h_{k+1}h_{k+2}' + h_{k+1}'h_{k+2}}{k!}\\
        & \qqqqquad \pm \calO\paren{\delta_r^3}\sum_{k=0}^{\infty}\frac{h_{k+2}h_{k+3}' + h_{k+2}'h_{k+3}}{k!} \pm \calO\paren{\delta_r^3}\sum_{k=0}^{\infty}\frac{h_kh_{k+3}' + h_k'h_{k+3}}{k!}\\
        & \qqqqquad \pm \calO\paren{\delta_r^3}\sum_{k=0}^{\infty}\frac{h_{k+1}h_{k+4}' + h_{k+1}'h_{k+4}}{k!} \pm \calO\paren{\delta_r^3}\sum_{k=0}^{\infty}\frac{h_{k}h_{k+5}' + h_{k}'h_{k+5}}{k!}\\
        & \qqquad = 6\sum_{k=0}^{\infty}\frac{h_{k+1}h_{k}'}{k!}\gamma_1^k\gamma_2^2\zeta_{12} + 6\sum_{k=0}^{\infty}\frac{h_{k}h_{k+1}'}{k!}\gamma_1^k\gamma_2^2\zeta_{22} \pm \calO\paren{\delta_r^3}
    \end{align*}
\end{proof}

\begin{lemma}
    \label{lem:ub_prod}
    Let $h_i\paren{\bfx} = \pi\paren{\bbfv_i^\top\bfx}\sigma\paren{\bbfw_i^\top\bfx}$ and $h_j\paren{\bfx} = \pi\paren{\bbfv_j^\top\bfx}\sigma\paren{\bbfw_j^\top\bfx}$. Suppose that
    \[
        \max\left\{\left|\bbfv_i^\top\bbfw_i\right|,\left|\bbfv_j^\top\bbfw_j\right|, \left|\bbfv_i^\top\bbfw_j\right|,\left|\bbfv_j^\top\bbfw_i\right|\right\} \leq \delta_r
    \]
    then we have that
    \[
        \EXP[\bfx]{h_i\paren{\bfx}^2} = 6\sum_{k=0}^{\infty}\frac{c_k^2}{k!} \pm \calO\paren{\delta_r^2}
    \]
    If it holds that $\left|\bbfw_i^\top\bbfw_j\right|\leq \delta_r$, then we have that
    \[
        \EXP[\bfx]{h_i\paren{\bfx}h_j\paren{\bfx}} = \pm \calO\paren{\delta_r^3};\quad
    \]
    If it holds that $\left|\bbfw_i^\top\bbfw_j\right|\leq \delta_r$, then we have that
    \[
        \EXP[\bfx]{h_i\paren{\bfx}h_j\paren{\bfx}} = 6\EXP[\bfx]{h_i(\bfx)^2} \pm \calO\paren{\delta_r}
    \]
\end{lemma}
\begin{proof}
    Adopting the Hermite expansion, by Lemma~\ref{lem:cs_hermite} we have that
    \begin{align*}
        \EXP[\bfx]{h_i\paren{\bfx}h_j\paren{\bfx}} & = \sum_{k,\ell=0}^{\infty}\frac{c_kc_\ell}{k!\ell!}\EXP[\bfx]{He_k\paren{\bbfv_i^\top\bfx}He_{\ell}\paren{\bbfv_j^\top\bfx}He_3\paren{\bbfw_i^\top\bfx}\paren{\bbfw_j^\top\bfx}}\\
        & = 6\paren{\bbfw_i^\top\bbfw_j}^3\sum_{k=0}^{\infty}\frac{c_k^2}{k!}\paren{\bbfv_i^\top\bbfv_j}^k \pm \calO\paren{\delta_r^2\paren{\bbfw_i^\top\bbfw_j}^2 + \delta_r^4}
    \end{align*}
    Thus, in the case where $i\neq j$ and $\left|\bbfw_i^\top\bbfw_j\right| \leq \delta_r$, we have that
    \[
        \EXP[\bfx]{h_i\paren{\bfx}h_j\paren{\bfx}} = \pm \calO\paren{\delta_r^3}
    \]
    On the other hand, if $i=j$, then we have that
    \[
        \EXP[\bfx]{h_i\paren{\bfx}^2} = 6\sum_{k=0}^{\infty}\frac{c_k^2}{k!} \pm \calO\paren{\delta_r^2}
    \]
    If $\bbfw_i^\top\bbfw_j \geq 1 - \delta_r$ and $\bbfv_i^\top\bbfv_j\geq 1 - \delta$, then we have that
    \begin{align*}
        \EXP[\bfx]{h_i\paren{\bfx}h_j\paren{\bfx}} & = 6\paren{1 - 3\delta_r}\sum_{k=0}^{\infty}\frac{c_k^2}{k!}\paren{\bbfv_i^\top\bbfv_j}^k \pm \calO\paren{\delta_r^2}\\
        & = 6\paren{1 - 3\delta_r}\EXP[\bfx]{\pi\paren{\bbfv_i^\top\bfx}\pi\paren{\bbfv_j^\top\bfx}} \pm \calO\paren{\delta_r^2}
    \end{align*}
    Applying the Taylor's expansion gives that
    \begin{align*}
        \EXP[\bfx]{\pi\paren{\bbfv_i^\top\bfx}\pi\paren{\bbfv_j^\top\bfx}} & = \EXP[z_1,z_2\sim\mathcal{N}(0,1),\text{Cov}\paren{z_1, z_2} = 1 - \delta_r]{\pi\paren{z_1}\pi\paren{z_2}}\\
        & = \EXP[z\sim\mathcal{N}\paren{0,1}]{\pi\paren{z_1}^2} - \frac{1}{2}\text{Var}\paren{\pi'\paren{z_1}}\delta_r \pm \calO\paren{\delta_r}\\
        & \geq \EXP[z\sim\mathcal{N}\paren{0,1}]{\pi\paren{z_1}^2} \pm \calO\paren{\delta_r}
    \end{align*}
    Thus, we have that
    \[
        \EXP[\bfx]{h_i\paren{\bfx}h_j\paren{\bfx}} = 6\EXP[\bfx]{h_i(\bfx)^2} \pm \calO\paren{\delta_r}
    \]
\end{proof}

\begin{lemma}
    \label{lem:hermite_residue_bound}
    Let $\bm{\theta}$ satisfy that $\norm{\bbfv_i^\top - \bbfv_i^\star}_2^2, \norm{\bbfw_i^\top - \bbfw_i^\star}_2^2 \leq \varepsilon$, $\left|\bbfv_i^\top\bbfv_j\right|, \left|\bbfw_i^\top\bbfw_j\right|, \left|\bbfv_i^\top\bbfv_j^\star\right|, \left|\bbfw_i^\top\bbfw_j^\star\right| \leq \varepsilon$ for all $i\neq j$, and $\left|\bbfv_i^\top\bbfw_j\right|, \left|\bbfw_i^\top\bbfv_j\right|, \left|\bbfv_i^\top\bbfw_j^\star\right|, \left|\bbfw_i^\top\bbfv_j^\star\right| \leq \varepsilon$ for all $i,j$. Then the following holds:
    \begin{itemize}
        \item $\left|\EXP[\bfx]{\paren{f\paren{\bm{\theta},\bfx} - f^\star\paren{\bfx}}\pi^{(a)}\paren{\bbfv_i^\top\bfx}\sigma^{(b)}\paren{\bbfw_i^\top\bfx}}\right|\leq \calO\paren{m^\star\varepsilon}$
        \item $\norm{\EXP[\bfx]{\nabla^2_{\bfx}\paren{f\paren{\bm{\theta},\bfx} - f^\star\paren{\bfx}}\pi^{(a)}\paren{\bbfv_i^\top\bfx}\sigma^{(b)}\paren{\bbfw_i^\top\bfx}}}_2\leq \calO\paren{m^\star\varepsilon}$
        \item $\norm{\EXP[\bfx]{\paren{f\paren{\bm{\theta},\bfx} - f^\star\paren{\bfx}}\nabla^2_{\bfx}\paren{\pi^{(a)}\paren{\bbfv_i^\top\bfx}\sigma^{(b)}\paren{\bbfw_i^\top\bfx}}}}_2\leq \calO\paren{m^\star\varepsilon}$
        \item $\norm{\EXP[\bfx]{\nabla_{\bfx}\paren{f\paren{\bm{\theta},\bfx} - f^\star\paren{\bfx}}\nabla_{\bfx}\paren{\pi^{(a)}\paren{\bbfv_i^\top\bfx}\sigma^{(b)}\paren{\bbfw_i^\top\bfx}}}}_2\leq \calO\paren{m^\star\varepsilon}$
        \item $\norm{\EXP[\bfx]{\paren{f\paren{\bm{\theta},\bfx} - f^\star\paren{\bfx}}\nabla_{\bfx}\paren{\pi^{(a)}\paren{\bbfv_i^\top\bfx}\sigma^{(b)}\paren{\bbfw_i^\top\bfx}}}}_2\leq  \calO\paren{m^\star\varepsilon}$
        \item $\norm{\EXP[\bfx]{\nabla_{\bfx}\paren{f\paren{\bm{\theta},\bfx} - f^\star\paren{\bfx}}\pi^{(a)}\paren{\bbfv_i^\top\bfx}\sigma^{(b)}\paren{\bbfw_i^\top\bfx}}}_2\leq  \calO\paren{m^\star\varepsilon}$
    \end{itemize}
\end{lemma}
\begin{proof}
    We first write out the gradient with respect to $\bfx$
    \begin{align*}
        \nabla_{\bfx}f\paren{\bm{\theta},\bfx} & = \sum_{j=1}^{m^\star}\pi'\paren{\bbfv_j^\top\bfx}\sigma\paren{\bbfw_j^\top\bfx}\bbfv_j + \sum_{j=1}^{m^\star}\pi\paren{\bbfv_j^\top\bfx}\sigma'\paren{\bbfw_j^\top\bfx}\bbfw_j\\
        \nabla_{\bfx}f^\star\paren{\bfx} & = \sum_{j=1}^{m^\star}\pi'\paren{\bbfv_j^{\star\top}\bfx}\sigma\paren{\bbfw_j^{\star\top}\bfx}\bbfv_j^\star + \sum_{j=1}^{m^\star}\pi\paren{\bbfv_j^{\star\top}\bfx}\sigma'\paren{\bbfw_j^{\star\top}\bfx}\bbfw_j^\star\\
        \nabla^2_{\bfx}f\paren{\bm{\theta},\bfx} & = \sum_{j=1}^{m^\star}\pi'\paren{\bbfv_j^\top\bfx}\sigma'\paren{\bbfw_j^\top\bfx}\paren{\bbfv_j\bbfw_j^\top + \bbfw_j\bbfv_j^\top}\\
        & \qqquad + \sum_{j=1}^{m^\star}\pi''\paren{\bbfv_j^\top\bfx}\sigma\paren{\bbfw_j^\top\bfx}\bbfv_j\bbfv_j^\top + \sum_{j=1}^{m^\star}\pi\paren{\bbfv_j^\top\bfx}\sigma''\paren{\bbfw_j^\top\bfx}\bbfw_j\bbfw_j^\top\\
        \nabla^2_{\bfx}f^\star\paren{\bfx} & = \sum_{j=1}^{m^\star}\pi'\paren{\bbfv_j^{\star\top}\bfx}\sigma'\paren{\bbfw_j^{\star\top}\bfx}\paren{\bbfv_j^\star\bbfw_j^{\star\top} + \bbfw_j^\star\bbfv_j^{\star\top}}\\
        & \qqquad + \sum_{j=1}^{m^\star}\pi''\paren{\bbfv_j^{\star\top}\bfx}\sigma\paren{\bbfw_j^{\star\top}\bfx}\bbfv_j^\star\bbfv_j^{\star\top} + \sum_{j=1}^{m^\star}\pi\paren{\bbfv_j^{\star\top}\bfx}\sigma''\paren{\bbfw_j^{\star\top}\bfx}\bbfw_j^\star\bbfw_j^{\star\top}
    \end{align*}
    Moreover, we also have that
    \begin{align*}
        \nabla_{\bfx}\paren{\pi^{(a)}\paren{\bbfv_i^\top\bfx}\sigma^{(b)}\paren{\bbfw_i^\top\bfx}} & = \pi^{(a+1)}\paren{\bbfv_i^\top\bfx}\sigma^{(b)}\paren{\bbfw_i^\top\bfx}\bbfv_i + \pi^{(a)}\paren{\bbfv_i^\top\bfx}\sigma^{(b+1)}\paren{\bbfw_i^\top\bfx}\bbfw_i\\
        \nabla^2_{\bfx}\paren{\pi^{(a)}\paren{\bbfv_i^\top\bfx}\sigma^{(b)}\paren{\bbfw_i^\top\bfx}} & = \pi^{(a+1)}\paren{\bbfv_i^\top\bfx}\sigma^{(b+1)}\paren{\bbfw_i^\top\bfx}\paren{\bbfv_i\bbfw_i^\top + \bbfw_i\bbfv_i^\top}\\
        & \qqquad + \pi^{(a+2)}\paren{\bbfv_i^\top\bfx}\sigma^{(b)}\paren{\bbfw_i^\top\bfx}\bbfv_i\bbfv_i^\top\\
        & \qqquad + \pi^{(a)}\paren{\bbfv_i^\top\bfx}\sigma^{(b+2)}\paren{\bbfw_i^\top\bfx}\bbfw_i\bbfw_i^\top
    \end{align*}
    Therefore, for the last two bounds, we have
    \begin{align*}
        \norm{\EXP[\bfx]{\paren{f\paren{\bm{\theta},\bfx} - f^\star\paren{\bfx}}\nabla_{\bfx}\paren{\pi^{(a)}\paren{\bbfv_i^\top\bfx}\sigma^{(b)}\paren{\bbfw_i^\top\bfx}}}}_2 & = \mathcal{E}_{1,1} + \mathcal{E}_{2,1} \\
        \norm{\EXP[\bfx]{\nabla_{\bfx}\paren{f\paren{\bm{\theta},\bfx} - f^\star\paren{\bfx}}\pi^{(a)}\paren{\bbfv_i^\top\bfx}\sigma^{(b)}\paren{\bbfw_i^\top\bfx}}}_2 & = \mathcal{E}_{1,2} + \mathcal{E}_{2,2}    \end{align*}
    where $\mathcal{E}_{1,1}$ and $\mathcal{E}_{1,2}$ are a summation of two terms in the form
    \[
        \left|\EXP[\bfx]{\paren{\pi\paren{\bbfv_i^\top\bfx}\sigma\paren{\bbfw_i^\top\bfx} - \pi\paren{\bbfv_i^{\star\top}\bfx}\sigma\paren{\bbfw_i^{\star\top}\bfx}}\pi^{(a')}\paren{\bbfv_i^\top\bfx}\sigma^{(b')}\paren{\bbfw_i^\top\bfx}}\right|
    \]
    Thus, by Taylor expansion, we obtain that $\mathcal{E}_{1,1}, \mathcal{E}_{1,2} \leq \calO\paren{\varepsilon}$.
    Moreover, $\mathcal{E}_{1,2}$ and $\mathcal{E}_{2,2}$ are a summation of $4m^\star$ terms of the form
    \begin{gather*}
        \left|\EXP[\bfx]{\pi\paren{\bbfv_j^\top\bfx}\sigma\paren{\bbfw_j^\top\bfx}\pi^{(a)}\paren{\bbfv_i^\top\bfx}\sigma^{(b')}\paren{\bbfw_i^\top\bfx}}\right|\\
        \left|\EXP[\bfx]{\pi\paren{\bbfv_j^{\star\top}\bfx}\sigma\paren{\bbfw_j^{\star\top}\bfx}\pi^{(a)}\paren{\bbfv_i^\top\bfx}\sigma^{(b')}\paren{\bbfw_i^\top\bfx}}\right|
    \end{gather*}
    By Lemma~\ref{lem:prod_hermite}, we have that $\mathcal{E}_{2,1}, \mathcal{E}_{2,2} \leq \calO\paren{m^\star\varepsilon}$.
    This gives the last two property. For the rest of the property, we can apply similar strategy to decompose the objective in terms of 
    \[
        \left|\EXP[\bfx]{\paren{\pi^{(a_0)}\paren{\bbfv_i^\top\bfx}\sigma^{(b_0)}\paren{\bbfw_i^\top\bfx} - \pi\paren{\bbfv_i^{\star\top}\bfx}\sigma\paren{\bbfw_i^{\star\top}\bfx}}\pi^{(a')}\paren{\bbfv_i^\top\bfx}\sigma^{(b')}\paren{\bbfw_i^\top\bfx}}\right|
    \]
    which can be upper bounded by Taylor expansion, and
    \begin{gather*}
        \left|\EXP[\bfx]{\pi\paren{\bbfv_j^\top\bfx}\sigma\paren{\bbfw_j^\top\bfx}\pi^{(a)}\paren{\bbfv_i^\top\bfx}\sigma^{(b')}\paren{\bbfw_i^\top\bfx}}\right|\\
        \left|\EXP[\bfx]{\pi\paren{\bbfv_j^{\star\top}\bfx}\sigma\paren{\bbfw_j^{\star\top}\bfx}\pi^{(a)}\paren{\bbfv_i^\top\bfx}\sigma^{(b')}\paren{\bbfw_i^\top\bfx}}\right|
    \end{gather*}
    which can be upper bounded by lemma~\ref{lem:misaligned_prod_bound}. Since there are in total $\calO\paren{m^\star}$ terms for each quantity, we can conclude the desired result.
\end{proof}

\begin{lemma}
    \label{lem:hermite_hess_bound}
    Let $\bfv_1,\bfv_2,\bfw_1,\bfw_2$ be unit vectors satisfying that any two of the four have an inner product with magnitude less than $\varepsilon$. Then for $a_1, a_2, b_1, b_2 > 0$ with $b_1 + b_2 \leq 3$, the following holds
    \begin{itemize}
        \item $\left|\EXP[\bfx]{\pi^{(a_1)}\paren{\bfv_1^\top\bfx}\sigma^{(b_1)}\paren{\bfw_1^\top\bfx}\pi^{(a_2)}\paren{\bfv_2^\top\bfx}\sigma^{(b_2)}\paren{\bfw_2^\top\bfx}}\right| \leq \calO\paren{\varepsilon}$
        \item $\norm{\EXP[\bfx]{\nabla_{\bfx}^2\paren{\pi^{(a_1)}\paren{\bfv_1^\top\bfx}\sigma^{(b_1)}\paren{\bfw_1^\top\bfx}}\pi^{(a_2)}\paren{\bfv_2^\top\bfx}\sigma^{(b_2)}\paren{\bfw_2^\top\bfx}}}_2 \leq \calO\paren{\varepsilon}$
        \item $\norm{\EXP[\bfx]{\pi^{(a_1)}\paren{\bfv_1^\top\bfx}\sigma^{(b_1)}\paren{\bfw_1^\top\bfx}\nabla^2_{\bfx}\paren{\pi^{(a_2)}\paren{\bfv_2^\top\bfx}\sigma^{(b_2)}\paren{\bfw_2^\top\bfx}}}}_2 \leq \calO\paren{\varepsilon}$
        \item $\norm{\EXP[\bfx]{\nabla_{\bfx}\paren{\pi^{(a_1)}\paren{\bfv_1^\top\bfx}\sigma^{(b_1)}\paren{\bfw_1^\top\bfx}}\nabla_{\bfx}\paren{\pi^{(a_2)}\paren{\bfv_2^\top\bfx}\sigma^{(b_2)}\paren{\bfw_2^\top\bfx}}}}_2 \leq \calO\paren{\varepsilon}$
    \end{itemize}
\end{lemma}
\begin{proof}
    The first quantity is directly bounded by applying Lemma~\ref{lem:misaligned_prod_bound}. For the rest, we write out the form of the gradients with respect to $\bfx$ as
    \begin{align*}
        \nabla_{\bfx}\paren{\pi^{(a)}\paren{\bfv^\top\bfx}\sigma^{(b)}\paren{\bfw^\top}} &  = \pi^{(a+1)}\paren{\bfv^\top\bfx}\sigma^{(b)}\paren{\bfw^\top}\bfv + \pi^{(a)}\paren{\bfv^\top\bfx}\sigma^{(b+1)}\paren{\bfw^\top}\bfw\\
        \nabla_{\bfx}^2\paren{\pi^{(a)}\paren{\bfv^\top\bfx}\sigma^{(b)}\paren{\bfw^\top}} & = \pi^{(a+2)}\paren{\bfv^\top\bfx}\sigma^{(b)}\paren{\bfw^\top}\bfv\bfv^\top + \pi^{(a)}\paren{\bfv^\top\bfx}\sigma^{(b+2)}\paren{\bfw^\top}\bfw\bfw^\top\\
        & \qqquad + \pi^{(a+1)}\paren{\bfv^\top\bfx}\sigma^{(b+1)}\paren{\bfw^\top}\paren{\bfv\bfw^\top + \bfw\bfv^\top}
    \end{align*}
    Therefore, for each of the rest property, it can be written in terms of a summation of terms of the form
    \[
        \EXP[\bfx]{\pi^{(a_1')}\paren{\bfv_1^\top\bfx}\sigma^{(b_1')}\paren{\bfw_1^\top\bfx}\pi^{(a_2')}\paren{\bfv_2^\top\bfx}\sigma^{(b_2')}\paren{\bfw_2^\top\bfx}}
    \]
    Since $b_1 + b_2\leq 3$, taking twice derivative gives $b_1' + b_2' \leq 5$. Therefore, applying Lemma~\ref{lem:misaligned_prod_bound} gives that all the rest terms are upper bounded by $\calO\paren{\varepsilon}$.
\end{proof}

\begin{lemma}
    \label{lem:misaligned_prod_bound}
    Let $\bfv_1,\bfv_2,\bfw_1,\bfw_2$ be unit vectors such that any two of the four have an inner product with magnitude upper bounded by $\varepsilon$. Then we have that for any $a_1, a_2, b_1, b_2$ such that $b_1 + b_2 \leq 5$
    \[
        \left|\EXP[\bfx]{\pi^{(a_1)}\paren{\bfv_1^\top\bfx}\pi^{(a_2)}\paren{\bfv_2^\top\bfx}\sigma^{(b_1)}\paren{\bfw_1^\top\bfx}\sigma^{(b_2)}\paren{\bfw_2^\top\bfx}}\right| \leq \calO\paren{\varepsilon}
    \]
\end{lemma}
\begin{proof}
    Taking the Hermite expansion
    \begin{align*}
        & \EXP[\bfx]{\pi^{(a_1)}\paren{\bfv_1^\top\bfx}\pi^{(a_2)}\paren{\bfv_2^\top\bfx}\sigma^{(b_1)}\paren{\bfw_1^\top\bfx}\sigma^{(b_2)}\paren{\bfw_2^\top\bfx}}\\
        & \qqquad = \sum_{k,\ell=0}^{\infty}\frac{c_{k+a_1}c_{\ell+a_2}}{k!\ell!}\EXP[\bfx]{He_k\paren{\bfv_1^\top\bfx}He_{\ell}\paren{\bfv_2^\top\bfx}He_{3-b_1}\paren{\bfw_1^\top\bfx}He_{3-b_2}\paren{\bfw_2^\top\bfx}}
    \end{align*}
    We could observe that at least one of $3-b_1$ and $3-b_2$ is nonzero. Therefore, by Lemma~\ref{lem:prod_hermite}, we have that $\EXP[\bfx]{He_k\paren{\bfv_1^\top\bfx}He_{\ell}\paren{\bfv_2^\top\bfx}He_{3-b_1}\paren{\bfw_1^\top\bfx}He_{3-b_2}\paren{\bfw_2^\top\bfx}}$ is a polynomial with lowest degree at most $1$. Therefore, we have that
    \[
        \left|\EXP[\bfx]{He_k\paren{\bfv_1^\top\bfx}He_{\ell}\paren{\bfv_2^\top\bfx}He_{3-b_1}\paren{\bfw_1^\top\bfx}He_{3-b_2}\paren{\bfw_2^\top\bfx}}\right| \leq \calO\paren{\varepsilon}
    \]
    Moreover, this quantity is nonzero only when $k - \ell \leq 6 - b_1 - b_2$. Thus, by the boundedness of the Hermite coefficients, we can conclude that
    \[
        \left|\EXP[\bfx]{\pi^{(a_1)}\paren{\bfv_1^\top\bfx}\pi^{(a_2)}\paren{\bfv_2^\top\bfx}\sigma^{(b_1)}\paren{\bfw_1^\top\bfx}\sigma^{(b_2)}\paren{\bfw_2^\top\bfx}}\right| \leq \calO\paren{\varepsilon}
    \]
\end{proof}

\subsection{Other Auxiliary Results}
\begin{lemma}
    Let $\bfv,\bfw\in\R^d$, and define $f\paren{\bfv} = \frac{1}{\norm{\bfv}_2}\paren{\bfI - \frac{\bfv\bfv^\top}{\norm{\bfv}_2^2}}\bfw$. Then we have that
    \[
        \mathcal{J}f\paren{\bfv} = -\frac{\bfv^\top\bfw}{\norm{\bfv}_2^3}\paren{\bfI - \frac{\bfv\bfv^\top}{\norm{\bfv}_2^2}} - \frac{1}{\norm{\bfv}_2^3}\paren{\bfv\bfw^\top + \bfw\bfv^\top}
    \]
\end{lemma}

\begin{lemma}
    \label{lem:sigmoid_property}
    Let $f(x) = \frac{1}{1 + e^{-x}}$ be the sigmoid function. Then we have that
    \begin{itemize}
        \item $f'''(x)^2 \leq 1$ for all $x\in\R$
        \item $\EXP[x\sim\mathcal{N}\paren{0,1}]{f(x)} = \frac{1}{2}$
        \item $\EXP[x\sim\mathcal{N}\paren{0,1}]{f^{(2k)}(x)} = 0$ for all $k\geq 1$
    \end{itemize}
\end{lemma}
\begin{proof}
    Using simple calculations, we can obtain that
    \[
        f'(x) = f(x)(1-f(x))
    \]
    This gives that
    \[
        f''(x) = f'(x)(1-f(x)) - f(x)f'(x) = f'(x)(1-2f(x)) 
    \]
    Thus, $f'''(x)$ can be written as
    \begin{align*}
        f'''(x) & = f''(x)(1-2f(x)) -2f'(x)^2\\
        & = f'(x)(1-2f(x))^2 - 2f'(x)^2\\
        & = f'(x)\paren{1 - 4f(x) + 4f(x)^2 - 2f'(x)}\\
        & = f'(x)\paren{1 - 6f(x) + 6f(x)^2}
    \end{align*}
    One the range $[0,1]$, the function $1 - 6y + 6y^2$ takes extremes at $y = 0, \frac{1}{2}, 1$. At $y = 0$ and $y = 1$, we have $1 - 6y + 6y^2 = 1$. At $y = \frac{1}{2}$, we have that $1 - 6y + 6y^2 = -\frac{1}{2}$. Thus, we can conclude that $\left|1 - 6y + 6y^2\right|\leq 1$ for all $y\in[0, 1]$. Moreover, we have that $f'(x) = f(x)(1- f(x))\in [0,1]$, since $f(x)\in [0,1]$. Therefore, we can conclude that $|f'''(x)| \leq 1$, which implies the first property. To prove the second, we notice that
    \[
        f(-x) = \frac{1}{1 + e^{x}} = \frac{e^{-x}}{1 + e^{-x}} = 1 - f(x)
    \]
    Therefore, due to the symmetry of Gaussian distribution, we have that
    \[
        \EXP[x\sim\mathcal{N}\paren{0,1}]{f(x)} = \EXP[x\sim\mathcal{N}\paren{0,1}]{f(-x)} = 1 - \EXP[x\sim\mathcal{N}\paren{0,1}]{f(x)}
    \]
    This gives that $\EXP[x\sim\mathcal{N}\paren{0,1}]{f(x)} = \frac{1}{2}$. To prove the third property, we notice that
    \[
        f'(-x) = f(-x)(1- f(-x)) = (1-f(x))f(x) = f'(x)
    \]
    which shows that $f'(-x)$ is even. Therefore, $f^{(2k)}(x)$ are odd functions for all $k\geq 1$. This implies the third property.
\end{proof}

\begin{lemma}
    \label{lem:ode_stability}
    Consider function $f(x), g(x), h(x)$ given by the ODE system
    \begin{align*}
        f'(x) = -a_1f(x) + b_1g(x) + c_1h(x) + p\\
        g'(x) = -a_2f(x) - b_2g(x) + c_2h(x) + p\\
        h'(x) = -a_3f(x) + b_3g(x) - c_3h(x) + p
    \end{align*}
    for some $a_1, a_2, a_3, b_1, b_2, b_3, c_1, c_2, c_3 > 0$. If $b_2c_3\geq b_3c_2$ and
    \[
        a_1^2b_2 + a_1^2c_3 + a_1a_2b_1 + a_1a_3c_1 + a_1b_2^2 + a_2b_1b_2  + a_1c_3^2 + a_3c_1c_3 \geq a_2b_3c_1 + a_3b_1c_2
    \]
    then we have that
    \[
        \max\left\{|f(x)|, |g(x)|, |h(x)|\right\} \leq e^{-\Omega\paren{x}}\paren{|f(0)|+ |g(0)| + |h(0)|} + \calO\paren{p}
    \]
    for all $x\geq 0$
\end{lemma}
\begin{proof}
    Let $\bfA\in\R^{3\times 3}$ be given by
    \[
        \bfA = \begin{bmatrix}
            -a_1 & b_1 & c_1\\
            -a_2 & -b_2 & c_2\\
            -a_3 & b_3 & -c_3
        \end{bmatrix}
    \]
    Then we have that
    \[
        \begin{bmatrix}
            f'(x)\\
            g'(x)\\
            h'(x)
        \end{bmatrix} = \bfA\begin{bmatrix}
            f(x)\\
            g(x)\\
            h(x)
        \end{bmatrix} + p\bm{1}
    \]
    Solving the system gives
    \[
        \begin{bmatrix}
            f(x)\\
            g(x)\\
            h(x)
        \end{bmatrix} = e^{\bfA x}\begin{bmatrix}
            f(0)\\
            g(0)\\
            h(0)
        \end{bmatrix} + p\bfA^{-1}\paren{e^{\bfA x} - \bfI}\bm{1}
    \]
    Let $\lambda$ be the eigenvalue of $A$ with the largest real part. Then we have that
    \[
        \norm{\begin{bmatrix}
            f(x)\\
            g(x)\\
            h(x)
        \end{bmatrix}}_2 \leq e^{\lambda x}\norm{\begin{bmatrix}
            f(0)\\
            g(0)\\
            h(0)
        \end{bmatrix}} + \calO\paren{p}
    \]
    Thus, it suffice to show that all eigenvalues of $\bfA$ has negative real parts. To do this, we write out the characteristic polynomial of $\bfA$ as
    \[
        P(y) = y^3 - \text{Tr}\paren{\bfA}y^2 + \frac{1}{2}\paren{\text{Tr}\paren{\bfA}^2 - \text{Tr}\paren{\bfA^2}}y - \text{det}\paren{\bfA}
    \]
    By the Routh-Hurwitz criteria, it suffice to show that
    \[
        \text{Tr}\paren{\bfA} < 0,\; \text{det}\paren{\bfA} < 0,\;\frac{1}{2}\text{Tr}\paren{\bfA}\paren{\text{Tr}\paren{\bfA}^2 - \text{Tr}\paren{\bfA^2}}  <\text{det}\paren{\bfA}
    \]
    With the form of $\bfA$, we obtain that
    \begin{gather*}
        \text{Tr}\paren{\bfA} = -\paren{a_1 + b_2 + c_3}\\
        \frac{1}{2}\text{Tr}\paren{\bfA}\paren{\text{Tr}\paren{\bfA}^2 - \text{Tr}\paren{\bfA^2}} = a_1b_2 + a_1c_3 + b_2c_3 + a_2b_1 + a_3c_1 - b_3c_2\\
        \text{det}\paren{\bfA} = -a_1b_2c_3 - a_2b_3c_1 - a_3b_1c_2 - a_3b_2c_1 - a_2b_1c_3 - a_1b_3c_2
    \end{gather*}
    Thus, it is easy to see that $\text{Tr}\paren{\bfA} < 0,\; \text{det}\paren{\bfA} < 0$. It remains to show that $\frac{1}{2}\text{Tr}\paren{\bfA}\paren{\text{Tr}\paren{\bfA}^2 - \text{Tr}\paren{\bfA^2}}  <\text{det}\paren{\bfA}$. This is equivalent to show that $S \geq 0$ with
    \begin{align*}
        S & = \text{det}\paren{\bfA} - \frac{1}{2}\text{Tr}\paren{\bfA}\paren{\text{Tr}\paren{\bfA}^2 - \text{Tr}\paren{\bfA^2}}\\
        & = \paren{a_1 + b_2 + c_3}\paren{a_1b_2 + a_1c_3 + b_2c_3 + a_2b_1 + a_3c_1 - b_3c_2}\\
        & \qqquad - \paren{a_1b_2c_3 + a_2b_3c_1 + a_3b_1c_2 + a_3b_2c_1 + a_2b_1c_3 + a_1b_3c_2}\\
        & = a_1^2b_2 + a_1^2c_3 + a_1b_2c_3 + a_1a_2b_1 + a_1a_3c_1 - a_1b_3c_2 + a_1b_2^2 + a_1b_2c_3 +b_2^2c_3 + a_2b_1b_2\\
        & \qqquad a_3b_2c_1 - b_2b_3c_2 + a_1b_2c_3 + a_1c_3^2 + b_2c_3^2 + a_2b_1c_3 + a_3c_1c_3 - b_3c_2c_3\\
        & \qqquad -a_1b_2c_3 - a_2b_3c_1 - a_3b_1c_2 - a_3b_2c_1 - a_2b_1c_3 - a_1b_3c_2\\
        & = a_1^2b_2 + a_1^2c_3 + 2a_1b_2c_3 + a_1a_2b_1 + a_1a_3c_1 - 2a_1b_3c_2 + a_1b_2^2 +b_2^2c_3 + a_2b_1b_2\\
        & \qqquad - b_2b_3c_2 + a_1c_3^2 + b_2c_3^2 + a_3c_1c_3 - b_3c_2c_3 - a_2b_3c_1 - a_3b_1c_2
    \end{align*}
    If $b_2c_3 \geq b_3c_2$, then we have that
    \[
        a_1b_2c_3 \geq a_1b_3c_2;\quad b_2c_3^2\geq b_3c_2c_3;\quad b_2^2c_3 \geq b_2b_3c_2
    \]
    This gives that
    \begin{align*}
        S \geq a_1^2b_2 + a_1^2c_3 + a_1a_2b_1 + a_1a_3c_1 + a_1b_2^2 + a_2b_1b_2  + a_1c_3^2 + a_3c_1c_3 - a_2b_3c_1 - a_3b_1c_2
    \end{align*}
\end{proof}

\begin{lemma}
    \label{lem:Q2_bound}
    Let $c_k$ be the $k$th order Hermite coefficient of $\pi\paren{x}$. Let $z_1, z_2\sim\mathcal{N}(0,1)$ with $\text{Cov}(z_1, z_2) = \rho$. If $\EXP[z_1,z_2]{\pi'\paren{z_1}\pi'''\paren{z_2}} \leq 0$ for all $\rho$, then we have that
    \[
        \sum_{k=0}^{\infty}\frac{c_kc_{k+2}}{k!}\gamma^k \leq 0;\;\forall \gamma > 0
    \]
\end{lemma}
\begin{proof}
    Notice that, by taking the Hermite expansion of $\pi(x)$ and $\pi''(x)$, we have that for $z_1, z_2\sim\mathcal{N}(0,1)$ with $\text{Cov}\paren{z_1, z_2} = \gamma$
    \[
        \EXP[z_1,z_2]{\pi(z_1)\pi''(z_2)} = \sum_{k=0}^{\infty}\frac{c_kc_{k+2}}{k!}\gamma^k
    \]
    Let $f(\gamma) = \EXP[z_1,z_2]{\pi(z_1)\pi''(z_2)}$. Then by Price's Theorem we have that
    \[
        f'(\gamma) = \EXP[z_1, z_2]{\pi'(z_1)\pi'''(z_2)} \leq 0
    \]
    Moreover, at $\gamma = 0$, we have that
    \[
        f(0) = \EXP[z_1]{\pi\paren{z_1}}\EXP[z_2]{\pi''\paren{z_2}} = 0
    \]
    where the last equality is due to Lemma~\ref{lem:sigmoid_property}. Therefore, we can conclude that
    \[
        \sum_{k=0}^{\infty}\frac{c_kc_{k+2}}{k!}\gamma^k = f(\gamma) \leq f(0)  = 0
    \]
\end{proof}

\begin{lemma}
    \label{lem:sum_bound}
    Let $x\in (-1,1)$, and consider $S = \sum_{k=0}^{\infty}\frac{c_k^2}{k!}x^k$. If $c_k\leq \sqrt{B\cdot k!}$, then we have that
    \[
        c_0^2 - \frac{|x|}{1 - |x|} \leq S \leq c_0^2 + \frac{|x|}{1 - |x|}
    \]
\end{lemma}
\begin{proof}
    We write $S$ as
    \[
        S = c_0^2 + \sum_{k=1}^{\infty}\frac{c_k^2}{k!}x^k
    \] 
    Notice that
    \[
        \left|\sum_{k=1}^{\infty}\frac{c_k^2}{k!}x^k\right| \leq \sum_{k=1}^{\infty}\frac{c_k^2}{k!}|x|^k \leq B\sum_{k=1}^{\infty}|x|^k  = \frac{B|x|}{1 - |x|}
    \]
    Therefore, we can conclude that
    \[  
        c_0^2 - \frac{B|x|}{1 - |x|} \leq S \leq c_0^2 + \frac{B|x|}{1 - |x|}
    \]
\end{proof}

\begin{lemma}
    Let $c_k$ denote the $k$th order Hermite coefficient of $\pi\paren{\cdot}$ such that $c_2 = 0$ and $c_{k+3}\leq \sqrt{B\cdot k!}$ for all $k\geq 0$. Let $\gamma\in(b, 1]$. If $\EXP[z_1,z_2\sim\mathcal{N}(0,1),\text{Cov}\paren{z_1,z_2} = \rho]{\pi'\paren{z_1}\pi^{(3)}\paren{z_2}}\leq 0$ for all $\rho\in[-1,1]$, then we have that
    \[
        \sum_{k=0}^{\infty}\frac{c_kc_{k+2}}{k!}\gamma^k\leq \frac{|b|}{1 - |b|}
    \]
\end{lemma}
\begin{proof}
    Notice that
    \[
        \EXP[z_1,z_2]{\pi\paren{z_1}\pi''\paren{z_2}} = \EXP[z_1,z_2]{\paren{\sum_{k=0}^{\infty}\frac{c_k}{k!}He_k\paren{z_1}}\paren{\sum_{k=0}^{\infty}\frac{c_{k+2}}{k!}He_k\paren{z_1}}} = \sum_{k=0}^{\infty}\frac{c_kc_{k+2}}{k!}
    \]
    Therefore, we can define
    \[
        h(\rho) = \sum_{k=0}^{\infty}\frac{c_kc_{k+2}}{k!}\rho^k = \EXP[z_1,z_2]{\pi\paren{z_1}\pi''\paren{z_2}}\rho^k
    \]
    By Price's Theorem, we have that
    \[
        h'\paren{\rho} = \EXP[z_1,z_2]{\pi'\paren{z_1}\pi^{(3)}\paren{z_2}} \leq 0
    \]
    Therefore, $h\paren{\rho}$ decreases monotonically, which implies that
    \[
        h(\rho) = \sum_{k=0}^{\infty}\frac{c_kc_{k+2}}{k!}\rho^k \leq h\paren{b}  = b\sum_{k=0}^{\infty}\frac{c_{k+1}c_{k+3}}{k!}b^k \leq \sum_{k=1}^{\infty}|b|^k \leq \frac{|b|}{1 - |b|}
    \]
\end{proof}

\section{Plotting Sigmoid Property}
In this section, we plot the simulation result of the properties of the sigmoid function. Figure~\ref{fig:sig_prod} is generated by taking $10^5$ samples of correlated Gaussian random variables for each covariance value $\rho \in [-1, 1]$. Figure~\ref{fig:sig_sq} is generated by taking $10^6$ samples of standard Gaussian random variables.
\label{sec:plot_sigmoid}
\begin{figure}
    \centering
    \begin{subfigure}[b]{0.42\textwidth}
        \centering
        \includegraphics[width=\textwidth]{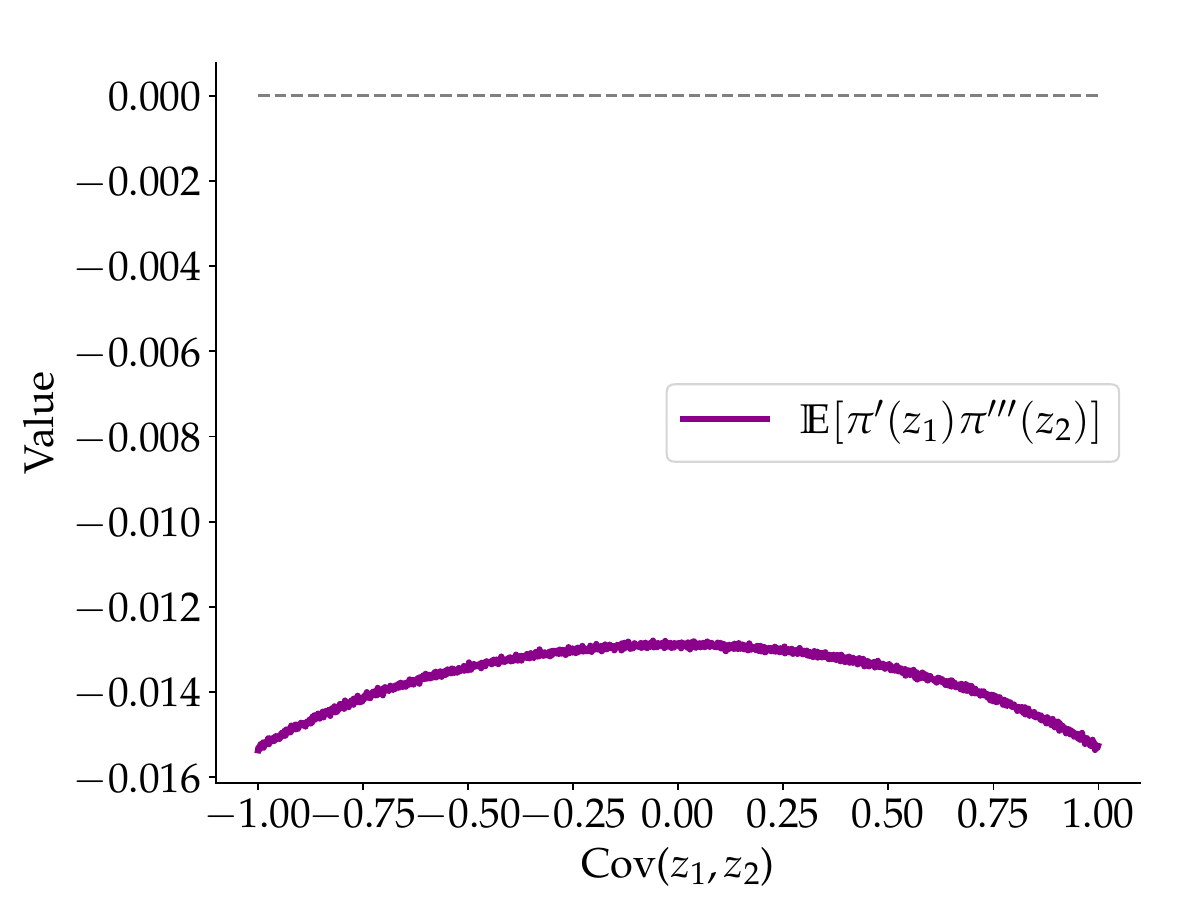}
        \caption{Value of $\EXP[z_1,z_2\sim\mathcal{N}(0,1)]{\pi'(z_1)\pi'''(z_2)}$ for different values of $\text{Cov}\paren{z_1, z_2}$. As in the figure, all values are negative.}
        \label{fig:sig_prod}
    \end{subfigure}
    \begin{subfigure}[b]{0.42\textwidth}
        \centering
        \includegraphics[width=\textwidth]{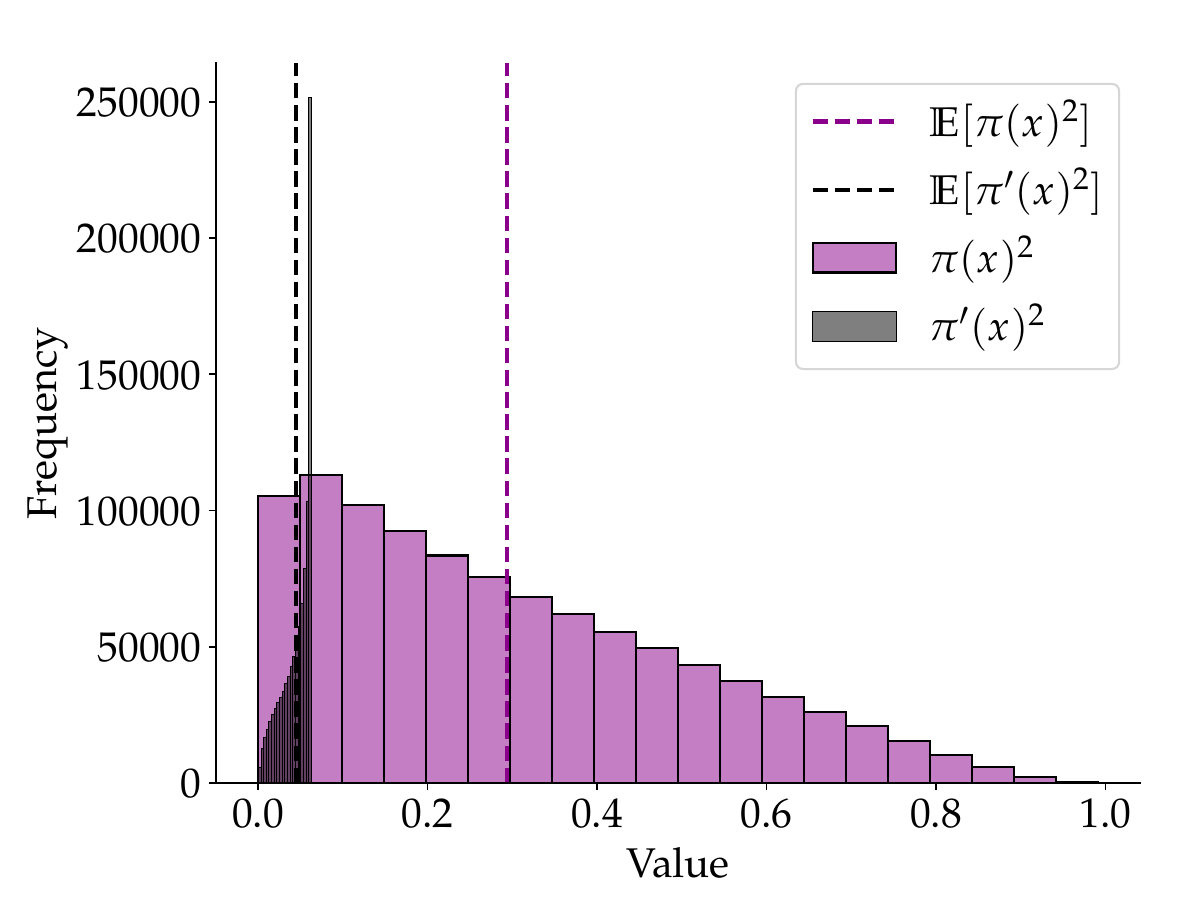}
        \caption{Values of $\EXP[x\sim\mathcal{N}(0,1)]{\pi(x)^2}$ and $\EXP[x\sim\mathcal{N}(0,1)]{\pi'(x)^2}$. Figures shows that $\EXP[x\sim\mathcal{N}(0,1)]{\pi(x)^2} \geq 1.1\EXP[x\sim\mathcal{N}(0,1)]{\pi'(x)^2}$.}
        \label{fig:sig_sq}
    \end{subfigure}
    \caption{Plot of the property of $\pi(x)$}
\end{figure}
\end{document}